\documentclass[onecolumn, 12pt, fullpage, a4paper]{report}

\newcommand{\fomuml}{{\tt FO-MuML}}
\newcommand{\fomaml}{{\tt FO-MAML}}
\newcommand{\imaml}{{\tt iMAML}}
\newcommand{\maml}{{\tt MAML}}
\newcommand{\mamlpp}{{\tt MAML++}}

\newcommand{\fedsplit}{{\tt FedSplit}}
\newcommand{\fedprox}{{\tt FedProx}}
\newcommand{\fedavg}{{\tt FedAvg}}

\newcommand{\sgd}{{\tt SGD}}
\newcommand{\gd}{{\tt GD}}
\newcommand{\agd}{{\tt AGD}}

\newcommand{\pgd}{{\tt PGD}}
\newcommand{\apgd}{{\tt APGD}}
\newcommand{\apgdo}{{\tt APGD1}}
\newcommand{\apgdt}{{\tt APGD2}}
\newcommand{\iapgd}{{\tt IAPGD}}
\newcommand{\katyusha}{{\tt Katyusha}}
\newcommand{\altsgdp}{{\tt AL2SGD+}}
\newcommand{\ltsgdp}{{\tt L2SGD+}}

\newcommand{\ltsgd}{{\tt L2SGD}}
\newcommand{\ltgd}{{\tt L2GD}}

\newcommand{\sag}{{\tt SAG}}
\newcommand{\saga}{{\tt SAGA}}
\newcommand{\svrg}{{\tt SVRG}}
\newcommand{\lsvrg}{{\tt L-SVRG}}

\newcommand{\aoa}{\emph{a1a}}
\newcommand{\mushroom}{\emph{mushroom}}
\newcommand{\bodyfat}{\emph{bodyfat}}
\newcommand{\mg}{\emph{mg}}
\newcommand{\phishing}{\emph{phishing}}
\newcommand{\mnist}{\emph{MNIST}}
\newcommand{\wea}{\emph{w8a}}
\newcommand{\ana}{\emph{a9a}}
\newcommand{\madelon}{\emph{madelon}}
\newcommand{\duke}{\emph{duke}}

\newcommand{\ain}{{\tt AICN}}
\newcommand{\sgn}{{\tt SGN}}
\newcommand{\newton}{{Newton}}
\newcommand{\cnewton}{{cubic Newton}}
\newcommand{\Cnewton}{{Cubic Newton}}
\newcommand{\dnewton}{{damped Newton}}
\newcommand{\Dnewton}{{Damped Newton}}
\newcommand{\rnewton}{{regularized Newton}}
\newcommand{\Rnewton}{{Regularized Newton}}

\newcommand{\Gnewton}{{Glob.~reg.~Newton}}

\newcommand{\jacsketch}{{\tt JacSketch}}
\newcommand{\sap}{{ sketch-and-project}}
\newcommand{\Sap}{{ Sketch-and-project}}
\newcommand{\en}{{ Exact Newton descent}}
\newcommand{\sscn}{{\tt SSCN}}
\newcommand{\rsn}{{\tt RSN}}
\newcommand{\cd}{{\tt CD}}
\newcommand{\acd}{{\tt ACD}}

\newcommand{\libsvm}{{\tt LIBSVM}}

\usepackage{hyperref}       
\usepackage{url}            
\usepackage{booktabs}       
\usepackage{amsfonts}       
\usepackage{nicefrac}       
\usepackage{microtype}      
\usepackage{wrapfig}
\usepackage{xargs}                      
\usepackage{comment}

\usepackage{amssymb,amsmath,amscd,amsfonts,amstext,amsthm,bbm, enumerate, dsfont, mathtools}
\usepackage{array}
\usepackage{multicol}
\usepackage{algorithm}
\usepackage{varwidth}
\usepackage{algpseudocode}
\usepackage[none]{hyphenat}
\usepackage{footnote}
\usepackage{makecell}
\usepackage{threeparttable}
\usepackage{booktabs}
\usepackage{tcolorbox}
\usepackage{caption}
\usepackage{subcaption}

\newcommand{\mycite}[1]{$\vartriangleright$ Based on work: \cite{#1}}
\newtheorem{theorem}{Theorem}
\newtheorem{proposition}{Proposition}
\newtheorem{lemma}{Lemma}

\newtheorem{assumption}{Assumption}
\newtheorem*{remark}{Remark}
\newtheorem{definition}{Definition}
\usepackage{varioref}
\usepackage{cleveref}

\DeclareMathOperator*{\argmin}{argmin}

\newcommand{\prox}{\mathop{\mathrm{prox}}\nolimits}

\newcommand{\EE}{\mathbb{E}}
\newcommand{\ec}[2][]{\ensuremath{\mathbb{E}_{#1} \left[#2\right]}}
\newcommand{\R}{\mathbb{R}}

\newcommand{\cL}{{\cal L}}

\newcommand{\eqdef}{\stackrel{\text{def}}{=}}
\newcommand{\cO}{{\cal O}}

\newcommand{\meta}{F}
\newcommand{\opt}{x^*}
\newcommand{\fopt}{f^*}

\usepackage{changepage}
\usepackage{stackengine}
\usepackage{pifont}

\usepackage{mdframed}
\mdfsetup{%
	linecolor=white,
	backgroundcolor=blue!5,
}
\newcolumntype{?}{!{\vrule width 1pt}}
\usepackage{colortbl}

\definecolor{mydarkgreen}{RGB}{39,130,67}

\definecolor{mydarkred}{RGB}{192,47,25}
\newcommand{\red}{\color{mydarkred}}

\newcommand{\cmark}{{\color{mydarkgreen}\ding{51}}}%
\newcommand{\xmark}{{\color{mydarkred} \ding{55}}}%
\newcommand{\norm}[1]{{\left \| #1 \right\|}}
\newcommand{\E}[1]{\mathbb{E}\left[#1\right] } 
\newcommand{\Ed}[2]{\mathbb{E}_{#1}\left[#2\right]}
\newcommand{\varopt}{\sigma_*^2}

\newcommand{\Lstandard}{{L_{\text{sc}}}}

\newcommand{\Lsemi}{{L_{\text{semi}}}}
\newcommand{\Lstrongly}{{L_{\text{str}}}}
\newcommand{\Lalg}{{L_{\text{est}}}}
\newcommand{\Lalt}{{L_{\text{alt}}}}

\newcommand{\cD}{\mathcal D}
\newcommand{\s}{\mathbf S}
\newcommand{\sk}{\mathbf S_k}
\newcommand{\gS}{\nabla_{\s} f}
\newcommand{\gSk}{\nabla_{\sk} f}
\newcommand{\hS}{\nabla_{\s}^2 f}
\newcommand{\hSk}{\nabla_{\sk}^2 f}

\newcommand{\st}{\s ^ \top}
\newcommand{\Ls}{L_{\s}}
\newcommand{\als}[1]{\alpha_{#1}}
\newcommand{\modelS}[1]{T_{\s} (#1)}
\newcommand{\modelSk}[1]{T_{\sk} (#1)}

\newcommand{\normMS}[2]{{\left \| #1 \right\|}_{#2, \s}}
\newcommand{\normsMS}[2]{{\left \| #1 \right\|}_{#2, \s}^2}
\newcommand{\normMSd}[2]{{\left \| #1 \right\|}_{#2, \s}^*}
\newcommand{\normsMSd}[2]{{\left \| #1 \right\|}_{#2, \s}^{*2}}
\newcommand{\normMSdk}[2]{{\left \| #1 \right\|}_{#2, \sk}^*}
\newcommand{\normsMSdk}[2]{{\left \| #1 \right\|}_{#2, \sk}^{*2}}

\newcommand{\level}{\mathcal Q(x^0)}
\newcommand{\p}{\mathbf P}
\newcommand{\pk}[1]{\mathbf P_{x_{#1}}}
\newcommand{\murel}{\hat \mu}
\newcommand{\Lrel}{\hat L}
\newcommand{\Lrels}{\hat L_{\sk}}

\newcommand{\okd}{\mathcal O \left( k^{-2} \right)}
\newcommand{\tn}[1]{\tnote{\color{red}(#1)}}
\newcommand{\afrac}{\frac}
\newcommand{\Tr}{\text{Tr}}

\newcommand{\h}{\nabla^2 f}
\newcommand{\g}{\nabla f}

\newcommand{\mI}{\mathbf I}

\newcommand{\normM}[2]{{\left \| #1 \right\|}_{#2}}
\newcommand{\normsM}[2]{{\left \| #1 \right\|}_{#2}^2}
\newcommand{\normMd}[2]{{\left \| #1 \right\|}_{#2}^*}
\newcommand{\normsMd}[2]{{\left \| #1 \right\|}_{#2}^{*2}}
\newcommand{\G}{G_k}
\newcommand{\gk}{g^k}
\newcommand{\hk}{h^k}

\newcommand{\err}{\beta}
\newcommand{\gn}{g^{k+1}} 
\newcommand{\gbox}{}
\newcommand{\gboxeq}{}

\usepackage{theoremref}

\usepackage[colorinlistoftodos,bordercolor=orange,backgroundcolor=orange!20,linecolor=orange,textsize=scriptsize]{todonotes}

\newcommand{\name}{Slavom\'ir Hanzely}
\newcommand{\mytitle}{Adaptive Optimization Algorithms for Machine Learning}

\newcommand{\sumin}[2]{ \sum \limits_{#1=1}^{#2}}
\newcommand{\avein}[2]{\frac 1 {#2} \sum \limits_{#1=1}^{#2}}
\newcommand{\norms}[1]{\left \| #1 \right\|^2}

\newcommand{\xdiff}{x^{k+1} -x^k}

\newcommand{\N}{\mathbb N}

\newcommand{\ip}[2]{\left\langle #1, #2  \right \rangle}

\def\<#1,#2>{\left\langle #1,#2\right\rangle}

\usepackage[round]{natbib}
\renewcommand{\bibname}{References}

\usepackage{enumitem}
\newlist{myitemize}{itemize}{3}
\setlist[myitemize,1]{label=\textbullet,leftmargin=0in}

\newcommand{\ls}{\left(}
\newcommand{\rs}{\right)}

\newcommand{\lb}{\left\lbrace}
\newcommand{\rb}{\right\rbrace}
\newcommand{\la}{\left\langle}
\newcommand{\ra}{\right\rangle}

\newcommand{\bbE}{\mathbb{E}}

\newcommand{\outers}{\beta}
\newcommand{\inners}{\alpha}

\newcommand{\Lmeta}{L_{\meta}}
\newcommand{\locstep}{\gamma}

\newcommand {\aveis}[2]{\frac 1{|#2|} \sum_{#1 \in #2}}

\def\<#1,#2>{\left\langle #1,#2\right\rangle}

\newcommand{\diam}{{\mathrm{Diam}}} 

\newcommand{\probx}{ \rho}
\newcommand{\proby}{p}

\newcommand{\tR}{{i}} 
\newcommand{\TR}{{n}} 

\newcommand{\Lloc}{{\tilde{L}_1}} 

\newcommand{\Span}{{\rm Span}} 

\newcommand{\cA}{{\cal A}}

\newcommand{\cK}{{\cal K}}

\newcommand{\cS}{{\cal S}}

\newcommand{\mM}{{\bf M}}

\newcommand{\mQ}{{\bf Q}_r}

\newcommand{\ones}{e}

\newcommand{\Null}[1]{{\rm Null}\left( #1\right)}

\newcommand{\Range}[1]{{\rm Range}\left( #1\right)}

\usepackage[colorinlistoftodos,bordercolor=orange,backgroundcolor=orange!20,linecolor=orange,textsize=scriptsize]{todonotes}

\newcommand{\nhalf}{M }

\newcommand{\flocc}{{ \tilde{f}}}

\newcommand{\locf}{\zeta}

\newcommand{\comm}{C}

\newcommand{\mat}{\begin{pmatrix}
		c& -c\\
		-c & c
\end{pmatrix} }

\newcommand{\mmat}{\begin{pmatrix}
		\tfrac{(\nhalf+1)c}{\nhalf} & -\tfrac{(\nhalf+1)c}{\nhalf}\\
		-\tfrac{(\nhalf+1)c}{\nhalf} & \tfrac{(\nhalf+1)c}{\nhalf}
\end{pmatrix} }

\newcommand{\QED}{ }

\newcommand{\sg}{{\sigma}}
\newcommand{\tk}{{\tau^k}}
\newcommand{\lk}{{\mathcal L^k}}

\newcommand{\Q}{\textbf{Q:} }
\newcommand{\A}{\textbf{A:} }

\newcommand{\efa}{&& \hspace{-1.5cm}}

\newcommand{\n}{& \xmark}
\newcommand{\y}{& \cmark}

\newcommand{\prob}{{ \mathbb P}}

\newcommand{\mycomment}[1]{}
\newcommand{\chap}[1]{
	\chapter{Appendix to Chapter #1}
	
}

\usepackage{verbatim}
\usepackage{amssymb}
\usepackage{amsthm}
\usepackage{arabtex} 
\usepackage{academicons}
\usepackage{breakcites} 
\usepackage{color}
\usepackage{epsfig}
\usepackage{epstopdf}
\usepackage{etoolbox}
\usepackage{enumitem}
\usepackage{float}
\usepackage{fancyhdr}
\usepackage{graphicx}
\usepackage{graphics}
\usepackage{latexsym,amsfonts}
\usepackage{longtable}
\usepackage{listings}
\usepackage{lscape} 
\usepackage{lipsum}
\usepackage{multirow}             
\usepackage{pdfpages}
\usepackage[figuresright]{rotating} 
\usepackage{sectsty}
\usepackage{setspace}
\usepackage{textcomp}
\usepackage{utf8} 
\usepackage{url} 
\usepackage{wrapfig} 
\usepackage{wasysym} 

\usepackage{footmisc}

\chapterfont{\fontsize{14}{15}\selectfont}   
\sectionfont{\fontsize{14}{15}\selectfont}
\subsectionfont{\fontsize{14}{15}\selectfont}
\subsubsectionfont{\fontsize{14}{15}\selectfont}
  
\fancyheadoffset[L]{0.01mm}

\makeatletter
\patchcmd{\@makechapterhead}{50\p@}{20pt}{}{}
\patchcmd{\@makeschapterhead}{50\p@}{20pt}{}{}
\makeatother

\fancypagestyle{plain}{
\fancyhf{} 
\fancyhead[C]{\thepage} 

}

   \usepackage{ifthen,xkeyval,xfor,amsgen}
   \usepackage[acronym,toc, nogroupskip]{glossaries}
   \newglossary[slg]{symbols}{syi}{sbl}{List of Symbols}
 
   \makeglossaries

\newglossaryentry{symb:Pi}{
name=$\pi$, type=symbols,
description=A mathematical constant whose value is the ratio of any circle's circumference to its diameter,
sort=symbolpi
}

\newglossaryentry{symb:Phi}{
name=$\varphi$, type=symbols,
description=An angle,
sort=symbolphi
}

\newglossaryentry{symb:Lambda}{
name=$\lambda$, type=symbols,
description=Lambda indicates usually an eigenvalue in linear algebra,
sort=symbollambda
}

\newacronym{toc}{ToC}{Table of Contents}
\newacronym{los}{LoS}{List of Symbols}
\newacronym{loa}{LoA}{List of Abbreviations}
\newacronym{phd}{PhD}{Doctoral}
\newacronym{MS}{MS}{Masters}
\newacronym{CD}{CD}{Compact Disc}
\newacronym{kaust}{KAUST}{King Abdullah University of Science and Technology}

\newacronym{AD}{AD}{Active Directory\protect\glsadd{glos:AD}}

\newglossaryentry{glos:AD}{
name=Active Directory,
description={Active Directory is the directory service for Windows based networks, that allows central organization and administration of any network resource. It allows a single-sign-on concept independent from network topologies or network protocols. As a prerequisite you need a Windows Server acting as Domain Controller. This computer stores all necessary data, e.g.~usernames and corresponding passwords}
}

\newglossaryentry{glos:RespF}{name=response file, description={A file 
that allows unattended software installation}}

\usepackage[lmargin=40mm, rmargin=25mm, vmargin=25mm, headsep=2.5mm]{geometry}
\newcommand{\mathsym}[1]{{}}
\newcommand{\unicode}[1]{{}}

\renewcommand\bibname{\centering BIBLIOGRAPHY}
\newcommand{\orcid}{\includegraphics[width=8pt]{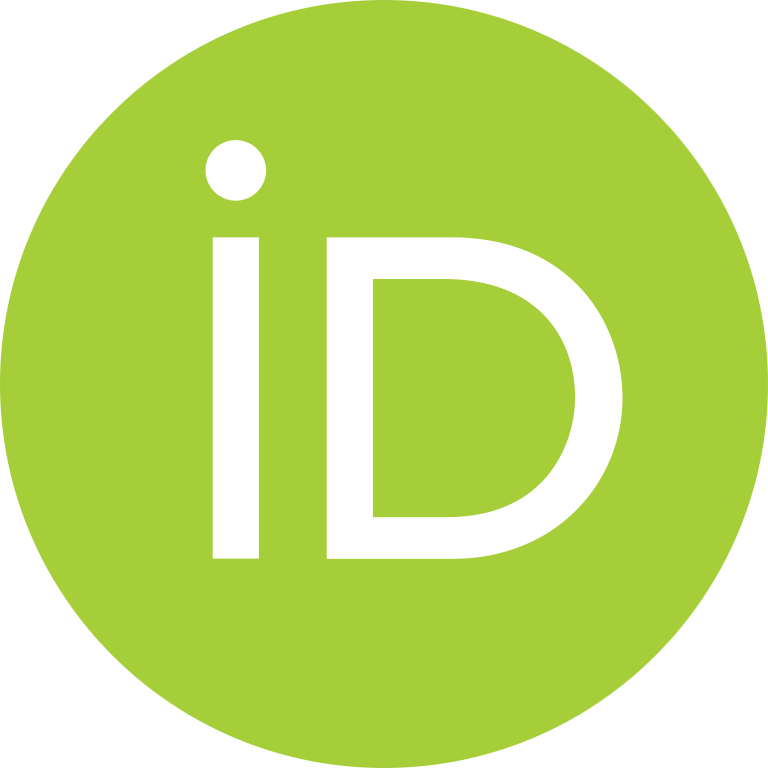}} 
\renewcommand{\R}{ \mathbb R}

\begin{document}

\thispagestyle{empty}
\addvspace{5mm}  

\begin{center}
\begin{doublespace}
{\textbf{{\large \mytitle }}}
\end{doublespace}

\vspace{10mm}
{Dissertation by}\\
{\name} 

\vspace{30mm}

{ In Partial Fulfillment of the Requirements}\\[12pt]
{ For the Degree of}\\[12pt]
{Doctor of Philosophy} \vfill
{King Abdullah University of Science and Technology }\\
{Thuwal, Kingdom of Saudi Arabia}
\vfill

\begin{onehalfspace}
{\copyright September, 2023}\\
\name\\               
All rights reserved\\

\orcid{} \small \href{https://orcid.org/0009-0006-2640-0354}{orcid.org/0009-0006-2640-0354}\\
\small \href{https://slavomir-hanzely.github.io/}{slavomir-hanzely.github.io}\\
\end{onehalfspace}

\end{center}
\newpage

%
\chaptertitlefont{\fontsize{14}{15}\selectfont\centering}  

\begin{singlespace}
    
\begin{center}

\end{center}

\begin{center}
{{\bf\fontsize{14pt}{14.5pt}\selectfont \uppercase{EXAMINATION COMMITTEE PAGE}}}
\end{center}

\addcontentsline{toc}{chapter}{Examination committee page}

\begin{center}
\end{center}

\noindent
The dissertation of Slavom\'ir Hanzely is approved by the examination committee.\\

\vspace{5cm}

\noindent
Committee Chairperson: Peter Richt\'arik\\

\noindent
Committee Members: Eric Moulines, Martin Jaggi, Ajay Jasra, Di Wang
    
\begin{center}

\end{center}

\begin{center}
{{\bf\fontsize{14pt}{14.5pt}\selectfont \uppercase{ABSTRACT}}}
\end{center}

\addcontentsline{toc}{chapter}{Abstract}

\begin{center}
	\begin{doublespace}
{\fontsize{14pt}{14.5pt}\selectfont {\mytitle}}\\
		{\fontsize{14pt}{14.5pt}\selectfont {\name}}\\
	\end{doublespace}
\end{center}

Machine learning assumes a pivotal role in our data-driven world.
The increasing scale of models and datasets necessitates quick and reliable algorithms for model training.
This dissertation investigates adaptivity in machine learning optimizers. The ensuing chapters are dedicated to various facets of adaptivity, including:
\begin{enumerate}[leftmargin=*]
	\item \textbf{personalization} and user-specific models via personalized loss (\Cref{sec:mixture_optimal} and \ref{sec:moreau_meta}), 
	\item provable \textbf{post-training model adaptations} via meta-learning (\Cref{sec:moreau_meta}), 
	\item learning \textbf{unknown hyperparameters} in real time via hyperparameter variance reduction (\Cref{sec:adaptive}), 
	\item fast $\okd$ global convergence of second-order methods via stepsized Newton method regardless of the \textbf{initialization and choice basis} (\Cref{sec:aicn}),
	\item fast and \textbf{scalable} second-order methods via low-dimensional updates (\Cref{sec:sgn}).
\end{enumerate}

This thesis contributes novel insights, introduces new algorithms with improved convergence guarantees, and improves analyses of popular practical algorithms.  
    

\begin{center}

\end{center}
\begin{center}

{\bf\fontsize{14pt}{14.5pt}\selectfont \uppercase{Acknowledgements}}\\\vspace{1cm}
\end{center}

\addcontentsline{toc}{chapter}{Acknowledgements} 

I would like to express my deepest gratitude and appreciation to all those who have contributed to my research. 
First and foremost, I am immensely grateful to my supervisor, Peter Richt\'arik, for their invaluable expertise, guidance, and continuous support throughout this research journey. Their deep understanding of the subject matter, insightful feedback, and commitment to academic excellence have been instrumental in shaping the direction and quality of my research.
I would like to extend my sincere appreciation to the members of my dissertation committee. 
I am indebted to my colleagues and peers in the research group for their intellectual contributions, fruitful discussions, and collaborative spirit. Namely, I'd like to thank
Abdurakhmon Sadiev,
Adil Salim,
Ahmed Khaled,
Alexander Tyurin,
Aritra Dutta,
Artavazd Maranjyan,
Avetik Karagulyan,
 Aymeric Dieleveut,
 Dmitry Kamzolov,
Dmitry Kovalev,
 Dmitry Pasechnyuk,
 Dan Alistarh,
Eduard Gorbunov,
Egor Shulgin,
El Houcine Bergou,
Elnur Gasanov,
Filip Hanzely,
Grigory Malinovsky,
Hanmin Li,
Igor Sokolov,
Ilyas Fatkhullin,
Ivan Agarsky,
Ivan Ilin,
Kay Yi,
Kaja Gruntkowska,
Konstantin Burlachenko,
Konstantin Mishchenko,
Laurent Condat,
Lukang Sun,
Martin Takáč,
Mher Safaryan,
Michal Grudzien,
Nicolas Loizou,
Peter Richtárik,
Rafal Szlendak,
Robert M. Gower,
Rustem Islamov,
Samuel Horváth,
Sarit Khirirat,
 Si Yi Meng,
Xun Qian,
Yury Demidovich,
Zhize Li.
Their diverse perspectives and collective knowledge have broadened my horizons and inspired new ideas throughout the course of this research.
I am deeply grateful to my friends and family for their unwavering support, understanding, and encouragement throughout this demanding journey. Their belief in my abilities and continuous motivation have been the driving force behind my perseverance.
\end{singlespace}

\begin{onehalfspacing}

\tableofcontents
\cleardoublepage

\printglossary[type=\acronymtype,style=long3col, title=\centerline{LIST OF ABBREVIATIONS}, toctitle=List of Abbreviations, nonumberlist=true] 
\printglossary[type=symbols,style=long3col, title=\centerline{LIST OF SYMBOLS}, toctitle=List of Symbols, nonumberlist=true]

\end{onehalfspacing}
\chaptertitlefont{\fontsize{14}{15}\selectfont}  
\begin{singlespace}


\thispagestyle{empty}

\chapter{Introduction}

\section{Adaptivity}

Machine learning, with its rich historical background \citep{Newton, Raphson, Simpson}, assumes a cardinal role in today's data-driven world \citep{beck-book-first-order, conn2000trust}. The expanse of available data, of profound magnitude, accentuates the necessity to cultivate robust and scalable machine learning models. This thesis focuses on the optimization facet of machine learning, aiming to enhance the performance and efficiency of training models. Central to this pursuit is the concept of adaptivity, enabling models to dynamically adjust and improve their performance in response to changing conditions, new data, or evolving requirements.

Within the realm of machine learning, adaptivity conveys the intrinsic capability of models or systems to autonomously recalibrate in response to changing environments, update their internal representations, and make robust decisions or predictions. It encompasses the capability to handle concept drift, personalize recommendations, optimize resource allocation, and respond to dynamic environments, among other aspects. The term ``adaptivity" within this thesis encapsulates the ensuing concepts:

\begin{enumerate}
	\item Resource optimization: Adaptive methods can optimize the allocation of computational resources based on changing or unforeseen circumstances. This engenders heightened efficiency, reduces computational expenses, and optimizes the utilization of available resources.
	
	\item Implementation streamlining: Adaptivity reduces reliance on hyperparameters and the need for fine-tuning initialization, data sampling, and stepsize schedules.
	
	\item Dynamic environments \& future-proofing: Real-world scenarios invariably manifest shifts in data distributions, user preferences, and system parameters. Adaptive models are designed to navigate these dynamic environments effectively, ensuring that their accuracy and effectiveness are preserved as circumstances evolve.
	
	\item Personalized experiences: Adaptivity enables models to assimilate and accommodate individual user preferences or contextual information. In scenarios where users with very heterogeneity of the data coalesce into a collective dataset, no singular model fits well for all users. Tailoring outputs to specific clients enables models to harness users' data to dispense personalized recommendations and services, thereby augmenting both the model's accuracy and user satisfaction.
\end{enumerate}

In subsequent chapters, we undertake an exploration of distinct modes of adaptivity in the context of machine learning optimization. Through investigating various techniques and algorithms, we aim to develop novel adaptive methods and quantify both the theoretical and practical benefits of adaptivity. This endeavor imparts insights into the challenges, opportunities, and future directions to amplify the prowess and versatility of machine learning models.
We now present a brief overview.

\section{Adaptation to unknown hyperparameters} 
The convergence of first-order methods is sensitive to the choice of hyperparameters. Choosing a good combination of hyperparameters is a difficult problem on its own. One of the popular approaches involves an exhaustive grid search for an optimal set of hyperparameters, which imposes an extra cost on the training process. Alternatively, one can use hyperparameters from theory, which guarantee convergence in even the worst case, but with a possible slower rate in the average case. However, even then formulas from theory are not always computable. As an example, \citet{SGD_general_analysis} presents a formula for the optimal minibatch size (minimizing the amount of gradient computation of SGD). However, it requires knowledge of the gradient variance at the optimum. In most cases, this is not known in advance. 
We address this limitation by demonstrating the feasibility of approximating this gradient variance and optimal minibatch size in real-time.

We consider finite-sum empirical risk minimization of function $f: \R^d \to \R$,
\begin{eqnarray} \label{eq:intro_erm}
    \min_{x \in \R^d} \left [ f(x) \eqdef \avein in f_i(x)\right],
\end{eqnarray}
where $f_i(x)$ is the losses of the model parametrized by vector $x$ on the datapoint $i$. Denote the solution $x^* \eqdef \argmin_{x \in \R^d} f(x)$. 
\citet{Gower2019} showed that for a random vector $v \in \R^d$ sampled from a user defined distribution $\cD$ satisfying $\bbE_{\cD}\left[ v_i\right] =1$, problem \eqref{eq:intro_erm} can be reformulated as 
\begin{equation}
    \min_{x \in \R^d} \bbE_{\cD} \left [ f_v(x) \eqdef \avein in v_i f_i(x)\right].
\end{equation}
This problem can be solved by applying generic minibatch \sgd{} method, which performs iterations 
of the form
\begin{equation} \label{eq:intro_sgd-MB}
 x^{k+1} = x^k - \gamma^k \g_{v_k}(x^k),
\end{equation}
where $v_k \sim \cD$ is sampled each iteration. Typically, the vector $v$ is defined by first choosing a random minibatch $\mathcal S^k \subseteq \{1,2,\dots,n\}$ and then defining  $v_i^k =0$ for $i \not \in \mathcal S^k$ and $v_i^k$ for $\mathcal S^k$ so that stochastic gradient $\g_{v^k}(x^k)$ is unbiased. 
For example, one standard choice is to fix a  batch size $\tau \in \{1,2,\dots,n\}$, and pick sampling $\mathcal S^k$ uniformly from all subsets of size $\tau$. In such case, $v_i^k= \frac n \tau$ for $i \in \mathcal S^k$ and $v_i^k= 0$ for $i \not \in \mathcal S^k$.

\subsection{Optimal minibatch size}
To discuss the optimal minibatch size, we delve into a theorem that identifies the optimal iteration complexity of \sgd{} under a fixed minibatch size $\tau$ \citep{SGD_general_analysis}. 
Our theory relies on standard assumptions \citep{nesterov2013introductory}: strong convexity, and expected smoothness \citep{Gower2019}. The former assumption provides a lower bound on the first-order Taylor approximation, the latter provides a stochastic upper bound on the first-order Taylor approximation.

\begin{theorem}[\citep{Gower2019}]
	For sampling $\mathcal S$, constant $\tau \eqdef E{|\mathcal S|}$ function $f:\R^d \to \R$ is $\mu$--strongly convex (\Cref{def:intro_convex}), $\mathcal L(\tau)$--expected smooth (\Cref{def:intro_expected_smoothness})\footnote{Note that constants $\mathcal L(\tau), \mu$ are dependent on $\tau$ and particular sampling strategy.} and with finite gradient variance at the optimum, $ \sg(x^*,\tau)<\infty$\footnote{Gradient variance at point $x$ is defined as $\sg(x, \tau) \eqdef \Ed{v \sim \mathcal{D}}{\norm{ \nabla f_v(x) }^2}$.}. For any $\epsilon > 0$, if the learning rate $\gamma$ is set to be
	\begin{equation}
		\gamma = \frac 1 2 \min \left\{ \frac{1}{\mathcal{L}(\tau)}, \frac{\epsilon \mu}{2\sg(x^*,\tau)} \right\} 
	\end{equation}
	and the number of iteration $k$ satisfies
	\begin{equation}  \label{eq:intro_iteration_complexity}
		k \geq \frac 2 {\mu} \max\left\{ \mathcal{L}(\tau), \frac{2\sg(x^*, \tau)}{\epsilon \mu} \right\} \log\left( \frac{2 \norm{x^0 - x^*}^2}{\epsilon} \right),
	\end{equation}
	then $\E{ \norm{ x^k - x^* }^2 } \leq \epsilon$.
\end{theorem}
By utilizing the preceding theorem, \citet{Gower2019} acquires estimates for both the expected smoothness constant and the gradient noise in terms of the minibatch size $\tau$. This allows us to calculate total computation complexity \eqref{eq:intro_iteration_complexity} as
\begin{equation} \label{eq:intro_T}
	T(\tau) \eqdef \frac 2 {\mu} \max \left\{ \tau\mathcal L(\tau), \frac{2}{\epsilon \mu}\tau \sg(x^*,\tau) \right\} \log\left( \frac{2 \norm{x^0 - x^*}^2}{\epsilon} \right),
\end{equation}
where $\tau$ denotes minibatch size, $\cL(\tau)$ is the expected smoothness constant, and $\sg(x^*,\tau)$ is the variance of gradients at the optimum. 

The optimal minibatch size $\tau^*$ minimizes $\max \left\{ \tau\mathcal L(\tau), \frac{2}{\epsilon \mu}\tau \sg(x^*,\tau) \right\}$.
Notably, both $\tau \mathcal L(\tau)$ and $\tau \sg(x^*,\tau)$ are piece-wise linear in $\tau$ for considered samplings (Lemmas \ref{le:L}, \ref{le:variance}). This enables expressing the $\tau^*$ in the closed form.

\subsection{Approximating the optimal minibatch size}
Note that as the gradient variance at the optimum is often unknown, the optimal minibatch size formula is not computable.
Our proposed approach involves a heuristic methodology for real-time estimation of the optimal minibatch size. The algorithm initializes with an optimal minibatch size estimate $\tau^0$. In the $k$-th iteration, given a sampling technique $\mathcal S^k$, it estimates gradient noise at the optimum $\sigma(x^*,\tk)$ by gradient noise at the current model, $\sigma(x^k, \tk)$. This estimate, coupled with the expected smoothness $\mathcal{L}(\tau^k)$, is then used to compute the stepsize, the next estimate for the optimal minibatch size $\tau^{k+1} (x^k)$, and to conduct an \sgd{} step 
\begin{equation} 
	x^{k+1} = x^k - \gamma^k \sum \limits_{i\in \mathcal S^k} \nabla f_i(x^k),
\end{equation}
where $\mathcal S^k \subseteq \{1,2,\dots,n\}, |\mathcal S^k| = \tau^k$. The summary can be found in Algorithm \ref{alg:intro_sgd_dyn_general}.

\begin{algorithm}[H]
	\caption{\sgd{} with Adaptive Batch size}
	\label{alg:intro_sgd_dyn_general}
	\begin{algorithmic}[1]
		\State \textbf{Input:} Smoothness constant $\cL(\tau)$, strong convexity constant $\mu$, target neighborhood $\epsilon$, strategy $S$ of sampling stochastic gradients $\g_v(x)$, upper bound on gradient variance $C \geq 0$ (for theoretical purposes), $x^0 \in \R^d$
		\For{$k=0, \dots, K$}
		\State $\tau^k = \tau(x^k)$ 
		\State $\lk = \mathcal L(\tau^k)$
		\State $\sg^k = \sg(x^k, \tau^k)$
		\State $\gamma^k = \frac 1 2 \min \left\{\frac 1 {\lk}, \frac{\epsilon \mu}{\min(C,2\sg^k)} \right\}$ 
		\State Sample $v^k$ using $S$
		\State $x^{k+1} = x^k - \gamma^k \nabla f_{v^k} (x^k)$
		\EndFor
		\State \textbf{Output:} $x^K$
	\end{algorithmic}
\end{algorithm}
Although \Cref{alg:intro_sgd_dyn_general} is a heuristic, we provide a stepsize bound and convergence guarantees under the assumption that gradient variance is bounded by a constant $C$. Denote $\gamma_{\max} \eqdef \frac 1 2\max_{\tau \in \{1,\dots, n\}} \left\{ \frac 1 {\mathcal L(\tau)} \right\}$ and $ \gamma_{\min} \eqdef \frac 1 2\min \left\{ \min_{\tau \in \{1,\dots, n\}} \left\{\frac 1 {\mathcal L(\tau)} \right\}, \frac{\epsilon \mu} C \right\}$. 
\begin{lemma}\label{le:intro_step_sizes_positive}
	Stepsizes $\gamma^k$ generated by Algorithm \ref{alg:intro_sgd_dyn_general} are bounded, $\gamma_{\min} \leq \gamma^k \leq \gamma_{\max}$.
\end{lemma}
\begin{theorem}\label{th:intro_convergence_our}
	Assume $f$ is $\mu$--strongly convex, $\cL(\tau)$--smooth, and let Assumption~\ref{as:gradient_noise} hold. 
	Then the iterates of Algorithm \ref{alg:intro_sgd_dyn_general} satisfy
	\begin{equation}
		\E{\norm{x^k - x^*}^2} \leq \left(1-\gamma_{\min}\mu\right)^k\norm{x^0-x^*}^2 + \frac{2\gamma_{\max}^2\sg^*}{\gamma_{\min}\mu}.
	\end{equation} 	
\end{theorem}
Theorem \ref{le:intro_step_sizes_positive} guarantees the convergence of the proposed algorithm. In the case of the fixed learning rate $\gamma^k = \gamma$, it recovers the results of \citet{SGD_general_analysis}.

It should be noted that while the presented results do not offer substantial theoretical improvements over the fixed minibatch size scenarios analyzed by \citet{SGD_general_analysis}, our contributions lie in practical considerations, where the proposed heuristic methodology proves advantageous. Using adaptive batch size we can get a comparable performance as if we knew the optimal minibatch size in advance. Detailed explanation can be found in \Cref{sec:adaptive}.

\section{Adaptation to geometry} 
Progressing further in reducing the number of hyperparameters, the utilization of second-order methods offers a noteworthy avenue. In contrast to the first-order methods, second-order algorithms possess a distinct advantage: they can remain agnostic to the choice of basis and have local convergence rates independent of the conditioning of the underlying problem.

One of the main benefits of second-order methods is that they have the potential to achieve up to quadratic convergence $\norm{\g(x^{k+1})} \leq c \norms
{\g(x^k)}$ (for a constant $c>0$ locally). However, the decrease is only local. That is, when initialized far from the solution, many classical second-order methods diverge both in theory and in practice (\citet{jarre2016simple}, \citet{mascarenhas2007divergence}). 
One recourse is to commence training with a first-order method, which guides convergence in the vicinity of the solution, and subsequently transition to second-order techniques for swift accuracy refinement. However, determining the optimal time for this transition introduces an additional layer of consideration, which runs counter to the ethos of adaptivity.

To tackle this problem, we propose a very simple second-order algorithm that is independent of the choice of basis and converges globally at a fast rate.

Our method is based on the work of
\citet{nesterov2006cubic}, which demonstrated that the incorporation of cubic regularization can facilitate fast $\okd$ global convergence, albeit at the cost of a more complex algorithm structure and implicit subproblem to be solved in each iteration. 
Compared to the work of \citet{nesterov2006cubic}, we show that complex algorithm structure and implicit subproblem are not necessary for the fast $\okd$ rate. We present a simple stepsize schedule for the Newton method with better geometrical properties and convergence rates than cubically regularized methods.

\subsection{Second-order methods}
In order to present the motivation behind our method, let us present second-order methods from a different viewpoint. One of the most famous algorithms in optimization, the Newton method \citep{Newton, Raphson}, takes iterates of the form 
\begin{equation} \label{eq:intro_newton_step}
	x^{k+1} \eqdef x^k - \left[\nabla^2 f(x^k) \right]^{-1} \nabla f(x^k).
\end{equation}
The update rule of \newton{} \eqref{eq:intro_newton_step} was chosen to minimize second order Taylor approximation
\begin{equation} \label{eq:intro_newton_approximation}
	T_{f}(y;x) \eqdef f(x) + \la\nabla f(x),y-x\ra + \frac{1}{2} \la \nabla^2 f(x)(y-x),y-x\ra,
\end{equation}
of $f(y)$ in $y$ at $x=x^k$. That is, $x^{k+1}=\argmin_{y \in \R^d} T_f(y;x)$.
The main problem with design of this algorithm is that Taylor approximation is not an upper bound on $f(y)$, and therefore, global convergence of \newton{} is not guaranteed \citep{jarre2016simple}, \citep{mascarenhas2007divergence}.
Many papers proposed globalization strategies; essentially all of them require some combination of the following: line-search \citep{nesterov1994interior}, trust regions \citep{conn2000trust}, damping/truncation \citep{nesterov1994interior}, and regularization \citep{nesterov2006cubic}. 
To this day, virtually all known global convergence guarantees are for regularized Newton methods (\rnewton{}), which can be written as
\begin{equation} \label{eq:intro_LM_update}
	x^{k+1} = x^k - \alpha^k \left( \h(x^k) + \lambda^k \mI \right)^{-1} \g(x^k),
\end{equation}
where $\mI$ is the $d \times d$ identity matrix and $\lambda^k \geq 0$.
Parameter $\lambda^k$ is also known as Levenberg-Marquardt regularization \citep{levenberg1944method, marquardt1963algorithm, more1978levenberg}, which was originally introduced for nonlinear least-squares objectives. 
The motivation behind \eqref{eq:intro_LM_update} is to replace Taylor approximation \eqref{eq:intro_newton_approximation} by an upper bound. 
Cubic Newton method \citep{nesterov2006cubic} for function $f$ with Lipschitz-continuous Hessian,
\begin{equation} 
	\gboxeq{ \norm{\nabla^2 f(x) - \nabla^2 f(y)} \leq L_2\norm{x-y}, \qquad \forall x,y \in \R^d}
\end{equation} 
leverages the upper bound
\begin{equation} \label{eq:intro_cubic_newton_approximation}
	f(y) \leq T_{f}(y;x) +\frac{L_2}{6}\norm{y-x}^3
\end{equation}
to set the next iterate as a minimizer of the right-hand side of \eqref{eq:intro_cubic_newton_approximation}
\begin{equation} \label{eq:intro_cubic_newton_step}
	x^{k+1} = \argmin_{y\in \R^d} \lb T_{f}(y;x^k) +\frac{L_2}{6}\norm{y-x^k}^3 \rb.
\end{equation}
Our newly-proposed algorithm \ain{} (\Cref{alg:intro_ain}) uses an almost identical step\footnote{Function $f$ is $\Lsemi$--semi-strongly self-concordant (\Cref{def:intro_semi-self-concordance}). Instead of $\Lsemi$, we use its upper bound $\Lalg$, $\Lalg \geq \Lsemi$.}:
\begin{equation} \label{eq:intro_aicn_step}
	\gboxeq{x^{k+1} =  \argmin_{y\in \R^d} \left\lbrace T_{f}(y;x^k) +\frac{\Lsemi}{6}\norm{y-x^k}_{x^k}^3\right\rbrace.}
\end{equation}
The difference between the update of \cnewton{} and \ain{} is that we measure the cubic regularization term in the local Hessian norms, which are for fixed $x$ defined as $\normsM a {\h(x)} \eqdef a^\top \h(x) a$, with shorthand notation $\normM \cdot x \eqdef \normM \cdot {\h(x)}$. As we measure distances in local norms, we replace Lipschitz-continuous Hessian with a variant of self-concordance with the constant $\Lsemi$ (\Cref{def:intro_semi-self-concordance}). This minor modifications turned out to be of great significance for two reasons:
\begin{enumerate}
	\item The model in \eqref{eq:intro_aicn_step} is affine-invariant and hence independent of the choice of the basis.
	\item Surprisingly, the next iterate of \eqref{eq:intro_aicn_step} lies in the direction of the \newton{} step and is obtainable without regularizer, $\lambda^k=0$ (\ain{} just sets the stepsize $\alpha^k$).
	
	On the other hand \cnewton{} \eqref{eq:intro_cubic_newton_step} can be equivalently expressed in form \eqref{eq:intro_LM_update} with $\alpha^k=1$ and $\lambda^k = L_2 \norm{x^k -x^{k+1}}$. However, since such $\lambda^k$ depends on $x^{k+1}$, the resulting algorithm requires an additional subroutine for solving its subproblem in each iteration.
\end{enumerate}
Note that the subroutine used in \cnewton{} can be avoided. \citet{polyak2009regularized} chose $\lambda^k\propto \norm{\g(x^k)}$. However, this came with a trade-off in terms of the slower convergence rate $\cO\left( k^{-1/4}\right)$.
Finally, \citet{mishchenko2021regularized} and \citet{doikov2021optimization} improved upon both of these works by using explicit regularization $\alpha^k=1$, $\lambda^k \propto \sqrt{L_2 \norm{\g(x^k)}}$, and for the cost of slower local convergence achieving global rate $\cO\left(k^{-2}\right)$.

\subsection{Affine-invariant geometry} \label{sec:intro_affine_invariance}
One of the main geometric properties of the \newton{} method is its \textit{affine invariance}, i.e., invariance to affine transformations of variables.
Let $\mathbf A: \R^d \rightarrow \R^d$ be a non-degenerate linear transformation enabling change of variables $x \to y$ via $y=\mathbf A^{-1} x$. Instead of minimizing $f(x)$ we minimize $\phi(y) \eqdef f(\mathbf Ay)=f(x)$. 
\paragraph{Adaptivity of local norms:}
The local Hessian norm $\norm{h}_{\nabla f(x)}$ is affine-invariant. Indeed, if $h=\mathbf A z$ and $x=\mathbf A y$, then
\begin{equation}
    \norm{z}_{\nabla^2 \phi(y)}^2=\la\nabla^2 \phi(y)z,z \ra = \la \mathbf A^\top \nabla^2 f(\mathbf A y) \mathbf Az,z\ra = \la \nabla^2 f(x)h,h\ra=\norm{h}_{\nabla^2 f(x)}^2,  
\end{equation}
where $\la a, b \ra \eqdef a^\top b$ denotes the inner product. On the other hand, the standard Euclidean norm $\norm{h}_\mI$ is not affine-invariant since
\begin{equation}
    \norm{z}_\mI^2=\la z,z \ra = \la \mathbf A^{-1}h, \mathbf A^{-1}h\ra = \norm{\mathbf A^{-1} h}^2_\mI \not = \normsM h {\mI}.  
\end{equation}
With respect to geometry, the local Hessian norm is more natural. From affine invariance follows that for this norm, the level sets $\lb y\in \R^d\,|\,\norm{ y-x }_x^2 \leq c\rb$ are balls centered around $x$ (all directions have the same scaling). In comparison, the scaling of the standard Euclidean norm is dependent on the eigenvalues of the Hessian. In terms of convergence, one direction in $l_2$ can significantly dominate others and slow down an algorithm.

\paragraph{Significance for algorithms:}
Algorithms that are not affine-invariant can suffer from a bad choice of the coordinate system. This is the case for \cnewton{}, as its model \eqref{eq:intro_cubic_newton_step} is bound to basis-dependent $l_2$ norm. The same is true for any other method regularized with an induced norm $\norm{h}_\mI$.

\subsection{Algorithm}
Finally, our algorithm \ain{} can be written as a \dnewton{} method with updates
\begin{equation} \label{eq:intro_update}
	x^{k+1} = x^k - \alpha^k \h(x^k)^{-1} \g(x^k),
\end{equation}
and stepsize dependent on the gradient measured in the dual local Hessian norm, $\normMd \cdot x \eqdef \normM \cdot {\left[\h(x)\right]^\dagger}$\footnote{Symbol $\dagger$ denotes Moore-Penrose peseudoinverse.},
\begin{equation}
    \alpha^k \eqdef \frac {-1+\sqrt{1+2 \Lalg \normMd {\g(x^k)} {x^k}}}{\Lalg \normMd {\g(x^k)} {x^k}},
\end{equation} 
as summarized in \Cref{alg:intro_ain}. The constant $\Lalg$ is an upper bound on the semi-strong self-concordance constant $\Lsemi$.
The stepsize satisfies $0<\alpha^k \leq 1$ and $\lim_{x^k \rightarrow \opt} \alpha^k = 1$, hence \eqref{eq:intro_update} converges to the \newton{} iteration.
Next, we are going to discuss the geometric properties of our algorithm.
\begin{algorithm} [h]
	\caption{\ain{}: Affine-Invariant Cubic Newton} \label{alg:intro_ain}
	\begin{algorithmic}[1]
		\State \textbf{Requires:} Initial point $x^0 \in \R^d$, constant $\Lalg$ such that $\Lalg \geq \Lsemi>0$
		\For {$k=0,1,2\dots$}
		\State $\alpha^k = \frac{-1 +\sqrt{1+2 \Lalg \norm{\g(x^k)}_{x^k}^* }}{\Lalg \norm{\g(x^k)}_{x^k}^* }$
		\State $x^{k+1} = x^k - \alpha^k \left[\h(x^k)\right]^{-1} \g(x^k)$ \Comment {Equivalent update: \eqref{eq:intro_aicn_step}.}
		\EndFor
	\end{algorithmic}
\end{algorithm}

\begin{itemize}[leftmargin=*]
	\item \textbf{Fast global convergence}: \ain{} converges globally with rate $\cO\left( k^{-2} \right)$, 
	which matches the state-of-the-art global rate for all regularized Newton methods \citep{nesterov2006cubic, mishchenko2021regularized,doikov2021local}. Furthermore, it is the first such rate for a damped (stepsized) \newton{} method.
	\begin{theorem} \label{thm:intro_convergence}
		Let $f:\R^d\to \R$ be a $\Lsemi$--semi-strongly self-concordant (\Cref{def:intro_semi-self-concordance}) convex function with positive-definite Hessian, constant $\Lalg$ satisfying $\Lalg \geq \Lsemi$ and with level set 
        $\level\eqdef\lb x \in \R^d : f(x)\leq f(x^0) \rb$ bounded by a constant $D$ as $\sup\limits_{x\in\level} \normM {x-x_{*}} {x} \leq D<+\infty$. Then, after $k+1$ iterations of Algorithm \ref{alg:intro_ain}, we have the following convergence: 
		\begin{equation}
			f(x^{k+1}) - f(\opt)\leq O\ls\frac{\Lalg D^{3}}{k^{2}}\rs.
		\end{equation}
	\end{theorem}
	
	\item \textbf{Fast local convergence:} In addition to the fast global rate, \ain{} decreases gradient norms locally at a quadratic rate. This result matches the best-known rates for both regularized Newton algorithms and damped Newton algorithms.
	\begin{theorem} \label{th:intro_local}
		Let function $f: \R^d \to \R$ be $\Lsemi$--semi-strongly self-concordant (\Cref{def:intro_semi-self-concordance}), 
		$\Lalg \geq \Lsemi$ and starting point $x^0$ be in the neighborhood of the solution such that $\normMd{\g(x^0)} {x^0} \leq \frac{8} {9\Lalg} $.
		For $k\geq 0$, we have a quadratic decrease of the gradient norms,
		\begin{align}
			\normMd{\g(x^{k})}{x^{k}}
			&\leq \left( \frac{3} {2} \Lalg \right )^k 
			\left( \normMd{\g(x^0)} {x^0} \right)^{2^k} .
		\end{align}
	\end{theorem}
	
	\item \textbf{Iteration computation:}
	Given a smoothness constant $\Lsemi$, the next iterate of \ain{} can be computed directly. 
	This is an improvement over \cnewton{} \citep{nesterov2006cubic}, which for a given constant  $L_2$ needs to run a line-search subroutine to solve its implicit subproblem in each iteration.
	
	Avoiding the extra subroutine also yields theoretical improvements.
	If we compute matrix inverses naively, the iteration cost of \ain{} is $\cO(d^3)$ (where $d$ is a dimension of the problem), which is an improvement over the 
	$\cO(d^3 \log \varepsilon^{-1})$ iteration cost of \cnewton{} \citep{nesterov2006cubic}.
	
	\item \textbf{Geometric adaptivity:} We analyze \ain{} under more geometrically natural assumptions. Instead of smoothness, we use a version of self-concordance (\Cref{def:intro_semi-self-concordance}), which is invariant to affine transformations and hence also to a choice of a basis. 
	\ain{} preserves affine invariance obtained from assumptions throughout the convergence.
	In contrast, \cnewton{} uses the basis-dependent $l_2$ norm and hence depends on a choice of a basis. This represents an extra layer of complexity.
	
	\item \textbf{Practical performance:} We show that in practice, \ain{} outperforms all algorithms sharing the same convergence guarantees: \cnewton{} method \citep{nesterov2006cubic}, and globally \rnewton{} method \citep{mishchenko2021regularized, doikov2021optimization}, and fixed stepsize \newton{} method \citep{Newton, KSJ-Newton2018, RSN} (Sec.~\ref{sec:experiments_aicn}).

 More detailed explanations and proofs can be found in \Cref{sec:aicn}.
\end{itemize} 

\section{Adaptivity to high dimensions}
The main drawback of second-order algorithms is their poor scalability in the context of modern large-scale machine learning. Large datasets characterized by numerous features necessitate well-scalable algorithms. 
While strategies such as employing approximations or inexact computations to circumvent the computation of the inverse Hessian are feasible, inverting or just simply storing the Hessian becomes infeasible as the dimensionality $d$ grows.
This challenge has served as a catalyst for the recent developments in the field. To address curse of dimensionality, \citet{SDNA}, \citet{luo2016efficient}, \citet{RSN}, \citet{RBCN}, and \citet{hanzely2020stochastic} proposed Newton-like methods that operate in random low-dimensional subspaces. For a model $x^k \in \R^d$, a thin sketching matrix $\sk \in \R^{d \times \tau(\sk)}$, and a low-dimensional update $h^k \in \R^{\tau(\sk)}$ (dependent on $x^k$ and $\sk$), its update rule can be written as
\begin{equation}
    x^{k+1}=x^k+ \sk h^k.
\end{equation}
This approach, also known as \sap{} \citep{gower2015randomized} has the advantage of substantially reducing the computational cost per iteration. However, this happens at the cost of slower, $\cO \left( k^{-1}\right)$, convergence rate \citep{gower2020variance, hanzely2020stochastic}.

In this work, we argue that the \sap{} adaptations of second-order methods can be improved. To this end, we introduce the first \sap{} method (Sketchy Global Newton, \sgn{}, \Cref{alg:intro_sgn}) which boasts a global \textbf{$\okd$} convex convergence rate. This achievement aligns with the rapid global rate of full-dimensional \rnewton{} methods. 
In particular, sketching on $1$-dimensional subspaces engenders $\okd$ global convex convergence with an iteration cost of $\cO(1)$. 
As a cherry on the top, our approach offers a local linear convergence rate independent of the condition number	and a global linear convergence rate under the assumption of relative convexity (\Cref{def:intro_sc_rel}).

\subsection{Three faces of the algorithm}
Our algorithm accommodates the strengths of three distinct lines of research (as outlined in \Cref{tab:intro_three_ways}), and can be expressed in three equivalent formulations. From now on, we shall denote the gradients and Hessians along the subspace spanned by columns of $\s$ as {$\gS (x) \eqdef \st \g (x)$} and {$\hS (x) \eqdef \st \h(x) \s$}, the sketched local norms in the range of $\s$ will be denoted by $\normMS h x \eqdef \normM {\s^\top h} x$, $\normMSd \cdot x \eqdef \normM {\cdot} {\left[ \hS(x)\right]^{-1} }$.

\begin{algorithm} [t!]
	\caption{\sgn{}: Sketchy Global Newton } \label{alg:intro_sgn}
	\begin{algorithmic}[1]
		\State \textbf{Requires:} Initial point $x^0 \in \R^d$, distribution of sketch matrices $\cD$ such that $\mathbb E_{\s \sim \cD} \left[ \p \right] = \afrac \tau d \mI$, 
		constant $\Lalg$ upper bounding semi-strong self-concordance constants in the sketched directions (\Cref{def:intro_scs}) $\Lalg \geq \sup_{\s \sim \cD} \Ls$ 
		\For {$k=0,1,2\dots$}
		\State Sample $\s_k \sim \cD$
		\State $\als k = \afrac {-1 +\sqrt{1+2 \Lalg \normMSdk{\gSk(x^k)} {x^k} }}{\Lalg \normMSdk {\gSk(x^k)} {x^k} }$
		\State $x^{k+1} = x^k - \als k \s_k \left[\hSk(x^k)\right]^{\dagger} \gSk(x^k)$        
		\Comment {Equiv. to \eqref{eq:intro_sgn_reg} and \eqref{eq:intro_sgn_sap}}
		\EndFor
	\end{algorithmic}
\end{algorithm}
\begin{theorem} 
	\label{th:intro_three}
	If $\g(x^k) \in \Range{\h(x^k)}$\footnote{$\Range{\mathcal A}$ denotes column space of the matrix $\mathcal A$.}, then the update rules \eqref{eq:intro_sgn_reg}, \eqref{eq:intro_sgn_aicn}, and \eqref{eq:intro_sgn_sap} are equivalent. 
	\begin{align}
		\text{\Rnewton{} step:} \quad x^{k+1} &= x^k + \sk\argmin_{h\in \mathbb{R}^d} \modelSk{x^k, h}, \label{eq:intro_sgn_reg}\\
		\text{\Dnewton{} step:} \quad x^{k+1} &= x^k - \als k \sk [\hSk(x^k)]^\dagger \gSk(x^k), \label{eq:intro_sgn_aicn}\\
		\text{\Sap{} step:} \quad x^{k+1} &= x^k - \als k \pk k [\h(x^k)]^\dagger \g(x^k), \label{eq:intro_sgn_sap}
	\end{align}
	where $\p \eqdef \s \left( \st \h (x) \s \right)^\dagger \st \h(x)$ is a projection matrix onto $\Range{\s}$ with respect to the norm $\normM \cdot x$,
	\begin{eqnarray}
		\modelS {x,h} &\eqdef& f(x) + \la \g(x), \s h \ra + \frac 12 \normsM {\s h} x + \frac {\Lalg}6 \normM {\s h} x ^3 \nonumber \\
		&=& f(x) + \la \gS(x), h \ra + \frac 12 \normsMS {h} x + \frac {\Lalg}6 \normMS {h} x ^3, \label{eq:intro_sgn_Ts} \\
		\als k &\eqdef& \frac{-1 +\sqrt{1+2 \Lalg \normMSdk{\gSk(x^k)} {x^k} }}{\Lalg \normMSdk {\gSk(x^k)} {x^k} }. \label{eq:intro_sgn_alpha}
	\end{eqnarray}    
    We call this algorithm Sketchy Global Newton, \sgn{}, it is formalized as \Cref{alg:intro_sgn}.
\end{theorem}
Notice that $0< \als k \leq 1$ and in the limit cases
\begin{equation}
\als k \xrightarrow{\Lalg \normMSdk{\gSk(x^k)} {x^k} \rightarrow 0} 1
\end{equation}
and
\begin{equation}
\als k \xrightarrow{\Lalg \normMSdk{\gSk(x^k)} {x^k} \rightarrow \infty} 0.
\end{equation}

\subsection{Benefits}
We prove convergence guarantees for \sgn{} under Lipschitz smoothness of theHessians formulated in local norms, which \citet{hanzely2022damped} introduced as semi-strong self-concordance (\Cref{def:intro_semi-self-concordance}).
We summarize its strengths below; for comprehensive comparison please see \Cref{sec:sgn}.

\begin{itemize}[leftmargin=*] 
	\setlength\itemsep{0.2em}
	\item \textbf{One connects all:} We present \sgn{} through three orthogonal viewpoints: \sap{} method, subspace \newton{} method, and subspace \rnewton{} method. Compared to established algorithms, \sgn{} can be viewed as \ain{} operating in subspaces, \sscn{} operating in local norms, or \rsn{} with the new stepsize schedule (\Cref{tab:intro_three_ways}).
	
	\item \textbf{Fast global convergence:}
	\sgn{} is the first low-rank method that solves \textbf{convex} optimization problems with $\okd$ global rate. This matches the state-of-the-art rates of full-rank Newton-like methods. Other \sap{} methods, in particular, \sscn{} and \rsn{}, have slower $\mathcal O \left( k^{-1} \right)$ rate (\Cref{tab:intro_setup_comparison}).
	
	\begin{theorem} \label{th:intro_global_convergence}
		For $\Lsemi$--semi-strongly concordant (\Cref{def:intro_semi-self-concordance}) function $f: \R^d \to \R$ with finite diameter of initial level set $\level$, $D<\infty$ and distribution of sketching matrices such that $\mathbb E_{\s \sim \cD} \left[ \p \right] = \afrac \tau d \mI$, \sgn{} has the following global convergence rate
		\begin{align}        
			\E{f(x^k) -\fopt} 
			\leq  \afrac {4 d^3 (f(x^0)-\fopt)}{\tau^3 k^3} + \afrac {9(\max \Lalg + \Lsemi) d^2 D^3} {2 \tau^2 k^2} = \cO(k^{-2}).
		\end{align}
	\end{theorem}

	\begin{table*}[h]
		\centering
		\setlength\tabcolsep{3pt} 
		\begin{threeparttable}[t]
			{\scriptsize
				\renewcommand\arraystretch{3}
				\caption[Global convex convergence rates of low-rank Newton methods]{
					Global convergence rates of low-rank Newton methods for convex functions with Lipschitz smooth gradients and Hessians. For simplicity, we disregard differences between various notions of smoothness. We use the fastest full-dimensional algorithms as the baseline, we highlight the best rate in blue.}
				\label{tab:intro_setup_comparison}
				\centering 
				\begin{tabular}{?c?c|c?}
					\Xhline{2\arrayrulewidth}
					
					\makecell{\\ \\ \textbf{Update} \\ \textbf{direction}} \makecell{ \textbf{Update}\\ \textbf{oracle} \\ \\} & \makecell{\textbf{Full-dimensional} \\ (direction is deterministic)} & \makecell{\textbf{Low-rank} \\ (direction in expectation)}\\
					
					\Xhline{2\arrayrulewidth}
					\makecell{\textbf{Non-Newton} \\ \textbf{direction} }
					& \makecell{ {\color{blue} $\cO(k^{-2})$} \\
						Cubically regularized Newton \\  \citep{nesterov2006cubic}, \\
						Globally regularized Newton \\ \citep{mishchenko2021regularized, doikov2021optimization} } 
					
					& \makecell{$\cO(k^{-1})$ \\ 
						Stochastic Subspace Cubic Newton \\ \citep{hanzely2020stochastic} }\\               
					
					\hline
					\makecell{\textbf{Newton} \\ \textbf{direction}} 
					& \makecell{ {\color{blue} $\cO(k^{-2})$} \\ 
						Affine-Invariant Cubic Newton  \\ \citep{hanzely2022damped}}
					
					& \makecell{ {\color{blue} $\cO(k^{-2})$} \\ 
						\textbf{Sketchy Global Newton}  \textbf{(this work)} \\ \hline
						{$ \cO(k^{-1})$}\\ 
						Randomized Subspace Newton \\ \citep{RSN}}\\
					\Xhline{2\arrayrulewidth}
				\end{tabular}
			}
		\end{threeparttable}
	\end{table*}
	
	\item \textbf{Cheap iterations:} \sgn{} uses $\tau$--dimensional updates. Per-iteration cost is $\mathcal O \left( d\tau^2 \right)$ and in the $\tau=1$ case it is even $\mathcal O \left(1\right).$ Conversely, full-rank Newton-like methods have cost proportional to $d^3$  and $d \gg \tau$.
	
	\item \textbf{Linear local rate:} \sgn{} has local linear rate $\cO \left( \frac d \tau \log \frac 1 \varepsilon \right)$ dependent only on the ranks of the sketching matrices. This improves over the condition-dependent linear rate of \rsn{} or any rate of first-order methods.
	
	\begin{theorem}\label{th:intro_local_linear}
		Let function $f:\R^d \to \R$ be $\Ls$--self-concordant in subspaces $\s \sim \cD$ (\Cref{def:intro_scs}), and distribution of sketching matrices such that $\mathbb E_{\s \sim \cD} \left[ \p \right] = \afrac \tau d \mI$.
		For iterates $x^0, \dots, x^k$ of \sgn{} such that $\normMSdk {\gSk(x^k)} {x^k} \leq \frac 4 {L_{\sk}}$, we have local linear convergence rate
		\begin{equation}
			\E{f(x^{k})-\fopt)} \leq \left( 1- \afrac \tau {4d} \right)^k (f(x^0)-\fopt),
		\end{equation}
		and the local complexity of \sgn{} is independent on the conditioning, $\mathcal O \left( \frac d \tau \log \frac 1 \varepsilon \right).$
	\end{theorem}
	
	We also show a quadratic decrease of the gradient norm in the sketched direction.
	\begin{lemma} \label{le:intro_local_step_subspace}
		For $\Lsemi$--semi-strong self-concordant function $f: \R^d \to \R$ and parameter choice $\Lalg \geq\Lsemi$, one step of \sgn{} has quadratic decrease in $\Range{\sk}$, i.e.,
		\begin{equation} \label{eq:intro_local_step_subspace}
			\normMSdk {\gSk(x^{k+1})} {x^k} \leq \Lalg \als k^2\normsMSdk{\gSk(x^k) }{x^k}.
		\end{equation}
	\end{lemma}
	Nevertheless, this is insufficient for superlinear local convergence; sketch-and-project can achieve a linear rate at best (\Cref{sec:limits}).
	
	\item \textbf{Global linear rate:}  Under $\murel$--relative convexity, \sgn{} achieves global linear rate $\cO \left( \frac {\Lalg} {\rho \murel } \log \frac 1 \varepsilon \right)$ to a neighborhood of the solution\footnote{$\rho$ is condition number of a projection matrix \eqref{eq:rho}, constant $\Lalg$ affects stepsize \eqref{eq:intro_sgn_alpha}.}. 
	
	\begin{theorem} \label{th:intro_global_linear}
		Let $f:\R^d \to \R$ be $\Ls$--relative smooth in subspaces $\s$ and $\murel$--relative convex. Let sampling $\s\sim \cD$ 
		satisfy $\Null{\s^\top \h(x) \s} = \Null{\s}$
		and $\Range{\h(x)} \subset \Range {\bbE_{\s \sim \cD}\left[\s \s ^\top\right] }$.     
		Then $0 < \rho \leq 1$.
		Choose parameter $\Lalg = \sup_{\s \sim \cD} \afrac 98 \Ls \Lrel _\s^2$.
		
		While iterates $x^0, \dots, x^k$ satisfy $\normMSdk {\gSk(x^k)} {x^k} \geq \frac 4 {L_{\sk}}$, then \sgn{} has decrease
		\begin{equation}        
			\E{f(x^{k}) - \fopt}  \leq \left(1 - \afrac 43 \rho \murel \right)^k (f(x^0) - \fopt), 
		\end{equation}
		and global linear $\cO\left( \afrac 1 {\rho \murel} \log \afrac 1 \varepsilon \right)$ convergence.
	\end{theorem}
	
	\item \textbf{Geometric adaptivity:}
	Update of \sgn{} uses well-understood projections\footnote{\citet{gower2020variance} describes six equivalent viewpoints.} of Newton method with stepsize schedule \ain{}. Algorithm \sgn{} is affine-invariant (which removes the choice of basis from parameter tuning) and enjoys simple and explicit update rule \eqref{eq:intro_sgn_aicn}. On the other hand, implicit steps of regularized Newton methods including \sscn{} lack geometric interpretability and often require an extra subproblem to solve in each iteration.
\end{itemize}

\begin{table*}[!h]
	\centering
	\setlength\tabcolsep{2pt} 
	\begin{threeparttable}[h]
		{
			\scriptsize
			\caption[Three approaches for second-order global minimalization]{
				Three approaches for second-order global minimization. We denote $x^k\in \mathbb{R}^d$ iterates, $\sk \sim \cD$ distribution of sketches of rank $\tau \ll d$, $\alpha^k, \als k$ stepsizes, $L_2, \Ls$ smoothness constants, $c_{\text{stab}}$ Hessian stability constant. For simplicity, we disregard differences in assumptions and report complexities for matrix inverses implemented naively. }
			\label{tab:intro_three_ways}
			\centering 
			\begin{tabular}{?c?c|c|c?}
				\bottomrule
				\makecell{\textbf{Orthogonal}\\ \textbf{lines of}\\{\textbf{work}}} 
				& \makecell{\textbf{Sketch-and-Project} \citenum{gower2015randomized} \\ (various update rules)} 
				& \makecell{\textbf{Damped Newton} \citenum{nesterov1994interior},  \citenum{KSJ-Newton2018} } 
				& \makecell{\textbf{Globally Reg. Newton} \\ \citenum{nesterov2006cubic}, \citenum{polyak2009regularized},  \citenum{mishchenko2021regularized}, \citenum{doikov2021optimization}\tn{1} }\\
				
				\Xhline{2\arrayrulewidth}
				\makecell{\textbf{Update}\\ $x^{k+1} - x^k =$} 
				& \makecell{$\alpha_{k,\sk} \pk k \left( \text{update}(x^k) \right)$,\\ for $\sk \sim \cD$}
				& $\alpha^k [\h(x^k)]^\dagger \g(x^k)$
				& \makecell{$\argmin_{h\in \mathbb{R}^d} T(x^k, h)$, \,\\ for $T(x,h) \eqdef \la \g(x), h \ra +$\\$+ \frac 12 \normsM {h} x + \frac {L_2}6 \normM {h} {} ^3$}\\
				
				\hline
				\makecell{\textbf{Characteristics}}
				&\makecell[l]{
					\color{mydarkgreen}
					+ cheap, low-rank updates\\
					\color{mydarkgreen}
					+ global linear convergence\\
					\, (conditioning-dependent)\\
					\color{mydarkred}
					-- optimal rate: linear\\
				}
				& \makecell[l]{
					\color{mydarkgreen}
					+ affine-invariant geometry\\
					\color{mydarkred}
					-- iteration cost $\cO\left(d^3 \right)$\\
					Fixed $\alpha^k = c_{\text{stab}}^{-1}$:\\
					\color{mydarkgreen}
					\, + global linear convergence\\
					Schedule $\alpha^k \nearrow 1$:\\
					\color{mydarkgreen}
					\, + local quadratic rate\\
				}             
				& \makecell[l]{
					\color{mydarkgreen}
					+ global convex rate $\okd$\\
					\color{mydarkgreen}
					+ local quadratic rate\\
					\color{mydarkred}
					-- implicit updates\\
					\color{mydarkred}
					-- iteration cost $\cO\left(d^3 \log \frac 1 \varepsilon \right)$\\
				}\\
				\toprule
				
             \bottomrule
             \makecell{\textbf{Combinations}\\\textbf{+ benefits}} 
             & \makecell{\textbf{Sketch-and-Project}} 
             & \makecell{\textbf{Damped Newton} } 
             & \makecell{\textbf{Globally Reg. Newton} }\\

             \Xhline{2\arrayrulewidth}
             \makecell{\rsn{} \citenum{RSN} \\ \Cref{alg:rsn}}
             & \makecell{ \cmark\\ \makecell[l]{
             \color{mydarkgreen}
             + iter. cost $\cO\left(d \tau^2 \right)$\\ 
             \color{mydarkgreen}
             + iter. cost $\cO\left(1 \right)$ if $\tau=1$\\ 
             }} & \makecell{\cmark\\
             \color{mydarkgreen}
             + global rate $\cO\left( \frac 1\rho \frac \Lrel \murel \log \frac 1 \varepsilon\right)$ \\ }& \xmark\\
             \hline
             \makecell{\sscn{} \citenum{hanzely2020stochastic} \\ \Cref{alg:sscn}}
             &\makecell{\cmark\\ \makecell[l]{
             \color{mydarkgreen}
             + iter. cost $\cO\left(d\tau^2 +\tau^3 \log \frac 1 \varepsilon \right)$\\
             \color{mydarkgreen}
             + iter. cost $\cO\left(\log \frac 1 \varepsilon \right)$ if $\tau=1$\\
             \color{mydarkgreen}
             + local rate $\cO\left( \frac d \tau \log \frac 1 \varepsilon\right)$ \\ }} & \xmark & \makecell{ \cmark \\
             \color{mydarkgreen}
             + global convex rate $\okd$ \\}\\
             \hline
             \makecell{\ain{} \citenum{hanzely2022damped} \\ \Cref{alg:aicn}}
             & \xmark & \makecell{\cmark\\ \makecell[l]{
             \color{mydarkgreen}
             + affine-invariant geometry\\                          
             \color{mydarkred}
             -- no glob. lin. rate proof\tn{2}\\} \, } & \makecell{\cmark\\ \makecell[l]{
             \color{mydarkgreen}
             + global convex rate $\okd$\\
             \color{mydarkgreen}
             + local quadratic rate\\
             \color{mydarkgreen}
             + iteration cost $\cO\left(d^3 \right)$\\
             \color{mydarkgreen}
             + simple, explicit updates\\
             }}\\
             \hline
             \makecell{\textbf{\sgn{}}\\ \textbf{(this work)} \\ \Cref{alg:sgn}}
             & \makecell{ \cmark \\ \makecell[l]{
             \color{mydarkgreen}
             + iter. cost $\cO\left(d\tau^2 \right)$\\
             \color{mydarkgreen}
             + iter. cost $\cO\left(1 \right)$ if $\tau=1$\\
             \color{mydarkgreen}
             + local rate $\cO\left( \frac d \tau \log \frac 1 \varepsilon\right)$ \\
             -- quadratic rate unachievable}}
             & \makecell{\cmark\\ \makecell[l]{
             \color{mydarkgreen}
             + affine-invariant geometry\\
             \color{mydarkgreen}
             + global rate $\cO\left( \frac 1\rho \frac \Lrel \murel \log \frac 1 \varepsilon\right)$\\}} 
             & \makecell{\cmark\\ \makecell[l]{
             \color{mydarkgreen}
             + global convex rate $\okd$\\
             \color{mydarkgreen}
             + simple, explicit updates\\
             \\}} \\

             \toprule
				\bottomrule
				\makecell{\textbf{Three views} \\\textbf{of} \color{mydarkgreen} \textbf{\sgn{}}} & \makecell{\textbf{Sketch-and-Project of}\\\textbf{\color{mydarkgreen} Damped Newton method}} & \makecell{\textbf{Damped Newton} \\ \textbf{\color{mydarkgreen}in sketched subspaces} } & \makecell{\textbf{{Affine-Invariant} Newton} \\ \textbf{\color{mydarkgreen} in sketched subspaces}}\\
				\Xhline{2\arrayrulewidth}
				\makecell{\textbf{Update}\\ $x^{k+1} - x^k =$} 
				& $\als k \pk k \color{mydarkgreen} [\h(x^k)]^\dagger \g(x^k)$ 
				& ${\color{mydarkgreen}\als k} {\color{mydarkgreen} \sk } [\nabla_{\color{mydarkgreen}\sk}f(x^k)]^\dagger \nabla_{\color{mydarkgreen} \sk}f(x^k)$
				& \makecell{${\color{mydarkgreen}\sk} \argmin_{h\in \mathbb{R}^d} T_{\color{mydarkgreen}\sk}(x^k, h)$, \\ for 
					$T_{\color{mydarkgreen} \s} (x,h) \eqdef \la \g(x), {\color{mydarkgreen}\s} h \ra +$\\$+ \frac 12 \normsM {{\color{mydarkgreen}\s} h} x + \frac {\Ls}6 \Vert {{\color{mydarkgreen}\s} h} \Vert ^3 _ {\color{mydarkgreen}x} $} \\          
				\toprule
			\end{tabular}
		}
		\begin{tablenotes}
			{\scriptsize
				\item [\color{red}(1)]Works \citet{polyak2009regularized}, \citet{mishchenko2021regularized}, \citet{doikov2021optimization} have explicit updates and iteration cost $\cO \left(d^3 \right)$, but for the costs of slower global rate, slower local rate, and slower local rate, respectively.
				\item [\color{red}(2)] \citet{hanzely2022damped} didn't show global linear rate of \ain{}. However, it follows from our Theorems \ref{th:global_linear}, \ref{th:local_linear} for $\sk = \mI$.
			}
		\end{tablenotes}
	\end{threeparttable}
\end{table*}

In the preceding sections, we discussed adaptivity to the hyperparameters and geometric properties of the objective function. In the subsequent section, we are going to investigate the adaptivity of the model to the heterogeneity of the data distributions. The imperative for such adaptivity is particularly pronounced in the field of distributed optimization and especially in federated learning. We start by motivating adaptivity in terms of personalization.

\section{Adaptivity to heterogeneity of the data distribution}
Federated learning (FL)~\citep{konevcny2016federated, mcmahan17a, hard2018federated, FL-big} has emerged as a prominent domain within distributed machine learning, facilitating training on dataset residing locally across a multitude of clients. Unlike conventional distributed learning within centralized data centers, the distinctive trait of FL is that each client accesses solely his/her data, which might differ from the overall population significantly. 

While the difference between FL and the rest of machine learning is primarily rooted in the way training is conducted, both scenarios are identical from the modeling perspective. Specifically, the objective of the standard FL is to minimize the aggregate population loss,
\begin{align}\label{eq:intro_fl_standard}
	\min_{z\in \R^d} \frac1n \sum_{i=1}^n f_i(z) = & \min_{x_1, x_2, \dots, x_n \in \R^d} \, \frac1n \sum_{i=1}^n f_i(x_i),  \\ \nonumber
	&\, \text{subject to}\,  x_1 = x_2 = \dots = x_n
\end{align}
where $f_i$ is the loss of the client $i$ that only depends on their local data. 

However, the objective~\eqref{eq:fl_standard} faces major criticism for many FL applications~\citep{wu2020personalized, kulkarni2020survey, deng2020adaptive}. Specifically, the minimizer of the collective population loss might not represent an ideal model for a given client, especially when their data distribution significantly deviates from the population. An exemplary case that underscores the necessity of personalized FL models is mobile keyboard next word prediction, wherein personalized approach~\citep{hard2018federated} significantly outperformed the non-personalized counterparts.

Numerous strategies within the literature seek to integrate personalization into FL. Those include multi-task learning~\citep{vanhaesebrouck2016decentralized, smith2017federated, fallah2020personalized}, transfer learning~\citep{zhao2018federated, khodak2019adaptive}, variational inference~\citep{corinzia2019variational}, mixing of local and global models~\citep{peterson2019private, hanzely2020federated, mansour2020three, deng2020adaptive} and others~\citep{eichner2019semi}. See also~\citep{kulkarni2020survey, FL-big} for a personalized FL survey.

In this work, we analyze the mixing objective studied by \citet{hanzely2020federated}, \citet{gasanov2021flix}, \citet{yi2023explicit} which was previously considered in the area of distributed optimizaton~\citep{lan2018communication, gorbunov2019optimal} and distributed transfer learning~\citep{liu2017distributed, wang2018distributed}. This approach permits distinct local models $x_i$ for each client $i$ while simultaneously penalizing their dissimilarity, i.e.,

\begin{equation} \label{eq:intro_main}
	\min_{x = [x_1,\dots,x_n]\in \R^{nd}, \forall i:\, x_i\in \R^d} \left\{ F(x) \eqdef  \underbrace{\frac{1}{n}\sum \limits_{i=1}^n  f_i(x_i)}_{\eqdef f(x)} + \lambda \underbrace{\frac{1}{2 n}\sum \limits_{i=1}^n \norm{x_i-\bar{x}}^2}_{\eqdef  \psi(x)} \right\}.
\end{equation}

Perhaps suprisingly, the optimal solution $x^\star = [x_1^\star, x_2^\star, \dots, x_n^\star] \in \R^{nd}$ of~\eqref{eq:intro_main} can be expressed as $x_i^\star = \bar{x}^\star - \frac{1}{\lambda}\nabla f_i(x_i^\star)$, where $\bar{x}^\star  \eqdef \frac1n\sum_{i=1}^n x^\star_i$~\citep{hanzely2020federated}, which strongly resembles the inner workings of the famous \maml{} algorithm \citep{finn2017model}.\\

In this paper, we investigate the personalized FL formulation~\eqref{eq:intro_main}.
Given the geographical dispersal of clients in FL, communication emerges as the most prominent bottleneck. Consequently, our principal focus is centered on mitigating communication complexity, while also giving due consideration to the reduction of local computation.

Our inquiry commences by addressing the question of the potential efficiency of algorithms in solving \eqref{eq:intro_main}. We achieve this by introducing communication and computation lower bound. Subsequently, we present four algorithms that attain this lower bound, thereby demonstrating the optimality of both the established lower bound and the proposed algorithms.

\subsection{Lower bounds on communication}
Before presenting the lower complexity bounds for solving~\eqref{eq:main}, let us formalize the notion of an oracle that an algorithm interacts with.

As we are interested in both communication and local computation, we will also distinguish between two different oracles: the communication oracle and the local oracle. While the communication oracle allows the optimization history to be shared among the clients, the local oracle $\text{Loc}(x_i,i)$ provides either a local proximal operator, local gradient, or local gradient of a summand given that a local loss is of a finite-sum structure itself $
f_\tR(x_{\tR}) = \frac1m \sum_{j =1}^m \flocc_{i,j}(x_{\tR})$:

\[
\text{Loc}(x,i) = 
\begin{cases}
	\{ \nabla f_i(x_i), \prox_{\beta_{i} {f_i}} (x_i)  \} & \text{for \emph{proximal} oracle ($\beta_i\geq0$)},
	\\
	\{ \nabla f_i(x_i)  \} & \text{for \emph{gradient} oracle},
	\\
	\{ \nabla \flocc_{i,j_i}(x_i)  \} & \text{for \emph{summand gradient} oracle ($1\leq j_i\leq m$)},
\end{cases}
\]
for all clients $i$ simultaneously, which we refer to as a single local oracle call.

Next, we restrict ourselves to algorithms whose iterates lie in the span of previously observed oracle queries. Assumption~\ref{as:intro_oracle} formalizes the mentioned notion. 

\begin{assumption}\label{as:intro_oracle}
	Let $\{x^k\}_{k=1}^\infty$ be iterates generated by algorithm $\cA$. For $1\leq i\leq n$ let $\{S_i^k\}_{k=0}^\infty$ be a sequence of sets
	defined recursively as follows:
	\begin{align*}
		S_i^{0}&=  \Span(x_i^0)\\
		S_i^{k+1} &= \begin{cases}
			\Span\left(S_i^k,  \mathrm{Loc}(x^k,i)\right)& \text{if  } \locf(k) = 1
			\\
			\Span\left( S_1^k, S_2^k, \dots, S_n^k\right) & \text{otherwise,}
		\end{cases}
	\end{align*}
	where $\locf(k)= 1$ if the local oracle was queried at the iteration $k$, otherwise $\locf(k)=0$. 
	Then, assume that $x_i^k\in S_i^k$ .
\end{assumption}

The lower bound on the communication complexity of problem~\eqref{eq:intro_main} is as follows.

\begin{theorem}\label{thm:intro_lb}	
	Let $k\geq 0, {L_1}\geq \mu,  \lambda\geq \mu$. Then, there exist ${L_1}$--smooth $\mu$--strongly convex functions $f_1, f_2, \dots, f_n: \R^d \rightarrow \R $ and a starting point $x^0\in \R^{nd}$, such that the sequence of iterates $\{x^t\}_{t=1}^k$ generated by any algorithm $\cA$ meeting Assumption~\ref{as:intro_oracle} satisfies
	\begin{equation}
		\norm{ x^{k} - x^\star}^2  \geq \frac1 4 \left(1-10\max\left\{ \sqrt{\frac{\mu}{\lambda}}, \sqrt{\frac{\mu}{{L_1}-\mu}}\right \} \right)^{\comm(k)+1} \norm{ x^0 - x^\star}^2.
	\end{equation}
	Above, $\comm(k)$ stands for the number of communication oracle queries at the first $k$ iterations of $\cA$.
	
\end{theorem}

Theorem~\ref{thm:intro_lb} shows that in order get $\varepsilon$-close to the optimum, one needs at least $\cO\left(\sqrt{\frac{\min\{ {L_1}, \lambda\}}{\mu}}\log \frac1\varepsilon \right)$ rounds of communication. This reduces to known communication complexity $\cO\left(\sqrt{\frac{{L_1}}{\mu}}\log \frac1\varepsilon \right)$ for standard FL objective~\eqref{eq:intro_fl_standard}\citep{scaman2018optimal, hendrikx2020optimal} when $\lambda=\infty$\footnote{Also~\citet{woodworth2018graph} provides similar lower bound in a slightly different setup.}.

\subsection{Lower bounds on the local computation}
Lower bounds on the local computation depend on the strength of the local oracle. We obtain bounds for the following local oracles:  

\begin{enumerate}
	\item \textbf{ Proximal oracle.} The construction from Theorem~\ref{thm:intro_lb} requires at least $\cO\left(  \sqrt{\frac{\min\left\{\lambda, {L_1}\right\}}{\mu}} \log \frac1\varepsilon \right)$ calls of any local oracle, which serves as the lower bound on the local proximal oracle. 
	
	\item \textbf{ Gradient oracle.} 
	Setting $x^0= 0\in \R^{nd}$ and $f_1=f_2=\dots = f_n$, the problem~\eqref{eq:intro_main} reduces to minimizing a single local objective $f_1$. Selecting next $f_1$ as the worst-case quadratic function \citep{nesterov2018lectures}, the corresponding objective requires at least $\cO\left(\sqrt{\frac{{L_1}}{\mu}} \log\frac{1}{\varepsilon}\right)$ gradient calls to reach $\varepsilon$--neighborhood, which serves as our lower bound. 
	
	Observe that the parallelism does not help as the starting point is identical on all machines and the construction of $f$ only allows to explore a single new coordinate per a local call, regardless of the communication.

	\item \textbf{ Summand gradient oracle.}
	Consider case $f_i=\avein jn \flocc_{i,j}$ and $\flocc_{i,j}$ is $\Lloc$--smooth for all $1\leq j\leq m, 1\leq i\leq n$. 	
	Setting again $x^0= 0\in \R^{nd}$ and $f_1=f_2=\dots = f_n$, the described algorithm restriction yields $x_1^k=x_2^k= \dots = x_n^k$ for all $k\geq 0$. Consequently, the problem reduces to minimizing a single finite sum objective $f_1$ which requires at least $\cO\left(m+ \sqrt{\frac{m\Lloc}{\mu}} \log\frac{1}{\varepsilon}\right)$ summand gradient calls~\citep{lan2018optimal, woodworth2016tight}. 
	
	For the simplicity of arguments (for summand gradient oracle), we restricted ourselves to a class of client-symmetric algorithms such that $x^{k+1}_i = \cA(H_i^k, H_{-i}^k, C^k)$, where $H_i$ is history of local gradients gathered by client $i$, $H_{-i}$ is an \emph{unordered} set with elements $H_l$ for all $l\neq i$ and $C^k$ are indices of the communication rounds of the past. We assume that $\cA$ is either deterministic or generated from a given seed that is identical for all clients initially\footnote{We believe that assuming the perfect symmetry across nodes is not necessary and can be omitted using more complex arguments. We believe that symmetry can be omitted by allowing for a varying scale of the local problem across the workers so that the condition number remains constant and adapting the approach from~\citet{hendrikx2020optimal}.}. 
\end{enumerate}

\subsection{Optimal algorithms for the personalized objective}
Specializing the approach from~\citet{wang2018distributed} to our problem, we apply Accelerated Proximal Gradient Descent (\apgd{}) in two different ways -- either we take a gradient step with respect to $f$ and proximal step with respect to $\lambda \psi$ or vice versa. In the first case, we get both the communication complexity and local gradient complexity of the order $\cO\left( \sqrt{\frac{{L_1}}{\mu}}\log \frac1\varepsilon\right)$, which is optimal if ${L_1}\leq \lambda$. In the second case, we get both the communication complexity and the local prox complexity of the order $\cO\left( \sqrt{\frac{\lambda}{\mu}}\log \frac1\varepsilon\right)$, thus optimal if ${L_1}\geq \lambda$. 

Motivated again by~\citet{wang2018distributed}, we argue that local prox steps can be evaluated inexactly\footnote{Such an approach was already considered for the standard FL formulation~\eqref{eq:fl_standard} \citep{li2018federated, pathak2020fedsplit}.} either by running locally 

\begin{enumerate}
	\item Accelerated Gradient Descent (\agd{})~\citep{nesterov1983method} preserving $\cO\left(\sqrt{\frac{\lambda}{\mu}}\log \frac1\varepsilon \right)$ communication complexity and yielding $\tilde{\cO}\left(\sqrt{\frac{{L_1}+\lambda}{\mu}} \right)$ local gradient complexity, optimal for ${L_1}\geq \lambda$ (up to log factors).
	
	\item \katyusha{}~\citep{allen2017katyusha} given that the local objective is of an $m$--finite sum structure with $\Lloc$--smooth summands. Employing \katyusha{} locally yields communication complexity $\cO\left(\sqrt{\frac{\lambda}{\mu}}\log \frac1\varepsilon\right)$ and the local gradient complexity of order $\tilde{\cO}\left(m\sqrt{\frac{\lambda}{\mu}} + \sqrt{m\frac{\Lloc}{\mu}} \right) $ (up to $\log$ factor) optimal once $  m\lambda \leq  \Lloc$.
\end{enumerate}

There are three drawbacks of inexact \apgd{} with a local randomized solver: 
\begin{enumerate}
	\item there are extra $\log$ factors in the local gradient complexity, 
	\item assumption on the boundedness of the algorithm iterates is required and 
	\item the communication complexity is suboptimal for $\lambda>{L_1}$.
\end{enumerate}
In order to avoid them, we accelerate the \ltsgdp{} algorithm of \citet{hanzely2020federated}. The proposed algorithm, \altsgdp{}, enjoys the optimal communication complexity $\cO\left(  \sqrt{\frac{ \min \{ \Lloc, \lambda\}}{\mu}}\log\frac1\varepsilon\right)$ and the local summand gradient complexity $ \cO\left(
\left(m +\sqrt{\frac{m( \Lloc +  \lambda)}{\mu}}\right)\log\frac1\varepsilon
\right)$, which is optimal for $\lambda \leq \Lloc$. Unfortunately, the two bounds are not achieved at the same time, as is summarized in \Cref{tbl:intro_optimal}. We can summarize our contributions as follows:
\begin{itemize}[leftmargin=*]
	\item \textbf{Communication complexity lower bound:} We show that for any algorithm that satisfies a certain reasonable assumption (see Assumption~\ref{as:intro_oracle}) there is an instance of~\eqref{eq:intro_main} with ${L_1}$--smooth, $\mu$--strongly convex
	local objectives $f_i$ requiring at least $\cO\left( \sqrt{\frac{\min \{ {L_1}, \lambda\}}{\mu}}\log \frac1\varepsilon\right)$ communication rounds to get to the $\varepsilon$--neighborhood of the optimum.
	
	\item \textbf{Local computation lower bound}: We show that one requires at least $\cO\left( \sqrt{\frac{\min \{ {L_1}, \lambda\}}{\mu}}\log \frac1\varepsilon\right)$ proximal oracle calls
	or at least $\cO\left( \sqrt{\frac{{L_1}}{\mu}}\log \frac1\varepsilon\right)$  evaluations of local gradients. Similarly, given that each of the local objectives is of a $m$--finite-sum structure with $\Lloc$--smooth summands, we show that at least $\cO\left(\left(m +  \sqrt{\frac{m\Lloc}{\mu}}\right)\log \frac1\varepsilon\right)$ gradients of the local summands are required. 
	
	\item \textbf{Algorithms with optimal communication \& computation complexity:} We adjust multiple algorithms presented for empirical risk minimization to solve objective~\eqref{eq:intro_main} and with optimal rates under various circumstances. Summarized in \Cref{tbl:intro_optimal}.
	
	\begin{table}[!h]
		\setlength \tabcolsep{2pt}
        \begin{threeparttable}     
			\small
    		\caption[Optimal algorithms for personalized federated learning]{Matching (up to $\log$ and constant factors) lower and upper complexity bounds for solving~\eqref{eq:intro_main}. Indicator  \cmark\, means that the lower and upper bound are matching up to constant and log factors, while \xmark\, means the opposite.
    		}
    		\label{tbl:intro_optimal}
			\begin{tabular}{|c|c|c|c|}
				\hline
				{\bf Local oracle}& \begin{tabular}{c} 
					{\bf Optimal}
					\\
					{ \bf  \# Comm }
				\end{tabular} 
				&   \begin{tabular}{c} 
					{\bf Optimal}
					\\
					{ \bf  \# Local calls }
				\end{tabular}&  {\bf Algorithm}     \\
				\hline
				\hline
				Proximal & \cmark &  \cmark  &
				{
					$
					\begin{cases}
						\lambda \geq {L_1}: &\apgdt{}~\text{\citenum{wang2018distributed}} \text{(A.~\ref{alg:fista_2})}
						\\
						\lambda \leq {L_1}: &\apgdo{}~\text{\citenum{wang2018distributed}} \text{(A.~\ref{alg:fista})}
					\end{cases}
					$
				}
				\\
				\hline
				Gradient &  \cmark &   \cmark &
				{
					$
					\begin{cases}
						\lambda \geq {L_1}:& \apgdt{}~\text{\citenum{wang2018distributed}} \text{(A.~\ref{alg:fista_2})}
						\\
						\lambda \leq {L_1}:&  \makecell{\iapgd{}~\text{\citenum{wang2018distributed}} \text{(A.~\ref{alg:fista_inex})}  \\+ \agd{}~\text{\citenum{nesterov1983method}}}
					\end{cases}$
				}
				\\
				\hline
				Stoch grad & \cmark  &  
				{ 
					$
					\text{\cmark} \,\,\, \, \text{if } m\lambda\leq \Lloc
					$
				} 
				&
				{ 
					$
					\begin{cases}
						\lambda \geq {L_1}:& \apgdt{}~\text{\citenum{wang2018distributed}} \text{(A.~\ref{alg:fista_2}) }
						\\
						\lambda \leq {L_1}:&  \makecell{\iapgd{}~\text{\citenum{wang2018distributed}} \text{(A.~\ref{alg:fista_inex})}  \\+ \katyusha{}~\text{\citenum{allen2017katyusha}}}
					\end{cases}$
				}
				\\
				\hline
				Stoch grad & 
				{ 
					\begin{tabular}{c l} 
						\cmark
						\\
						\xmark
					\end{tabular}
				} 
				&  
				{ 
					\begin{tabular}{c l} 
						\xmark &
						\\
						\cmark &  if  $\lambda\leq \Lloc$
					\end{tabular}
				} 
				&
				\altsgdp{}\tn{1}
				\\
				\hline
			\end{tabular}
            \begin{tablenotes}
			{
				\item [\color{red}(1)] \altsgdp{} can be optimal either in terms of communication or in terms of local computation; the two cases require a slightly different parameter setup.
			}
		\end{tablenotes}
        \end{threeparttable}		
	\end{table}	
	
	\item \textbf{Optimality of local methods:} We have demonstrated that the aforementioned algorithms are optimal for solving FL problem~\eqref{eq:intro_main} with \textbf{heterogeneity} of the data. Prior to this work, local algorithms were considered optimal only when all clients possess identical datasets, a condition that often does not align with real-world FL. Establishing the optimality of local methods thereby justifies their applicability in practical non-iid FL scenarios. 
	
\end{itemize}

A more detailed explanation can be found in \Cref{sec:sgn}.

\subsection{Beyond the personalized objective}
In the preceding section, we showed how to tackle heterogeneity of local data. We proposed adjustments of the loss function that personalizes models to clients data, and introduced optimal algorithms for the training of adjusted loss. 
Nevertheless, this approach exhibits two practical limitations:
\begin{enumerate}
	\item \textbf{New clients:} Proposed personalization is explicit and fixed pre-training. Although the server knows the average of the local models after the training, it is unclear how to deploy the model on the clients that didn't participate in the training.
	\item \textbf{Choice of $\lambda$ and algorithm:} The optimal choice of training algorithm depends on the relative value of personalization parameter $\lambda$ and smoothness constant $L_1$. Determining whether $\lambda \geq L_1$ or$\lambda \leq L_1$ might be hard.
\end{enumerate}

To circumvent the first limitation, we draw inspiration from the field of meta-learning.
Instead of fixing personalization in the loss pre-training, meta-learning proposes a search for a single, generalized model that can be fine-tuned for individual client distributions post-training. This approach is very practical as it decouples personalization from training and offers a method for adaptation to previously unseen data distributions. It is remarkably successful in practice, but a considerable gap persists between the practical performance of state-of-the-art methods and their convergence guarantees. 

Our objective is to bridge this theory-practice gap by providing an alternative perspective on the renowned \maml{} algorithm, accompanied by improved convergence rates and a natural adaptation to the FL setup.

\section{Adaptation to users/clients} 
In particular, we analyze the very popular meta-learning algorithm \maml{} \citep{finn2017model}. Algorithm \maml{} has much weaker convergence guarantees compared to similar algorithms outside of meta-learning. Prior to this work, it had been unclear whether weak guarantees are a result of a fundamental restriction of meta-learning or whether they result from a less effective analysis approach.
We analyze convergence of the first-order version of \maml{} (\fomaml{}, \Cref{alg:introfo_maml}) towards a personalized objective \eqref{eq:intro_new_pb}. Through this lens, we achieve faster convergence guarantees that match rates of analogous non-meta-learning counterparts. This outcome suggests that contrary to the usual presentation of \maml{}, \fomaml{} should be viewed as the solver of a personalized objective \eqref{eq:intro_new_pb}.

\begin{algorithm}[h]
	\caption{\fomaml{}: First-Order \maml{}}
	\label{alg:introfo_maml}
	\begin{algorithmic}[1]
		\State \textbf{Input:} $x^0$, $\alpha, \beta > 0$
		\For{$k=0,1,\dotsc$}
		\State Sample a subset of tasks $T^k$
		\For{each sampled task $i$ \textbf{in} $T^k$}
		\State $z_i^k = x^k - \alpha \nabla f_{i}(x^k)$
		\EndFor
		\State $x^{k+1} = x^k - \beta \frac{1}{|T^k|}\sum_{i\in T^k} \nabla f_i(z_i^k)$
		\EndFor
	\end{algorithmic}
\end{algorithm}

\subsection{Moreau envelopes personalization}
We consider the Moreau envelopes formulation of meta-learning
\begin{align}
	&\min_{x\in\R^d}  \meta(x) \eqdef \frac{1}{n}\sum_{i=1}^n \meta_i(x),\quad \label{eq:intro_new_pb} \\
	&\text{where}\quad
	\meta_i(x)\eqdef \min_{z\in\R^d} \left\{f_i(z) + \frac{1}{2\alpha}\norm{z - x}^2\right\}, \notag
\end{align}
and $\alpha>0$ is a parameter controlling the level of adaptation to the local data. In this bilevel optimization, we refer to $f_i$ as the task/loss function and $\meta_i$ as the meta-function. In the inner level, we seek to find a parameter vector $z_i$ somewhere close to $x$ that $f_i(z)$ is sufficiently small. This formulation of meta-learning was first introduced by \citet{zhou2019efficient} and it has been used by \citet{Hanzely2020} and \citet{t2020personalized} to study personalization in federated learning.
Throughout the paper, we use the following variables for minimizers of meta-problems $\meta_i$:
\begin{align}
	z_i(x) \eqdef \argmin_{z\in\R^d} \left\{f_i(z) + \frac{1}{2\alpha}\norm{z - x}^2\right\},  i=1,\dotsc, n. \label{eq:intro_z_i}
\end{align}
One can notice that if $\alpha\to 0$, then $\meta_i(x)\approx f_i(x)$, and Problem~\eqref{eq:intro_new_pb} reduces to the well-known empirical risk minimization:
\begin{align}
	\min_{x\in\R^d} f(x)
	\eqdef \frac{1}{n}\sum_{i=1}^n f_i(x).
\end{align}
If, on the other hand, $\alpha\to +\infty$, the minimization problem in~\eqref{eq:intro_new_pb} becomes essentially independent of $x$ and it holds $z_i(x)\approx \argmin_{z\in\R^d} f_i(z)$. Thus parameter $\alpha$ is part of the objective that controls the similarity between the task-specific parameters.

Now let's discuss the properties of our formulation \ref{eq:intro_new_pb}. Firstly, we state a standard result from~\citet{beck-book-first-order} that only holds for convex functions. 

\begin{proposition}\citep[Theorem 6.60]{beck-book-first-order}\label{pr:intro_moreau_is_smooth}
	Let $\meta_i$ be defined as in \cref{eq:intro_new_pb} and $z_i(x)$ be defined as in \cref{eq:intro_z_i}.
	If $f_i$ is convex, proper and closed, then $\meta_i$ is differentiable and $\frac{1}{\alpha}$--smooth:
	\begin{align}
		&\nabla \meta_i(x)
		= \frac{1}{\alpha}(x-z_i(x)) = \nabla f_i(z_i(x)), \label{eq:intro_implicit} \\
		&\norm{\nabla \meta_i(x) - \nabla \meta_i(y)}\le \frac{1}{\alpha}\norm{x-y}.
	\end{align}
\end{proposition}

In meta-learning, the tasks are often defined by a neural network, whose landscape is nonconvex. We refine \Cref{pr:intro_moreau_is_smooth} for the nonconvex case and also improve the smoothness constant in the convex case\footnote{This result is similar to Lemma 2.5 of \citet{davis2021proximal}, except their guarantee is a bit weaker because they consider more general assumptions.}.

\begin{lemma}\label{lem:intro_moreau_is_str_cvx_and_smooth}
	Let function $f_i$ be ${L_1}$--smooth.
	\begin{itemize}
		\item If $f_i$ is nonconvex and $\alpha<\frac{1}{{L_1}}$, then $\meta_i$ is $\frac{{L_1}}{1-\alpha {L_1}}$--smooth. If $\alpha \le \frac{1}{2{L_1}}$, then $\meta_i$ is $2{L_1}$--smooth.
		\item If $f_i$ is convex, then $\meta_i$ is $\frac{{L_1}}{1+\alpha {L_1}}$--smooth. Moreover, for any $\alpha$, it is ${L_1}$--smooth.
		\item If $f_i$ is $\mu$--strongly convex, then $\meta_i$ is $\frac{\mu}{1+\alpha\mu}$--strongly convex. If $\alpha \le \frac{1}{\mu}$, then $\meta_i$ is $\frac{\mu}{2}$--strongly convex.
	\end{itemize}
	Whenever $\meta_i$ is smooth, its gradient is given as in \eqref{eq:intro_implicit}, i.e., $\nabla \meta_i(x) = \nabla f_i(z_i(x))$.
\end{lemma}
The takeaway message of \Cref{lem:intro_moreau_is_str_cvx_and_smooth} is that the optimization properties of $\meta_i$ are always at least as good as those of $f_i$ up to a constant factor. Furthermore, the \emph{conditioning} (ratio of smoothness to strong convexity constants) of $\meta_i$ is upper bounded by that of $f_i$ (up to a constant factor). And even if $f_i$ is convex but nonsmooth (${L_1}\to+\infty)$, $F_i$ is still smooth with constant $\frac{1}{\alpha}$.

\subsection{Solving the meta function}
Note that computing the exact gradient of $F_i$ requires solving its inner problem as per equation~\eqref{eq:intro_implicit}. Even if the gradient of task $\nabla f_i(x)$ is easy to compute, we still cannot obtain $ \nabla \meta_i(x)$ through standard differentiation or backpropagation. However, one can approximate $\nabla \meta_i(x)$ in various ways.
We show that it is possible to decrease it iteratively.
\begin{lemma}\label{lem:intro_approx_implicit}
	Let task losses $f_i$ be ${L_1}$--smooth and $\alpha>0$. Given $i$ and $x\in\R^d$, we define recursively 
	\begin{equation}\label{eq:intro_maml_loop}
		{\color{mydarkred}z_{i,0}} \eqdef {\color{blue}x}, \qquad \text{and} \qquad {\color{mydarkred}z_{i,j+1}} \eqdef {\color{blue}x} - \alpha \nabla f_i({\color{mydarkred}z_{i,j}}).
	\end{equation}		
	Then, for any $s\ge 0$ holds
	\begin{align}
		\norm{ \nabla f_i(z_{i,s}) - \nabla \meta_i(x) } \le (\alpha {L_1})^{s+1} \norm{\nabla \meta_i(x)}.
	\end{align}
	In particular, the \fomaml{} (\Cref{alg:introfo_maml}) uses sets $z_i^k = {\color{mydarkred}z_{i,1}}$ from \eqref{eq:intro_maml_loop} and hence satisfy
	\begin{align}
		\norm{ \nabla f_i(z_i^k) - \nabla \meta_i(x^k) } \le (\alpha {L_1})^2 \norm{\nabla \meta_i(x^k)}.
	\end{align}
\end{lemma}
\Cref{lem:intro_approx_implicit} shows that \fomaml{} approximates \sgd{} step with relative error proportional to the norm of the meta-gradient.

\subsection{Convergence guarantees}
Most standard way to analyze \Cref{alg:intro_mamlP} is as inexact \sgd{}, by expressing meta-gradient using
\begin{align}
	\nabla f_i(z_i^k)
	= \nabla F(x^k) + \underbrace{\nabla F_i(x^k) - \nabla F(x^k)}_{\eqdef \xi_i^k\ (\mathrm{noise})} + \underbrace{b_i^k}_{\mathrm{bias}},
\end{align}  
where it holds $\E{\xi_i^k}=0$, and $b_i^k$ is a bias vector that also depends on $i$ but does not have zero mean. The best-known guarantees for inexact \sgd{} are provided by \citet{ajalloeian2020analysis}, but they are not applicable because their proofs use independence of $\xi_i^k$ and $b_i^k$. The analysis of \citet{zhou2019efficient} is not applicable either because their inexactness assumption requires the error to be smaller than a fixed predefined constant $\varepsilon$. To resolve these issues, we provide a refined analysis.
The key idea of our analysis is to establish the existence of variables $y_i^k$ such that $\nabla f_i(z_i^k)=\nabla F_i(y_i^k)$. This allows us to better express 
\begin{align}
	\nabla f_i(z_i^k)
	= \nabla F_i(y_i^k) 
	= \nabla F(x^k) + \underbrace{\nabla F_i(x^k) - \nabla F(x^k)}_{\mathrm{noise}}
	+ \underbrace{\nabla F_i(y_i^k) - \nabla F_i(x^k)}_{\textrm{reduced bias}}&.
\end{align}  

\begin{algorithm}[t]
	\caption{\fomuml{}: First-Order Multistep Meta-Learning 
		\newline \Comment{\textbf{Remark:} \fomaml{} (\Cref{alg:introfo_maml}) is a special case for $z_i^k$ set as $z_{i,1}$ from \eqref{eq:intro_maml_loop} }
		\newline \Comment{\textbf{Remark:} Using \eqref{eq:intro_maml_loop} (with some $s$) in Line 5 yields method with $s$ local steps.}
	}
	\label{alg:intro_mamlP}
	\begin{algorithmic}[1]
		\State \textbf{Input:} $x^0$, $\beta>0$, accuracy $\delta\geq0$ or $\varepsilon\ge 0$.
		\For{$k=0,1,\dotsc$}
		\State Sample a subset of tasks $T^k$
		\For{each sampled task $i$ \textbf{in} $T^k$}
		\State  Find $z_i^k$ such that\ $\norm{\frac 1 \inners \left(x^k -z_i^k \right) - \nabla \meta_i(x^k)} \leq \delta \norm{\nabla \meta_i(x^k)}$ \Comment{E.g., as \eqref{eq:intro_maml_loop}}          
		\EndFor
		\State $x^{k+1} = x^k - \beta\frac{1}{|T^k|}\sum_{i\in T^k} \nabla f_i(z_i^k)$
		\EndFor
	\end{algorithmic}
\end{algorithm}

We can use this to obtain a convergence guarantee to a neighborhood even with a small number of steps in the inner loop.
\begin{theorem} \label{th:intro_convengence_of_mamlP_no_stepsize}
	Consider the iterates of \Cref{alg:intro_mamlP} (with general $\delta$ or \fomaml{} $\delta=\alpha {L_1}$).
	Let task losses be ${L_1}$--smooth and $\mu$--strongly convex and let objective parameter satisfy $\inners \leq \frac {1}{\sqrt 6 {L_1}}$. Choose stepsize $ \beta \leq \frac \tau {4 {L_1}}$, where $\tau = |T^k|$ is the batch size. Then we have
	\begin{align}
		\E{\norm{x^k-x^*}^2} &\leq \left(1 - \frac {\beta \mu}{12}  \right)^k \norm{x^0 - x^*}^2
		+ \frac { 6\left( \frac \beta \tau + 3 \delta^2 \inners^2{L_1}\right) \varopt} {\mu}.
	\end{align}
\end{theorem}
The radius of the convergence neighborhood depends on the variance of meta-functions at the optimum,
\begin{align}
	\sigma_*^2\eqdef \frac{1}{n}\sumin in \norm{\nabla \meta_i(\opt)}^2.  \label{eq:intro_def_sigma}
\end{align}

The theorem above guarantees linear convergence to a neighborhood of radius $\cO\left(\frac{\frac{\beta}{\tau} + \alpha^2{L_1}}{\mu} \right)$ in contrast to radius $\cO\left(\frac{\beta + \kappa\alpha^2{L_1}}{\mu} \right)$ obtainable by using inexact \sgd{} approach. 
If the first term is dominating, then it implies an improvement proportional to the batch size $\tau$. If the second term is larger, then the improvement is $\cO(\kappa)$ times better, which is often a very large constant.

\subsection{Convergence in the nonconvex setup}
Next, we extend this result to the practical nonconvex objective functions under the assumption of the uniformly bounded variance of meta-gradients.
\begin{definition}\label{def:intro_bounded_var}
	We assume that the variance of meta-loss gradients is uniformly bounded by some $\sigma^2$, i.e.,
	\begin{align}
		\E{\norm{\nabla F_i(x) - \nabla F(x)}^2}
		\le \sigma^2.
	\end{align}
\end{definition}
This assumption is stronger than the one previously used, variance being finite at the optimum \eqref{eq:intro_def_sigma}. At the same time, it is very common in literature on stochastic optimization when studying convergence on nonconvex functions.
\begin{theorem}\label{th:intro_nonconvex_fo_maml}
	Let variance of meta-loss gradients be uniformly bounded by some $\sigma^2$, functions $f_1,\dotsc, f_n$ be ${L_1}$--smooth and $F$ be lower bounded by $F^*>-\infty$. Assume $\alpha\le \frac{1}{4{L_1}}, \beta\le \frac{1}{16{L_1}}$. If we consider the iterates of  \Cref{alg:intro_mamlP} (with general $\delta$ or \fomaml{} $\delta=\alpha {L_1}$), then
	\begin{align}
		\min_{t\le k}\E{\norm{\nabla F(x^t)}^2}
		&\le \frac{4}{\beta k}\E{F(x^0)-F^*} 
		+ 4(\alpha {L_1})^2\delta^2 \sigma^2 \nonumber \\
		& \qquad+ 32 \beta(\alpha {L_1})^2 \left(\frac{1}{|T^k|} + (\alpha {L_1})^2\delta^2\right) \sigma^2.
	\end{align}
\end{theorem}
Notice that this convergence is also only until some neighborhood of first-order stationarity since the second term does not decrease with $k$. This size of the upper bound depends on the product $\cO((\alpha {L_1})^2 \delta^2)$, so to obtain better convergence one can simply increase approximation accuracy to make $\delta$ smaller. However, the standard \fomaml{} corresponds to $\delta=\alpha {L_1}$, so \fomaml{} convergence to \eqref{eq:intro_new_pb} directly depend on the problem parameter $\alpha$.

For Algorithm~\ref{alg:intro_mamlP} with $s$ steps of \eqref{eq:intro_maml_loop} in the inner loop we have $\delta=\cO((\alpha {L_1})^s)$ and the radius of the neighborhood is  $\cO((\alpha {L_1})^2\delta^2)=\cO((\alpha {L_1})^{2s+2})$. Therefore, to converge to a given target accuracy $\varepsilon>0$, we need at most $s=\cO(\log\frac{1}{\varepsilon})$ inner-loop iterations. If we can plug-in $s=1$, we also get that \fomaml{} converges to a neighborhood of size $\cO((\alpha {L_1})^4)$.

A more detailed explanation can be found in \Cref{sec:moreau_meta}.

\clearpage

\section{Overview of objective functions}
Throughout the thesis, we generally consider the optimization problem:
\begin{equation} \label{eq:intro_min}
	\min\limits_{x\in \R^d}  f(x),
\end{equation} 
where $f:\R^d \to \R$ is a lower bounded convex function with continuous first derivatives. 
In centralized machine learning, 
$f$ often represents prediction loss over particular datapoints, motivating the empirical risk minimization structure of the objective
\begin{align}
	\min_{x\in\R^d} \left\{ f(x) \eqdef \frac{1}{n}\sum_{i=1}^n f_i(x) \right\}. 
\end{align}
Parameter $n$ represents the number of datapoints and $d$ represents the number of features.
We use this loss function in Chapters \ref{sec:mixture_optimal} and \ref{sec:adaptive}.

\subsection{Federated learning}
In the context of decentralized machine learning and particularly in federated learning, 
$f_i$ represents the loss at client $i$, which motivates further assumptions on the structure of $f_i$.
In \Cref{sec:mixture_optimal} we focus on the mixing FL objective from~\citet{hanzely2020federated}, which is well-known from the area of distributed optimization and distributed transfer learning. The mentioned formulation allows the local models $x_i$ to be mutually different while penalizing their dissimilarity with the regularizer $\psi$,

\begin{equation} \label{eq:intro_mixture}
	\min_{\stackrel{x = [x_1,\dots,x_n]\in \R^{nd},} {\forall i:\, x_i\in \R^d}} \left\{ F(x) \eqdef  \underbrace{\frac{1}{n}\sum \limits_{i=1}^n  f_i(x_i)}_{\eqdef f(x)} + \lambda \underbrace{\frac{1}{2 n}\sum \limits_{i=1}^n \norm{x_i-\bar{x}}^2}_{\eqdef  \psi(x)} \right\}.
\end{equation}

\subsection{Meta-learning}
In pursuit of post-training personalization, in \Cref{sec:moreau_meta} we explore the \emph{Moreau envelope} formulation of meta-learning \citep{zhou2019efficient},
\begin{equation}
	\min_{x\in\R^d}   \left\{\meta(x) \eqdef \frac{1}{n}\sum_{i=1}^n \meta_i(x) \right\},\;
	\text{where}\;
	\meta_i(x)\eqdef \min_{z\in\R^d} \left\{f_i(z) + \frac{1}{2\alpha}\norm{z - x}^2\right\}, \label{eq:intro_moreau}
\end{equation}
we refer to $f_i$ as task/loss/data functions, which can be distributed across various devices, and $F_i$ as \emph{meta functions}, Moreau envelopes of functions $f_i$. Parameter $\alpha>0$ controls the adaptation level.

As we mentioned before, objectives \eqref{eq:intro_mixture} and \eqref{eq:intro_moreau} are similar, but they differ drastically in their usability and practical aspects.
In the former, every client starts training with their own model and improves it throughout.
In the latter and in meta-learning in general, the goal is to find a single, more general model, which they fine-tune to their data after the training. This is very practical, as it allows deploying the model on clients that didn't participate in the training as well. Also, it provides a recipe for how to adapt to previously unseen data distributions.

\subsection{Second-order methods}
In Chapters \ref{sec:aicn} and \ref{sec:sgn}, we go beyond first-order optimization and we consider minimization problem \eqref{eq:intro_min}:
\begin{equation}
	\min\limits_{x\in \R^d}  f(x),
\end{equation} 	
where $f$ is a convex function with continuous first and second derivatives and positive definite Hessian. 
We assume that the loss $f$ has a unique minimizer $$x_{\ast}\in \argmin\limits_{x\in \R^d}  f(x),$$ and that diameter of the level set $\cL(x_0)\eqdef\lb x \in \R^d, f(x)\leq f(x_0) \rb$ is bounded.

\section{Overview of assumptions} \label{sec:intro_def}
Let us introduce assumptions on function properties that are used in particular chapters. Starting with standard function properties that are commonly used in the optimization literature \citep{Nesterov2013}.

\begin{definition}\label{def:intro_smooth}
	We say that a function $f$ is $L_1$--\emph{smooth} if its gradient is $L_1$--Lipschitz, i.e., for any $x, y\in\R^d$,
	\begin{align}
		\norm{\nabla f(x) - \nabla f(y)} \le L_1\norm{x-y}.
	\end{align}
\end{definition}
\begin{definition}\label{def:intro_convex}
	Given a function $f$, we call it $\mu$--\emph{strongly convex} if it satisfies for any $x, y\in\R^d$,
	\begin{align}
		f(y)\ge f(x) + \la\nabla f(x), y-x\ra + \frac{\mu}{2}\norm{y-x}^2.
	\end{align}
	If the property above holds with $\mu=0$, we call $f$ to be \emph{convex}. If the property does not hold even with $\mu=0$, we say that $f$ is \emph{nonconvex}.
\end{definition}

Smoothness and convexity provide upper and lower bound on the functional values. Note that $L_1$--smoothness and $\mu$--strong convexity constants can be expressed in terms of largest and smallest eigenvalues of Hessian, as we have $\mu \mI \preceq \h(x) \preceq L_1 \mI$.\\

If $f$ has ERM of losses at datapoints
\begin{equation}
	f(x) = \avein in f_i(x),
\end{equation}
then one can assume $\Lloc$--smoothness of datapoint losess $f_i$.
This implies $\Lloc$--smoothness of function $f$ and hence, $L_1 \leq \Lloc$.

Finite form structure motivates stochastic subsampling and formulation of stochastic gradients. For subset $\mathcal S \subseteq \{1, \dots,n \}$, and scalars weights $v_i\geq 0$, $v_i=0$ for $i \not \in \mathcal S$ and $\E{v_i} = 1$, denote stochastic gradient
\begin{equation}
    \nabla f_v (x) \eqdef \aveis i {\mathcal S} v_i \g_i(x).
\end{equation}

Stochastic gradients motivate the stochastic notion of smoothness, \citet{JacSketch} introduces expected smoothness assumption that can be used to obtain tighter bounds.

\begin{definition} \label{def:intro_expected_smoothness}
	A function $f$ is \emph{expected $\cL(\tau)$--smooth} with respect to a datasets $\mathcal{D}$ if for constant $\cL(\tau) > 0$ and minibatch sampling $\mathcal S \sim \cD$, so that $\E{\mathcal S}=\tau$ holds
	\begin{equation}
		\E{\norm{ \nabla f_v(x) - \nabla f_v(x_*) }^2} \leq 2 \cL(f(x) - f(x_*)). 
	\end{equation}
    When $\tau$ is clear from the context, we will use shorthand notation $\cL \eqdef \cL(\tau)$.
\end{definition}

Note that the smoothness and convexity constants above depend on the $l_2$ norm. One can go beyond $l_2$ norm and formulate smoothness and convexity in local norms \citep{lu2018relatively, RSN}. This proves to be useful for the analysis of second-order methods.
\begin{definition} 
	We call \emph{relative convexity} and \emph{relative smoothness} constants $\murel, \Lrel$ for which following inequalities hold $\forall x, y \in \mathbb \R^d$:
	\begin{align}
		f(y) &\leq f(x) + \la \g(x), y-x \ra + \afrac {\Lrel} {2}  \normsM {y-x}{x}, \\
		f(y) &\geq f(x) + \la \g(x), y-x \ra + \afrac {\murel}2 \normsM{y-x} x.
	\end{align}
\end{definition}

Previous smoothness definitions are based on first-order approximations of the Taylor polynomial. One can go beyond that by considering second-order approximations.
\begin{definition}
	Function $f$ with continuous first and second derivatives is called \emph{second-order smooth} or  \emph{Hessian smooth} if there exists $L_2\geq0$ such that $\forall x,y \in \R^d$ holds
	\begin{equation} \label{eq:intro_L2-smooth}
		\gboxeq{ \norm{ \nabla^2 f(x) - \nabla^2 f(y)} \leq L_2\norm{x-y}.}
	\end{equation} 
\end{definition}
Work \citet{nesterov2006cubic} uses \eqref{eq:intro_L2-smooth} to bound higher-order terms of Taylor polynomial. Resulting algorithm, Cubically Regularized \newton{} method, have fast $\okd$ rate for convex functions for the cost of more expensive iterates.
Work \citet{hanzely2022damped} shows that the expensive iterates of cubic Newton method are result of measuring in $l_2$ norms, and can be avoided by using more geometrically natural \emph{local norms} that were commonly used with interior point methods \citep{nesterov1994interior}.

\subsection{Local norms} 

Denote $x,g,h \in \R^d$. For a symmetric positive semi-definite matrix $\mathbf H: \R^d \rightarrow \R^d$, we can define norms\footnote{In \Cref{sec:aicn} we define those norms in more general Euclidean spaces $\mathbb E$ and $\mathbb E^{\ast}$.}
\begin{equation}
    \norm{x}_ {\mathbf H} \eqdef  \la \mathbf Hx,x\ra^{1/2}, \, x \in \R^d, \qquad \norm{g}_{\mathbf H}^{\ast}\eqdef\la g,\mathbf H^{-1}g\ra^{1/2},  \, g \in \R^d.
\end{equation}
For identity $\mathbf H= \mI$, we get classical $l_2$ norm $\norm{x}_\mI = \la x,x\ra^{1/2}$. We call \emph{local Hessian norm} choice $\mathbf H = \nabla^2 f(x)$, for which we use shortened notation
\begin{equation}
	\gboxeq{\norm{h}_x \eqdef \la \nabla^2 f(x) h,h\ra^{1/2}, \, h \in\R^d,} \qquad \gboxeq{\norm{g}_{x}^{\ast} \eqdef \la g,\nabla^2 f(x)^{-1}g\ra^{1/2},  \, g \in \R^d.}
\end{equation}

Local Hessian norm $\norm{h}_{\nabla f(x)}$ is affine-invariant, because for $h=\mathbf Az$ we have
\begin{equation}
    \norm{z}_{\nabla^2 \phi(y)}^2=\la\nabla^2 \phi(y)z,z \ra = \la \mathbf A^\top \nabla^2 f(\mathbf A y) \mathbf Az,z\ra = \la \nabla^2 f(x)h,h\ra=\norm{h}_{\nabla^2 f(x)}^2.  
\end{equation}
On the other hand, induced norm $\norm{h}_\mI$ is not, because
\begin{equation}
    \norm{z}_\mI^2=\la z,z \ra = \la \mathbf A^{-1}h, \mathbf A^{-1}h\ra = \norm{ \mathbf A^{-1} h}^2_\mI.  
\end{equation}
With respect to geometry around point $x$, the more natural norm is the local Hessian norm, $\norm{h}_{\nabla f(x)}$. From affine invariance follows that its level sets $\lb y\in \R^d\,|\,\norm{ y-x }_x^2 \leq c\rb$ are balls centered around $x$ (all directions have the same scaling). In comparison, the scaling of the $l_2$ norm is dependent on the eigenvalues of the Hessian. In terms of convergence, one direction in $l_2$ can significantly dominate others and slow down an algorithm.

\subsection{Smoothness in local norms}
In order to work with local norms, we will use variant of the smoothness assumption defined in terms of local norms \citep{nesterov1994interior}.
\begin{definition} \label{def:intro_self-concordance}
	Convex function $f$ with continuous first, second, and third derivatives is called \textit{self-concordant}  if
	\begin{equation}
		|D^3 f(x)[h]^3| \leq \Lstandard\norm{h}_{x}^3, \quad \forall x,h\in \R^d,
	\end{equation}
	where for any integer $p\geq 1$, by $D^p f(x)[h]^p \eqdef D^p f(x)[h,\ldots,h]$ we denote the $p$--th order directional derivative\footnote{For example, $D^1 f(x)[h] =\langle\nabla  f(x),h\rangle$ and $D^2 f(x)[h]^2 = \la \nabla^2 f(x) h, h \ra $.} of $f$ at $x\in \R^d$ along direction $h\in \R^d$.
\end{definition}
Both sides of the inequality are affine-invariant. This assumption corresponds to a big class of optimization methods called interior-point methods \citep{nesterov1994interior}. Self-concordance implies the uniqueness of the solution of the lower bounded function \citep[Theorem 5.1.16]{nesterov2018lectures}.

Similarly to \eqref{eq:intro_L2-smooth}, self-concordance implies smoothness-like upper bound on the function value,
\begin{equation} 
	f(y) - f(x) \leq \ip{\g(x)} {y-x} + \frac1 2 \normM {y-x} x ^2 + \frac\Lstandard 6 \normM {y-x} x ^3, \qquad \forall x,y \in \R^d.
\end{equation}    	

Going beyond self-concordance, \citet{rodomanov2021greedy} introduced a stronger version of the self-concordance assumption.
\begin{definition}	\label{def:strong-self-concordance}
	Twice differentiable convex function $f$  is called \textit{strongly self-concordant} if
	\begin{equation}
		\label{eq:strong-self-concordance}
		\nabla^2 f(y)-\nabla^2 f(x) \preceq \Lstrongly \norm{y-x}_{z} \nabla^2 f(w), \quad \forall y,x,z,w \in \R^d .
	\end{equation} 
\end{definition}

In this work we will be utilizing a definition for a class of functions between self-concordant and strongly self-concordant functions \citep{hanzely2022damped} that is analogy to \eqref{eq:intro_L2-smooth} in local norms.

\begin{definition} \label{def:intro_semi-self-concordance}
	Twice differentiable convex function $f$  is called \textit{semi-strongly self-concordant} if
	\begin{equation}
		\norm{\nabla^2 f(y)-\nabla^2 f(x)}_{op} \leq \Lsemi \norm{y-x}_{x} , \quad \forall y,x \in \R^d,
	\end{equation} 
    \noindent
    where operator norm is, for given $x \in \R^d$, defined for any matrix $\mathbf H \in \R^{d \times d}$ as 
    \begin{equation} \label{eq:intro_matrix_operator_norm}
        \gboxeq{\normM{\mathbf H} {op} \eqdef \sup_{v\in \R^d} \frac {\normMd {\mathbf H v} x}{\normM v x}.}
    \end{equation}
    Note that the operator norm of Hessian at the same point is one, $\normM{\h(x)} {op} =1$.
\end{definition}
All of the Definitions \ref{def:intro_self-concordance} - \ref{def:intro_semi-self-concordance} are affine-invariant and their respective classes satisfy \citep{hanzely2022damped}
\begin{gather*}
	\textit{strong self-concordance} \subseteq \textit{semi-strong self-concordance}
	\subseteq \textit{self-concordance}.
\end{gather*}
Also, for a fixed strongly self-concordant function $f$ and smallest such  $\Lstandard, \Lsemi, \Lstrongly$ holds $\Lstandard \leq \Lsemi \leq \Lstrongly$ \citep{hanzely2022damped}.

All notions of self-concordance are closely related to the standard convexity and smoothness; strong self-concordance follows from function $L_2$--Lipschitz continuous Hessian and strong convexity.
\begin{proposition} \citep[Example 4.1]{rodomanov2021greedy} \label{le:sscf}
	Let $\mathbf H: \R^d \rightarrow \R^d$ be a self-adjoint positive definite operator. Suppose there exist $\mu>0$ and $L_2 \geq 0$ such that the function $f$ is $\mu$--strongly convex and its Hessian is $L_2$--Lipschitz continuous \eqref{eq:intro_L2-smooth} with respect to the norm $\norm{\cdot}_{\mathbf H}$. Then $f$ is strongly self-concordant with constant $\Lstrongly=\frac{L_2}{\mu^{3 / 2}}$. 
\end{proposition}

\subsection{Subspace methods}
In \Cref{sec:sgn}, we analyze subspace methods that use cheap sparse updates \citep{gower2015randomized}
\begin{equation} 
	x_+ = x + \s h,
\end{equation}
where $\s \in \R^{d \times \tau(\s)}, \s \sim \cD$ is a thin matrix and $h \in \R^{\tau(\s)}$.
This allows us to replace the computation of full-dimensional gradients/Hessians with the computation of subspace gradients/Hessians {$\gS (x) \eqdef \st \g (x)$} and {$\hS (x) \eqdef \st \h(x) \s$}, which are much cheaper.
In particular, \citet{RSN} shows that for $\tau \eqdef \tau(\s)$, $\s ^\top \h(x) \s$ can be obtained by twice differentiating function $\lambda \to f(x + \s \lambda)$ at cost of $\tau$ times of evaluating function $f(x+\s \lambda)$ by using reverse accumulation techniques \citep{christianson1992automatic, gower2012new}. This results in the cost of $\cO \left(d\tau^2\right)$ arithmetic operations and if $\tau=1$, then the cost is even $\cO \left(1\right)$.

This motivates formulating assumptions in the sketched directions. 
Given a sketching matrix $\s \in \R^{d \times \tau(\s)}$, we define self-concordance in its range. 
\begin{definition} \label{def:intro_scs}
	Convex function $f \in C^3$ is \textit{$\Ls$--self-concordant in range of $\s$} if
	\begin{equation}
		\Ls \eqdef \max_{x \in \R^d} \max_{\stackrel{h \in \R^{\tau(\s)}}{h \neq 0}} \frac {\left \vert \nabla^3 f(x) [\s h]^3 \right \vert}{\normM {\s h} x ^3},
	\end{equation}
	where 
	$\nabla^3 f(x)[h]^3 \eqdef \nabla^3 f(x) [h, h, h]$ is $3$--rd order directional derivative of $f$ at $x$ along $h \in \R^d$.

\end{definition}
In case $\s = \mI$, \Cref{def:intro_scs} matches definition of self-concordance, hence $\Ls \leq \Lstandard$. 

\begin{proposition} \citep[Lemma 2.2]{hanzely2020stochastic}
	Constant $\Ls$ is determined from $\Range{\s}$, and $\Range{\s} = \Range {\s'}$ implies $\Ls = L_{\s'} $.
\end{proposition}
This proposition motivates measuring local norms in sketched subspaces. For given $\s$, we denote local norm in the range of $\s$, $\normMS h x \eqdef \normM {\s^\top h} x$.
Lastly, in \Cref{sec:sgn} we also formulate relative smoothness in sketched subspaces.

\begin{definition} \label{def:intro_sc_rel}
	We call relative smoothness in subspace $\s$ positive constant $ \Lrel _\s$ for which following inequality holds $\forall x \in \mathbb \R^d$ and $y_\s = x+\s h$, $\forall h \in \R^{\tau(\s)}$:
	\begin{align}
		f(y_\s) &\leq f(x) + \la \gS(x), y_\s - x \ra + \afrac {\Lrel _\s} {2}  \normsMS {y_\s - x}{x}.
	\end{align}
\end{definition}

\subsection{Empirical evaluations}
Our contributions are mostly theoretical. We perform empirical evaluations to support our theoretical findings, usually on empirical risk minimization, for data functions
\begin{align}
	\min \limits_{x \in \R^d} \Big\{ f_{\text{logistic}}(x)& \eqdef \frac{1}{n} \sum \limits_{i=1}^n  \log\left(1-e^{-b_i a_i^\top x}\right) + \frac{\mu}{2} \norm{x}^2 \Big\},\\
	\min \limits_{x \in \R^d} \Big\{ f_{\text{ridge}}(x) &\eqdef \frac{1}{2n}\sum \limits_{i=1}^n \norm{a_i^\top x - b_i}^2 + \frac{\mu}{2}\norm{x}^2 \Big\},
\end{align}
from LIBSVM datasets \citep{chang2011libsvm}, with features $\{(a_i, b_i)\}_{i=1}^n$, labels $b_i \in \{-1, 1\}$ and $\mu > 0$ is the regularization coefficient.

\section{Organization of the thesis}
In this section, we provide concise overviews of individual chapters, followed by a summary of key insights (\Cref{tab:intro_chapter_references}).

\subsection{Content of \Cref{sec:mixture_optimal}}
In this chapter, we consider the optimization formulation of personalized federated learning recently introduced by~\citet{hanzely2020federated} which was shown to give an alternative explanation to the
workings of local \sgd{} methods. Our first contribution is establishing the first lower bounds for this formulation, for both the communication complexity and the local oracle complexity.
Our second contribution is the design of several optimal methods matching these lower bounds in almost all regimes. These are the first provably optimal methods for personalized federated learning.
Our optimal methods include an accelerated variant of \fedprox, and an accelerated variance-reduced version of \fedavg/\texttt{Local} \sgd{}. We demonstrate the practical superiority of our methods through
extensive numerical experiments.

\mycite{hanzely2020federated}

\subsection{Content of \Cref{sec:moreau_meta}}
In this chapter, we consider the problem of minimizing the sum of Moreau envelopes of given functions, which has previously appeared in the context of meta-learning and personalized federated learning. In contrast to the existing theory that requires running subsolvers until a certain precision is reached, we only assume that a finite number of gradient steps is taken at each iteration. As a special case, our theory allows us to show the convergence of First-order model-agnostic meta-learning (\fomaml{}) to the vicinity of a solution of the Moreau envelope objective. We also study a more general family of first-order algorithms that can be viewed as a generalization of \fomaml{}. Our main theoretical achievement is a theoretical improvement upon the inexact \sgd{} framework. In particular, our perturbed-iterate analysis allows for tighter guarantees that improve the dependency on the problem's conditioning. In contrast to the related work on meta-learning, ours does not require any assumptions on the Hessian smoothness and can leverage smoothness and convexity of the reformulation based on Moreau envelopes. Furthermore, to fill the gaps in the comparison of \fomaml{} to the implicit \maml{} (\imaml{}), we show that the objective of \imaml{} is neither smooth nor convex, implying that it has no convergence guarantees based on the existing theory.

\mycite{mishchenko2023convergence}

\subsection{Content of \Cref{sec:adaptive}}
Recent advances in the theoretical understanding of \sgd{} \citep{SGD_general_analysis} led to a formula for the optimal minibatch size, minimizing the number of effective data passes, i.e., the number of iterations times the minibatch size. However, this formula is of no practical value as it depends on the knowledge of the variance of the stochastic gradients evaluated at the optimum.  In this chapter, we design a practical \sgd{} method capable of learning the optimal minibatch size adaptively throughout its iterations for strongly convex and smooth functions. Our method does this provably, and in our experiments with synthetic and real data robustly exhibit nearly optimal behavior; that is, it works as if the optimal minibatch size was known a-priori. Further,  we generalize our method to several new sampling strategies not considered in the literature before, including a sampling suitable for distributed implementations.

\mycite{Alfarra2020}

\subsection{Content of \Cref{sec:aicn}}
In this chapter, we present the first stepsize schedule for \newton{} method resulting in fast global and local convergence guarantees. In particular, a) we prove an $\okd$ global rate, which matches the state-of-the-art global rate of cubically regularized  Newton method of \citet{nesterov2006cubic} and of regularized Newton method of \citet{mishchenko2021regularized} and \citet{doikov2021optimization}, b) we prove a local quadratic rate, which matches the best-known local rate of second-order methods, and c) our stepsize formula is simple, explicit, and does not require solving any subproblem. Our convergence proofs hold under affine invariance assumptions closely related to the notion of self-concordance. Finally, our method has competitive performance when compared to existing baselines, which share the same fast global convergence guarantees.

\mycite{hanzely2022damped}

\subsection{Content of \Cref{sec:sgn}}
In this chapter, we propose the first \sap{} \newton{} method with fast $\mathcal O \left( k^{-2} \right)$ global convergence rate for self-concordant functions. 
Our method, \sgn{}, can be viewed in three ways: i) as a sketch-and-project algorithm projecting updates of the Newton method, ii) as a cubically regularized Newton method in sketched subspaces, and iii) as a damped Newton method in sketched subspaces.

\sgn{} inherits the best of all three worlds: cheap iteration costs of sketch-and-project methods, state-of-the-art $\mathcal O \left( k^{-2} \right)$ global convergence rate of full-rank Newton-like methods, and the algorithm simplicity of damped Newton methods. 
Finally, we demonstrate its comparable empirical performance to baseline algorithms.

\mycite{hanzely2023sketchandproject}

\subsection{Chapter takeaways}
In the \Cref{tab:intro_chapter_references} we summarize takeaway messages of particular chapters.

\begin{table*}[h!]
	\centering
	\setlength\tabcolsep{3pt} 
	\begin{threeparttable}[t]
		{
			\renewcommand\arraystretch{3}
			\caption{Takeaway messages of thesis chapters.}
			\label{tab:intro_chapter_references}
			\centering 
			\begin{tabular}{?c|l?}
				\Xhline{2\arrayrulewidth}
				\makecell{\textbf{Chapter} } & \makecell{\textbf{Underpinned paper and its takeaway message}} \\
				
				\Xhline{2\arrayrulewidth}
				
				\Cref{sec:mixture_optimal} &  \makecell[l]{ \bf Lower bounds and optimal algorithms for \\ \bf personalized federated learning \citenum{hanzely2020federated}\\ 
					\hline 
					\Q What is the price of the personalization fixed pre-training? \\
					\A As expensive as centralized training.
				}\\
				
				\Xhline{2\arrayrulewidth}
				\Cref{sec:moreau_meta} &  \makecell[l]{ \bf Convergence of first-order algorithms for \\ \bf meta-learning with Moreau envelopes \citenum{mishchenko2023convergence}\\
					\hline
					\Q How can we obtain faster convergence guarantees for \fomaml{}?\\
					\A \fomaml{} matches \sgd{} rate if viewed as solver of a \\ \qquad personalized objective.} \\
				
				\Xhline{2\arrayrulewidth}
				\Cref{sec:adaptive} &  \makecell[l]{ \bf Adaptive learning of the optimal minibatch size \\ \bf of \sgd{} \citenum{Alfarra2020}\\
					\hline
					\Q How to use impractical minibatch size formula (from \citenum{SGD_general_analysis})?\\
					\A One can learn it during training.\\
				} \\
				
				\Xhline{2\arrayrulewidth}
				\Cref{sec:aicn} &  \makecell[l]{ \bf A damped Newton method achieves global $\okd$ \\ \bf and local quadratic convergence rate \citenum{hanzely2022damped} \\
					\hline
					\Q Can simple Newton-like methods have $\okd$ global rate?\\
					\A Yes.
				} \\
				
				\Xhline{2\arrayrulewidth}
				\Cref{sec:sgn} &  \makecell[l]{ \bf Sketch-and-project meets Newton method: global \\ \bf $\okd$ convergence with low-rank updates \citenum{hanzely2023sketchandproject}\\
					\hline
					\Q Can second-order methods scale with a dimension?\\
					\A Yes.}
				\\             
				\Xhline{2\arrayrulewidth}
			\end{tabular}
		}
	\end{threeparttable}
\end{table*}

\mycomment{
	\begin{table*}[h!]
		\centering
		\setlength\tabcolsep{3pt} 
		\begin{threeparttable}[t]
			{
				\renewcommand\arraystretch{3}
				\caption{Chapter connections.}
				\label{tab:intro_chapter_connections}
				\centering 
				\begin{tabular}{?c?c|c|c|c|c?}
					\Xhline{2\arrayrulewidth}
					\makecell{\\ \\ \\ \textbf{Aspect}} \makecell{ \textbf{Chapter}\\ \\ \\} & mixture & \fomaml{} & minibatch & \ain{} & \sgn{}\\
					\Xhline{2\arrayrulewidth}
					second-order \n \n \n \y \y\\
					\hline
					personalization \y \y \n \n \n\\
					\hline 
					ERM \y \y \y \n \n\\
					\hline
					scalable \y \y \y \n \y\\
					\hline
					optimal \y \n \n \n \y \\ 
					\hline
					published \y \n \n \y \n \\
					\hline
					strongly convex analysis \y \y \y \y \y\\
					\hline
					nonconvex analysis \n \y \n \n \n\\
					\hline
					virtual iterates \n \y \n \y \y \\
					\Xhline{2\arrayrulewidth}
				\end{tabular}
			}
		\end{threeparttable}
\end{table*}}

\chapter{Lower bounds and optimal algorithms for personalized federated learning} \label{sec:mixture_optimal}
\thispagestyle{empty}

	\section{Introduction}

	Federated learning (FL)~\citep{mcmahan17a,konevcny2016federated} is a relatively new field that attracted much attention recently. Specifically, FL is a subset of distributed machine learning that aims to fit the data stored locally on plentiful clients. Unlike typical distributed learning inside a data center, each client only sees his/her data, which might differ from the population average significantly. Furthermore, as the clients are often physically located far away from the central server, communication becomes a notable bottleneck, which is far more significant compared to in-datacenter learning.

	While the main difference between FL and the rest of the machine learning lies in means of the training, the two scenarios are often identical from the modeling perspective. In particular, the standard FL aims to find the minimizer of the overall population loss,
	\begin{align}\label{eq:fl_standard}
	\min_{z\in \R^d} \frac1n \sum_{i=1}^n f_i(z) = & \min_{x_1, x_2, \dots, x_n \in \R^d} \, \frac1n \sum_{i=1}^n f_i(x_i),  \\ \nonumber
	&\, \text{subject to}\,  x_1 = x_2 = \dots = x_n
	\end{align}
	where $f_i$ is the loss of the client $i$ that only depends on his/her own local data. 
	
	However, there is major criticism of the objective~\eqref{eq:fl_standard} for many of the FL applications~\citep{wu2020personalized, kulkarni2020survey, deng2020adaptive}. Specifically, the minimizer of the overall population loss might not be the ideal model for a given client, given that his/her data distribution differs from the population significantly. A good example to illustrate the requirement of personalized FL models is the prediction of the next word written on a mobile keyboard, where a personalized FL approach~\citep{hard2018federated} significantly outperformed the non-personalized one.
	
There are multiple strategies in the literature for incorporating the personalization into FL: multi-task learning~\citep{vanhaesebrouck2016decentralized, smith2017federated, fallah2020personalized}, transfer learning~\citep{zhao2018federated, khodak2019adaptive}, variational inference~\citep{corinzia2019variational}, mixing of the local and global models~\citep{peterson2019private, hanzely2020federated, mansour2020three, deng2020adaptive} and others~\citep{eichner2019semi}. See also~\citep{kulkarni2020survey, FL-big} for a personalized FL survey.
	
	In this work, we analyze the mixing objective studied by \citet{hanzely2020federated} which was previously studied in the area of distributed optimizaton~\citep{lan2018communication, gorbunov2019optimal} and distributed transfer learning~\citep{liu2017distributed, wang2018distributed}. This approach permits distinct local models $x_i$ for each client $i$ while simultaneously penalizing their dissimilarity, i.e.,
	
	\begin{equation}\label{eq:main}
	\min_{x = [x_1,\dots,x_n]\in \R^{nd}, \forall i:\, x_i\in \R^d} \left\{ F(x) \eqdef  \underbrace{\frac{1}{n}\sum \limits_{i=1}^n  f_i(x_i)}_{\eqdef f(x)} + \lambda \underbrace{\frac{1}{2 n}\sum \limits_{i=1}^n \norm{x_i-\bar{x}}^2}_{\eqdef  \psi(x)} \right\}
	\end{equation}

	Suprisingly enough, the optimal solution $x^\star = [x_1^\star, x_2^\star, \dots, x_n^\star] \in \R^{nd}$ of~\eqref{eq:main} can be expressed as $x_i^\star = \bar{x}^\star - \frac{1}{\lambda}\nabla f_i(x_i^\star)$, where $\bar{x}^\star  = \frac1n\sum_{i=1}^n x^\star_i$~\citep{hanzely2020federated}, which strongly resembles the inner workings of the famous \maml{}~\citep{finn2017model}.

	In addition to personalization, the above formulation sheds light on the most prominent FL optimizer -- local \sgd{}/\fedavg{} \citep{mcmahan2016communication}. Specifically, it was shown that a simple version of Stochastic Gradient Descent (\sgd{}) applied on~\eqref{eq:main} is essentially\footnote{Up to the stepsize and random number of the local gradient steps.} equivalent to \fedavg{} algorithm~\citep{hanzely2020federated}. Furthermore, the FL formulation~\eqref{eq:main} enabled local gradient methods to outperform their non-local cousins when applied to heterogeneous data problems.\footnote{Surprisingly enough, the non-local algorithms outperform their local counterparts when applied to solve the classical FL formulation~\eqref{eq:fl_standard} with heterogeneous data.}

	\section{Contributions}
	In this paper, we study the personalized FL formulation~\eqref{eq:main}. We propose a lower complexity bounds for communication and local computation, and develop several algorithms capable of achieving it. Our contributions can be listed as follows:
	
	\begin{itemize}[leftmargin=*]
	\item \textbf{Lower bound on the communication complexity} of FL formulation~\eqref{eq:main}\textbf{:} We show that for any algorithm that satisfies a certain reasonable assumption (see Assumption~\ref{as:oracle}) there is an instance of~\eqref{eq:main} with ${L_1}$--smooth, $\mu$--strongly convex\footnote{We say that function $h:\R^d \rightarrow \R$ is ${L_1}$--smooth if for each $z,z' \in \R^d $ we have $h(z) \leq h(z') + \langle \nabla h(z), z'-z\rangle + \frac {L_1} 2 \norm{z-z'}^2.$ Similarly, a function $h:\R^d \rightarrow \R$ is $\mu$--strongly convex, if for each $z,z' \in \R^d $ it holds $h(z) \geq h(z') + \langle \nabla h(z), z'-z\rangle + \frac \mu 2 \norm{z-z'}^2.$} local objectives $f_i$ requiring at least $\cO\left( \sqrt{\frac{\min \{ {L_1}, \lambda\}}{\mu}}\log \frac1\varepsilon\right)$ communication rounds to get to the $\varepsilon$--neighborhood of the optimum.
	
	\item \textbf{Lower bound on the local computation complexity} of FL formulation~\eqref{eq:main}\textbf{:}
	We show that one requires at least $\cO\left( \sqrt{\frac{\min \{ {L_1}, \lambda\}}{\mu}}\log \frac1\varepsilon\right)$ local proximal oracle calls\footnote{Local proximal oracle reveals $\{ \prox_{\beta f}(x), \nabla f(x)\}$ for any $x\in \R^{nd}, \beta>0$. Local gradient oracle reveals $\{ \nabla f(x)\}$ for any $x\in \R^{nd}$.} or at least $\cO\left( \sqrt{\frac{{L_1}}{\mu}}\log \frac1\varepsilon\right)$  evaluations of local gradients. Similarly, given that each of the local objectives is of a $m$--finite-sum structure\footnote{Denoted as
	$f_\tR(x_{\tR}) = \frac1m \sum_{j =1}^m \flocc_{i,j}(x_{\tR})$.} with $\Lloc$--smooth summands, we show that at least $\cO\left(\left(m +  \sqrt{\frac{m\Lloc}{\mu}}\right)\log \frac1\varepsilon\right)$ gradients of the local summands are required.

	\item \textbf{Optimal algorithms:} We discuss several approaches to solve~\eqref{eq:main} which achieve the \emph{optimal communication complexity and optimal local gradient complexity} under various circumstances. Specializing the approach from~\citep{wang2018distributed} to our problem, we apply Accelerated Proximal Gradient Descent (\apgd{}) in two different ways -- either we take a gradient step with respect to $f$ and proximal step with respect to $\lambda \psi$ or vice versa. In the first case, we get both the communication complexity and local gradient complexity of the order $\cO\left( \sqrt{\frac{{L_1}}{\mu}}\log \frac1\varepsilon\right)$ which is optimal if ${L_1}\leq \lambda$. In the second case, we get both the communication complexity and the local prox complexity of the order $\cO\left( \sqrt{\frac{\lambda}{\mu}}\log \frac1\varepsilon\right)$, thus optimal if ${L_1}\geq \lambda$. Motivated again by~\citep{wang2018distributed}, we argue that local prox steps can be evaluated inexactly\footnote{Such an approach was already considered in~\citep{li2018federated, pathak2020fedsplit} for the standard FL formulation~\eqref{eq:fl_standard}.} either by running locally Accelerated Gradient Descent (\agd{})~\citep{nesterov1983method} or \katyusha{}~\citep{allen2017katyusha} given that the local objective is of an $m$--finite sum structure with $\Lloc$--smooth summands. Local \agd{} approach preserves $\cO\left(\sqrt{\frac{\lambda}{\mu}}\log \frac1\varepsilon \right)$ communication complexity and yields $\tilde{\cO}\left(\sqrt{\frac{{L_1}+\lambda}{\mu}} \right)$ local gradient complexity, both of them optimal for ${L_1}\geq \lambda$ (up to log factors). Similarly, employing \katyusha{} locally, we obtain the communication complexity of order $\cO\left(\sqrt{\frac{\lambda}{\mu}}\log \frac1\varepsilon\right)$ and the local gradient complexity of order $\tilde{\cO}\left(m\sqrt{\frac{\lambda}{\mu}} + \sqrt{m\frac{\Lloc}{\mu}} \right) $; the former is optimal once ${L_1}\geq \lambda$, while the latter is (up to $\log$ factor) optimal once $  m\lambda \leq  \Lloc$.

	\item \textbf{New accelerated algorithm:} Observe that the inexact \apgd{} with local randomized solver has three drawbacks: (i) there are extra $\log$ factors in the local gradient complexity, (ii) boundedness of the algorithm iterates as an assumption is required and (iii) the communication complexity is suboptimal for $\lambda>{L_1}$. In order to fix all the issues, \emph{we accelerate the \ltsgdp{} algorithm} of~\citet{hanzely2020federated}. The proposed algorithm, \altsgdp{}, enjoys the optimal communication complexity $\cO\left(  \sqrt{\frac{ \min \{ \Lloc, \lambda\}}{\mu}}\log\frac1\varepsilon\right)$ and the local summand gradient complexity $ \cO\left(
	\left(m +\sqrt{\frac{m( \Lloc +  \lambda)}{\mu}}\right)\log\frac1\varepsilon
	\right)$, which is optimal for $\lambda \leq \Lloc$. Unfortunately, the two bounds are not achieved at the same time, as we shall see.

	\item \textbf{Optimality of algorithms and bounds:} As a consequence of all aforementioned points, we show the optimality of local algorithms applied on FL problem~\eqref{eq:main} with heterogeneous data. We believe this is an important piece that was missing in the literature. Until now, the local algorithms were known to be optimal only when all nodes own an identical set of data, which is questionable for the FL applications. By showing the optimality of local methods, we justify the standard FL practices (i.e., using local methods in the practical scenarios with non-iid data). 
	
	\end{itemize}
	
	Table~\ref{tbl:algs2} presents a summary of the described results: for each algorithm, it indicates the local oracle requirement and the circumstances under which the corresponding complexities are optimal.
	
	\begin{table}[!t]
    \centering
		\begin{threeparttable}
			\small
            \caption[Algorithms for solving personalized federated learning objective]{Algorithms for solving~\eqref{eq:main} and their (optimal) complexities.
    		}
    		\label{tbl:algs2}
			\begin{tabular}{|c|c|c|c|}
				\hline
				{\bf Algorithm}  & {\bf Local oracle}& { \bf  Optimal \# comm}&  { \bf  Optimal \# local }    \\
				\hline
				\hline
				\ltgd{}~\citenum{hanzely2020federated} &  Grad & \xmark & \xmark \\
				\hline
				\ltsgdp{}~\citenum{hanzely2020federated} &  Stoch grad &  \xmark & \xmark  \\
				\hline
				\apgdo{}~\citenum{wang2018distributed}  (A.~\ref{alg:fista}) &  Prox & \cmark \, (if $\lambda \leq {L_1} $)&  \cmark  \, (if $\lambda \leq {L_1} $) \\
				\hline
				\apgdt{}~\citenum{wang2018distributed} (A.~\ref{alg:fista_2}) &  Grad & \cmark \, (if $\lambda \geq {L_1} $)&  \cmark  \\
				\hline
				\apgdt{}~\citenum{wang2018distributed} (A.~\ref{alg:fista_2}) &  Stoch grad & \cmark \, (if $\lambda \geq {L_1} $)&  \xmark  \\
				\hline
				\makecell{\iapgd{}~\citenum{wang2018distributed} (A.~\ref{alg:fista_inex})  \\ + \agd{}~\citenum{nesterov1983method}} &  Grad & \cmark \, (if $\lambda \leq {L_1} $)&  \cmark  \, (if $\lambda \leq {L_1} $)
				\\
				\hline
				\makecell{\iapgd{}~\citenum{wang2018distributed} (A.~\ref{alg:fista_inex})  \\+ \katyusha{}~\citenum{allen2017katyusha}}
				&  Stoch grad & \cmark \, (if $\lambda \leq {L_1} $)  & \cmark \, (if $m\lambda \leq \Lloc $)
				\\
				\hline
				\altsgdp{}
				{(A.~\ref{alg:acc_stoch})}
				&  Stoch grad &  \cmark  &   \cmark \, $\left( \text{if } \lambda \leq \Lloc \right)$ \\
				\hline
			\end{tabular}
		\end{threeparttable}		
	\end{table}

	\paragraph{Optimality table.} Next we present Table~\ref{tbl:optimal} which carries an information orthogonal to Table~\ref{tbl:algs2}. In particular, Table~\ref{tbl:optimal} indicates whether our lower and upper complexities match for a given pair of \{local oracle, type of complexity\}. The lower and upper complexity bounds on the number of communication rounds match regardless of the local oracle. 
	Similarly, the local oracle calls match almost always with one exception when the local oracle provides summand gradients and $\lambda>\Lloc$. 
	
	\begin{remark}
		Our upper and lower bounds do not match for the local summand gradient oracle once we are in the classical FL setup~\eqref{eq:main}, which we recover for $\lambda=\infty$. In such a case, an optimal algorithm was developed only very recently~\citep{hendrikx2020optimal} under a slightly stronger oracle -- the proximal oracle for the local summands. 
	\end{remark}

	\begin{table}[!t]
        \setlength \tabcolsep{2pt}
        \centering
		\begin{threeparttable}
			\small
            \caption[Optimal algorithms for personalized federated learning]{Matching (up to $\log$ and constant factors) lower and upper complexity bounds for solving~\eqref{eq:main}. Indicator  \cmark\, means that the lower and upper bound are matching up to constant and log factors, while \xmark\, means the opposite. 
    		}
    		\label{tbl:optimal}
			\begin{tabular}{|c|c|c|c|}
				\hline
				{\bf Local oracle}& \begin{tabular}{c} 
					{\bf Optimal}
					\\
					{ \bf  \# Comm }
				\end{tabular} 
				&   \begin{tabular}{c} 
					{\bf Optimal}
					\\
					{ \bf  \# Local calls }
				\end{tabular}&  {\bf Algorithm}     \\
				\hline
				\hline
				Proximal & \cmark &  \cmark  &
				{
					$
					\begin{cases}
					\lambda \geq {L_1}: &\apgdt{}~\text{\citenum{wang2018distributed}} \text{(A.~\ref{alg:fista_2})}
					\\
					\lambda \leq {L_1}: &\apgdo{}~\text{\citenum{wang2018distributed}} \text{(A.~\ref{alg:fista})}
					\end{cases}
					$
				}
				\\
				\hline
				Gradient &  \cmark &   \cmark &
				{
					$
					\begin{cases}
					\lambda \geq {L_1}:& \apgdt{}~\text{\citenum{wang2018distributed}} \text{(A.~\ref{alg:fista_2})}
					\\
					\lambda \leq {L_1}:&  \makecell{\iapgd{}~\text{\citenum{wang2018distributed}} \text{(A.~\ref{alg:fista_inex})}  \\+ \agd{}~\text{\citenum{nesterov1983method}}}
					\end{cases}$
				}
				\\
				\hline
				Stoch grad & \cmark  &  
				{ 
					$
					\text{\cmark} \,\,\, \, \text{if } m\lambda\leq \Lloc
					$
				} 
				&
				{ 
					$
					\begin{cases}
					\lambda \geq {L_1}:& \apgdt{}~\text{\citenum{wang2018distributed}} \text{(A.~\ref{alg:fista_2}) }
					\\
					\lambda \leq {L_1}:&  \makecell{\iapgd{}~\text{\citenum{wang2018distributed}} \text{(A.~\ref{alg:fista_inex})}  \\+ \katyusha{}~\text{\citenum{allen2017katyusha}}}
					\end{cases}$
				}
				\\
				\hline
				Stoch grad & 
				{ 
					\begin{tabular}{c l} 
						\cmark
						\\
						\xmark
					\end{tabular}
				} 
				&  
				{ 
					\begin{tabular}{c l} 
						\xmark &
						\\
						\cmark &  if  $\lambda\leq \Lloc$
					\end{tabular}
				} 
				&
				\altsgdp{}\tn{1}
				\\
				\hline
			\end{tabular}
            \begin{tablenotes}
			{
				\item [\color{red}(1)]\altsgdp{} can be optimal either in terms of the communication or in terms of the local computation; the two cases require a slightly different parameter setup.
			}
		      \end{tablenotes}		
		\end{threeparttable}
	\end{table}

	\section{Lower complexity bounds \label{sec:lower}}
	
	Before stating the lower complexity bounds for solving~\eqref{eq:main}, let us formalize the notion of an oracle that an algorithm interacts with.
	
	As we are interested in both communication and local computation, we will also distinguish between two different oracles: the communication oracle and the local oracle. While the communication oracle allows the optimization history to be shared among the clients, the local oracle $\text{Loc}(x_i,i)$ provides either a local proximal operator, local gradient, or local gradient of a summand given that a local loss is of a finite-sum structure itself $
	f_\tR(x_{\tR}) = \frac1m \sum_{j =1}^m \flocc_{i,j}(x_{\tR})$:

	\[
	\text{Loc}(x,i) = 
	\begin{cases}
	\{ \nabla f_i(x_i), \prox_{\beta_{i} {f_i}} (x_i)  \} & \text{for \emph{proximal} oracle ($\beta_i\geq0$)}
	\\
	\{ \nabla f_i(x_i)  \} & \text{for \emph{gradient} oracle}
	\\
	\{ \nabla \flocc_{i,j_i}(x_i)  \} & \text{for \emph{summand gradient} oracle ($1\leq j_i\leq m$)}
	\end{cases}
	\]
	for all clients $i$ simultaneously, which we refer to as a single local oracle call.
	
	Next, we restrict ourselves to algorithms whose iterates lie in the span of previously observed oracle queries. Assumption~\ref{as:oracle} formalizes the mentioned notion. 
	
	\begin{assumption}\label{as:oracle}
		Let $\{x^k\}_{k=1}^\infty$ be iterates generated by algorithm $\cA$. For $1\leq i\leq n$ let $\{S_i^k\}_{k=0}^\infty$ be a sequence of sets
		defined recursively as follows:
		\begin{align*}
		S_i^{0}&=  \Span(x_i^0)\\
		S_i^{k+1} &= \begin{cases}
		\Span\left(S_i^k,  \mathrm{Loc}(x^k,i)\right)& \text{if  } \locf(k) = 1
		\\
		\Span\left( S_1^k, S_2^k, \dots, S_n^k\right) & \text{otherwise,}
		\end{cases}
		\end{align*}
		where $\locf(k)= 1$ if the local oracle was queried at the iteration $k$, otherwise $\locf(k)=0$. 
		Then, assume that $x_i^k\in S_i^k$ .
	\end{assumption}

Assumption~\ref{as:oracle} is rahter standard in the literature of distributed optimization~\citep{scaman2018optimal, hendrikx2020optimal}; it informally means that the iterates of $\cA$ lie in the span of explored directions only. A similar restriction is in place for several standard optimization lower complexity bounds~\citep{nesterov2018lectures, lan2018optimal}. We shall, however, note that Assumption~\ref{as:oracle} can be omitted by choosing the worst-case objective adversarially based on the algorithm decisions~\citep{nemirovsky1983problem, woodworth2016tight, woodworth2018graph}. We do not explore this direction for the sake of simplicity.

	\subsection{Lower complexity bounds on communication}
	
	Next, we present the lower bound on the communication complexity of problem~\eqref{eq:main}. It is similar to lower bounds from \citep{arjevani2015communication} for standard FL objective~\eqref{eq:fl_standard}.
	
	\begin{theorem}\label{thm:lb}
		
		Let $k\geq 0, {L_1}\geq \mu,  \lambda\geq \mu$. Then, there exist ${L_1}$--smooth $\mu$--strongly convex functions $f_1, f_2, \dots f_n: \R^d \rightarrow \R $ and a starting point $x^0\in \R^{nd}$, such that the sequence of iterates $\{x^t\}_{t=1}^k$ generated by any algorithm $\cA$ meeting Assumption~\ref{as:oracle} satisfies
		
		\begin{equation}\label{eq:lb}
		\norm{ x^{k} - x^\star}^2  \geq \frac1 4 \left(1-10\max\left\{ \sqrt{\frac{\mu}{\lambda}}, \sqrt{\frac{\mu}{{L_1}-\mu}}\right \} \right)^{\comm(k)+1} \norm{ x^0 - x^\star}^2.
		\end{equation}
		
		Above, $\comm(k)$ stands for the number of communication oracle queries at the first $k$ iterations of $\cA$.
		
	\end{theorem}

	Theorem~\ref{thm:lb} shows that in order get $\varepsilon$-close to the optimum, one needs at least $\cO\left(\sqrt{\frac{\min\{ {L_1}, \lambda\}}{\mu}}\log \frac1\varepsilon \right)$ rounds of communication. This reduces to known communication complexity $\cO\left(\sqrt{\frac{{L_1}}{\mu}}\log \frac1\varepsilon \right)$ for standard FL objective~\eqref{eq:fl_standard} from~\citep{scaman2018optimal, hendrikx2020optimal} when $\lambda=\infty$.\footnote{See also~\citep{woodworth2018graph} for a similar lower bound in a slightly different setup.}

	\subsection{Lower complexity bounds on local computation\label{sec:lower_local}}
	
	Next, we present the lower complexity bounds on the number of the local oracle calls for three different types of a local oracle. In a special case when $\lambda=\infty$, we recover known local oracle bounds for the classical FL objective~\eqref{eq:fl_standard} from~\citep{hendrikx2020optimal}.
	
	\begin{itemize}[leftmargin=*]
		\item {\bf Proximal oracle.} The construction from Theorem~\ref{thm:lb} not only requires $\cO\left(  \sqrt{\frac{\min\left\{\lambda, {L_1}\right\}}{\mu}} \log \frac1\varepsilon \right)$ communication rounds to reach $\varepsilon$--neighborhood of the optimum, it also requires at least $\cO\left(  \sqrt{\frac{\min\left\{\lambda, {L_1}\right\}}{\mu}} \log \frac1\varepsilon \right)$ calls of any local oracle, which serves as the lower bound on the local proximal oracle. 
	
	\item {\bf Gradient oracle.} 
	Setting $x^0= 0\in \R^{nd}$ and $f_1=f_2=\dots = f_n$, the problem~\eqref{eq:main} reduces to minimizing a single local objective $f_1$. Selecting next $f_1$ as the worst-case quadratic function from~\citep{nesterov2018lectures}, the corresponding objective requires at least $\cO\left(\sqrt{\frac{{L_1}}{\mu}} \log\frac{1}{\varepsilon}\right)$ gradient calls to reach $\varepsilon$--neighborhood, which serves as our lower bound. Note that the parallelism does not help as the starting point is identical on all machines and the construction of $f$ only allows to explore a single coordinate per a local call, regardless of the communication.

	\item {\bf Summand gradient oracle.}
	Suppose that $\flocc_{i,j}$ is $\Lloc$--smooth for all $1\leq j\leq m, 1\leq i\leq n$. Let us restrict ourselves on a class of client-symmetric algorithms such that $x^{k+1}_i = \cA(H_i^k, H_{-i}^k, C^k)$, where $H_i$ is history of local gradients gathered by client $i$, $H_{-i}$ is an \emph{unordered} set with elements $H_l$ for all $l\neq i$ and $C^k$ are indices of the communication rounds of the past. We assume that $\cA$ is either deterministic, or generated from given seed that is identical for all clients initially.\footnote{We suspect that assuming the perfect symmetry across nodes is not necessary and can be omitted using more complex arguments. In fact, we believe that allowing for a varying scale of the local problem across the workers so that the condition number remains constant, we can adapt the approach from~\citep{hendrikx2020optimal} to obtain the desired local summand gradient complexity without assuming the symmetry.} Setting again $x^0= 0\in \R^{nd}$ and $f_1=f_2=\dots = f_n$, the described algorithm restriction yields $x_1^k=x_2^k= \dots = x_n^k$ for all $k\geq 0$. Consequently, the problem reduces to minimizing a single finite sum objective $f_1$ which requires at least $\cO\left(m+ \sqrt{\frac{m\Lloc}{\mu}} \log\frac{1}{\varepsilon}\right)$ summand gradient calls~\citep{lan2018optimal, woodworth2016tight}. 
	\end{itemize}

	\section{Optimal algorithms \label{sec:upperbound}}
	In this section, we present several algorithms that match the lower complexity bound on the number of communication rounds and the local steps obtained in Section~\ref{sec:lower}.
	
	\subsection{Accelerated proximal gradient descent (\apgd{}) for federated learning \label{sec:apgd_simple}}
	
	The first algorithm we mention is a version of the accelerated proximal gradient descent~\citep{beck2009fast}. In order to see how the method specializes in our setup, let us first describe the non-accelerated counterpart -- proximal gradient descent (\pgd{}).

	Let a function $h: \R^{nd}\rightarrow \R$ be $L_h$--smooth and $\mu_h$--strongly convex, and function $\phi: \R^{nd}\rightarrow \R \cup \{\infty\}$ be convex. In its most basic form, iterates of \pgd{} to minimize a regularized convex objective $h(x) + \phi(x)$ are generated recursively as follows
	\begin{eqnarray} \label{eq:pgd}
	x^{k+1} &=& \prox_{\frac{1}{L_h }\phi}\left( x^k - \frac{1}{L_h} \nabla h(x^k)\right) \\
	&=& \argmin_{x\in \R^{nd}} \phi(x) - \frac{L_h}{2} \norm{ x - \left( x^k - \frac{1}{L_h} \nabla h(x^k) \right)}^2.
	\end{eqnarray}
	The iteration complexity of the above process is $\cO\left( \frac{L_h}{\mu_h}\log\frac{1}{\varepsilon}\right)$. 
	
	Motivated by~\citep{wang2018distributed}\footnote{Iterative process~\eqref{eq:pgd_specialized} is in fact a special case of algorithms proposed in~\citep{wang2018distributed}. See Remark~\ref{rem:graph_transfer} for details.}, there are two different ways to apply the process~\eqref{eq:pgd} to the problem~\eqref{eq:main}. A more straightforward option is to set $h=f, \phi  = \lambda \psi$, which results in the following update rule
	\begin{equation} \label{eq:pgd_specialized_simple}
	x^{k+1}_i =    \frac{{L_1}y^k_i + \lambda \bar{y}^k}{{L_1}+\lambda}, \quad \text{where} \quad  y^{k}_i =x^{k}_i - \frac1{L_1} \nabla f(x^k_i),\quad \bar{y}^k = \frac{1}{n}\sum_{i=1}^n y^{k}_i, 
	\end{equation}
	and it yields $\cO\left( \frac{{L_1}}{\mu}\log\frac{1}{\varepsilon}\right)$ rate. The second option is to set $h (x)=  \lambda \psi(x) + \frac{\mu}{2n}\norm{ x}^2$ and $\phi(x) = f(x) -  \frac{\mu}{2n}\norm{ x}^2$. Consequently, the update rule~\eqref{eq:pgd} becomes (see Lemma~\ref{lem:mnadjnjks} in the Appendix):
	\begin{equation}
	x^{k+1}_i = \prox_{\frac{1}{\lambda}f_i}(\bar{x}^k) = \argmin_{z\in \R^d} f_i(z) + \frac{\lambda}{2} \norm{ z - \bar{x}^k}^2 \quad \text{for all }i,
	\label{eq:pgd_specialized}
	\end{equation}
	matching the \fedprox{}~\citep{li2018federated} algorithm. The iteration complexity we obtain is, however, $\cO\left( \frac{\lambda}{\mu}\log\frac{1}{\varepsilon}\right)$ (see Lemma~\ref{lem:mnadjnjks} again). 
	
	As both~\eqref{eq:pgd_specialized_simple} and~\eqref{eq:pgd_specialized} require a single communication round per iteration, the corresponding communication complexity becomes $\cO\left(\frac{{L_1}}{\mu} \log\frac1\varepsilon\right)$ and $\cO\left(\frac{\lambda}{\mu} \log\frac1\varepsilon\right)$ respectively, which is suboptimal in the light of Theorem~\ref{thm:lb}. 
	
	Fortunately, incorporating the Nesterov's momentum~\citep{nesterov1983method, beck2009fast} on top of the procedure~\eqref{eq:pgd_specialized} yields both an optimal communication complexity and optimal local prox complexity once $\lambda\leq {L_1}$. We will refer to such method as  \apgdo{} (Algorithm~\ref{alg:fista} in the Appendix). Similarly, incorporating the acceleration into~\eqref{eq:pgd_specialized_simple} yields both an optimal communication complexity and optimal local prox complexity once $\lambda\geq {L_1}$. Furthermore, such an approach yields the optimal local gradient complexity regardless of the relative comparison of ${L_1}, \lambda$. We refer to such method \apgdt{} (Algorithm~\ref{alg:fista_2} in the Appendix).

	\subsection{Beyond proximal oracle: inexact \apgd{} (\iapgd{}) \label{sec:iapgd}}
	In most cases, the local proximal oracle is impractical as it requires the exact minimization of the regularized local problem at each iteration. In this section, we describe an accelerated inexact~\citep{schmidt2011convergence} version of~\eqref{eq:pgd_specialized} (Algorithm~\ref{alg:fista_inex}), which only requires a local (either full or summand) gradient oracle. We present two different approaches to achieve so: \agd{}~\citep{nesterov1983method} (under the gradient oracle) and \katyusha{}~\citep{allen2017katyusha} (under the summand gradient oracle). Both strategies, however, share a common characteristic: they progressively increase the effort to inexactly evaluate the local prox, which is essential in order to preserve the optimal communication complexity.

	\begin{algorithm}[h]
		\caption{\iapgd{} +$ \cA$}
		\label{alg:fista_inex}
		\begin{algorithmic}
			\State \textbf{Requires:} Starting point $y^0 = x^0\in\R^{nd}$
			\For{ $k=0,1,2,\ldots$ }
			\State{{\color{blue}Central server} computes the average $\bar{y}^k = \frac1n \sum_{i=1}^n y^k_i$}
			\State{All {\color{red} clients} $i=1,\dots,n$: 
   
				\qquad Set $h_i^{k+1}(z) \eqdef  f_i(z) + \frac{\lambda}{2} \norm{ z -\bar{y}^k }^2$
    
                \qquad Find $x^{k+1}_i$ using local solver $\cA$ for $T_k$ iterations
				\begin{equation} \label{eq:algo_suboptimality}
				h_{i}^{k+1}(x^{k+1}_i) \leq \epsilon_k+ \min_{z\in \R^d}  h_{i}^{k+1}(z).
				\end{equation}
			}
            \State \qquad Take the momentum step $y^{k+1}_i = x^{k+1}_i +\frac{ \sqrt{\lambda}- \sqrt{\mu}}{ \sqrt{\lambda}+ \sqrt{\mu}} ( x^{k+1}_i - x^{k}_i ) $
			\EndFor
		\end{algorithmic}
	\end{algorithm}
	
	\begin{remark}\label{rem:graph_transfer}
		As already mentioned, the idea of applying \iapgd{} to solve~\eqref{eq:main} is not new; it was already explored in~\citep{wang2018distributed}.\footnote{The work \citep{wang2018distributed} considers the distributed multi-task learning objective that is more general than~\eqref{eq:main}.} However,~\citep{wang2018distributed} does not argue about the optimality of \iapgd{}. Less importantly, our analysis is slightly more careful, and it supports \katyusha{} as a local sub-solver as well.  
	\end{remark}

	{\bf \iapgd{} + \agd{}}
	
	The next theorem states the convergence rate of \iapgd{} with \agd{}~\citep{nesterov1983method} as a local subsolver.

	\begin{theorem}\label{thm:inexact}
		Suppose that $f_i$ is ${L_1}$--smooth and $\mu$--strongly convex for all $i$. Let \agd{}  with starting point $y^k_i$ be employed for
		$$T_k \eqdef   \sqrt{\frac{{L_1}+ \lambda}{\mu+\lambda}} \log \left( 1152 {L_1} \lambda n^2 \left(2 \sqrt{\frac{\lambda}{\mu}} +1 \right)^2\mu^{-2}\right)
		+4\sqrt{\frac{\mu({L_1}+ \lambda)}{\lambda(\mu+\lambda)}} k$$
		iterations to approximately solve~\eqref{eq:algo_suboptimality} at iteration $k$.
		Then, we have
		$
		F(x^k) -  F^\star
		\leq
		8 \left( 1- \sqrt{\frac{\mu}{\lambda}}\right)^k (F(x^0) - F^\star),$ 
		where $F^\star = F(x^\star)$.
		As a result, the total number of communications required to reach $\varepsilon$--approximate solution is $\cO\left( \sqrt{\frac{\lambda}{\mu}}\log\frac1\varepsilon\right)$. The corresponding local gradient complexity is
		$$
		\cO\left(
		\sqrt{\frac{{L_1}+ \lambda}{\mu}} \log\frac1\varepsilon
		\left( \log \frac{ {L_1} \lambda n}{\mu} +\log\frac1\varepsilon \right)
		\right) = \tilde{\cO}\left(\sqrt{\frac{{L_1}+ \lambda}{\mu}}  \right).$$
	\end{theorem}

	As expected, the communication complexity of \iapgd{} + \agd{} is $\cO\left( \sqrt{\frac\lambda\mu} \log \frac1\varepsilon\right)$, thus optimal. On the other hand, the local gradient complexity is $\tilde{\cO}\left(\sqrt{\frac{{L_1}+ \lambda}{\mu}}  \right)$. For $\lambda = \cO({L_1})$ this simplifies to $\tilde{\cO}\left(\sqrt{\frac{{L_1}}{\mu}}  \right)$, which is, up to $\log$ and constant factors identical to the lower bound on the local gradient calls. \\

{\bf \iapgd{} + \katyusha{}}

	In practice, the local objectives $f_i$'s often correspond to a loss of some model on the given client's data. In such a case, each function $f_i$ is of the finite-sum structure:
	$
	f_\tR(x_{\tR}) = \frac1m \sum_{j =1}^m \flocc_{i,j}(x_{\tR}).
	$

	Clearly, if $m$ is large, solving the local subproblem with \agd{} is rather inefficient as it does not take an advantage of the finite-sum structure. To tackle this issue, we propose solving the local subproblem~\eqref{eq:algo_suboptimality} using \katyusha{}.\footnote{Essentially any accelerated variance reduced algorithm can be used instead of \katyusha{}, for example {\tt ASDCA}~\citep{shalev2014accelerated}, {\tt APCG}~\citep{lin2015accelerated}, Point-{\tt SAGA}~\citep{defazio2016simple}, {\tt MiG}~\citep{zhou2018simple}, {\tt SAGA-SSNM}~\citep{zhou2018direct} and others.}

	\begin{theorem}\label{thm:inexact_stoch}
		Let $\flocc_{i,j}$ be $\Lloc$--smooth and $f_i$ be $\mu$--strongly convex for all $1\leq i\leq n, 1\leq j\leq m$.\footnote{Consequently, we have $\Lloc\geq {L_1} \geq \frac{\Lloc}{m}$.}		
		Let \katyusha{}  with starting point $y^k_i$ be employed for
				$$T_k = 
		\cO\left(\left(m+  \sqrt{m \frac{\Lloc + \lambda}{\mu + \lambda}}\right)\left( \log\frac{1}{R^2} +  k\sqrt{\frac{\mu}{\lambda}} \right)\right)$$
		iterations to approximately solve~\eqref{eq:algo_suboptimality} at iteration $k$ of \iapgd{} for some small $R$ (see proof for details). Given that the iterate sequence $\{x^k\}_{k=0}^\infty$ is bounded, the expected communication complexity of \iapgd{}+\katyusha{} is $\cO\left(\sqrt{\frac{\lambda}{\mu}} \log \frac1\epsilon\right)$, while the local summand gradient complexity is
		$\tilde{\cO}\left(m\sqrt{\frac{\lambda}{\mu}} + \sqrt{m\frac{\Lloc}{\mu}} \right) $.
	\end{theorem}
	
	Theorem~\ref{thm:inexact_stoch} shows that local \katyusha{} enjoys the optimal communication complexity. Furthermore, if $\sqrt{m}\lambda = \cO(\Lloc)$, the total expected number of local gradients becomes optimal as well (see \Cref{sec:lower_local}).

	There is, however, a notable drawback of Theorem~\ref{thm:inexact_stoch} over Theorem~\ref{thm:inexact} -- Theorem~\ref{thm:inexact_stoch} requires a boundedness of the sequence $\{x^k\}_{k=0}^\infty$ as an assumption, while this piece is not required for \iapgd{}+\agd{} due to its deterministic nature. In the next section, we devise a stochastic algorithm \altsgdp{} that does not require such an assumption. Furthermore, the local (summand) gradient complexity of \altsgdp{} does not depend on the extra log factors, and, at the same time, \altsgdp{} is optimal in a broader range of scenarios.

	\subsection{Accelerated \ltsgdp{}}
	\label{sec:al2sgd+}

	In this section, we introduce accelerated version of \ltsgdp{}~\citep{hanzely2020federated}, which can be viewed as a variance-reduced variant of \fedavg{} devised to solve~\eqref{eq:main}. The proposed algorithm, \altsgdp{}, is stated as Algorithm~\ref{alg:acc_stoch} in the Appendix. From high-level point of view, \altsgdp{} is nothing but {\tt L-Katyusha} with non-uniform minibatch sampling.\footnote{ {\tt L-Katyusha}~\citep{qian2019svrg} is a variant of \katyusha{} with a random inner loop length.} In contrast to the approach from Section~\ref{sec:iapgd}, \altsgdp{} does not treat $f$ as a proximable regularizer, but rather directly constructs $g^k$ --- a non-uniform minibatch variance reduced stochastic estimator of $\nabla F(x^k)$. Next, we state the communication and the local summand gradient complexity of \altsgdp{}.
	
	\begin{theorem} \label{thm:a2}
		Suppose that the parameters of \altsgdp{} are chosen as stated in Proposition~\ref{prop:acc} in the Appendix.  In such case, the communication complexity of \altsgdp{} with $\probx = \proby(1-\proby)$, where $\proby = \frac{\lambda}{\lambda+\Lloc}$, is 
		$
		\cO\left(  \sqrt{\frac{\min\{\Lloc, \lambda \}}{\mu}}\log\frac1\varepsilon\right)
		$
		(see \Cref{sec:a2_proof} of the Appendix)

		while the local gradient complexity of \altsgdp{} for $\probx = \frac1m$ and $\proby = \frac{\lambda}{\lambda+\Lloc}$ is
		$
		\cO\left(
		\left(m +
		\sqrt{\frac{m (\Lloc +  \lambda)}{\mu}}\right)\log\frac1\varepsilon
		\right).
		$
		
	\end{theorem}
	
	The communication complexity of \altsgdp{} is optimal regardless of the relative comparison of $\Lloc, \lambda$, which is an improvement over the previous methods. Furthermore, \altsgdp{} with a slightly different parameters choice enjoys the local gradient complexity which is optimal once $\lambda = \cO(\Lloc)$.

	\section{Experiments}
	
 In this section we present empirical evidence to support the theoretical claims of this work. 
	
	In the first experiment, we study the most practical scenario with where the local objective is of a finite-sum structure, while the local oracle provides us with gradients of the summands. In this work, we developed two algorithms capable of dealing with the summand oracle efficiently: \iapgd{}+\katyusha{} and \altsgdp{}. We compare both methods against the baseline \ltsgdp{} from~\citep{hanzely2020federated}.The results are presented in Figure \ref{fig:com_met}. In terms of the number of communication rounds, both \altsgdp{} and {\tt IAPGD+Katyusha} are significantly superior to the \ltsgdp{}, as theory predicts. The situation is, however, very different when looking at the local computation. While \altsgdp{} performs clearly the best, {\tt IAPGD+Katyusha} falls behind \ltsgd{}. We presume this happened due to the large constant and $\log$ factors in the local complexity of {\tt IAPGD+Katyusha}.

	\begin{figure}[!h]
		\centering
			\includegraphics[width = 0.4 \textwidth ]{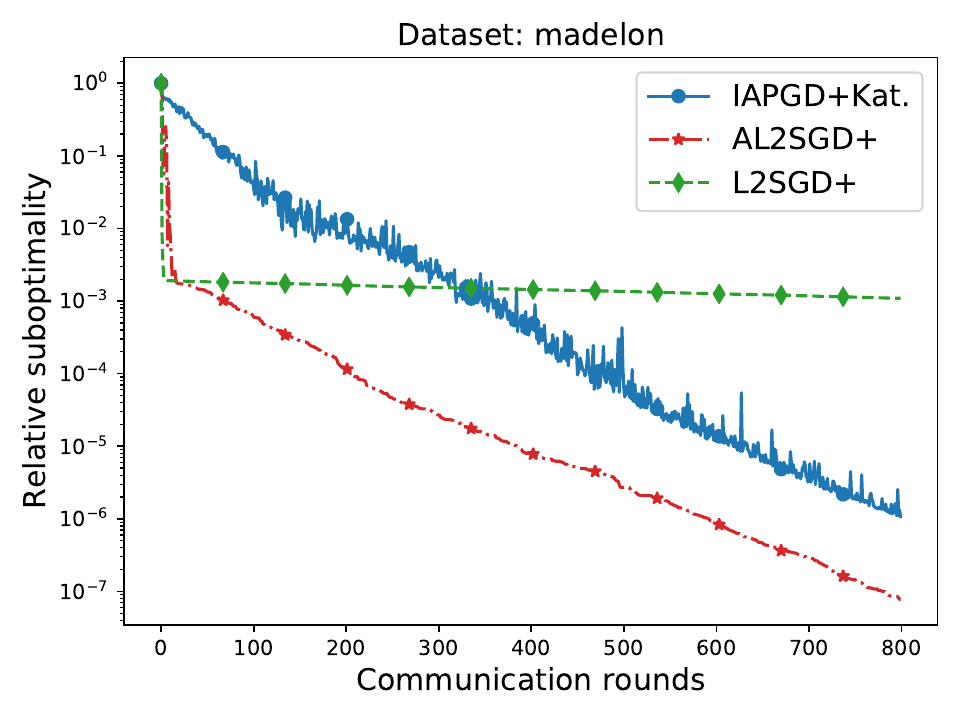}
			\includegraphics[width = 0.4 \textwidth ]{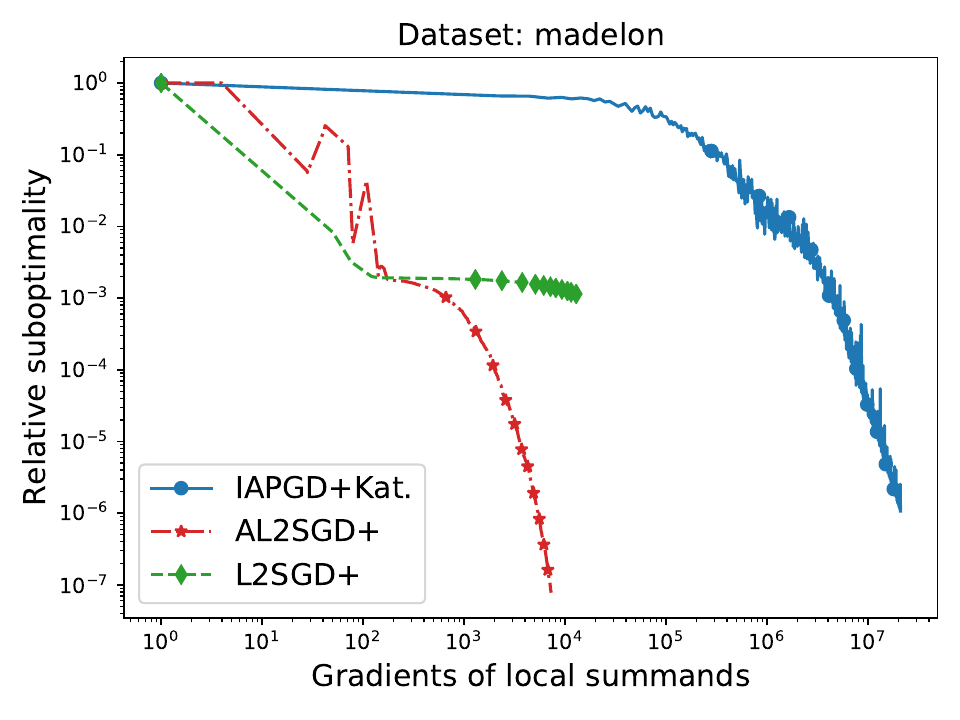}
			\includegraphics[width = 0.4 \textwidth ]{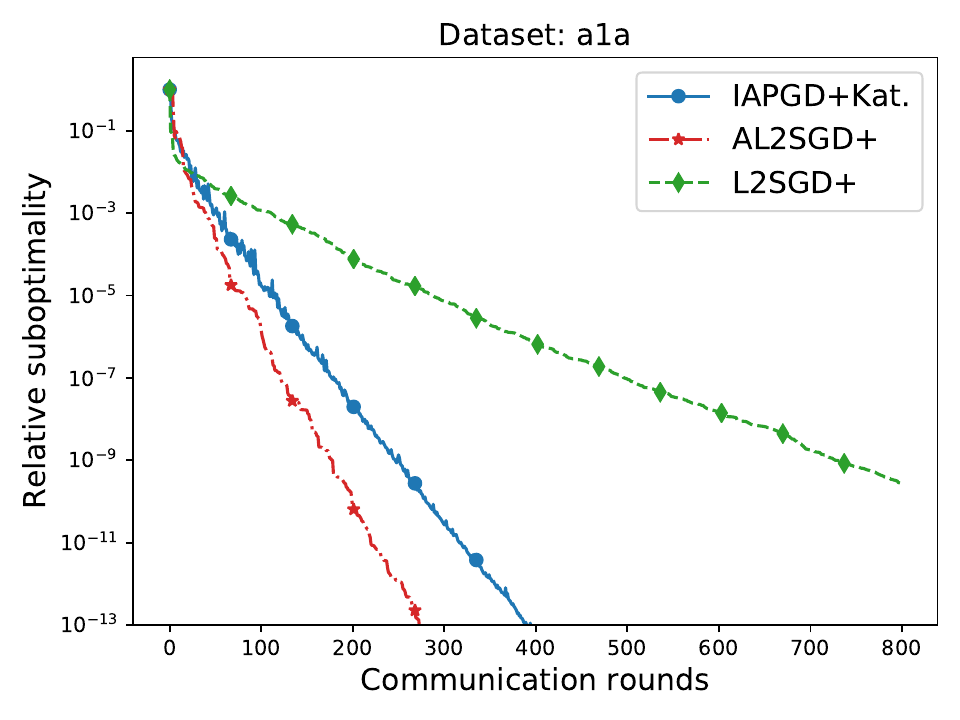}
			\includegraphics[width = 0.4 \textwidth ]{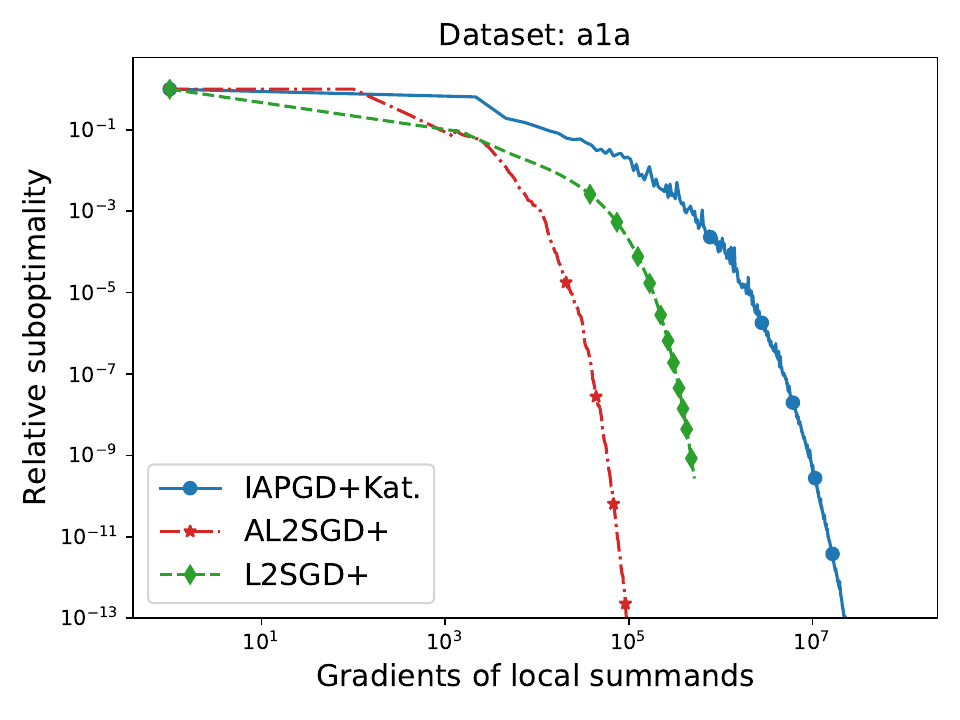}
			\includegraphics[width = 0.4 \textwidth ]{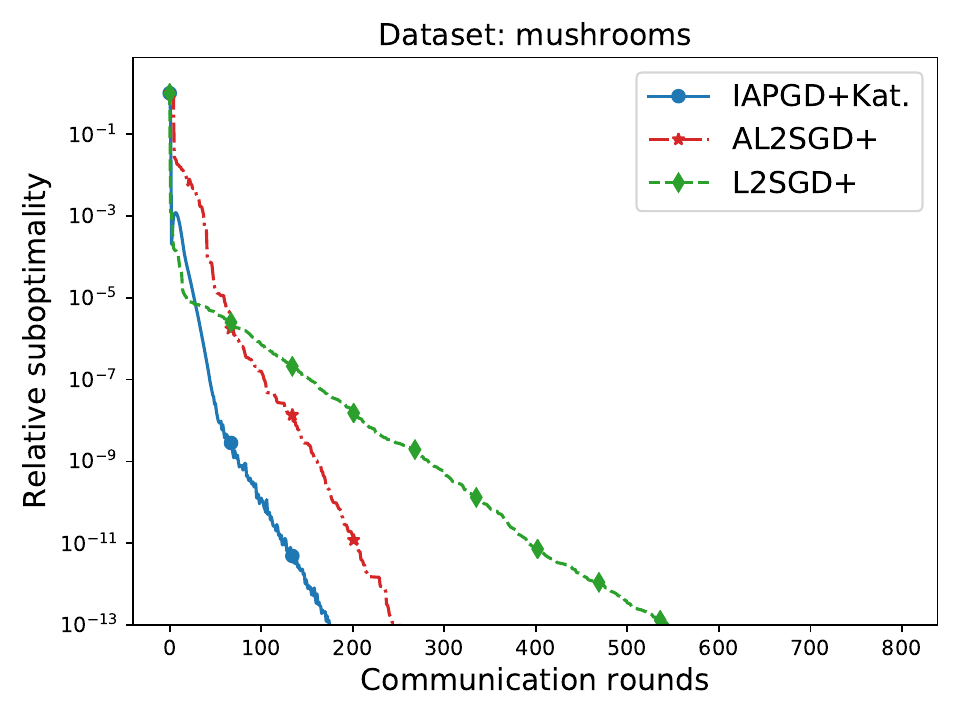}
			\includegraphics[width = 0.4 \textwidth ]{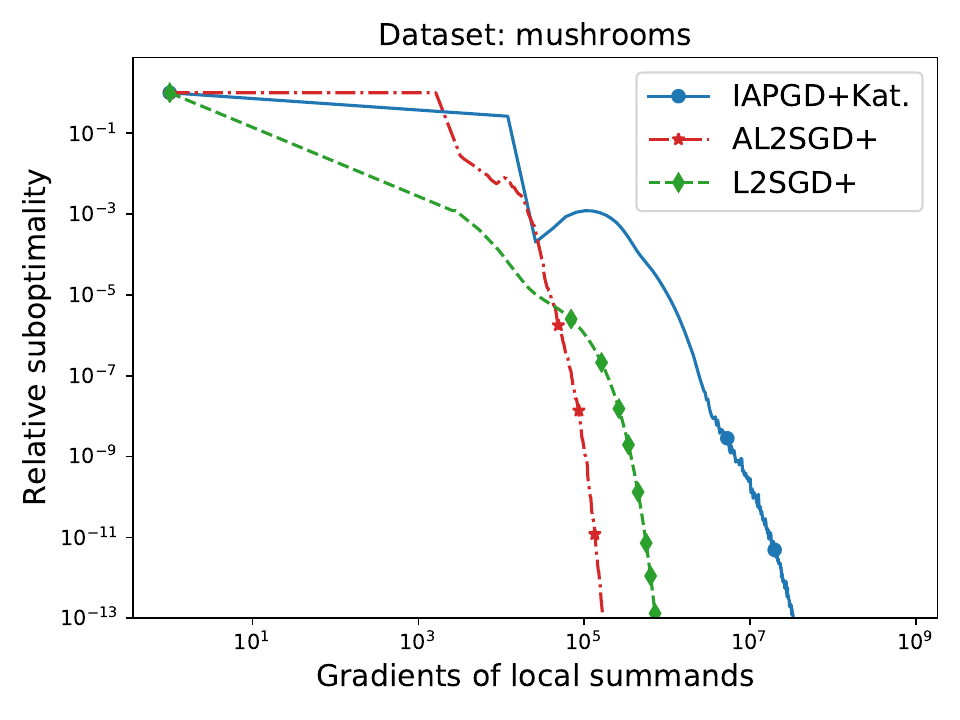}
			\includegraphics[width = 0.4 \textwidth ]{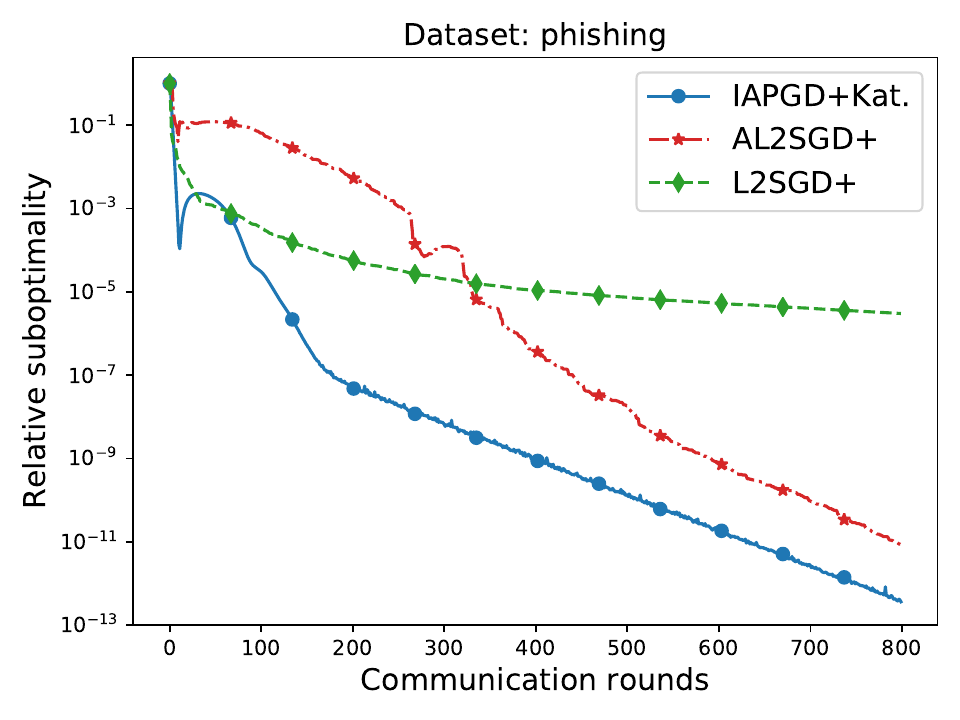}
			\includegraphics[width = 0.4 \textwidth ]{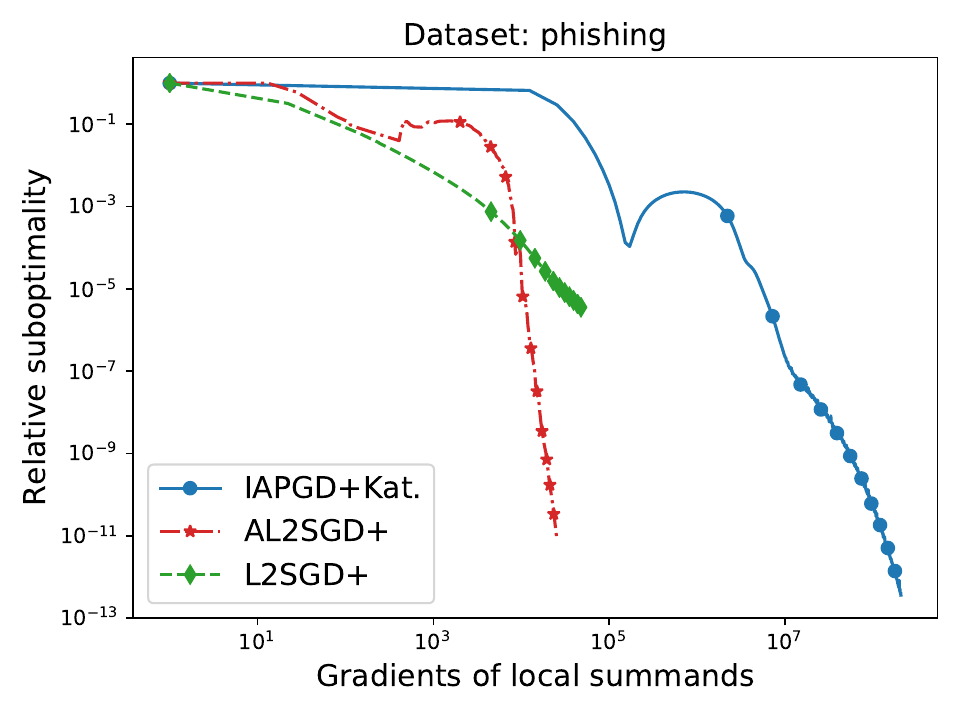}
			\includegraphics[width = 0.4 \textwidth ]{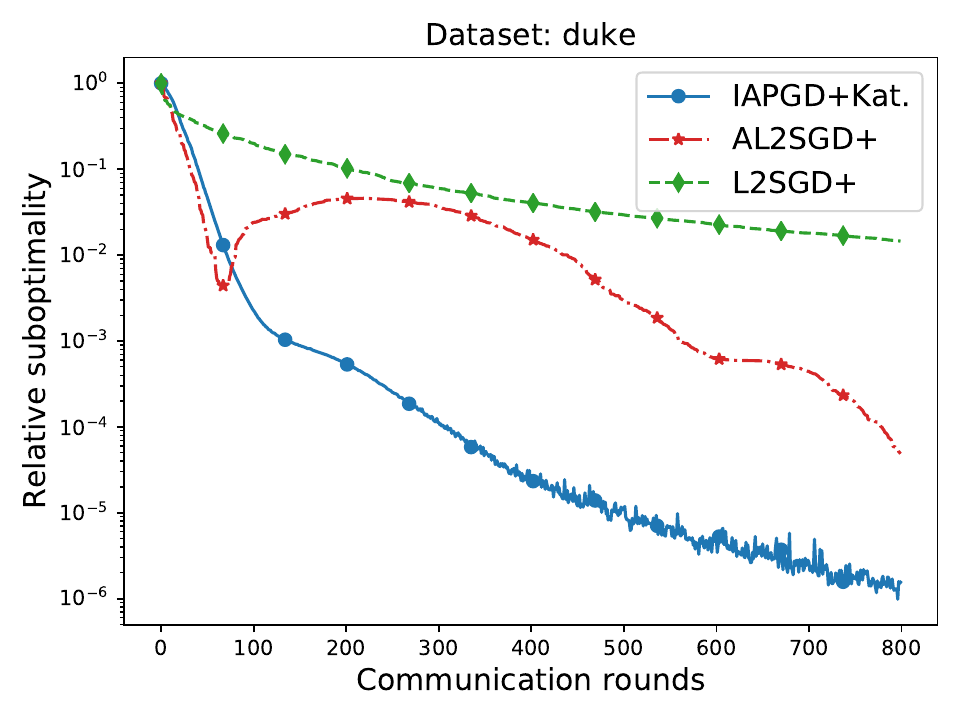}
			\includegraphics[width = 0.4 \textwidth ]{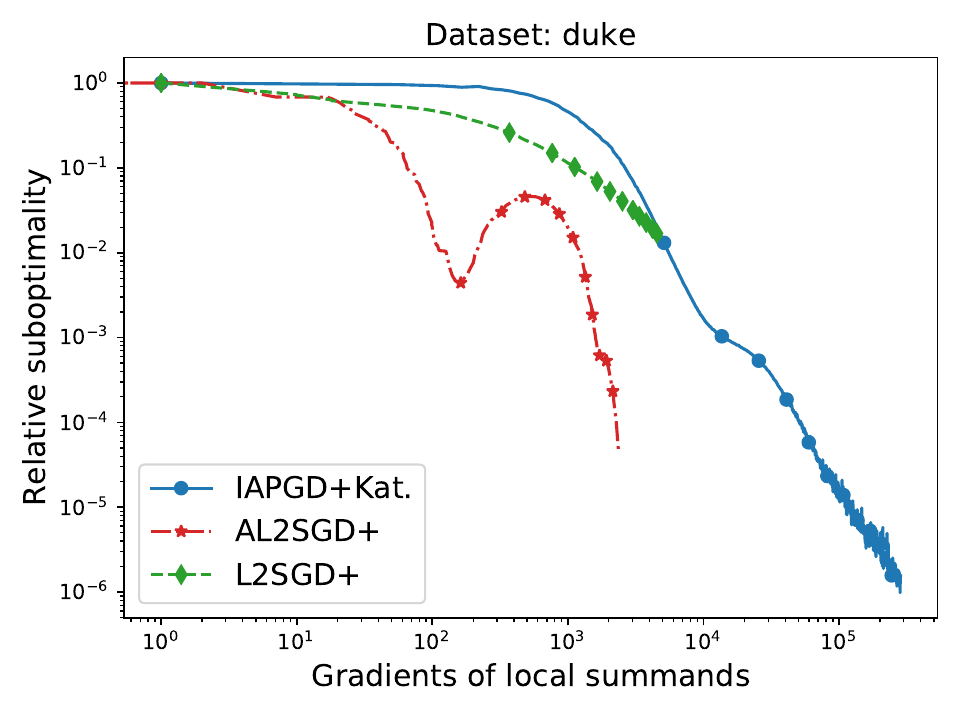}
		\caption[Comparison of optimal algorithms for personalized federated learning with randomized data split]{Comparison of {\tt IAPGD+Katyusha}, \altsgdp{} and \ltsgdp{} on logistic regression with \libsvm{} datasets~\citep{chang2011libsvm}. Each client owns a random, mutually disjoint subset of the full dataset. First column: communication complexity, second column: local computation complexity for the same experiment.
		}
		\label{fig:com_met}
	\end{figure}

In the second experiment, we investigate the heterogeneous split of the data among the clients for the same setup as described in the previous paragraph. Figure~\ref{fig:stoch_hetero} shows the result. We can see that the data heterogeneity does not influence the convergence significantly and we observe a similar behaviour compared to the homogenous case.

	\begin{figure}[!h]
		\centering
			\includegraphics[width = 0.4 \textwidth ]{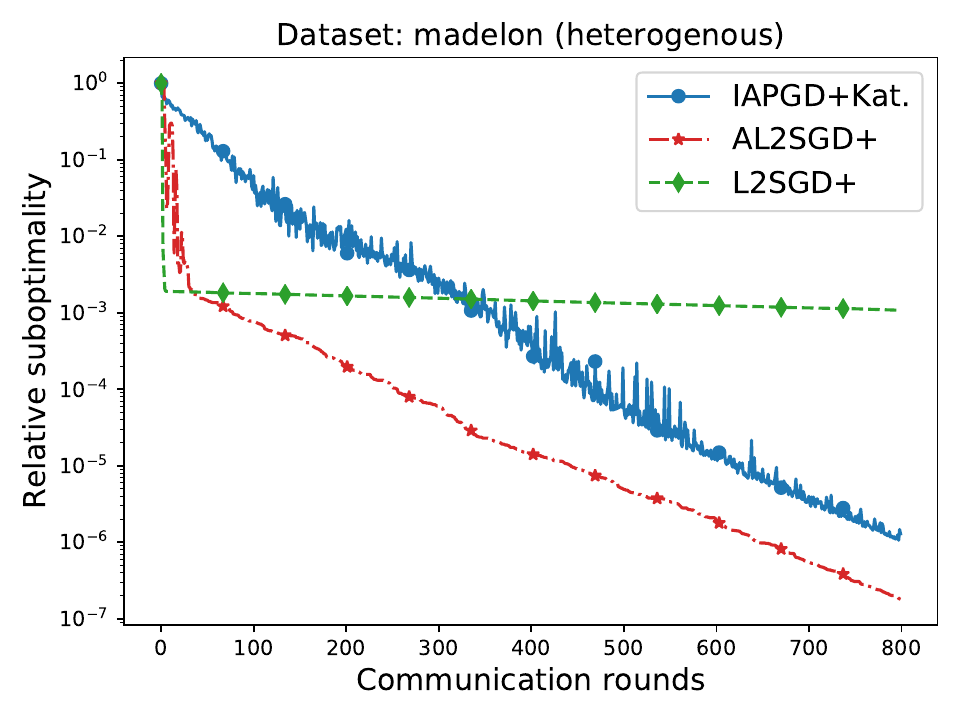}
			\includegraphics[width = 0.4 \textwidth ]{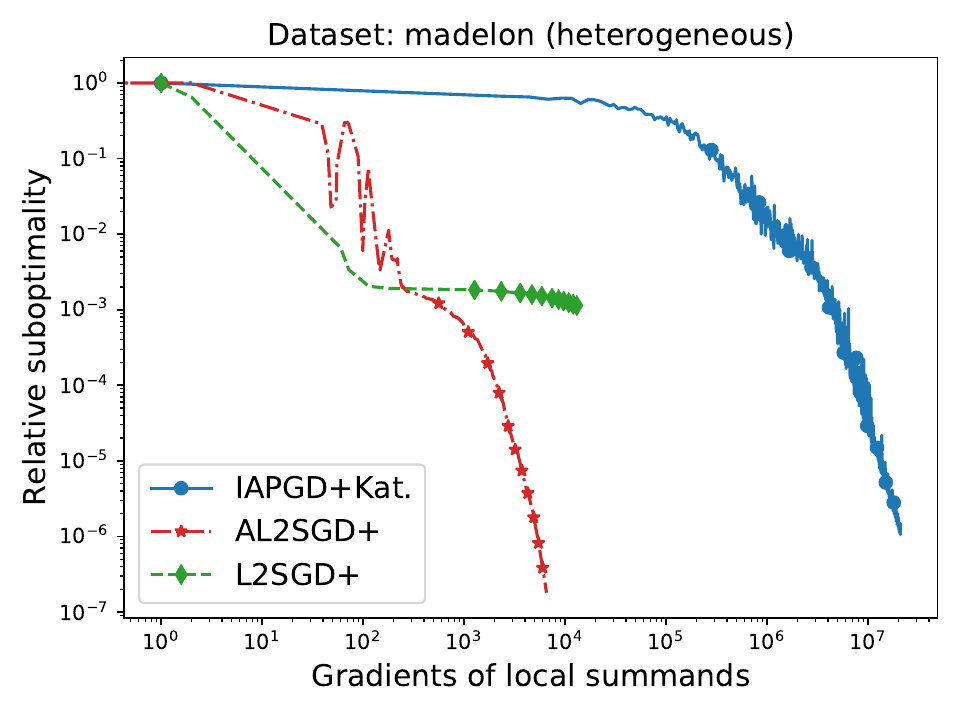}
			\includegraphics[width = 0.4 \textwidth ]{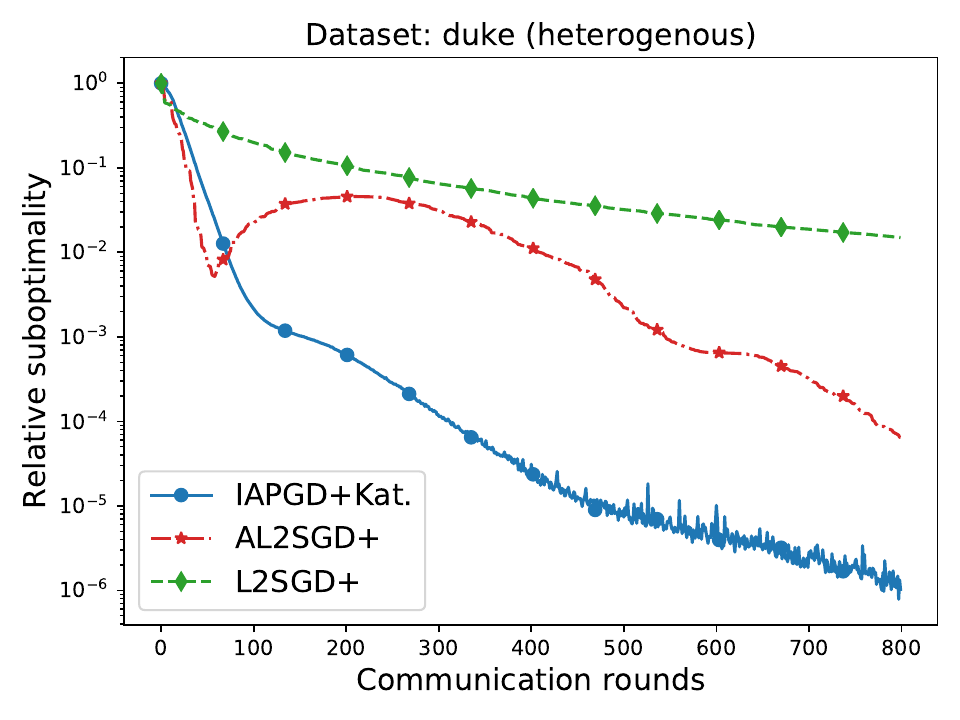}
			\includegraphics[width = 0.4 \textwidth ]{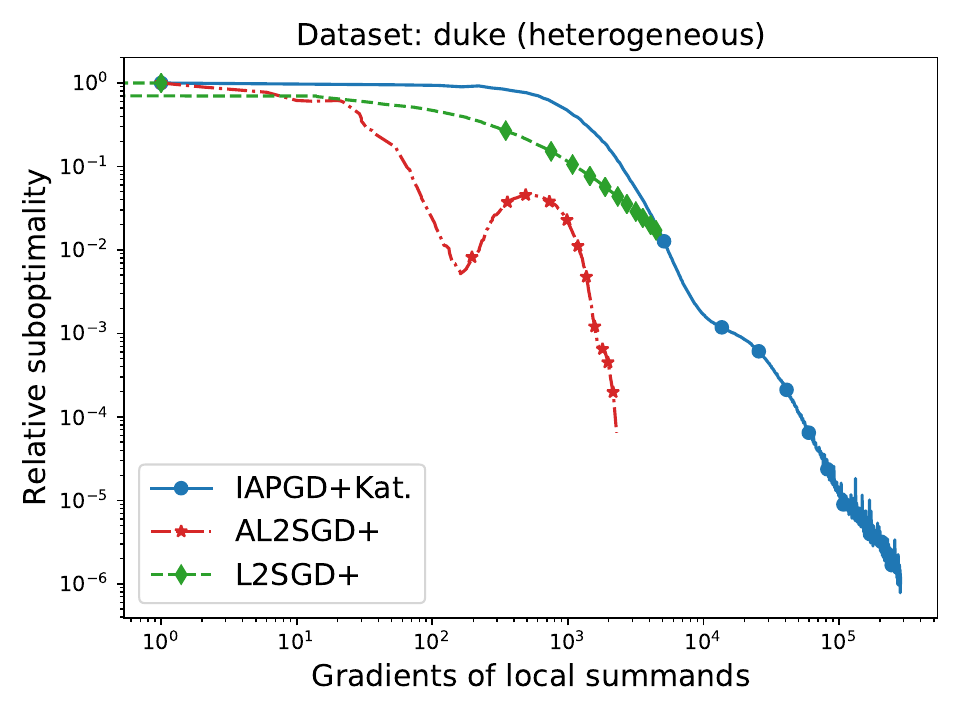}
			\includegraphics[width = 0.4 \textwidth ]{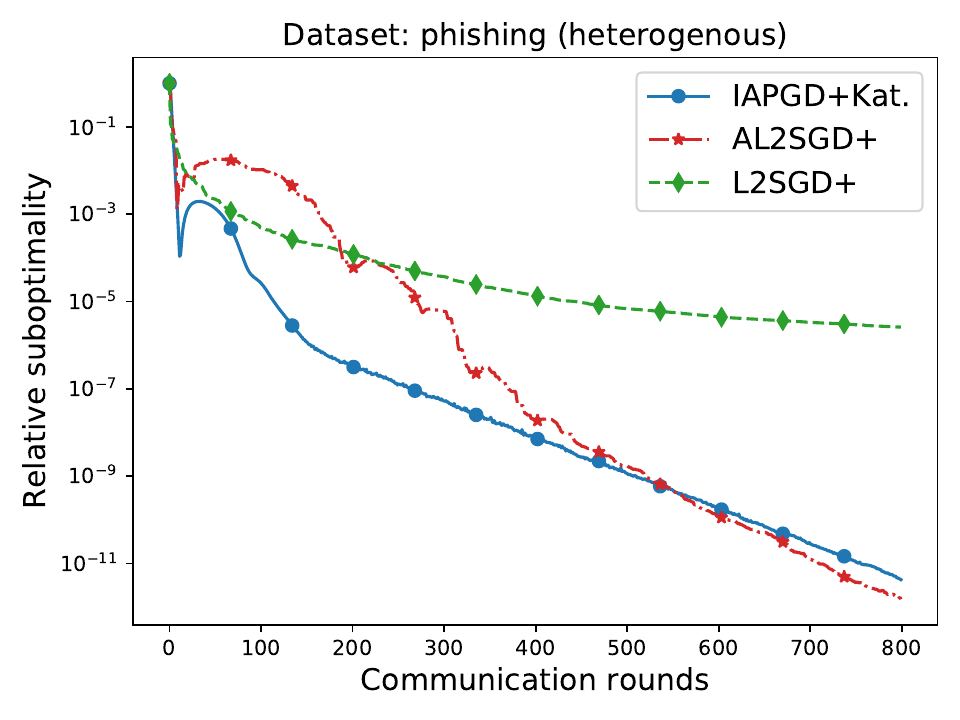}
			\includegraphics[width = 0.4 \textwidth ]{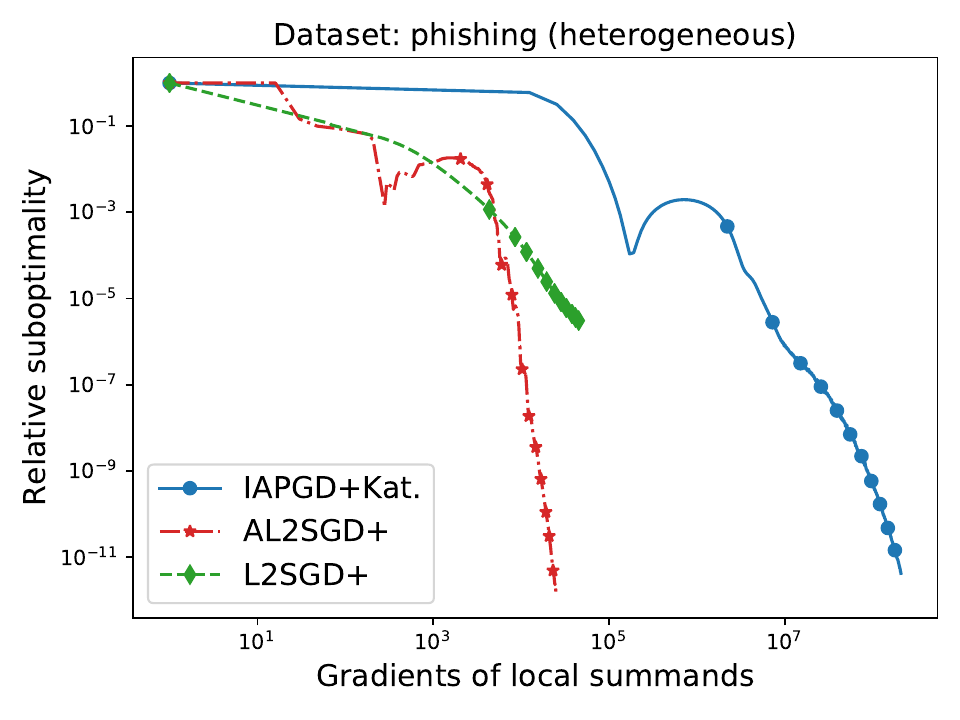}
\caption[Comparison of optimal algorithms for personalized federated learning with heterogeneous data split]{Same experiment as Figure~\ref{fig:com_met}, but a heterogeneous data split. First column: communication complexity, second column: local computation complexity for the same experiment.}
			\label{fig:stoch_hetero}
	\end{figure}

	In the third experiments, we compare two variants of \apgd{} presented in Section~\ref{sec:apgd_simple}: \apgdo{} (Algorithm~\ref{alg:fista}) and \apgdt{} (Algorithm~\ref{alg:fista_2}). We consider several synthetic instances of~\eqref{eq:main} where we vary $\lambda$ and keep remaining parameters (i.e., ${L_1}, \mu$) fixed. Our theory predicts that while the rate of \apgdt{} should not be influenced by varying $\lambda$, the rate of \apgdo{} should grow as $\cal{O}(\sqrt{\lambda})$. Similarly, \apgdo{} should be favourable if $\lambda\leq {L_1}=1$, while \apgdt{} should be the algorithm of choice for $\lambda > {L_1}=1$. As expected, Figure~\ref{fig:fistas} confirms both claims. 

	\begin{figure}[!h]
		\begin{center}
			\centerline{
				\includegraphics[width=0.5\columnwidth]{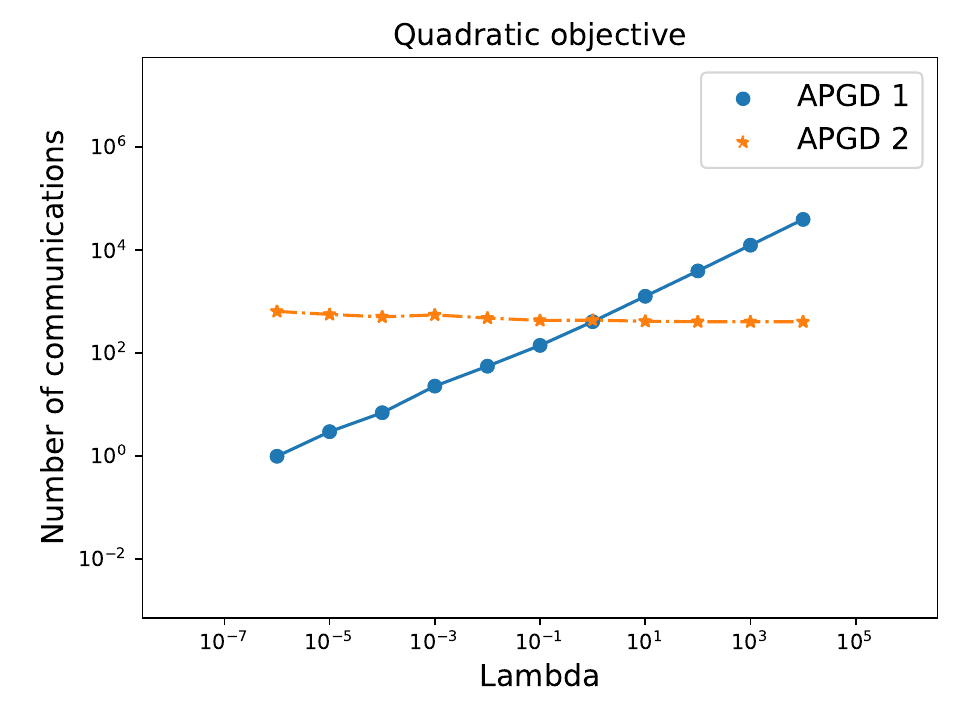}
			}
			\caption[Effect of personalization parameter $\lambda$ on communication complexity.]{Effect of the parameter $ \lambda$ on the communication complexity of \apgdo{} and \apgdt{}. For each value of $\lambda$ the y-axis indicates the number of communication required to get $10^4$--times closer to the optimum compared to the starting point. Quadratic objective with $n=50$, $d=50$. 
			}
			\label{fig:fistas}
		\end{center}
	\end{figure}

\paragraph{Experimental setup}

	In this section, we provide additional experiments comparing introduced algorithms on logistic regression with \libsvm{} data.\footnote{Logistic regression loss for on the $j$--th data point $a_j\in \R^d$ is defined as $\phi_j(x) = \log\left(1+\exp \left(b_j a_j^\top x\right)\right) + \frac{\lambda}{2}\norm{x}^2$, where $b_j \in \{-1,1 \}$ is the corresponding label. } The local objectives are constructed by evenly dividing to the workers. We vary the parameters $m,n$ among the datasets as specified in Table \ref{tbl:dataset_division}.

We consider two types of assignment of data to the clients: \emph{homogeneous} assignment, where local data are assigned uniformly at random and \emph{heterogeneous} assignment, where we first sort the dataset according to labels, and then assign it to the clients in the given order. The heterogeneous assignment is supposed to better simulate the real-world scenarios. Next, we normalize the data $a_1, a_2, \dots,$ so that $\flocc_{i,j}$ is $1$--smooth and set $\mu = 10^{-4}$. 

For each dataset we select rather small value of $\lambda$, specifically $\lambda = \frac1m$. Lastly, for  \ltsgdp{} and \altsgdp{}, we choose $p=\rho=1/m$, which is in the given setup optimal up to a constant factor in terms of the communication. We run the algorithms for $10^3$ communication rounds and track relative suboptimality\footnote{Relative suboptimality means that for iterates $\left\{ x^k \right\}_{k=1}^K$ we plot $\left\{ \frac{f(x^k)-f(x^\star)}{f(x^0)-f(x^\star)} \right\}_{k=1}^{K}$.} after each aggregation. Similarly to Figure~\ref{fig:com_met}, we plot relative suboptimality agains the number of communication rounds and local gradients computed. 

The remaining parameters are selected according to theory for each algorithm with one exception: For \texttt{IAPGD+Katyusha} we run \katyusha{} as a local subsolver at the iteration $k$ for 
\[
\sqrt{\frac{m({L_1}+\lambda)}{\mu+\lambda}} + 
\sqrt{\frac{m \mu ({L_1}+\lambda)}{\lambda(\mu+\lambda)}}k
    \]
iterations (slightly smaller than what our theory suggests). 


\begin{table}
    \centering
    \setlength\tabcolsep{10pt} 
    \begin{threeparttable}
        \caption[LIBSVM dataset partitioning for experiments]{Number of workers and local functions on workers for different datasets for Figures~\ref{fig:com_met} and~\ref{fig:stoch_hetero}.}
        \label{tbl:dataset_division}
        \begin{tabular}{|c|c|c|c|c|c|}
            \hline
            {\bf Dataset}& $n$ &$m$ &$d$ & $\lambda$ & $p=\rho$ \\
            \hline
            \hline
            \aoa{} & 5 & 321 & 119 & 0.003 & 0.003  \\
            \hline
            \duke{} & 11 & 4 & 7129 & 0.333 & 0.250 \\
            \hline
            \mushroom{} & 12 & 677 & 112 & 0.001 & 0.001 \\
            \hline
            \madelon{} & 200 & 10 &  500 & 0.111 & 0.100\\
            \hline
            \phishing{} & 335 & 33 & 68  & 0.031 & 0.030\\
            \hline 
        \end{tabular}
    \end{threeparttable}
\end{table}

\chapter{Convergence of first-order algorithms for meta-learning with Moreau envelopes} \label{sec:moreau_meta}
\thispagestyle{empty}

\section{Introduction}
Efficient optimization methods for empirical risk minimization have helped the breakthroughs in many areas of machine learning such as computer vision~\citep{krizhevsky2012imagenet} and speech recognition~\citep{hinton2012deep}. More recently, elaborate training algorithms have enabled fast progress in the area of meta-learning, also known as learning to learn~\citep{schmidhuber1987evolutionary}. At its core lies the idea that one can find a model capable of retraining for a new task with just a few data samples from the task. Algorithmically, this corresponds to solving a bilevel optimization problem~\citep{franceschi2018bilevel}, where the inner problem corresponds to a single task, and the outer problem is that of minimizing the post-training error on a wide range of tasks.

From optimization side, successful Model-Agnostic Meta-Learning (\maml{}) \citep{finn2017model} and its first-order version (\fomaml{}) \citep{finn2017model} in meta-learning applications has propelled the development of new gradient-based meta-learning methods. However, most new algorithms effectively lead to new formulations of meta-learning. For instance, \imaml{} \citep{rajeswaran2019meta} and proximal meta-learning \citep{zhou2019efficient} define two \maml{}-like objectives with implicit gradients, while Reptile \citep{nichol2018first} was proposed without defining any objective at all. These dissimilarities cause fragmentation of the field and make it particularly hard to have a clear comparison of meta-learning theory. Nonetheless, having a good theory helps to compare algorithms as well as identify and fix their limitations.

Unfortunately, for most of the existing methods, the theory is either incomplete as is the case with \imaml{} or even completely missing. In this work, we set out to at least partially mitigate this issue by proposing a new analysis for minimization of Moreau envelopes. We show that a general family of algorithms with multiple gradient steps is stable on this objective and, as a special case, we obtain results even for \fomaml{}. Previously, \fomaml{} was viewed as a heuristic to approximate \maml{} \citep{fallah2020convergence}, but our approach reveals that \fomaml{} can be regarded as an algorithm for a the sum of Moreau envelopes. While both perspectives show only approximate convergence, the main justification for the sum of Moreau envelopes is that requires unprecedentedly mild assumptions. In addition, the Moreau formulation of meta-learning does not require Hessian information and is easily implementable by any first-order optimizer, which \citet{zhou2019efficient} showed to give good empirical performance.

\subsection{Related work}\label{sec:related_work}

\begin{table*}[t]
    \centering
\setlength\tabcolsep{1.5pt} 
  \begin{threeparttable}[b]
    {\scriptsize
	\renewcommand\arraystretch{2.5}
	\caption[Difference between meta-learning approaches]{A summary of related work and conceptual differences to our approach. We mark as ``N/A'' unknown properties that have not been established in prior literature or our work. We say that $F_i$ ``Preserves convexity'' if for convex $f_i$, $F_i$ is convex as well, which implies that $F_i$ has no extra local minima or saddle points. We say that $F_i$ ``Preserves smoothness'' if its gradients are Lipschitz whenever the gradients of $f_i$ are, which corresponds to more stable gradients. We refer to \citep{fallah2020convergence} for the claims regarding nonconvexity and nonsmoothness of the \maml{} objective.}
  \label{tab:compare_with_others}
	\centering 
	\begin{tabular}{ccccccccc}\toprule[.1em]
		Algorithm & \makecell{$\meta_i$: meta-loss of task $i$} & \makecell{Hessian\\ free} & \makecell{Arbitrary\\ number \\ of steps} & \makecell{No\\ matrix\\ inversion} & \makecell{Preserves\\ convexity} & \makecell{Preserves\\ smoothness} & Reference \\
		\midrule
		 \maml{} & $f_i(x-\alpha \nabla f_i(x))$ & \xmark & \xmark & \cmark & \xmark & \xmark &  \citenum{finn2017model} \\
		 \makecell{Multi-step\\ \maml{}} & $f_i(GD(f_i, x))$\tnote{\color{red}(1)} & \xmark & \cmark & \cmark & \xmark & \xmark &  \makecell{\citenum{finn2017model}, \citenum{ji2020theoretical}} \\
		 \imaml{}\tnote{\red(2)} & \makecell{$f_i(z_i(x))$, where\\ $z_i(x)=x-\alpha \nabla f_i(z_i(x))$}& \xmark & \cmark & \xmark & \makecell{\xmark\\{\scriptsize (Th.~\ref{th:imaml_nonconvex})}} & \makecell{\xmark\\{\scriptsize(Th.~\ref{th:imaml_nonsmooth})}} & \citenum{rajeswaran2019meta}  \\
		  Reptile & N/A\tnote{\red(3)}  & \cmark & \cmark & \cmark & N/A & N/A & \citenum{nichol2018first} \\
		  \makecell{\fomaml{}\\ (original)} & $f_i(x-\alpha \nabla f_i(x))$ & \cmark & \xmark & \cmark & \xmark & \xmark & \citenum{finn2017model} \\
		  \makecell{ \tt Meta-\\ \tt MinibatchProx} & $\min\limits_{x_i}\{f_i(x_i) + \frac{1}{2\alpha}\norm{x_i-x}^2 \}$ & \cmark & \xmark\tnote{\red(4)} & \cmark & \cmark & \cmark &   \makecell{\citenum{zhou2019efficient}} \\
		  \midrule
		  \makecell{\textbf{\fomuml{}}\\ \textbf{(extended} \\ \textbf{\fomaml{})}} & $\min\limits_{x_i}\{f_i(x_i) + \frac{1}{2\alpha}\norm{x_i-x}^2 \}$ & \cmark & \cmark & \cmark & \cmark & \cmark &   \makecell{\textbf{This work}} \\
		\bottomrule[.1em]
    \end{tabular}
    }
    \begin{tablenotes}
      {\footnotesize
        \item [\color{red}(1)] Multi-step \maml{} runs an inner loop with gradient descent applied to task loss $f_i$, so the objective of multi-step \maml{} is $F_i(x)=f_i(x_s(x))$, where $x_0=x$ and $x_{j+1}=x_j - \alpha \nabla f_i(x_j)$ for $j=0,\dotsc, s-1$.
        \item [\color{red}(2)] To the best of our knowledge, \imaml{} is not guaranteed to work; \citet{rajeswaran2019meta} studied only the approximation error for gradient computation, see the discussion in our section on \imaml{}. 
        \item [\color{red}(3)] Reptile was proposed as an algorithm on its own, without providing any optimization problem. This makes it hard to say how it affects smoothness and convexity. \citet{balcan2019provable} and \citet{khodak2019adaptive} studied convergence of Reptile on the average loss over the produced iterates, i.e., $F_i(x)=\frac{1}{m}\sum_{j=0}^s f_i(x_j)$, where $x_0=x$ and $x_{j+1}=x_j - \alpha \nabla f_i(x_j)$ for $j=0,\dotsc, s-1$. Analogously to the loss of \maml{}, this objective seems nonconvex and nonsmooth.
        \item [\color{red}(4)] \citet{zhou2019efficient} assumed that the subproblems are solved to precision $\varepsilon$, i.e., $x_i$ is found such that $\norm{\nabla f_i(x_i) + \frac{1}{\alpha}(x_i - x)}\le \varepsilon$ with an absolute constant $\varepsilon$.
      }
    \end{tablenotes}
  \end{threeparttable}
\end{table*}
\maml{} \citep{finn2017model} has attracted a lot of attention due to its success in practice. Many improvements have been proposed for \maml{}, for instance, \citet{zhou2020task} suggested augmenting each group of tasks with its own global variable, and \citet{antoniou2018train} proposed \mamlpp{} that uses intermediate task losses with weights to improve the stability of \maml{}. \citet{rajeswaran2019meta} proposed \imaml{} that makes the objective optimizer-independent by relying on \emph{implicit} gradients. \citet{zhou2019efficient} used a similar implicit objective to that of \imaml{} with an additional regularization term that, unlike \imaml{}, does not require inverting matrices. Reptile \citep{nichol2018first} is an even simpler method that merely runs gradient descent on each sampled task. Based on generalization guarantees, \citet{zhou2020task} also provided a trade-off between the optimization and statistical errors for a multi-step variant \maml{}, which shows that it may not improve significantly from increasing the number of gradient steps in the inner loop. We refer to \citep{hospedales2021meta} for a recent survey of the literature on meta-learning with neural networks.

On the theoretical side, the most relevant works to ours is that of \citep{zhou2019efficient}, whose main limitation is that it requires a high-precision solution of the inner problem in Moreau envelope at each iteration. Another relevant work that studied convergence of \maml{} and \fomaml{} on the standard \maml{} objective is by \citep{fallah2020convergence}, but they do not provide any guarantees for the sum of Moreau envelopes and their assumptions are more stringent. \citet{fallah2020convergence} also study a Hessian-free variant of \maml{}, but its convergence guarantees still require posing assumptions on the Hessian Lipschitzness and variance.

Some works treat meta-learning as a special case of compositional optimization~\citep{sun2021optimal} or bilevel programming \citep{franceschi2018bilevel} and develop theory for the more general problem. Unfortunately, both approaches lead to worse dependence on the conditioning numbers of both inner and outer objective, and provide very pessimistic guarantees. Bilevel programming, even more importantly, requires computation of certain inverse matrices, which is prohibitive in large dimensions. One could also view minimization-based formulations of meta-learning as instances of empirical risk minimization, for which \fomaml{} can be seen as instance of inexact (biased) \sgd{}. For example, \citet{ajalloeian2020analysis} analyzed \sgd{} with deterministic bias and some of our proofs are inspired by theirs, except in our problem the bias is not deterministic. We will discuss the limitations of their approach in the section on inexact \sgd{}.

Several works have also addressed meta-learning from the statistical perspective, for instance, \citet{yoon2018bayesian} proposed a Bayesian variant of \maml{}, and \citet{finn2019online} analyzed convergence of \maml{} in online learning. Another example is the work of \citet{konobeev2021distributiondependent} who studied the setting of linear regression with task-dependent solutions that are sampled from same normal distribution. These directions are orthogonal to ours, as we want to study the optimization properties of meta-learning.

\section{Background and mathematical formulation}

Before we introduce the considered formulation of meta-learning, let us provide the problem background and define all notions. As the notation in meta-learning varies between papers, we correspond our notation to that of other works in the next subsection.

\subsection{Notation}
We assume that training is performed over $n$ tasks with task losses $f_1,\dotsc, f_n$ and we will introduce \emph{implicit} and \emph{proximal} meta-losses $\{F_i\}$ in the next section. 
We denote by $x$ the vector of parameters that we aim to train, which is often called \emph{model}, \emph{meta-model} or \emph{meta-parameters} in the meta-learning literature, and \emph{outer variable} in the bilevel literature. Similarly, given task $i$, we denote by $z_i$ the \emph{task-specific parameters} that are also called as \emph{ground model}, \emph{base-model}, or \emph{inner variable}. We will use letters $\alpha, \beta, \gamma$ to denote scalar hyper-parameters such as stepsize or regularization coefficient.

Given a function $\varphi(\cdot)$, we call the following function its \emph{Moreau envelope}:
\begin{align}
	\meta(x) \eqdef \min_{z\in\R^d}\left\{\varphi(x) + \frac{1}{2\alpha}\norm{z-x}^2 \right\},
\end{align}
where $\alpha>0$ is some parameter. Given the Moreau envelope $F_i$ of a task loss $f_i$, we denote by $z_i(x)$ the solution to the inner objective of $F_i$, i.e., $z_i(x)\eqdef \argmin_{z\in\R^d} \left\{f_i(z) + \frac{1}{2\alpha}\norm{z - x}^2\right\}$.

\subsection{\maml{} objective}
Assume that we are given $n$ tasks, and that the performance on task $i$ is evaluated according to some loss function $f_i(x)$. \maml{} has been proposed as an algorithm for solving the following objective:
\begin{align}
\min_{x\in\R^d} \frac{1}{n}\sumin in f_i(x - \alpha \nabla f_i(x)), \label{eq:bad_problem}
\end{align}
where $\alpha>0$ is a stepsize. Ignoring for simplicity minibatching, \maml{} update computes the gradient of a task meta-loss $\varphi_i(x)=f_i(x - \alpha \nabla f_i(x))$ through backpropagation and can be explicitly written as
\begin{align}
	x^{k+1}
	&= x^k - \beta \left(\mathbf{I}-\alpha\nabla^2 f_i(x^k) \right)\nabla f_i (x^k-\alpha\nabla f_i(x^k)),     \tag{\maml{} update}
\end{align}
where $\beta>0$ is a stepsize, $i$ is sampled uniformly from $\{1,\dotsc, n\}$ and $\mathbf{I}\in\R^{d\times d}$ is the identity matrix. Sometimes, \maml{} update evaluates the gradient of $\varphi_i$ using an additional data sample, but \citet{bai2021important} recently showed that this is often unnecessary, and we, thus, skip it.

Unfortunately, objective~\eqref{eq:bad_problem} might be nonsmooth and nonconvex even if the task losses $\{f_i\}$ are convex and smooth~\citep{fallah2020convergence}. Moreover, if we generalize this objective for more than one gradient step inside $f_i(\cdot)$, its smoothness properties deteriorate further, which complicates the development and analysis of multistep methods. 

\subsection{\imaml{} objective}\label{sec:imaml}
To avoid differentiating through a graph, \citet{rajeswaran2019meta} proposed an alternative objective to~\eqref{eq:bad_problem} that replaces the gradient step inside each function with an \emph{implicit} gradient step. In particular, if we define $z_i(x)\eqdef \argmin_{z\in\R^d} \left\{f_i(z) + \frac{1}{2\alpha}\norm{z - x}^2\right\}$, then the objective of \imaml{} is
\[
	\min_{x\in\R^d} \avein in f_i\left(x-\alpha \nabla f_i(z_i(x))\right).
\]
The idea of \imaml{} is to optimize this objective during training so that at inference, given a new function $f_{n+1}$ and solution $x_{\mathrm{\imaml{}}}$ of the problem above, one can find an approximate solution to $\min_{z\in\R^d} \left\{f_{n+1}(z) + \frac{1}{2\alpha}\norm{z - x_{\mathrm{\imaml{}}}}^2\right\}$ and use it as a new model for task $f_{n+1}$.

\citet{rajeswaran2019meta} proved, under some mild assumptions, that one can efficiently obtain an estimate of the gradient of $\varphi_i(x)\eqdef f_i\left(x-\alpha \nabla f_i(z_i(x))\right)$ with access only to gradients and Hessian-vector products of $f_i$, which rely on standard backpropagation operations. In particular, \citet{rajeswaran2019meta} showed that
\[
	\nabla \varphi_i(x)
	= \left(\mathbf{I} + \alpha \nabla^2 f_i(z(x)) \right)^{-1} \nabla f_i(z(x)),
\]
where $\mathbf{I}$ is the identity matrix, and they proposed to run the conjugate gradient method to find $\nabla \varphi_i(x)$.
However, it is not shown in \citep{rajeswaran2019meta} if the objective of \imaml{} is solvable and what properties it has. Moreover, we are not aware of any result that would show when the problem is convex or smooth. Since \sgd{} is not guaranteed to work unless the objective satisfies at least some properties \citep{zhang2020complexity}, nothing is known about convergence of \sgd{} when applied to the \imaml{} objective. 

As a sign that the problem is rather ill-designed, we present the following theorem that gives a negative example on the problem's convexity.
\begin{theorem}\label{th:imaml_nonconvex}
	There exists a convex function $f:\R^d\to\R$ with Lipschitz gradient and Lipschitz Hessian such that the \imaml{} meta-objective $\varphi(x)\eqdef f(z(x))$ is nonconvex, where $z(x)=x - \alpha \nabla f(z(x))$.
\end{theorem}
Similarly, we also show that the objective of \imaml{} may be harder to solve due to its worse smoothness properties as given by the next theorem.
\begin{theorem}\label{th:imaml_nonsmooth}
	There exists a convex function $f:\R^d\to\R$ with Lipschitz gradient and Lipschitz Hessian such that the \imaml{} meta-objective $\varphi(x)\eqdef f(z(x))$ is nonsmooth for any $\alpha>0$, where $z(x)=x - \alpha \nabla f(z(x))$.
\end{theorem}

\begin{algorithm}[t]
   \caption{\fomaml{}: First-Order \maml{}}
   \label{alg:fo_maml}
\begin{algorithmic}[1]
   \State \textbf{Input:} $x^0$, $\alpha, \beta > 0$
   \For{$k=0,1,\dotsc$}
        \State Sample a subset of tasks $T^k$
        \For{each sampled task $i$ \textbf{in} $T^k$}
            \State $z_i^k = x^k - \alpha \nabla f_{i}(x^k)$
        \EndFor
        \State $x^{k+1} = x^k - \beta \frac{1}{|T^k|}\sum_{i\in T^k} \nabla f_i(z_i^k)$
 \EndFor
\end{algorithmic}
\end{algorithm}

\subsection{Our main objective: Moreau envelopes}\label{sec:our_reformulation}
In this work we consider the following formulation of meta-learning
\begin{align}
&\min_{x\in\R^d}  \meta(x) \eqdef \frac{1}{n}\sum_{i=1}^n \meta_i(x),\quad \label{eq:new_pb} \\
&\text{where}\quad
\meta_i(x)\eqdef \min_{z\in\R^d} \left\{f_i(z) + \frac{1}{2\alpha}\norm{z - x}^2\right\}, \notag
\end{align}
and $\alpha>0$ is a parameter controlling the level of adaptation to the problem. In other words, we seek to find a parameter vector $x$ such that somewhere close to $x$ there exists a vector $z_i$ that verifies that $f_i(z)$ is sufficiently small. This formulation of meta-learning was first introduced by \citet{zhou2019efficient} and it has been used by \citet{Hanzely2020} and \citet{t2020personalized} to study personalization in federated learning.
Throughout the paper we use the following variables for minimizers of meta-problems $\meta_i$:
\begin{align}
	z_i(x) \eqdef \argmin_{z\in\R^d} \left\{f_i(z) + \frac{1}{2\alpha}\norm{z - x}^2\right\},  i=1,\dotsc, n. \label{eq:z_i}
\end{align}
One can notice that if $\alpha\to 0$, then $\meta_i(x)\approx f_i(x)$, and Problem~\eqref{eq:new_pb} reduces to the well-known empirical risk minimization:
\begin{align}
\min_{x\in\R^d} f(x)
\eqdef \frac{1}{n}\sum_{i=1}^n f_i(x).
\end{align}
If, on the other hand, $\alpha\to +\infty$, the minimization problem in~\eqref{eq:new_pb} becomes essentially independent of $x$ and it holds $z_i(x)\approx \argmin_{z\in\R^d} f_i(z)$. Thus, one has to treat the parameter $\alpha$ as part of the objective that controls the similarity between the task-specific parameters.

We denote the solution to Problem~\eqref{eq:new_pb} as
\begin{align}
x^* \eqdef \arg\min_{x\in\R^d} \meta(x).
\end{align}
One can notice that $F(x)$ and $x^*$ depend on $\alpha$. For notational simplicity, we keep $\alpha$ constant throughout the paper and do not explicitly write the dependence of $x^*, F, F_1, z_1, \dotsc, F_n, z_n$ on $\alpha$.

\subsection{Formulation properties}\label{sec:our_formulation}
We will also use the following quantity to express the difficulty of Problem~\eqref{eq:new_pb}:
\begin{align}
\sigma_*^2\eqdef \frac{1}{n}\sumin in \norm{\nabla \meta_i(\opt)}^2.  \label{eq:def_sigma}
\end{align}
Because $\nabla \meta(\opt)=0$ by first-order optimality of $\opt$, $\sigma_*^2$ serves as a measure of gradient variance at the optimum. Note that $\sigma_*$ is always finite because it is defined on a single point, in contrast to the \emph{maximum} gradient variance over all space, which might be infinite.
Now let's discuss properties of our formulation \ref{eq:new_pb}. Firstly, we state a standard result from~\citep{beck-book-first-order}.

\begin{proposition}\citep[Theorem 6.60]{beck-book-first-order}\label{pr:moreau_is_smooth}
	Let $\meta_i$ be defined as in \cref{eq:new_pb} and $z_i(x)$ be defined as in \cref{eq:z_i}.
	If $f_i$ is convex, proper and closed, then $\meta_i$ is differentiable and $\frac{1}{\alpha}$--smooth:
	\begin{align}
	&\nabla \meta_i(x)
	= \frac{1}{\alpha}(x-z_i(x)) = \nabla f_i(z_i(x)), \label{eq:implicit} \\
	&\norm{\nabla \meta_i(x) - \nabla \meta_i(y)}\le \frac{1}{\alpha}\norm{x-y}.
	\end{align}
\end{proposition}

The results above only hold for convex functions, while in meta-learning, the tasks are often defined by training a neural network, whose landscape is nonconvex. To address such applications, we also refine \Cref{pr:moreau_is_smooth} in the lemma bellow, which also improves the smoothness constant in the convex case. This result is similar to Lemma 2.5 of \citep{davis2021proximal}, except their guarantee is a bit weaker because they consider more general assumptions.

\begin{lemma}\label{lem:moreau_is_str_cvx_and_smooth}
	Let function $f_i$ be ${L_1}$--smooth.
	\begin{itemize}
		\item If $f_i$ is nonconvex and $\alpha<\frac{1}{{L_1}}$, then $\meta_i$ is $\frac{{L_1}}{1-\alpha {L_1}}$--smooth. If $\alpha \le \frac{1}{2{L_1}}$, then $\meta_i$ is $2{L_1}$--smooth.
		\item If $f_i$ is convex, then $\meta_i$ is $\frac{{L_1}}{1+\alpha {L_1}}$--smooth. Moreover, for any $\alpha$, it is ${L_1}$--smooth.
		\item If $f_i$ is $\mu$--strongly convex, then $\meta_i$ is $\frac{\mu}{1+\alpha\mu}$--strongly convex. If $\alpha \le \frac{1}{\mu}$, then $\meta_i$ is $\frac{\mu}{2}$--strongly convex.
	\end{itemize}
	Whenever $\meta_i$ is smooth, its gradient is given as in \eqref{eq:implicit}, i.e., $\nabla \meta_i(x) = \nabla f_i(z_i(x))$.
\end{lemma}
The takeaway message of \Cref{lem:moreau_is_str_cvx_and_smooth} is that the optimization properties of $\meta_i$ are always at least as good as those of $f_i$ (up to constant factors). Furthermore, the \emph{conditioning}, i.e., the ratio of smoothness to strong convexity, of $\meta_i$ is upper bounded, up to a constant factor, by that of $f_i$. And even if $f_i$ is convex but nonsmooth (${L_1}to+\infty)$, $F_i$ is still smooth with constant $\frac{1}{\alpha}$.

Finally, note that computing the exact gradient of $F_i$ requires solving its inner problem as per equation~\eqref{eq:implicit}. Even if the gradient of task $\nabla f_i(x)$ is easy to compute, we still cannot obtain $ \nabla \meta_i(x)$ through standard differentiation or backpropagation. However, one can approximate $\nabla \meta_i(x)$ in various ways, as we will discuss later.

\section{Can we analyze \fomaml{} as inexact \sgd{}?}\label{sec:inexact_sgd}
As we mentioned before
, the prior literature has viewed \fomaml{} as an inexact version of \maml{} for problem~\eqref{eq:bad_problem}. If, instead, we are interested in problem~\eqref{eq:new_pb}, one could still try to take the same perspective of inexact \sgd{} and see what convergence guarantees it gives for~\eqref{eq:new_pb}. The goal of this section, thus, is to refine the existing theory of inexact \sgd{} to make it applicable to \fomaml{}. We will see, however, that such approach if fundamentally limited and we will present a better alternative analysis in a future section. 
\subsection{Why existing theory is not applicable}
Let us start with a simple lemma for \fomaml{} that shows why it approximates \sgd{} for objective~\eqref{eq:new_pb}.
\begin{lemma}\label{lem:approx_implicit}
	Let task losses $f_i$ be ${L_1}$--smooth and $\alpha>0$. Given $i$ and $x\in\R^d$, we define recursively 
	\begin{equation*}
		{\color{mydarkred}z_{i,0}} \eqdef {\color{blue}x}, \qquad \text{and} \qquad {\color{mydarkred}z_{i,j+1}} \eqdef {\color{blue}x} - \alpha \nabla f_i({\color{mydarkred}z_{i,j}}).
	\end{equation*}		
	Then, for any $s\ge 0$ holds
	\begin{align}
		\norm{ \nabla f_i(z_{i,s}) - \nabla \meta_i(x) } \le (\alpha {L_1})^{s+1} \norm{\nabla \meta_i(x)}.
	\end{align}
	In particular, the \fomaml{} (\Cref{alg:fo_maml}) uses sets $z_i^k = {\color{mydarkred}z_{i,1}}$ from \eqref{eq:intro_maml_loop} and hence satisfy
	\begin{align}
		\norm{ \nabla f_i(z_i^k) - \nabla \meta_i(x^k) } \le (\alpha {L_1})^2 \norm{\nabla \meta_i(x^k)}.
	\end{align}
\end{lemma}
 \Cref{lem:approx_implicit} shows that \fomaml{} approximates \sgd{} step with error proportional to the stochastic gradient norm. Therefore, we can write
\begin{align}
 	\nabla f_i(z_i^k)
 	= \nabla F(x^k) + \underbrace{\nabla F_i(x^k) - \nabla F(x^k)}_{\eqdef \xi_i^k\ (\mathrm{noise})} + \underbrace{b_i^k}_{\mathrm{bias}},
\end{align}  
where it holds $\E{\xi_i^k}=0$, and $b_i^k$ is a bias vector that also depends on $i$ but does not have zero mean. The best known guarantees for inexact \sgd{} are provided by \citet{ajalloeian2020analysis}, but they are, unfortunately, not applicable because their proofs use independence of $\xi_i^k$ and $b_i^k$. The analysis of \citet{zhou2019efficient} is not applicable either because their inexactness assumption requires the error to be smaller than a predefined constant $\varepsilon$, while the error in \Cref{lem:approx_implicit} can be unbounded. To resolve these issues, we provide a refined analysis in the next subsection.

\subsection{A new result for inexact \sgd{}}
For strongly convex objectives, we give the following result by modifying the analysis of \citet{ajalloeian2020analysis}.

\begin{theorem}[Convergence of \fomaml{}, weak result]\label{th:fo_maml}
	Let task losses $f_1,\dotsc, f_n$ be ${L_1}$--smooth and $\mu$--strongly convex. If $|T^k|=\tau$ for all $k$, $\outers \leq \frac 1 {20{L_1}}$ and $\alpha\le \frac{1}{4\sqrt{\kappa}{L_1}}$, where $\kappa\eqdef \frac{{L_1}}{\mu}$, then for the iterates $x^1, x^2\dots$ of \Cref{alg:fo_maml}, it holds
	\begin{align}
		\E{\norm{x^k-\opt}^2}
		&\le \left(1 - \frac{\outers\mu}{4}\right)^k\norm{x^0-\opt}^2 
		 + \frac{16}{\mu} \left( \frac {2\alpha^2 {L_1}^2} {\mu} + \frac \outers \tau + \outers  \right) \sigma_*^2.
	\end{align}
\end{theorem}
Let us try to compare this result to that of vanilla \sgd{} as studied by \citet{Gower2019}. Since the first term decreases exponentially, it requires us $\cO\left(\frac{1}{\beta\mu}\log\frac{1}{\varepsilon} \right)$ iterations to make it smaller than $\varepsilon$. The second term, on the other hand, only decreases if we decrease $\alpha$ and $\beta$. Decreasing $\beta$ corresponds to using decreasing stepsizes in \sgd{}, which is fine, but $\alpha$ is a parameter that defines the objective, so in most cases, we do not want to decrease it. Moreover, the assumptions of \Cref{th:fo_maml} require $\alpha$ to be smaller than $\frac{1}{\sqrt{\kappa}{L_1}}$, which seems quite restrictive. This is the main limitation of this result as it shows that \fomaml{} as given in \Cref{alg:fo_maml} may not converge to the problem solution.

\begin{algorithm}[t]
   \caption{\fomuml{}: First-Order Multistep Meta-Learning (general formulation)}
   \label{alg:mamlP}
\begin{algorithmic}[1]
   \State \textbf{Input:} $x^0$, $\beta>0$, accuracy $\delta\geq0$ or $\varepsilon\ge 0$.
   \For{$k=0,1,\dotsc$}
        \State Sample a subset of tasks $T^k$
        \For{each sampled task $i$ \textbf{in} $T^k$}
            \State  Find $z_i^k$ s.t.\ $\norm{\frac 1 \inners \left(x^k -z_i^k \right) - \nabla \meta_i(x^k)} \leq \delta \norm{\nabla \meta_i(x^k)}$          
        \EndFor
        \State $x^{k+1} = x^k - \beta\frac{1}{|T^k|}\sum_{i\in T^k} \nabla f_i(z_i^k)$
   \EndFor
\end{algorithmic}
\end{algorithm}

To fix the nonconvergence of \fomaml{}, let us turn our attention to \Cref{alg:mamlP}, which may perform multiple first-order steps.

\begin{theorem} \label{th:convergence_of_mamlP}
	Let task losses $f_1,\dotsc, f_n$ be ${L_1}$--smooth and $\mu$--strongly convex. If $|T^k|=\tau$ for all $k$, $\alpha\le \frac{1}{{L_1}}, \outers \leq \frac 1 {20{L_1}}$, and $\delta \leq \frac 1 {4 \sqrt{ \kappa}}$, where $\kappa\eqdef \frac{{L_1}}{\mu}$, then the iterates of \Cref{alg:mamlP} satisfy
	\begin{align}
		\E{\norm{x^k-\opt}^2}
		&\le \left(1 - \frac{\outers\mu}{4}\right)^k\norm{x^0-\opt}^2
		 + \frac{16}{\mu} \left( \frac {2\delta^2} {\mu} + \frac \outers \tau + \outers \delta^2 \right) \sigma_*^2.
	\end{align}
\end{theorem}
The result of \Cref{th:convergence_of_mamlP} is better than that of \Cref{th:fo_maml} since it only requires the inexactness parameter $\delta$ to go to 0 rather than $\alpha$, so we can solve the meta-learning problem \eqref{eq:new_pb} for any $\alpha\le \frac{1}{{L_1}}$. The rate itself, however, is not optimal, as we show in the next section with a more elaborate approach.

\section{Improved theory}\label{sec:better_theory}

\begin{algorithm}[t]
   \caption{\fomuml{} (example of implementation)}
   \label{alg:maml2}
\begin{algorithmic}[1]
   \State \textbf{Input:} $x^0$, number of steps $s$, $\alpha > 0$, $\beta>0$
   \For{$k=0,1,\dotsc$}
        \State Sample a subset of tasks $T^k$
        \For{each sampled task $i$ \textbf{in} $T^k$}
            \State $z_{i, 0}^k = x^k$
            \For{$l = 0, \dotsc, s-1$}
                \State $z^k_{i, l+1} = { x^k} - \alpha \nabla f_{i}({z_{i, l}^k})$
            \EndFor
            \State $z_i^k = z_{i,s}^k$
        \EndFor
        \State $x^{k+1} = x^k - \beta\frac{1}{|T^k|}\sum_{i\in T^k} \nabla f_i(z_i^k)$
   \EndFor
\end{algorithmic}
\end{algorithm}

In this section, we provide improved convergence theory of \fomaml{} and \fomuml{} based on a sequence of virtual iterates that appear only in the analysis. Surprisingly, even though the sequence never appears in the algorithm, it allows us to obtain tighter convergence bounds.

\subsection{Perturbed iterate is better than inexact gradient}
Before we introduce the sequence, let us make some observations from prior literature on inexact and biased variants of \sgd{}. For instance, the literature on asynchronous optimization has established that getting gradient at a wrong point does not significantly worsen its rate of convergence \citep{mania2017perturbed}. A similar analysis with additional virtual sequence was used in the so-called error-feedback for compression \citep{stich2018sparsified}, where the goal of the sequence is to follow the path of \emph{exact} gradients even if \emph{compressed} gradients are used by the algorithm itself. Motivated by these observations, we set out to find a virtual sequence that could help us analyze \fomaml{}.
\subsection{On what vector do we evaluate the gradients?}
The main difficulty that we face is that we never get access to the gradients of $\{F_i\}$ and have to use the gradients of $\{f_i\}$. However, we would still like to write
\begin{align}
x^{k+1} 
=x^k - \frac \inners \tau \sum_{i \in T^k} \nabla f_i(z_i^k) 
= x^k - \frac \inners \tau \sum_{i \in T^k} \nabla \meta_i(y_i^k),
\end{align}
for some point $y_i^k$. If this is possible, using point $y_i^k$ would allow us to avoid working with functions $f_i$ in some of our recursion.

Why exactly would this sequence help? As mentioned before, \fomaml{} is a biased method, so we cannot evaluate expectation of $\E{\nabla f_i(z_i^k)}$. However, if we had access to $\nabla F_i(x^k)$, its expectation would be exactly $\nabla F(x^k)$. This suggests that if we find $y_i^k$ that satisfies
$\nabla \meta_i(y_i^k) \approx \nabla \meta_i(x^k)$, then 
\[
	x^{k+1} = x^k - \frac \inners \tau \sum_{i \in T^k} \nabla \meta_i(y_i^k) \approx x^k - \frac \inners \tau \sum_{i \in T^k} \nabla \meta_i(x^k),
\]
which would allow us to put the bias \emph{inside} the gradient. 
Fortunately, Moreau Envelopes objective \eqref{eq:new_pb} allows us to find such point easily. 

\begin{lemma}\label{lem:explicit_grad_to_implicit}
	For any points $z, y \in \R^d$ it holds $y= z + \inners \nabla f_i(z)$ if and only if $z = y - \inners \nabla \meta_i(y)$. Therefore, given $z$, we can define $y=z+\alpha \nabla f_i(z)$ and obtain $\nabla f_i(z)=\nabla F_i(y)$.
\end{lemma}
\begin{proof}
	The result follows immediately from the last statement of Lemma~\ref{lem:moreau_is_str_cvx_and_smooth}.
\end{proof}

The second part of \Cref{lem:explicit_grad_to_implicit} is exactly what we need. Indeed, we can choose $y_i^k \eqdef z_{i}^k + \inners \nabla  f_i(z_{i}^k)$ so that $z_{i}^k = y_i^k - \inners \nabla \meta_i(y_i^k)$ and $\nabla f_i(z_i^k)=\nabla F_i(y_i^k)$. As we have explained, this can help us to tackle the bias of \fomaml{}.

\subsection{Main results}
We have established the existence of variables $y_i^k$ such that $\nabla f_i(z_i^k)=\nabla F_i(y_i^k)$. This allows us to write 
\begin{align}
 	\nabla f_i(z_i^k)
 	= \nabla F_i(y_i^k) 
 	= \nabla F(x^k) + \underbrace{\nabla F_i(x^k) - \nabla F(x^k)}_{\mathrm{noise}}
 	+ \underbrace{\nabla F_i(y_i^k) - \nabla F_i(x^k)}_{\textrm{reduced bias}}&.
\end{align}  
As the next theorem shows, we can use this to obtain convergence guarantee to a neighborhood even with a small number of steps in the inner loop.
\begin{theorem} \label{th:convengence_of_mamlP_no_stepsize}
	Consider the iterates of \Cref{alg:mamlP} (with general $\delta$) or \Cref{alg:fo_maml} (for which $\delta=\alpha {L_1}$).
Let task losses be ${L_1}$--smooth and $\mu$--strongly convex and let objective parameter satisfy $\inners \leq \frac {1}{\sqrt 6 {L_1}}$. Choose stepsize $ \beta \leq \frac \tau {4 {L_1}}$, where $\tau = |T^k|$ is the batch size. Then we have
	\begin{align}
		\E{\norm{x^k-x^*}^2} &\leq \left(1 - \frac {\beta \mu}{12}  \right)^k \norm{x^0 - x^*}^2
	  + \frac { 6\left( \frac \beta \tau + 3 \delta^2 \inners^2{L_1}\right) \varopt} {\mu}.
	\end{align}
\end{theorem}
Similarly to \Cref{th:fo_maml}, the theorem above guarantees convergence to a neighborhood only. However, the radius of convergence is now $\cO\left(\frac{\frac{\beta}{\tau} + \alpha^2{L_1}}{\mu} \right)$ in contrast to $\cO\left(\frac{\beta + \kappa\alpha^2{L_1}}{\mu} \right)$. If the first term is dominating, then it implies an improvement proportional to the batch size $\tau$. If, in contrast, the second term is larger, then the improvement is even more significant and the guarantee is $\cO(\kappa)$ times better, which is often a very large constant.

The proof technique for this theorem also uses recent advances on the analysis of biased \sgd{} methods by \citet{mishchenko2020random}. In particular, we show that the three-point identity (provided in the Appendix) is useful for getting a tighter recursion.

Next, we extend this result to the nonconvex convergence as given under the assumption on bounded variance.
\begin{definition}\label{def:bounded_var}
	We assume that the variance of meta-loss gradients is uniformly bounded by some $\sigma^2$, i.e.,
	\begin{align}
		\E{\norm{\nabla F_i(x) - \nabla F(x)}^2}
		\le \sigma^2. \label{eq:bounded_var}
	\end{align}
\end{definition}
The new assumption on bounded variance is different from the one we used previously of variance being finite at the optimum, which was given in equation~\eqref{eq:def_sigma}. At the same time, it is very common in literature on stochastic optimization when studying convergence on nonconvex functions.
\begin{theorem}\label{th:nonconvex_fo_maml}
	Let the variance of meta-loss gradients is uniformly bounded by some $\sigma^2$ (\Cref{def:bounded_var}), functions $f_1,\dotsc, f_n$ be ${L_1}$--smooth and $F$ be lower bounded by $F^*>-\infty$. Assume $\alpha\le \frac{1}{4{L_1}}, \beta\le \frac{1}{16{L_1}}$. If we consider the iterates of \Cref{alg:fo_maml} (with $\delta=\alpha {L_1}$) or \Cref{alg:mamlP} (with general $\delta$), then
	\begin{align}
		\min_{t\le k}\E{\norm{\nabla F(x^t)}^2}
		&\le \frac{4}{\beta k}\E{F(x^0)-F^*} 
        + 4(\alpha {L_1})^2\delta^2 \sigma^2 \nonumber\\
		& \qquad+ 32 \beta(\alpha {L_1})^2 \left(\frac{1}{|T^k|} + (\alpha {L_1})^2\delta^2\right) \sigma^2.
	\end{align}
\end{theorem}
Notice that this convergence is also only until some neighborhood of first-order stationarity, since the second term does not decrease with $k$. This size of the upper bound depends on the product $\cO((\alpha {L_1})^2 \delta^2)$, so to obtain better convergence one can simply increase approximation accuracy to make $\delta$ smaller. However, the standard \fomaml{} corresponds to $\delta=\alpha {L_1}$, so its convergence guarantees directly depend on the problem parameter $\alpha$.

For Algorithm~\ref{alg:maml2}, we have $\delta=\cO((\alpha {L_1})^s)$ as per Lemma~\ref{lem:approx_implicit}, and we recover convergence guarantee up to a neighborhood of size $\cO((\alpha {L_1})^2\delta^2)=\cO((\alpha {L_1})^{2s+2})$. Therefore, to make this smaller than some given target accuracy $\varepsilon>0$, we need at most $s=\cO(\log\frac{1}{\varepsilon})$ inner-loop iterations. If we can plug-in $s=1$, we also get that \fomaml{} converges to a neighborhood of size $\cO((\alpha {L_1})^4)$.

Our \Cref{th:nonconvex_fo_maml} is very similar to the one obtained by \citet{fallah2020convergence}, except their convergence neighborhood depends on $\alpha$ as $\cO(\alpha^2)$, whereas ours is of size $\cO(\alpha^4)$, which goes to 0 much faster when $\alpha\to 0$. Moreover, in contrast to their theory, ours does not require any assumptions on the Hessian smoothness. Note, in addition, that the main difference comes from the kind of objectives that we study, as \citet{fallah2020convergence} considered minimization of problems not involving Moreau envelopes.

\section{Conclusion}
In this paper, we presented a new analysis of first-order meta-learning algorithms for minimization of Moreau envelopes. Our theory covers both nonconvex and strongly convex smooth losses and guarantees convergence of the family of methods covered by Algorithm~\ref{alg:mamlP}. As a special case, all convergence bounds apply to Algorithm~\ref{alg:maml2} with an arbitrary number of inner-loop steps. Compared to other results available in the literature, ours are more general as they hold with an arbitrary number of inner steps and do not require Hessian smoothness. The main theoretical difficulty we faced was the limitation of the inexact \sgd{} framework, which we overcame by presenting a refined analysis using virtual iterates. As a minor contribution, we also pointed out that standard algorithms, such as \sgd{}, are not immediately guaranteed to work on the \imaml{} objective, which might be nonconvex and nonsmooth even for convex and smooth losses. To show this, we presented examples of losses whose convexity and smoothness cease when the \imaml{} objective is constructed.

\chapter{Adaptive learning of the optimal mini-batch size of \sgd{}} \label{sec:adaptive}
\thispagestyle{empty}

\section{Introduction}

Stochastic Gradient Descent (\sgd{}), in one disguise or another, is undoubtedly the backbone of modern systems for training supervised machine learning models~\citep{robbins1951stochastic, nemirovski2009robust, bottou2010large}. The method earns its popularity due to its superior performance on very large datasets where more traditional methods such as Gradient Descent (\gd{}), relying on a pass through the entire training dataset before adjusting the model parameters, are simply too slow to be useful. In contrast, \sgd{} in each iteration uses a small portion of the training data only (a  batch) to adjust the model parameters, and this process repeats until a model of suitable quality is found. \\

In practice,  batch \sgd{} is virtually always  applied to a ``finite-sum'' problem of the form
\begin{equation} \label{eq:ERM}
    x^* \eqdef \arg \min \limits_{x\in \R^d} \frac{1}{n} \sum \limits_{i = 1}^n f_i(x),
\end{equation}
where $n$ is the number of training data and $f(x)=\frac{1}{n} \sum_{i = 1}^n f_i(x)$ represents the average loss, i.e.\ empirical risk, of model $x$ on the training dataset. 
\citet{Gower2019} showed that for a random vector $v \in \R^d$ sampled from a user defined distribution $\cD$ satisfying $\bbE_{\cD}\left[ v_i\right] =1$, problem \eqref{eq:intro_erm} can be reformulated as 
\begin{equation}
    \min_{x \in \R^d} \bbE_{\cD} \left [ f_v(x) \eqdef \avein in v_i f_i(x)\right].
\end{equation}
This problem can be solved by applying generic minibatch \sgd{} method, which performs iterations 
of the form
\begin{equation} \label{eq:sgd-MB}
 x^{k+1} = x^k - \gamma^k \g_{v_k}(x^k),
\end{equation}
where $v_k \sim \cD$ is sampled each iteration. Typically, the vector $v$ is defined by first choosing a random minibatch $\mathcal S^k \subseteq \{1,2,\dots,n\}$ and then defining $v_i^k =0$ for $i \not \in \mathcal S^k$ and $v_i^k$ for $\mathcal S^k$ so that stochastic gradient $\g_{v^k}(x^k)$ is unbiased. 
For example, one standard choice is to fix a  batch size $\tau \in \{1,2,\dots,n\}$, and pick sampling $\mathcal S^k$ uniformly from all subsets of size $\tau$. In such case, $v_i^k= \frac n \tau$ for $i \in \mathcal S^k$ and $v_i^k= 0$ for $i \not \in \mathcal S^k$.


\subsection{Accurate modelling of the second moment}
Despite the vast empirical evidence of the efficiency of  batch \sgd{} \citep{minibatch}, and despite ample theoretical research in this area \citep{cotter2011better, li2014efficient, gazagnadou2019optimal, dont_decrease_step_size}, our understanding of this method is far from complete.  For instance, while it is intuitively clear that as the  batch size approaches $n$ the method should increasingly behave as the full batch \gd{} method, the prevalent theory does not support this. The key reason for this, as explained by \citep{SGD_general_analysis} who consider the regime where all functions $f_i$ are smooth, and $f$ is (quasi) strongly convex, is inappropriate theoretical modelling of the second moment of the stochastic gradient $g^k$, which is not satisfied for stochastic gradients having the  batch structure. Indeed, a typical analysis of \sgd{} assumes that there exists a constant $\sigma^2>0$ such that $\E{\norm{g^k}^2 \;|\ x^k} \leq \sigma^2$
holds throughout the iterations of \sgd{}. However, this is not true for \gd{}, seen as an extreme case of  batch \sgd{}, which always sets $\mathcal S^k=\{1,2,\dots,n\}$, and does not hold in general for any  batch size. Hence, analyses relying on this assumption are not only incorrect for  batch \sgd{}~\citep{nguyen2018sgd}, but also make predictions which do not reflect empirical evidence. For instance, while practitioners have noticed that the choice $\tau>1$ leads to a faster method in practice even in a single processor regime, theory based on strong assumptions  does not predict this. Also, several works claim that the optimal batch size for \sgd{} is one \citep{li2014efficient}.\\

Recently, \citet{SGD_general_analysis} studied the convergence of  batch \sgd{} under a new type of assumption, called expected smoothness (ES). It is provably satisfied for stochastic gradients arising from  batching, and which is precise enough for all  batch sizes to lead to nontrivial predictions. Their assumption has the form $\E{\norm{g^k}^2 \;|\ x^k} \leq 2 A (f(x^k) - f(x^*)) + B$.
When compared to the previous assumption, ES includes an additional term proportional to the functional suboptimality $f(x^k) - f(x^*)$.  In particular,  their analysis recovers the fast linear rate of \gd{} in the extreme case when $\tau=n$ (in this case, ES holds with $B=0$ and $A=L$, where $L$ is the smoothness constant of $f$), and otherwise gives a linear rate to a neighborhood of the solution $x^*$ whose size depends on $B$. 

\subsection{Search for the optimal  batch size}

Denote
\begin{equation} \label{eq:T} 
T(\tau) \eqdef  \max \left\{ \cK_1(\tau), \cK_2(\tau) \right\},
\end{equation}
where $\cK_1 \eqdef \tau(nL-L_{\max})+nL$ and $\cK_2\eqdef \frac{2(n-\tau)\Bar{h}(x^*)}{\mu\epsilon}$, $L$ is the smoothness constant of $f$, $L_{\max}$ is the maximum of the smoothness constants of the functions $f_i$, $\mu$ is the modulus of (quasi) strong convexity of $f$ and 
$ \Bar{h}(x) \eqdef \frac{1}{n}\sum \limits_{i=1}^n h_i(x),$
where $h_i(x) \eqdef \norm{\nabla f_i(x)}^2$. 
Since both $\cK_1$ and $\cK_2$ are both
linear functions in
$\tau$, the expression  $T(\tau)$ in \eqref{eq:T} can be analytically minimized in $\tau$, giving a formula for the optimal  batch size $\tau^*$ of \sgd{}.

\subsection{Motivation}
While the existence of a theoretically grounded formula for an optimal  batch size provides a marked step in the understanding of  batch \sgd{}, the formula is not implementable in practice for several reasons. 
First, the quantity $\Bar{h}(x^*)$ is unknown as it depends on a-priori knowledge of the optimal solution $x^*$, which we are trying to find. One exception to this is the case of interpolated/overparameterized models characterized by the property that $\nabla f_i(x^*)=0$ for all $i$.
\begin{figure}
\begin{center}
\centerline{\includegraphics[width = 0.7\textwidth]{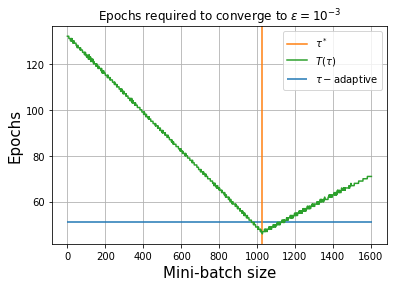}}
\caption[Number of epochs to converge for \sgd{} with fixed mini-batch size]{\textbf{Number of epochs to converge} to $\epsilon$--neighborhood of the solution for different  batch sizes, tested on Logistic Regression using $\tau$--partition nice sampling on \aoa{} dataset. Our adaptive method outperforms \sgd{} with $87.9\%$ of different fixed  batch sizes, achieving speedup up to $250\%$.}
\label{fig:convergence_different_taus} 
\end{center}
\end{figure}
However, in general, this property may not hold, and when it does, we may know this before we train the model. 
Second, the formula depends on the strong (quasi)-convexity parameter $\mu$, which also may not be known. Third, while the formula depends on the target accuracy $\epsilon$, in practice, we may not know what accuracy is most desirable in advance.

Putting these parameter estimation issues aside, we may further ask: assuming these parameters are known, does the formula provide a good predictor for the behavior of  batch \sgd{}? If it did not, it would be less useful, and perhaps inappropriate as a guide for setting the  batch size of \sgd{}. \citet{SGD_general_analysis} showed through experiments with synthetic data that the optimal  batch size does well in cases when all these parameters are assumed to be known. However, they did not explore whether their formula \eqref{eq:T}  provides a good model for the {\em empirical} total complexity of  batch \sgd{} across {\em all} values of $\tau$. The starting point of our research in this paper was investigating this question; the answer was rather surprising to us, as we illustrate in  Figure~\ref{fig:convergence_different_taus}. Our experiment shows that the theoretical predictions are astonishingly tight. The green V-shaped line reflects the empirical performance of \sgd{} for various levels of  batch sizes, and correlates with the shape of the theoretical total complexity function $T(\tau)$. The orange vertical line depicts the theoretically optimal  batch size $\tau^*$, i.e., the solution to the problem $\min_{1 \leq \tau \leq n} T(\tau)$, and acts as a virtually perfect predictor of the optimal empirical performance. Encouraged by the above observation, we believe it is important to deliver the theoretically impractical formula for $\tau^*$ into the practical world, and this is the main focus of our paper. 

\subsection{Summary of contributions}
In particular, we make the following contributions:

\begin{itemize}[leftmargin=*]
\item \textbf {Effective online learning of the optimal  batch size.} We make a step towards the development of a practical variant of optimal  batch \sgd{}, aiming to learn the optimal  batch size $\tau^*$ on the fly. To the best of our knowledge, our method (Algorithm~\ref{alg:sgd_dyn_general}) is the first variant of \sgd{} able to learn the optimal  batch.  The blue horizontal line in Figure~\ref{fig:convergence_different_taus} depicts the typical performance of our online method compared to  \sgd{} with various fixed  batch sizes $\tau$. Our method can perform as if we had an almost perfect knowledge of $\tau^*$ available to us before the start of the method. Indeed, the blue line cuts across the green V-shaped curve in a small neighborhood of the optimal  batch size. 

\item  \textbf {Sampling strategies.} We do not limit our selves to the uniform sampling strategy we used for illustration purposes above and develop closed-form expressions for the optimal  batch size for other sampling techniques (Section~\ref{sec:analysis}). Our adaptive method works well for all of them.

\item  \textbf {Convergence theory.} We prove that our adaptive method converges, and moreover learns the optimal  batch size (Section~\ref{sec:algo}; Theorem~\ref{th:convergence_our}).

\item  \textbf {Practical robustness.} We show the algorithm's robustness by conducting extensive experiments using different sampling techniques and different machine learning models (Section~\ref{sec:experiments}).
\end{itemize}

\subsection{Related work}
A separate stream of research attempts to progressively reduce the variance of the stochastic gradient estimator aiming to obtain fast convergence to the solution, and not only to a neighborhood. \sag{} \citep{SAG}, \saga{} \citep{SAGA}, \svrg{} \citep{SVRG},  \lsvrg{} \citep{L-SVRG-kovalev, L-SVRG-hoffman} and \jacsketch{} \citep{JacSketch} are all examples of a large family of such variance reduced \sgd{} methods. 

While we do not pursue this direction here, we point out that \citet{gazagnadou2019optimal} specialized the \jacsketch{} theory of \citep{JacSketch}  to derive expressions for the optimal  batch size of \saga. A blend of this work and ours is perhaps possible and would lead to an adaptive method for learning the optimal  batch size of \saga.\\

Another approach to boost the performance of \sgd{} in practice is tuning its hyperparameters such as learning rate, and  batch size while training. In this context, a lot of work has been done in proposing various learning rate schedulers \citep{schumer1968adaptive, mathews1993stochastic,barzilai1988two,zeiler2012adadelta,tan2016barzilai,adaptive_step_size}. \citet{adaptive-mini-batch} showed that one can reduce the variance by increasing the  batch size without decreasing stepsize (to maintain the constant signal to noise ratio). Besides, \citet{dont_decrease_step_size} demonstrated the effect of increasing the batch size instead of decreasing the learning rate in training a deep neural network. However, most of the strategies for hyper-parameter tuning are based on empirical results only. For example, \citet{you2017large,you2017scaling} show empirically the advantage of training on large  batch size, while \citet{masters2018revisiting} claim that it is preferable to train on small  batch sizes.

\section{\sgd{} overview}\label{sec:overview}
To proceed in developing the theory related to the optimal  batch size, we give a quick review of  batch \sgd{}. To study such methods for virtually all (stationary) subsampling rules, we adopt the stochastic reformulation paradigm for finite-sum problems proposed by \citet{SGD_general_analysis}.

 For a given random vector $v \in \mathbb{R}^n$ sampled from a user defined distribution $\cD$ and satisfying $\Ed{\mathcal{D}}{v_i} = 1$, 
 the empirical risk in \eqref{eq:ERM} can be solved by taking the update step \eqref{eq:sgd-MB}.
 
Typically, the vector $v$ is defined by first choosing a random  batch $\cS^k\subseteq \{1,2,\dots,n\}$, and then defining $v_i^k  = 0$ for $i\notin \cS^k$, and choosing $v_i^k$ to an appropriate value for $i\notin \cS^k$ in order to make sure the stochastic gradient $\nabla f_{v^k}(x^k)$ is unbiased. 
We will now briefly review four particular choices of the probability law governing the selection of the sets $\cS^k$ we consider in this paper.

\begin{itemize}[leftmargin=*]
\item \textbf{$\tau$--independent sampling} \citep{SGD_general_analysis}. In this sampling, each element is drawn independently with probability $p_i$ with $\sum_i p_i = \tau$ and $\E{\mathcal{S}|} = \tau$.
 For each element $i$, we have $
\prob[v_i=\mathbbm 1_{i \in \mathcal S}]=p_i$. 

\item \textbf{$\tau$--nice sampling (without replacement)} \citep{SGD_general_analysis}. In this sampling, we choose subset of $\tau$ examples from the total of $n$ examples in the training set at each iteration. In this case, $\prob[|\mathcal{S}| = \tau] = 1$. 
For each subset $C \subseteq \{1,\dots, n \}$ of size $\tau$, we have $\prob[v_C=\mathbbm 1_{C \in \mathcal S}]=1/\binom n \tau$.

\item \textbf{$\tau$--partition nice sampling} In this sampling, we divide the training set into partitions $\mathcal C_j$ (of possibly different sizes $n_{\mathcal C_j}$), and each of them has at least a cardinality of $\tau$. At each iteration, one of the sets $\mathcal C_j$ is chosen with probability $q_{\mathcal C_j}$, and then $\tau$--nice sampling (without replacement) is applied on the chosen set. For each subset $C$ cardinality $\tau$ of partition $C_j$ cardinality $n_{C_j}$, we have $\prob[v_C=\mathbbm 1_{C \in \mathcal S}]=q_j/\binom {n_{C_j}} \tau$. 

\item \textbf{$\tau$--partition independent sampling} Similar to $\tau$--partition nice sampling, we divide the training set into partitions $\mathcal C_j$, and each of them has at least a cardinality of $\tau$. At each iteration one of the sets $\mathcal C_j$ is chosen with probability $q_{\mathcal C_j}$, and then $\tau$--independent sampling is applied on the chosen set. For each element $i$ of partition $C_j$, we have $\prob[v_i=\mathbbm 1_{C \in \mathcal S}]=q_jp_i$. 
\end{itemize}

Sampling using $\tau$--partition nice and $\tau$--partition independent sampling are the generalization of both $\tau$--nice (without replacement) and $\tau$--independent sampling respectively. Generalized samplings arise more often in practice whenever big datasets are distributed throughout different nodes. Consequently, it allows us to model a real situation where distributions of training examples throughout nodes might be very heterogeneous. 
These samplings are especially interesting to be analyzed because it can model distributed optimization where the data is partitioned across multiple servers within which the gradient of a batch is computed and communicated to a parameter server. 
The stochastic formulation naturally leads to the concept of \emph{expected smoothness} (ES), which can be defined as follows \citep{sega}.
\begin{definition} \label{def:expected_smoothness}
	A function $f: \R^d \to \R$ is \emph{expected $\cL$--smooth} with respect to a datasets $\mathcal{D}$ if there exist $\cL > 0$ such that
	\begin{equation}\label{eq:exp_smoothn_cL}
		\E{\norm{ \nabla f_v(x) - \nabla f_v(x_*) }^2} \leq 2 \cL(f(x) - f(x_*)). 
	\end{equation}
\end{definition}

Apart of expected smoothness, the above formulation depends on is the finite gradient noise at the optimum which can be stated as follows \citep{SGD_general_analysis}
\begin{definition}
\label{def:gradient_noise}
For sampling $\mathcal S$ s.t. $\E{|\mathcal S|} = \tau$ Denote gradient noise at point $x$ as
\begin{equation}
    \sg(x, \tau) \eqdef \Ed{v \sim \mathcal{D}}{\norm{ \nabla f_v(x) }^2 }.
\end{equation}
and the gradient noise at optimum at the optimum $\sg \eqdef \sg(x^*, \tau)$.
\end{definition}
\begin{assumption}
	\label{as:gradient_noise}
	The gradient noise is finite at the optimum,  $\sg \< \infty$.
\end{assumption}

Based on the expected smoothenss and finite gradient noise at optimum, the update steps in \eqref{eq:sgd-MB} provide linear convergence up to a neighbourhood of the optimal point $x^*$. This convergence is described in the following theorem \citep{SGD_general_analysis}.
\begin{theorem}
\label{th:convergence_general}
Assume $f:\R^d\to \R$ is $\mu$--strongly convex, $\mathcal L$--expected smooth and Assumption \ref{as:gradient_noise} 
be satisfied. For any $\epsilon > 0$, if the learning rate $\gamma$ is set to be
\begin{equation} 
    \gamma = \frac 1 2 \min \left\{ \frac{1}{\mathcal{L}}, \frac{\epsilon \mu}{2\sg} \right\} \qquad  \text{and}\qquad k \geq \frac 2 {\mu} \max\left\{ \mathcal{L}, \frac{2\sg}{\epsilon \mu} \right\} \log\left( \frac{2 \norm{x^0 - x^*}^2}{\epsilon} \right), \label{eq:iteration_complexity}
\end{equation}
then $\E{ \norm{ x^k - x^* }^2 } \leq \epsilon$.
\end{theorem}

\section{Deriving optimal batch size} \label{sec:analysis}
After giving this thorough introduction to the stochastic reformulation of \sgd{}, we can move on to study the effect of theminibatch size on the total iteration complexity. In fact, for each sampling technique, theminibatch size will affect both the expected smoothness $\mathcal{L}$ and the gradient noise $\sg$. This effect reflects on the number of iterations required to reach to $\epsilon$ neighborhood around the optimal model \eqref{eq:iteration_complexity}.

\subsection{Formulas for $\mathcal L$ and $\sigma$}
Before proceeding, we establish some terminologies. In addition of having $f$ to be $L$--smooth, we also assume each $f_i$ to be $L_i$--smooth.
In $\tau$--partition samplings (both nice and independent), let $n_{\mathcal C_j}$ be number of data-points in the partition $\mathcal C_j$, where $n_{\mathcal C_j} \geq \tau$. Let $L_{\mathcal C_j}$ be the smoothness constants of the function $f_{\mathcal C_j} = \frac{1}{n_{\mathcal C_j}}\sum_{i \in \mathcal C_j} f_i$. Also, let $\overline L_{\mathcal C_j} = \frac{1}{n_{\mathcal C_j}} \sum_{i\in\mathcal C_j} L_i$ be the average of the Lipschitz smoothness constants of the functions in partition $\mathcal C_j$. In addition, let $h_{\mathcal C_j}(x) = \norm{\nabla f_{\mathcal{C}_j}(x)}^2$ be the norm of the gradient of $f_{\mathcal{C}_j}$ at $x$. Finally, let $\overline {h}_{\mathcal C_j}(x) = \frac {1} {n_{\mathcal C_j}} \sum_{i \in C_j} h_i(x)$. For ease of notation, we will drop $x$ from all of the expression since it is understood from the context $(h_i = h_i(x))$. Also, superscripts with $(*,k)$ refer to evaluating the function at $x^*$ and $x^k$ respectively $(e.g.\,\, h_i^* = h_i(x^*))$.\\

Now we introduce our first key lemma, which gives an estimate of the expected smoothness for different sampling techniques, where the proof is left for the appendix.

\begin{lemma} \label{le:L}
For the considered samplings, the expected smoothness constants $\mathcal L$ can be upper bounded by $\mathcal L(\tau)$ (i.e. $\mathcal L \leq \mathcal L(\tau)$), where $\mathcal L(\tau)$ is expressed as follows.
(i) For $\tau$--partition nice sampling,
\begin{equation*} \label{eq:L_part_nice}
\mathcal L (\tau) = \frac{1}{n\tau} \max_{\mathcal C_j} \frac{n_{\mathcal C_j}}{q_{\mathcal C_j}(n_{\mathcal C_j}-1)} 
\Big[(\tau-1) L_{\mathcal C_j}n_{\mathcal C_j}
+ (n_{\mathcal C_j}-\tau) \max_{i \in{\mathcal C_j}}{L_i} \Big].
\end{equation*}

\noindent
(ii) For $\tau$--partition independent sampling, $\mathcal L (\tau) = \frac{1}{n}\max_{\mathcal C_j}{\frac{n_{\mathcal C_j} L_{\mathcal C_j}}{q_{\mathcal C_j}} + \max_{i \in C_j}{\frac{L_i(1-p_{i})}{q_{C_j}p_{i}}}}.$
\end{lemma}

Note that (i), (ii) give as special case, for having all of the data in a single partition, the expected smoothness bound $\mathcal L(\tau)$ for $\tau$--nice sampling (without replacement) and $\tau$--independent sampling, which are results of \citep{SGD_general_analysis}. The bound on the expected smoothness constant for $\tau$--nice sampling without replacement is given by
$\mathcal L (\tau) = \frac {n(\tau-1)}{\tau(n-1)}L + \frac {n-\tau}{\tau(n-1)}\max_i L_{i},$
and that for $\tau$--independent sampling is given by
$\mathcal L (\tau) = L + \max_i \frac {1-p_i}{p_i}\frac {L_i} n.$\\

Next, we move to estimate the noise gradient $\sg (x^*, \tau)$ for different sampling techniques. The following lemma can accomplish this, where the proof is left for the appendix.

\begin{lemma} \label{le:variance}
For the considered samplings, the gradient noise is given by $\sg (x^*, \tau)$, where $\sg (x, \tau)$ is defined as:

(i) For $\tau$--partition nice sampling, $\sg(x,\tau) = \frac{1}{n^2\tau}\sum_{\mathcal C_j} \frac{n_{\mathcal C_j}^2}{q_{\mathcal C_j}(n_{\mathcal C_j}-1)} \Big[(\tau-1)  h_{\mathcal C_j}n_{\mathcal C_j} 
 + (n_{\mathcal C_j}-\tau) \overline {h}_{\mathcal C_j}\Big].$

(ii) For $\tau$--partition independent sampling, $\mathcal \sg(x,\tau) = \frac{1}{n^2}\sum_{\mathcal C_j}\frac{n_{\mathcal C_j}^2  h_{\mathcal C_j} + \sum_{i \in C_j}{\frac{1-p_{i}}{p_{i}} h_i}}{q_{C_j}}.$


\end{lemma}

Similarly, (i) and (ii) can be specialized in the case of having all of the data in a single partition to the gradient noise $\sg(x^*, \tau)$ for $\tau$--nice sampling (without replacement) and $\tau$--independent samplings, which are the results of \citep{SGD_general_analysis}. For $\tau$--nice sampling without replacement, we have
$\sg(x^*,\tau)=\frac {n-\tau} {n\tau(n-1)} \sum_{i \in [n]} h_i^*,$
and for $\tau$--independent sampling, we have
$\sg(x^*,\tau)=\frac 1 {n^2} \sum_{i \in [n]} \frac {1-p_i}{p_i}h_i^*.$\\

In what follows, we are going to deploy these two key lemmas to find expressions for the optimal minibatch size that will minimize the total iteration complexity.

\subsection{Optimal batch size}
Our goal is to estimate total iteration complexity as a function of $\tau$. In each iteration, we work with $\tau$ gradients, thus we can lower bound on the total iteration complexity by multiplying lower bound on iteration complexity \eqref{eq:iteration_complexity} by $\tau$.
We can apply similar analysis as in \citep{SGD_general_analysis}. 
Since we have estimates on both the expected smoothness constant and the gradient noise in terms of the minibatch size $\tau$, we can lower bound total iteration complexity \eqref{eq:iteration_complexity} as
\begin{equation}
    T(\tau) = \frac 2 {\mu} \max \left\{ \tau\mathcal L(\tau), \frac{2}{\epsilon \mu}\tau \sg(x^*,\tau) \right\} \log\left( \frac{2 \norm{x^0 - x^*}^2}{\epsilon} \right).
\end{equation}
Note that if we are interested in minimizer of $T(\tau)$, we can drop all constant terms in $\tau$. Therefore, optimal minibatch size $\tau^*$ minimizes $\max \left\{ \tau\mathcal L(\tau), \frac{2}{\epsilon \mu}\tau \sg(x^*,\tau) \right\}$. It turns out that all $\tau \mathcal L(\tau)$, and $\tau \sg(x^*,\tau)$ from Lemmas \ref{le:L}, \ref{le:variance} are piece-wise linear functions in $\tau$, which is cruicial in helping us find the optimal $\tau^*$ that minimizes $T(\tau)$.  Indeed, for the proposed sampling techniques, the optimal minibatch size can be calculated from the following theorem, where its proof is left for the appendix.

\begin{theorem}\label{le:tau_star}
For $\tau$--partition nice sampling, the optimal minibatch size is $\tau(x^*)$, where 
\begin{eqnarray} \label{eq:tau_part_nice}
 \tau(x) &=&    \min_{\mathcal C_r}\frac{
    \frac{nn_{\mathcal C_r}^2}{e_{\mathcal{C}_r}} (L_{\mathcal C_r}-L_{\max}^{\mathcal C_r}) + \frac{2}{\epsilon \mu} \sum_{\mathcal C_j} \frac{n_{\mathcal C_j}^3}{e_{\mathcal C_j}}\left(\overline {h}_{\mathcal C_j}- h_{\mathcal C_j}\right)}
    {\frac {nn_{\mathcal{C}_r}} {e_{\mathcal C_r}}(n_{\mathcal C_r}L_{\mathcal C_r}-L_{\max}^{\mathcal{C}_r}) + \frac{2}{\epsilon \mu}\sum_{\mathcal C_j} 
    \frac{n_{\mathcal C_j}^2 }{e_{\mathcal C_j}}({\overline {h}_{\mathcal C_j} -n_{\mathcal C_j}  h_{\mathcal C_j} })},\\
&& \qquad  \text{if } \sum_{\mathcal C_j} \frac{n_{\mathcal C_j}^2}{e_j} ( h_{\mathcal C_j}^*n_{\mathcal C_j}-\overline {h}^*_{\mathcal C_j}) \le 0, \nonumber
\end{eqnarray}
where $e_{\mathcal C_k}=q_{\mathcal C_k} (n_{\mathcal C_k} -1)$, $L_{\max}^{\mathcal{C}_r}=\max_{i \in \mathcal C_r}{L_i}$. Otherwise: $\tau(x^*)=1$.

For $\tau$--partition independent sampling with $p_i = \frac{\tau}{n_{\mathcal C_j}}$, where $\mathcal C_j$ is the partition where i belongs, we get the optimal minibatch size $\tau(x^*)$, where 
\begin{eqnarray} \label{eq:tau_part_ind}
    \tau(x) &=& \min_{\mathcal C_r}{\frac{
    \frac{2}{\epsilon \mu}
    \sum_{\mathcal C_j} \frac{n_{\mathcal C_j}^2} { q_{\mathcal C_j}} \overline {h}_{\mathcal C_j} - \frac{n}{q_{\mathcal C_r}} L_{\max}^{\mathcal{C}_r} }
    {
    \frac{2}{\epsilon \mu}\sum_{\mathcal C_j} \frac {n_{\mathcal C_j}} { q_{\mathcal C_j}}
    (\overline {h}_{\mathcal C_j} - n_{\mathcal C_j} h_{\mathcal C_j})
    +
    \frac n {q_{\mathcal C_r}} (n_{\mathcal C_r}L_{\mathcal C_r}-L_{\max}^{\mathcal C_r}) 
    }}, \\
	&& \qquad \text{if } \sum_{\mathcal C_j} \frac{n_{\mathcal C_j}}
    { q_{\mathcal C_j}} (n_{\mathcal C_j}  h_{\mathcal C_j}^* - \overline {h}^*_{\mathcal C_j}) \le 0, \nonumber
\end{eqnarray}
where $L_{\max}^{\mathcal C_r}=\max_{i \in \mathcal C_r}{L_i}$. Otherwise: $\tau(x^*)=1$.
\end{theorem}

Moreover, if all of the data are stored in one partition, then \eqref{eq:tau_part_nice} becomes formulas for $\tau$--nice sampling and $\tau$--independent sampling which as an earlier result of \citep{SGD_general_analysis}.
Note that as $\epsilon \rightarrow 0$ then $\tau^* \rightarrow n$. It makes intuitive sense, for it means that if one aims to converge to actual minimizer, then \sgd{} will turn to be gradient descent. On the other hand, as $\epsilon$ grows to be very large, optimal minibatch size shrinks to $1$. 
Also note that the larger the variance at the optimum, the larger the optimal minibatch size is. If the variance at the optimum is negligible, as in the case in overparameterized models, then the optimal minibatch size will be 1.

\begin{figure}[t]
\begin{center}
\centerline{
\includegraphics[width=0.5\textwidth]{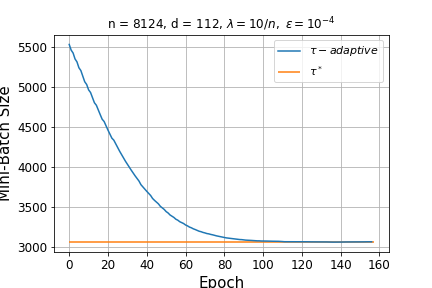}
\includegraphics[width=0.5\textwidth]{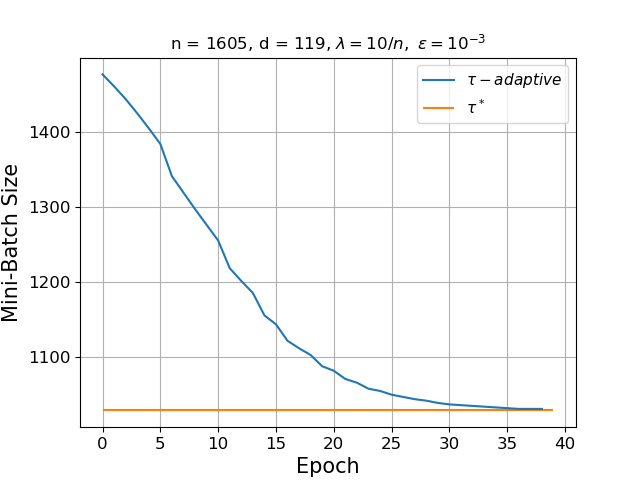}
}

\centerline{
\includegraphics[width=0.5\textwidth]{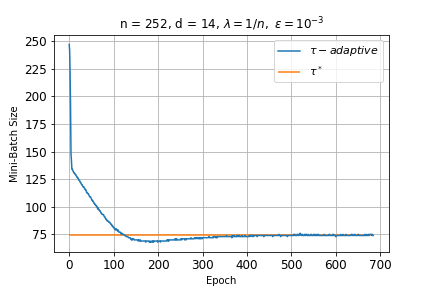}
\includegraphics[width=0.5\textwidth]{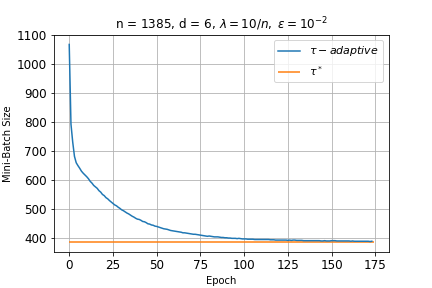}
}
\caption[Convergence of \sgd{} with adaptive mini-batch size]{\textbf{Behavior of $\tau$--adaptive throughout training epochs.} First row: logistic regression on \mushroom{} and \aoa{} datasets, respectively. Second row: ridge regression on \bodyfat{} and \mg{} datasets, respectively.}
\label{fig:behaviour_of_tau}
\end{center}
\end{figure}
\begin{algorithm}[H]
\caption{\sgd{} with Adaptive Batch size}
\label{alg:sgd_dyn_general}
\begin{algorithmic}[1]
\State \textbf{Input:} Smoothness constants $L$, $L_i$, strong convexity constant $\mu$, target neighborhood $\epsilon$, Sampling Strategy $S$, variance cap $C \geq 0$, $x^0$
\For{$k=0, \dots, K$}
    \State $\tau^k = \tau(x^k)$ 
    \State $\lk = \mathcal L(\tau^k)$
    \State $\sg^k = \sg(x^k, \tau^k)$
    \State $\gamma^k = \frac 1 2 \min \left\{\frac 1 {\lk}, \frac{\epsilon \mu}{\min(C,2\sg^k)} \right\}$ 
    \State Sample $v^k$ from $S$
    \State $x^{k+1} = x^k - \gamma^k \nabla f_{v^k} (x^k)$
\EndFor
\State \textbf{Output:} $x^K$
\end{algorithmic}
\end{algorithm}

\section{Proposed algorithm} \label{sec:algo}

The theoretical analysis gives us the optimal minibatch size for each of the proposed sampling techniques. However, we are unable to use these formulas directly since all of the expressions of optimal minibatch size depend on the knowledge of $x^*$ through the values of $h_i^* \forall i \in [n]$. Our algorithm overcomes this problem by estimating the values of $h_i^*$ at every iteration by $h_i^k$. Although this approach seems to be mathematically sound, it is costly because it requires passing through the whole training set every iteration. Alternatively, a more practical approach is to store $h_i^0=h_i(x^0)\,\, \forall i \in [n]$, then set $h_i^k = \norm{\nabla f_i(x^k)}^2\,\, for\,\, i\in\mathcal S_k$ and $h_i^k = h_i^{k-1} \,\,for\,\, i \notin \mathcal S_k$, where $\mathcal S_k$ is the set of indices considered in the $k^{th}$ iteration. In addition to storing an extra $n$ dimensional vector, this approach costs only computing the norms of the stochastic gradients that we already used in the \sgd{} step. Both options lead to convergence in a similar number of epochs, so we let our proposed algorithm adopt the second (more practical) option of estimating $h_i^*$.\\

In our algorithm, for a given sampling technique, we use the current estimate of the model $x^k$ to estimate the sub-optimal minibatch size $\tk \eqdef \tau(x^k)$ at the $k^{\text{th}}$ iteration. Based on this estimate, we use Lemmas \ref{le:L},\ref{le:variance} in calculating an estimate for both the expected smoothness  $\mathcal{L}(\tau^k)$ and the noise gradient $\sg(x^k, \tau^k)$ at that iteration. After that, we compute the stepsize $\gamma^k$ using \eqref{eq:iteration_complexity} and finally conduct a \sgd{} step \eqref{eq:sgd-MB}. The summary can be found in Algorithm \ref{alg:sgd_dyn_general}.\\

Note that intuitively, as $x^k$ gets closer to $x^*$, our estimate for the gradients at the optimum $h_i^k$ gets closer to $h_i^*$ and $\tk$ get closer to optimal minibatch size $\tau(x^*)$. Later in this section, we show convergence of the iterates $x^0, x^1, \dots$ to a neighborhood of the solution $x^*$, and convergence of the sequence of batch estimates $\tau^1, \tau^2, \dots$ to a neighborhood around the optimal minibatch size $\tau^*$.
\begin{figure*}[t]
\begin{center}
\centerline{
\includegraphics[width=0.5\textwidth]{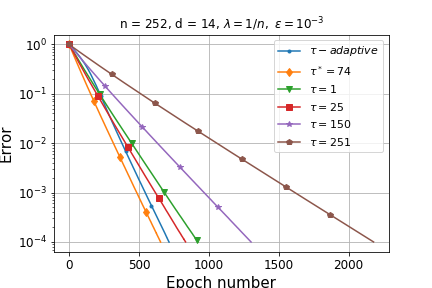}
\includegraphics[width=0.5\textwidth]{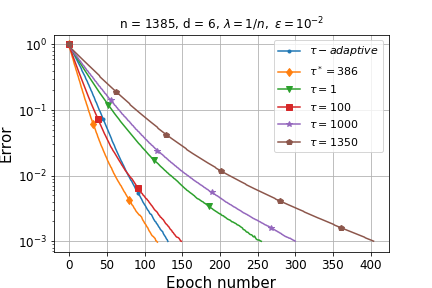}
}
\centerline{
\includegraphics[width=0.5\textwidth]{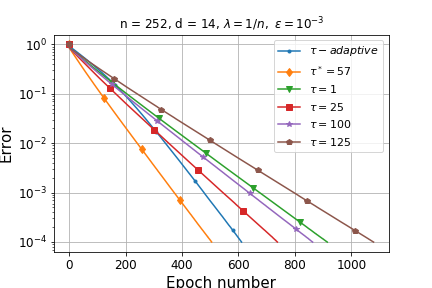}
\includegraphics[width=0.5\textwidth]{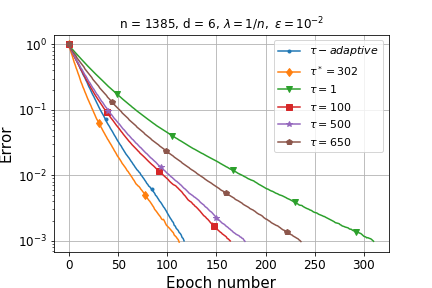}
}
\centerline{
\includegraphics[width=0.5\textwidth]{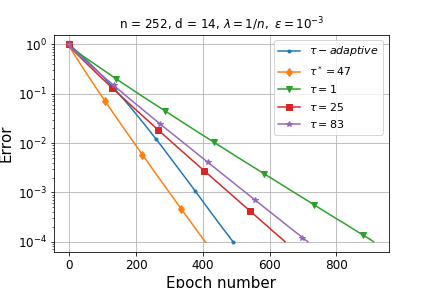}
\includegraphics[width=0.5\textwidth]{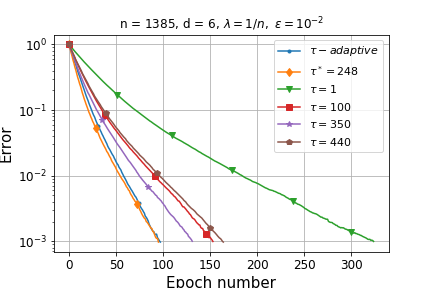}
}
\centerline{
\includegraphics[width=0.5\textwidth]{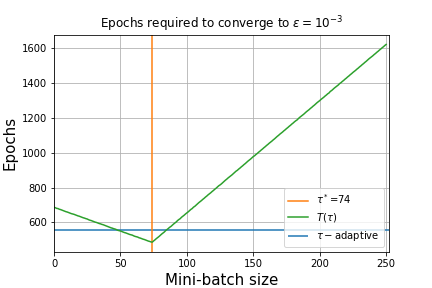}
\includegraphics[width=0.5\textwidth]{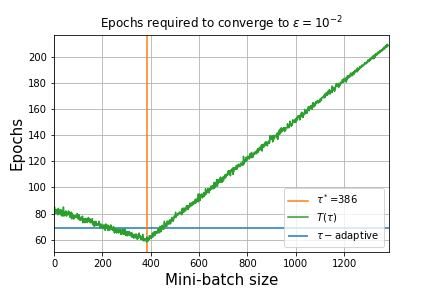}
}
\caption[Convergence of \sgd{} with adaptive mini-batch size, $\tau$--partition nice sampling on ridge regression]{\textbf{Convergence of ridge regression} using $\tau$--partition nice sampling on \bodyfat{} dataset (first column) and $\tau$--partition independent sampling on \mg{} dataset (second column). In first three rows, training set is distributed among $1$, $2$ and $3$ partitions, respectively. Figures in the fourth row show how many epochs does the \sgd{} need to converge for allminibatch sizes (single partition dataset).
}
\label{fig:ridge_nice}
\end{center}
\end{figure*}

\subsection{Convergence of Algorithm \ref{alg:sgd_dyn_general}}

Although the proposed algorithm is taking an \sgd{} step at each iteration and we can derive one-step estimate (for $\gamma^k \leq \frac 1 {2 \mathcal L^k}$)
\begin{equation} \label{eq:one_step}
\E{\norm{x^{k+1}-x^*}^2 |x^k} \leq (1-\gamma^k \mu)\norm{x^k-x^*}^2 +2(\gamma^k)^2 \sg
,\end{equation}
the global behavior of our algorithm differs from \sgd{}.
We dynamically update both the learning rate and theminibatch size. At each iteration, we set the learning rate $\gamma^k = \frac 1 2 \min \left\{\frac 1 {\lk}, \frac{\epsilon \mu}{2\sg^k} \right\}$, which depends on our estimates of both $\lk$ and $\sg^k$. If we were guaranteed that both of these terms are bounded above, then we can show that the sequence of generated learning rates will be lower bounded by a positive constant. In such cases, one can use inequality \eqref{eq:one_step} to prove the convergence of Algorithm \ref{alg:sgd_dyn_general}. To this end, we cap $\sg^k$ by a positive constant $C$, and we set the learning rate at each iteration to $\gamma^k = \frac 1 2 \min \left\{\frac 1 {\lk}, \frac{\epsilon \mu}{\min\{C,2\sg^k\}} \right\}$. This way, we guarantee an upper bound of our estimate to the noise gradient at each iteration.
It is worth mentioning that this addition of the capping limit $C$ is for the theoretical treatment only and it is dropped in all of the conducted experiments. Now, we follow this informal description by rigorous formulations and we keep the proofs to the appendix.

\begin{lemma}\label{le:step_sizes_positive}
Learning rates generated by Algorithm \ref{alg:sgd_dyn_general} are bounded by positive constants $\gamma_{\max}= \frac 1 2\max_{\tau \in [n]} \left\{ \frac 1 {\mathcal L(\tau)} \right\}$ and $ \gamma_{\min}= \frac 1 2\min \left\{ \min_{\tau \in [n]} \left\{\frac 1 {\mathcal L(\tau)} \right\}, \frac{\epsilon \mu} C \right\} $.
\end{lemma}
We use Lemma \ref{le:step_sizes_positive} to bound $\gamma^k$ in Inequality \eqref{eq:one_step} to obtain the following convergence theorem. 
\begin{theorem}\label{th:convergence_our}
Assume $f:\R^d \to \R$ is $\mu$--strongly convex, $\cL$--smooth and Assumption~\ref{as:gradient_noise} hold. 
Then the iterates of Algorithm \ref{alg:sgd_dyn_general} satisfy:
\begin{equation} \label{eq:x_distance_bound}
\E{\norm{x^k - x^*}^2} \leq \left(1-\gamma_{\min}\mu\right)^k\norm{x^0-x^*}^2 + R,
\end{equation} 
\end{theorem}
where $R = \frac{2\gamma_{\max}^2\sg^*}{\gamma_{\min}\mu}$. Note that the results above recover the earlier results of \citep{SGD_general_analysis} under fixed learning rate $\gamma^k = \gamma$. 
Theorem \ref{le:step_sizes_positive} guarantees the convergence of the proposed algorithm. Although there is no significant theoretical improvement here compared to previous \sgd{} results in the fixedminibatch and learning rate regimes, we measure the improvement to be significant in practice. 

\subsection{Convergence of $\tk$ to $\tau^*$}. The motivation behind the proposed algorithm is to learn the optimal minibatch size in an online fashion so that we get to $\epsilon$--neighborhood of the optimal model with the minimum number of epochs. For simplicity, let's assume that $\sigma^*=0$. As $x^k \rightarrow x^*$, then $h_i^k=\nabla f_i(x^k) \rightarrow \nabla f_i(x^*) = h_i^*$, and thus  $\tk \rightarrow \tau^*$. 
In Theorem \ref{th:convergence_our}, we showed the convergence of $x^k$ to a neighborhood around $x^*$. Hence the theory predicts that our estimate of the optimal minibatch size $\tk$ will converge to a neighborhood of the optimal minibatch size $\tau^*$ which was also confirmed by the experimental results in Figure \ref{fig:behaviour_of_tau}.

\section{Experiments} \label{sec:experiments}

\begin{figure*}[t]
\begin{center}
\centerline{
\includegraphics[width=0.5\textwidth]{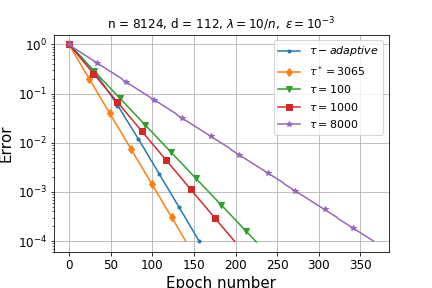}
\includegraphics[width=0.5\textwidth]{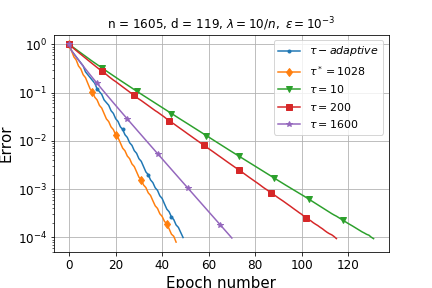}
}
\centerline{
\includegraphics[width=0.5\textwidth]{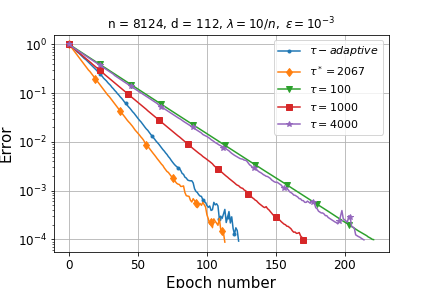}
\includegraphics[width=0.5\textwidth]{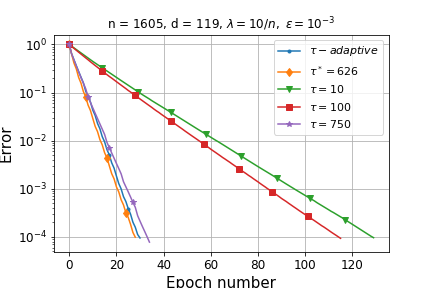}
}
\centerline{
\includegraphics[width=0.5\textwidth]{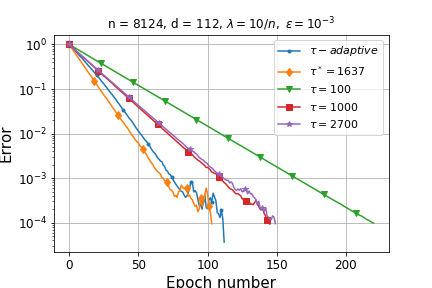}
\includegraphics[width=0.5\textwidth]{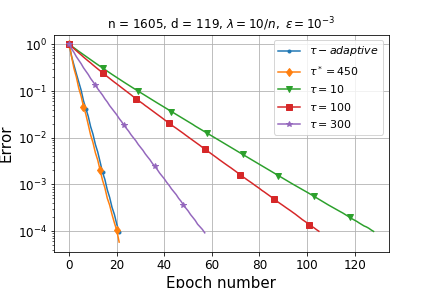}
}
\centerline{
\includegraphics[width=0.5\textwidth]{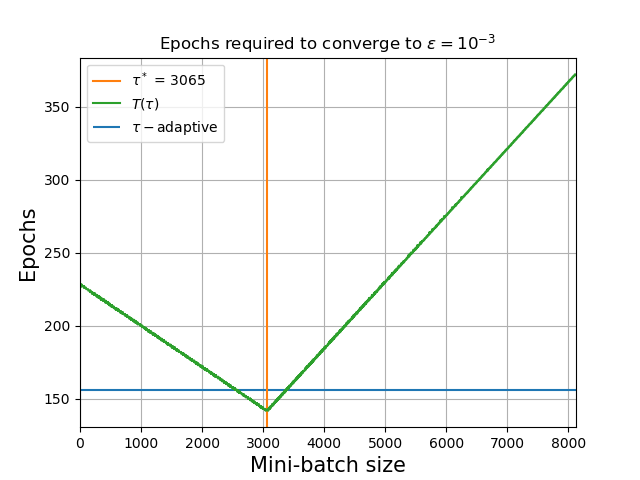}
\includegraphics[width=0.5\textwidth]{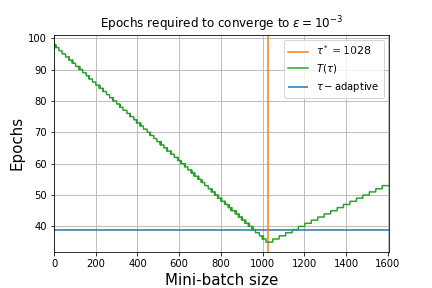}
}
\caption[Convergence of \sgd{} with adaptive mini-batch size, $\tau$--partition nice sampling on logistic regression]{\textbf{Convergence of logistic regression} using $\tau$--partition nice sampling on \mushroom{} dataset (first column) and $\tau$--partition independent sampling on \aoa{} dataset (second column). In first three rows, training set is distributed among $1$, $2$ and $3$ partitions, respectively. Figures in the fourth row show how many epochs does the \sgd{} need to converge for allminibatch sizes (single partition dataset).
}
\label{fig:logistic_nice}
\end{center}
\end{figure*}

In this section, we compare our algorithm to fixedminibatch size \sgd{} in terms of the number of epochs needed to reach a pre-specified neighborhood $\epsilon/10$. In the following results, we capture the convergence rate by recording the relative error $(\nicefrac{\norm{x^k - x^*}^2}{\norm{x^0 - x^*}^2})$ where $x^0$ is drawn from a standard normal distribution $\mathcal{N}(0, \mathbf{I})$. We also report the number of training examples $n$, the dimension of the machine learning model $d$, regularization factor $\lambda$, and the target neighborhood $\epsilon$ above each figure.
We consider the problem of regularized ridge regression where each $f_i$ is strongly convex and L-smooth, and $x^*$ can be known a-priori and the classification problem through regularized logistic regression where the optimal model is attained by running gradient descent for a long time. Both problem can be mathematically formulated in the form $\min_{x \in \mathbb R^d} f(x)$, where
\begin{align}
f_{\text{ridge}}(x) &= \frac{1}{2n}\sum \limits_{i=1}^n \norm{a_i^\top x - b_i}^2 + \frac{\lambda}{2}\norm{x}^2,\\
f_{\text{logistic}}(x) &= \frac{1}{2n}\sum \limits_{i=1}^n\log\left(1+\exp \left(b_i a_i^\top x\right)\right) + \frac{\lambda}{2}\norm{x}^2,
\end{align}
where $(a_i, b_i) \sim \mathcal{D}$ are pairs of data examples from the training set, and $\lambda > 0$ is the regularization coefficient.

For each of the considered problems, we performed experiments on at least two real datasets from \libsvm{} \citep{libsvm}. We tested our algorithm on ridge regression on \bodyfat{} and \mg{} datasets in Figure \ref{fig:ridge_nice}. Results on logistic regression on \mushroom{} and \aoa{} datasets are in Figure \ref{fig:logistic_nice}. \\

For each of these datasets, we considered $\tau$--partition independent and $\tau$--partition nice sampling with distributing the training set into one, two, and three partitions.  
Moreover, we take the previous experiments one step further by running a comparison of various fixedminibatch size \sgd{}, as well as our adaptive method with a single partition (last column of Figures \ref{fig:ridge_nice} and \ref{fig:logistic_nice}).
We plot the total iteration complexity for eachminibatch size, and highlight optimal minibatch size obtained from our theoretical analysis, and how many epochs our adaptive algorithm needs to converge. This plot can be viewed as a summary of grid-search for optimal minibatch size (throughout all possible fixedminibatch sizes). \\

Despite the fact that the optimal minibatch size is nontrivial and varies significantly with the model, dataset, sampling strategy, and number of partitions, our algorithm demonstrated consistent performance overall. In some cases, it was even able to cut down the number of epochs needed to reach the desired error to a factor of six. We refer the interested reader to the appendix for more experiments, on other datasets, different sampling strategies, and a different number of partitions.\\

The produced figures of our grid-search perfectly capture the tightness of our theoretical analysis. In particular, the total iteration complexity decreases linearly up to a neighborhood of $\tau^*$ and then increases linearly, as discussed in section \ref{sec:analysis}. In addition, Theorem \ref{le:tau_star} always captures the empirical minimum of $T(\tau)$ up to a negligible error. 
Moreover, these figures show how close $T_\text{adaptive}$ is to the total iteration complexity using optimal minibatch size $T(\tau^*)$.
In terms of practical performance, our algorithm can be better than running \sgd{} with $90\%$ of all possible fixedminibatch sizes. This demonstrates that running the proposed algorithm is equivalent to ``guessing'' the optimal minibatch size a-priori and hitting the interval of size $10\%$ around the optimal minibatch size.\\

Finally, in terms of running time, the proposed algorithm is slightly more expensive than running vanilla \sgd{}, since we estimate few parameters at each iteration. However, all of these estimates can be done through direct substitution to their formulas, which can be done cheaply. For example, on the \bodyfat{} experiment, the proposed algorithm requires $0.2322$ ms per epoch, while running \sgd{} with the optimal minibatch size requires $0.2298$ ms.

\section{Conclusion}
In this paper, We proposed a robust algorithm, which adaptively selectsminibatch sizes throughout training, that consistently performs almost as if we used \sgd{} with optimal minibatch size and theoretically grounded its convergence. Besides, we introduced generalized sampling strategies that are suitable for distributed optimization. We derived formulas for computing optimal minibatch size, and empirically showed its precision.
The performance of our algorithm does not depend on the complexity of these formulas even when their computation required passing through the whole dataset. we employ an effective technique to efficiently utilize memory usage and avoid seeing non-sampled data-points.
The proposed algorithm can be viewed as an enhancement on \sgd{} that is orthogonal to many other approaches (e.g., variance reduction). Thus, it might be possible to combine them in future work.

\chapter{A damped Newton method achieves global $\okd$ and local quadratic convergence rate} \label{sec:aicn}
\thispagestyle{empty}

\section{Introduction} \label{sec:intro}
Second-order optimization methods are the backbone of much of industrial and scientific computing. With origins that can be tracked back several centuries to the pioneering works of  Newton~\citep{Newton}, Raphson~\citep{Raphson} and Simpson~\citep{Simpson}, they were extensively studied, generalized, modified, and improved in the last century \citep{kantorovich1948functional, more1978levenberg, Griewank-cubic-1981}. For a review of the historical development of the classical Newton-Raphson method, we refer the reader to the work of \citet{ypma1995historical}. The number of extensions and applications of second-order optimization methods is enormous; for example, the survey of \citet{conn2000trust} on trust-region and quasi-Newton methods  cited over a thousand papers. 

\subsection{Second-order methods and modern machine learning}		
Despite the rich history of the field, research on second-order methods has been flourishing up to this day. Some of the most recent development in the area was motivated by the needs of modern machine learning. 
Data-oriented machine learning depends on large datasets (both in number of features and number of datapoints), which are often stored in distributed/decentalized fashion. Consequently, there is a need for scalable algorithms.

To tackle large number of features, \citet{SDNA, RSN, RBCN} and \citet{hanzely2020stochastic} proposed variants of Newton method operating in random low-dimensional subspaces. On the other hand, \citet{NewtonSketch,Roosta-MAPR} and \citet{SN2019} developed subsampled Newton methods for solving empirical risk minimization (ERM) problems with large training datasets. Additionally, \citet{Bordes2009a,OnlineBFGS,SBFGS,SQN2016} and \citet{RBFGS2020} proposed stochastic variants of  quasi-Newton methods.
To tackle non-centralized nature of datasets, \citet{DANE, AIDE, GIANT} and \citet{DINGO2019} considered distributed variants of Newton method, with improvements under various data/function similarity assumptions. \citet{NL2021, FedNL, BL2022, Newton-3PC} and \citet{agafonov2022flecs} developed communication-efficient distributed variants of Newton method using the idea of communication compression and error compensation, without the need for any similarity assumptions. 

We highlight two main research directions throughout of history of second-order methods: globally convergent methods under additional second-order smoothness \eqref{eq:L2-smooth} \citep{nesterov2006cubic} and local methods for self-concordant problems \eqref{eq:self-concordance} \citep{nesterov1989self}. 
	Former approach lead to various improvements such as acceleration \citet{nesterov2008accelerating, monteiro2013accelerated}, usage of inexact information \citet{ghadimi2017second, agafonov2020inexact}, generalization to tensor methods and their acceleration \citet{nesterov2021implementable, gasnikov2019near, kovalev2022first}, superfast second-order methods under higher smoothness \citet{nesterov2021superfast, nesterov2021inexact, kamzolov2020near}.
	Latter approach was a breakthrough in 1990s, it lead to interior-point methods. Summary of the results can be found in books \citet{nesterov1994interior}, \citet{nesterov2018lectures}. This direction is still popular up to this day \citet{dvurechensky2018global, hildebrand2020optimal, doikov2022affine, nesterov2022set}.

    As easy-to-scale alternative to second-order methods, first-order algorithms attracted a lot of attention. Many of their aspects have been explored, including strong results in variance reduction (\citet{roux2012stochastic},\citet{gower2020variance},
	\citet{johnson2013accelerating},\citet{nguyen2021inexact},\citet{nguyen2017sarah}), preconditioning \citep{jahani2022doubly}, acceleration (\citet{nesterov2013introductory},
	\citet{aspremont2021acceleration}) distributed/federated computation (\citet{konecny2016federated,chen2022distributed,berahas2016multi,takavc2015distributed,richtarik2016distributed}, \citet{kairouz2021advances}), and  decentralized computation \citep{koloskova2020unified,sadiev2022decentralized,borodich2021decentralized}.
 However, the convergence of first-order methods always depends on the conditioning of the underlying problem. Improving conditioning is fundamentally impossible without using higher-order information.
	Removing this conditioning dependence is possible by incorporating information about the Hessian. This results in second-order methods. Their most compelling advantage is that they can converge extremely quickly, usually in just a few iterations.
	
	\subsection{Newton method: benefits and limitations}
	One of the most famous algorithms in optimization, Newton method, takes iterates of form 
	\begin{equation} \label{eq:newton_step}
	x^{k+1} = x^k - \left[\nabla^2 f(x^k) \right]^{-1} \nabla f(x^k).
	\end{equation}
	Its iterates satisfy the recursion $\norm{\g(x^{k+1})} \leq c \norms
	{\g(x^k)}$ (for a constant $c>0$), which means that \newton{} converges locally quadratically. 
	However, convergence of \newton{} is limited to only to the neighborhood of the solution.  
	It is well-known that when initialized far from optimum, \newton{} can diverge, both in theory and practice (\citet{jarre2016simple}, \citet{mascarenhas2007divergence}). We can explain intuition why this happens.
	Update rule of \newton{} \eqref{eq:newton_step} was chosen to minimize right hand side of Taylor approximation
	\begin{equation} \label{eq:newton_approximation}
	\gboxeq{ f(y) \approx T_{f}(y;x) \eqdef f(x) + \la\nabla f(x),y-x\ra + \frac{1}{2} \la \nabla^2 f(x)(y-x),y-x\ra.}
	\end{equation}
	The main problem is that Taylor approximation is not an upper bound, and therefore, global convergence of \newton{} is not guaranteed.
	
	\subsection{Towards a fast globally convergent \newton{}}
	Even though second-order algorithms with superlinear local convergence rates are very common, global convergence guarantees of any form are surprisingly rare. Many papers proposed globalization strategies, essentially all of them require some combination of the following: line-search, trust regions, damping/truncation, regularization. 
	Some popular globalization strategies show non-increase of functional value during the training. However, this turned out to be insufficient for convergence to the optimum.  \citet{jarre2016simple}, \citet{mascarenhas2007divergence} designed simple functions (strictly convex with compact level sets) so that \newton{} with Armijo stepsizes does not converge to the optimum. 
	To this day, virtually all known global convergence guarantees are for regularized Newton methods (\rnewton{}), which can be written as
	\begin{equation} \label{eq:LM_update}
	x^{k+1} = x^k - \alpha_k \left( \h(x^k) + \lambda^k \mI \right)^{-1} \g(x^k),
	\end{equation}
	where $\lambda^k \geq 0$.
	Parameter $\lambda^k$ is also known as Levenberg-Marquardt regularization \citep{more1978levenberg}, which was first introduced for a nonlinear least-squares objective. For simplicity, we disregard differences in the objectives for the literature comparison.
	Motivation behind \eqref{eq:LM_update} is to replace Taylor approximation \eqref{eq:newton_approximation} by an upper bound. 
	The first method with proven global convergence rate $\cO\left(k^{-2}\right)$ is Cubic Newton method \citep{nesterov2006cubic} for function $f:\R^d\to\R$ with Lipschitz-continuous Hessian,
	\begin{equation} \label{eq:L2-smooth}
	\gboxeq{ \norm{\nabla^2 f(x) - \nabla^2 f(y)} \leq L_2\norm{x-y}.}
	\end{equation} 
	Under this condition, one can upper bound of Taylor approximation \cref{eq:newton_approximation} as
	\begin{equation} \label{eq:cubic_newton_approximation}
	f(y) \leq T_{f}(y;x) +\frac{L_2}{6}\norm{y-x}^3.
	\end{equation}
	The next iterate of \cnewton{} is a minimizer of right hand side of \eqref{eq:cubic_newton_approximation}\footnote{\label{ft:E}Where $\mathbb{E}$ a $d$--dimensional Euclidean space, defined in \Cref{ssec:notation}.}
	\begin{equation} \label{eq:cubic_newton_step}
	x^{k+1} = \argmin_{y\in \bbE} \lb T_{f}(y;x^k) +\frac{L_2}{6}\norm{y-x^k}^3 \rb.
	\end{equation}
	For our newly-proposed algorithm \ain{} (\Cref{alg:ain}), we are using almost identical step\footref{ft:E} \footnote{\label{ft:Lf}function $f:\R^d\to\R$ is $\Lsemi$--semi-strongly self-concordant (\Cref{def:semi-self-concordance}). Instead of $\Lsemi$, we will use its upper bound $\Lalg$, $\Lalg \geq \Lsemi$.}
	\begin{equation} \label{eq:aicn_step}
	\gboxeq{x^{k+1} =  \argmin_{y\in \bbE} \left\lbrace T_{f}(y;x^k) +\frac{\Lsemi}{6}\norm{y-x^k}_{x^k}^3\right\rbrace.}
	\end{equation}
	The difference between the update of \cnewton{} and \ain{} is that we measure the cubic regularization term in the local Hessian norms. This minor modification turned out to be of a great significance for two reasons:
	\begin{enumerate}
		\item The model in \eqref{eq:aicn_step} is affine-invariant,
		\item Surprisingly, the next iterate of \eqref{eq:aicn_step} lies in the direction of \newton{} step and is obtainable without regularizer $\lambda^k$ (\ain{} just needs to set stepsize $\alpha_k$).
	\end{enumerate}
	We elaborate on both of these points later in the paper.

	\Cnewton{} \eqref{eq:cubic_newton_step} can be equivalently expressed in form \eqref{eq:LM_update} with $\alpha_k=1$ and $\lambda^k = L_2 \norm{x^k -x^{k+1}}$. However, since such $\lambda^k$ depends on $x^{k+1}$, resulting algorithm requires additional subroutine for solving its subproblem each iteration.
	Next work showing convergence rate of \rnewton{} from \citep{polyak2009regularized} avoided implicit steps by choosing $\lambda^k\propto \norm{\g(x^k)}$. However, this came with a trade-off in terms of the slower convergence rate $\cO\left( k^{-1/4}\right)$.
	Finally, \citet{mishchenko2021regularized} and \citet{doikov2021optimization} improved upon both of these works by using explicit regularization $\alpha^k=1$, $\lambda^k \propto \sqrt{L_2 \norm{\g(x^k)}}$, and for the cost of slower local convergence achieving global rate $\cO\left(k^{-2}\right)$.
	\section{Contributions} 
	\subsection{\ain{} as a damped Newton method}
	In this work, we investigate global convergence for most basic globalization strategy, stepsized Newton method without any regularizer ($\lambda^k=0$). Such algorithms are also referred as damped/truncated Newton method; written as 
    \begin{equation}
        x^{k+1} = x^k - \alpha_k \h(x^k) ^{-1} \g(x^k).
    \end{equation}
	\dnewton{} methods were investigated in detail as interior-point methods. \citet{nesterov2018lectures} shows quadratic local convergence for stepsizes $\alpha_1 \eqdef \frac 1 {1 + G_1}, \alpha_2 \eqdef \frac {1+ G_1} {1 + G_1 + G_1^2}$, where\footnote{function $f:\R^d\to\R$ is $\Lstandard$--self-concordant (\Cref{def:self-concordance}).}\footnote{\label{ft:G}Dual norm $\normMd{\g(x^k)} {x^k} = \la\g(x^k), \h(x^k)^{-1} \g(x^k) \ra $ is defined in \Cref{ssec:notation}.} 
 \begin{equation}
     G_1 \eqdef \Lstandard \normMd {\g(x^k)} {x^k}.
 \end{equation}
	Our algorithm \ain{} is also \dnewton{} with stepsize \gbox{$\alpha = \frac {-1 + \sqrt{1+ 2G}}{G}$,} where\footref{ft:Lf}\footref{ft:G} 
 \begin{equation}
     G \eqdef \Lsemi \normMd {\g(x^k)} {x^k}.
 \end{equation}
 Mentioned stepsizes $\alpha_1, \alpha_2, \alpha$ share two characteristics. 	
	Firstly, all of them depends on gradient computed in the dual norm and scaled by a smoothness constant ($G_1$ or $G$). Secondly, all of these stepsizes converge to $1$ from below (for $\hat \alpha \in \{\alpha_1, \alpha_2, \alpha\}$ holds $0< \hat \alpha \leq 1$ and $\lim_{x \rightarrow \opt} \hat \alpha = 1$). Our algorithm uses stepsize bigger by orders of magnitude (see \Cref{fig:stepsizes} in \Cref{sec:ap_experiments} for detailed comparison).
	The main difference between already established stepsizes $\alpha_1, \alpha_2$ and our stepsize $\alpha$ are resulting global convergence rates. While stepsize $\alpha_2$ does not lead to a global convergence rate, and $\alpha_1$ leads to rate $\cO \left( k^{-1/2} \right)$, our stepsize $\alpha$ leads to a significantly faster, $\cO\left(k^{-2}\right)$ rate.
	Our rate matches best known global rates for \rnewton{}. We manage to achieve these results by carefully choosing assumptions. While rates for $\alpha_1$ and $\alpha_2$ follows from standard self-concordance, 
	 our assumptions are a consequence of a slightly stronger version of self-concordance. We will discuss this difference in detail later.

	We summarize important properties of \rnewton{} methods with fast global convergence guarantees and \dnewton{} methods in \Cref{tab:global_newton_like}.
	
	\begin{table*}[!t]
		\centering
		\setlength\tabcolsep{0.7pt} 
		\begin{threeparttable}[b]
			{\scriptsize
				\renewcommand\arraystretch{2.2}
				\caption[A summary of regularized Newton methods with global convergence guarantees]{A summary of \rnewton{} methods with global convergence guarantees. We consider algorithms with updates of form $x^{k+1} = x^k - \alpha_k \left( \h(x^k) + \lambda^k \mI \right)^{-1} \g(x^k)$. For simplicity of comparison, we disregard differences in objectives and assumptions. We assume $L_2$--smoothness of Hessian, $\Lsemi$--semi-strong self-concordance, convexity (\Cref{def:semi-self-concordance}), $\mu$--strong convexity locally and bounded level sets. For regularization parameter holds $\lambda^k\geq0$ and stepsize satisfy $0<\alpha_k \leq 1$. We highlight the best know rates in blue.}
				\label{tab:global_newton_like}
				\centering 
				\begin{tabular}{ccccccccc}\toprule[.1em]
					\bf Algorithm &\bf  \makecell{Regularizer \\$\lambda^k \propto$} & \bf \makecell{Stepsize\\ $\alpha_k=$} & \bf  \makecell{Affine\tnote{\color{red}(1)} \\ invariant?\\ (alg., as., rate)} &\bf  \makecell{Avoids \\ line\\ search?} &\bf  \makecell{Global \\ convergence \\ rate} & \bf \makecell{Local \tnote{\color{red}(2)} \\conv.\\exponent} & \bf Reference \\
					\midrule
					\makecell{\newton{}} & $0$ & $1$ & (\cmark, \xmark, \xmark) & \cmark & \xmark & {\color{blue} $ 2$} &  \citenum{kantorovich1948functional} \\
					\makecell{\newton{}} & $0$ & $1$ & (\cmark,\cmark,\cmark) & \cmark & \xmark & {\color{blue} $ 2$} &  \citenum{nesterov1994interior} \\
					\makecell{\Dnewton{} B} &$0$ & $\frac1 {1+G_1}$\tnote{\color{red}(4)} & (\cmark,\cmark,\cmark) & \cmark & $\cO\left( k^{-\frac 12} \right)$ & {\color{blue} $ 2$} &  \citenum{nesterov2018lectures}\\ 
					\makecell{\Dnewton{} C} &$0$ & $\frac{1 + G_1} {1+G_1 + G_1^2}$\tnote{\color{red}(4)} & (\cmark,\cmark,\cmark) & \cmark & \xmark & {\color{blue} $ 2$} &  \citenum{nesterov2018lectures}\\
					\makecell{\Cnewton{}} &$L_2\norm{\xdiff} $ & $1$ & (\xmark, \xmark, \xmark) & \xmark & $\cO\left( {\color{blue}{ k^{-2}} } \right)$ & {\color{blue} $ 2$} &  \makecell{\citenum{nesterov2006cubic}, \citenum{griewank1981modification}, \citenum{doikov2021local}} \\
					\makecell{ Loc.~Reg.~Newton} &$\norm{\g(x^k)} $ & $1$ & (\xmark, \xmark, \xmark) & \cmark & \xmark & {\color{blue} $ 2$} &  \makecell{\citenum{polyak2009regularized}} \\
					\makecell{ Glob.~Reg.~Newton} &$\norm{\g(x^k)} $ & $\frac{\mu +\norm{\g(x^k)}}{L_1}$\tnote{\color{red}(3)} & (\xmark, \xmark, \xmark) & \xmark & $\cO\left( k^{-\frac 14} \right)$ & {\color{blue} $ 2$} &  \makecell{\citenum{polyak2009regularized}} \\
					\makecell{ Glob.~Reg.~Newton} &$\sqrt{L_2\norm{\g(x^k)}}$ & $1$ & (\xmark, \xmark, \xmark) & \cmark & $\cO\left( {\color{blue}{ k^{-2}} } \right)$ & $ \frac32$ &  \makecell{\citenum{mishchenko2021regularized}, \citenum{doikov2021optimization}} \\
					\midrule
					\makecell{ \textbf{\ain{}}\\ (\Cref{alg:ain})} & $0$ & $\frac{-1 + \sqrt{1+2 G} }{G}$ \tnote{\color{red}(4)} & (\cmark, \cmark, \cmark) & \cmark & $\cO\left( {\color{blue} {k^{-2}} } \right)$ & {\color{blue} $ 2$} &   \textbf{This work} \\
					\bottomrule[.1em]
				\end{tabular}
			}
			\begin{tablenotes}
				{\scriptsize
					\item [\color{red}(1)] In triplets, we report whether algorithm, used assumptions, convergence rate are affine-invariant, respectively.
					\item [\color{red}(2)] For a Lyapunov function $\Phi^k$ and a constant $c$, we report exponent $\beta$ of $\Phi(x^{k+1}) \leq c \Phi (x^k)^\beta$.
					\item [\color{red}(3)] $f$ has $L_1$--Lipschitz continuous gradient.
					\item [\color{red}(4)] For simplicity, we denote $G_1 \eqdef \Lstandard \normMd {\g(x^k)} {x^k}$ and $G \eqdef \Lsemi \normMd {\g(x^k)} {x^k}$ (for $\Lalg \leftarrow \Lsemi$).
				}
			\end{tablenotes}
		\end{threeparttable}
	\end{table*}
	
\subsection{Summary of contributions}
	To summarize novelty in our work, we present a novel algorithm \ain{}. Our algorithm can be interpreted in two viewpoints \textbf{a)} as a regularized Newton method (version of \cnewton{}), \textbf{b)} as a \dnewton{}. \ain{} enjoys the best properties of these two worlds: 
\begin{itemize}[leftmargin=*]
    \item \textbf{Fast global convergence}: \ain{} converges globally with rate $\cO\left( k^{-2} \right)$ (Theorems~\ref{thm:convergence}, \ref{th:global}), 
    which matches the state-of-the-art global rate for all regularized Newton methods. Furthermore, it is the first such rate for \dnewton{}.
	
    \item \textbf{Fast local convergence:} In addition to the fast global rate, \ain{} decreases gradient norms locally in quadratic rate (\Cref{th:local}). This result matches the best-known rates for both regularized Newton algorithms and damped Newton algorithms.
	
    \item \textbf{Simplicity:} Previous works on Newton regularizations can be viewed as a popular global-convergence fix for the Newton method. We propose an even simpler fix in the form of a stepsize schedule (\Cref{sec:alg}).
    
    \item \textbf{Implementability:} 
    Step of \ain{} depends on a smoothness constant $\Lsemi$ (\Cref{def:semi-self-concordance}). Given this constant, next iterate of \ain{} can be computed directly. 
    
    This is improvement over \cnewton{} \citep{nesterov2006cubic}, which for a given constant  $L_2$ needs to run \textbf{line-search} subroutine each iteration to solve its subproblem
    
    \item \textbf{Improvement:} Avoiding latter subroutine yields theoretical improvements.
    If we compute matrix inverses naively, iteration cost of \ain{} is $\cO(d^3)$ (where $d$ is a dimension of the problem), which is improvement over 
    $\cO(d^3 \log \varepsilon^{-1})$ iteration cost of \cnewton{} \citep{nesterov2006cubic}.

    \item \textbf{Practical performance:} We show that in practice, \ain{} outperforms all algorithms sharing the same convergence guarantees: \cnewton{} \citep{nesterov2006cubic}, and globally \rnewton{} \citep{mishchenko2021regularized} and \citet{doikov2021optimization}, and fixed stepsize \dnewton{} (Sec.~\ref{sec:experiments_aicn}).
	
    \item \textbf{Geometric properties:} We analyze \ain{} under more geometrically natural assumptions. Instead of smoothness, we use a version of self-concordance (\Cref{sec:affine_invariance}), which is invariant to affine transformations and hence also to a choice of a basis. 
    \ain{} preserves affine invariance obtained from assumptions throughout the convergence.
    In contrast, \cnewton{} uses basis-dependent $l_2$ norm and hence depends on a choice of a basis. This represents an extra layer of complexity.

	\item \textbf{Alternative analysis:} We also provide alternative analysis under weaker assumptions
(\Cref{sec:ap_global_nsc}).
\end{itemize} 

The	rest of the paper is structured as follows. In \Cref{ssec:notation} we introduce our notation. In \Cref{sec:alg}, we discuss algorithm \ain{}, affine-invariant properties and self-concordance. In Sections~\ref{sec:global} and \ref{sec:local} we show global  and local convergence guarantees, respectively.
	In \Cref{sec:experiments_aicn} we present an empirical comparison of \ain{} with other algorithms sharing fast global convergence.
	
	\subsection{Minimization problem \& notation} \label{ssec:notation} 
		In the paper, we consider a $d$--dimensional Euclidean space $\mathbb{E}$. Its dual space, $\mathbb{E}^\ast$, is composed of all linear functionals on $\mathbb{E}$. For a functional $g \in \mathbb{E}^\ast$, we denote by $\la g,x\ra$ its value at $x\in \mathbb{E}$.
	We consider the following convex optimization problem:
	\begin{equation}
	\min\limits_{x\in \mathbb{E}}  f(x),
	\label{eq_pr}
	\end{equation} 
	where $f(x) \in C^2$ is a convex function with continuous first and second derivatives and positive definite Hessian. We assume that the problem has a unique minimizer $\opt\in \argmin\limits_{x\in \mathbb{E}}  f(x)$. 
	Note, that 
	$\nabla f(x) \in \bbE^{\ast}$ and $ \nabla^2 f(x) h \in \bbE^{\ast}.$

 We will work in more general Euclidean spaces $\bbE$ and $\bbE^{\ast}$. Denote $x,h \in \bbE, g \in \bbE^{\ast}$. For a self-adjoint positive-definite operator $\mathbf H : \bbE \rightarrow \bbE^\ast$, we can endow these spaces with conjugate Euclidean norms:
	\begin{equation}
	    \norm{x}_ {\mathbf H} \eqdef  \la \mathbf Hx,x\ra^{1/2}, \, x \in \bbE, \qquad \norm{g}_{\mathbf H}^{\ast}\eqdef\la g,\mathbf H^{-1}g\ra^{1/2},  \, g \in \bbE^{\ast}.
	\end{equation}
	For identity $\mathbf H= \mI$, we get classical Eucledian norm $\norm{x}_\mI = \la x,x\ra^{1/2}$. For local Hessian norm $\mathbf H = \nabla^2 f(x)$, we use shortened notation
	\begin{equation}
	\gboxeq{\norm{h}_x \eqdef \la \nabla^2 f(x) h,h\ra^{1/2}, \, h \in \bbE,} \qquad \gboxeq{\norm{g}_{x}^{\ast} \eqdef \la g,\nabla^2 f(x)^{-1}g\ra^{1/2},  \, g \in \bbE^{\ast}.}
	\end{equation}
	Operator norm is defined by 
	\begin{equation}
	\label{eq:matrix_operator_norm}
	    \gboxeq{\normM{\mathbf H} {op} \eqdef \sup_{v\in \mathbb E} \frac{\normMd {\mathbf H v} x}{\normM v x}},
	\end{equation} 
	for $\mathbf H :\bbE \rightarrow \bbE^\ast$ and a fixed $x \in \bbE$.
	If we consider a specific case $\bbE \leftarrow \R^d$, then $\mathbf H$ is a symmetric positive definite matrix.

	\section{Affine-invariant cubic Newton} \label{sec:alg}
    Finally, we are ready to present algorithm \ain{}. It is \dnewton{} with updates
	\begin{equation} \label{eq:update}
	x^{k+1} = x^k - \alpha_k \h(x^k)^{-1} \g(x^k),
	\end{equation}
	with stepsize 
    \begin{equation}
       \alpha_k \eqdef \frac {-1+\sqrt{1+2 \Lalg \normMd {\g(x^k)} {x^k}}}{\Lalg \normMd {\g(x^k)} {x^k}},
    \end{equation}
    as summarized in \Cref{alg:ain}. 
	Stepsize satisfy $\alpha_k \leq 1$ (from AG inequality, \eqref{eq:AG}). Also $\lim_{x^k \rightarrow \opt} \alpha_k = 1$, hence \eqref{eq:update} converges to \newton{}.
	Next, we are going to discuss geometric properties of our algorithm.
	\begin{algorithm} 
		\caption{\ain{}: Affine-Invariant Cubic Newton} \label{alg:ain}
		\begin{algorithmic}[1]
			\State \textbf{Requires:} Initial point $x^0 \in \bbE$, constant $\Lalg$ s.t. $\Lalg \geq \Lsemi>0$
			\For {$k=0,1,2\dots$}
			\State $\alpha_k = \frac{-1 +\sqrt{1+2 \Lalg \norm{\g(x^k)}_{x^k}^* }}{\Lalg \norm{\g(x^k)}_{x^k}^* }$
			\State $x^{k+1} = x^k - \alpha_k \left[\h(x^k)\right]^{-1} \g(x^k)$ \Comment {Note that  $x^{k+1} \stackrel{\eqref{eq:step}}= S_{f,\Lalg}(x^k)$.}
			\EndFor
		\end{algorithmic}
	\end{algorithm}
	\subsection{Geometric properties: affine invariance} \label{sec:affine_invariance}
	One of the main geometric properties of the \newton{} method is its \textit{affine invariance}, i.e., invariance to affine transformations of variables.
	Let $\mathbf A: \bbE \rightarrow \bbE^\ast$  be a non-degenerate linear transformation enabling change of variables $x \to y$ via $y=\mathbf A^{-1} x$. Instead of minimizing $f(x)$ we minimize $\phi(y) \eqdef f(\mathbf Ay)=f(x)$. 
    \paragraph{Significance of local norms:}
    The local Hessian norm $\norm{h}_{\nabla f(x)}$ is affine-invariant. Indeed, if $h=\mathbf A z$ and $x=\mathbf A y$, then
    \begin{equation}
        \norm{z}_{\nabla^2 \phi(y)}^2=\la\nabla^2 \phi(y)z,z \ra = \la \mathbf A^\top \nabla^2 f(\mathbf A y) \mathbf Az,z\ra = \la \nabla^2 f(x)h,h\ra=\norm{h}_{\nabla^2 f(x)}^2,  
    \end{equation}
    where $\la a, b \ra \eqdef a^\top b$ denotes the inner product. On the other hand, the standard Euclidean norm $\norm{h}_\mI$ is not affine-invariant since
    \begin{equation}
        \norm{z}_\mI^2=\la z,z \ra = \la \mathbf A^{-1}h, \mathbf A^{-1}h\ra = \norm{\mathbf A^{-1} h}^2_\mI \not = \normsM h {\mI}.  
    \end{equation}
    With respect to geometry, the local Hessian norm is more natural. From affine invariance follows that for this norm, the level sets $\lb y\in \R^d\,|\,\norm{ y-x }_x^2 \leq c\rb$ are balls centered around $x$ (all directions have the same scaling). In comparison, the scaling of the standard Euclidean norm is dependent on the eigenvalues of the Hessian. In terms of convergence, one direction in $l_2$ can significantly dominate others and slow down an algorithm.
	\paragraph{Significance for algorithms:}
    Algorithms that are not affine-invariant can suffer from a bad choice of the coordinate system. This is the case for \cnewton{}, as its model \eqref{eq:intro_cubic_newton_step} is bound to basis-dependent $l_2$ norm. The same is true for any other method regularized with an induced norm $\norm{h}_\mI$.
    On the other hand, (damped) \newton{} have affine-invariant models, and hence as algorithms independent of the chosen coordinate system. We prove this claim in following lemma (note: $\alpha_k=1$ and $\alpha_k$ from \eqref{eq:update} are affine-invariant).
	\begin{lemma}[Lemma 5.1.1 \citet{nesterov2018lectures}] \label{le:AI_newton_nesterov}
		Let the sequence $\lb x^k \rb$ be generated by a \dnewton{} with affine-invariant stepsize $\alpha_k$, applied to the function $f:\R^d\to\R$:
		$x^{k+1}=x^k - \alpha_k\left[\nabla^2 f(x^k) \right]^{-1}\nabla f(x^k).$
		For function $\phi(y)$, \dnewton{} generates  $\lb y^k \rb$:
		$y^{k+1}=y^k - \alpha_k\left[\nabla^2 \phi(y^k) \right]^{-1}\nabla \phi(y^k),$
		with $y_0 = \mathbf A^{-1} x^0$. Then $y^k= \mathbf A^{-1} x^k$.
	\end{lemma}
 
    \subsection{Significance in assumptions: self-concordance}

	We showed that \dnewton{} preserve affine invariance through iterations. Hence it is more fitting to analyze them under affine-invariant assumptions. Affine-invariant version of smoothness, \emph{self-concordance}, was introduced in \citet{nesterov1994interior}. 
    For more connection to the literature, see \Cref{sec:intro_def}.
	
    \begin{definition} \label{def:self-concordance}
        Convex function $f \in C^3$ is called \textit{self-concordant}  if
        \begin{equation}
        \label{eq:self-concordance}
        |D^3 f(x)[h]^3| \leq \Lstandard\norm{h}_{x}^3, \quad \forall x,h\in \bbE,
        \end{equation}
        where for any integer $p\geq 1$, by $D^p f(x)[h]^p \eqdef D^p f(x)[h,\ldots,h]$ we denote the $p$--th order directional derivative\footnote{For example, $D^1 f(x)[h] =\langle\nabla  f(x),h\rangle$ and $D^2 f(x)[h]^2 = \la \nabla^2 f(x) h, h \ra $.} of $f$ at $x\in \bbE$ along direction $h\in \bbE$.
    \end{definition}
	Both sides of inequality are affine-invariant. This assumption corresponds to a big class of optimization methods called interior-point methods. 
 We will use semi-strong self-concordance.
	
		\begin{definition} \label{def:semi-self-concordance}
			Convex function $f \in C^2$  is called \textit{semi-strongly self-concordant} if
			\begin{equation}
			\label{eq:semi-strong-self-concordance}
			\norm{\nabla^2 f(y)-\nabla^2 f(x)}_{op} \leq \Lsemi \norm{y-x}_{x} , \quad \forall y,x \in \bbE .
			\end{equation} 
		\end{definition}

    \subsection{From assumptions to algorithm}
	From semi-strong self-concordance we can get a second-order bounds on the function and model.
	
	\begin{lemma} \label{le:model}
		If $f$ is semi-strongly self-concordant, then
		\begin{equation}
		\label{eq:lower_upper_bound}
		\left|f(y) - T_f(y;x) \right|  \leq \frac{\Lsemi}{6}\norm{y-x}_x^3, \quad \forall x,y\in \bbE.
		\end{equation}
		Consequently, we have upper bound for function value in form
		\begin{equation} \label{eq:wssc_ub}
		f(y) \leq T_{f}(y;x) +\frac{\Lsemi}{6}\norm{y-x}_{x}^3.
		\end{equation}
	\end{lemma}
	One can show that \eqref{eq:wssc_ub} is not valid for just self-concordant functions. For example, there is no such upper bound for $--\log(x)$. Hence, the semi-strongly self-concordance is significant as an assumption.
	
	We can define iterates of optimization algorithm to be minimizers of the right hand side of \eqref{eq:wssc_ub}, 
    	\begin{gather} \label{eq:step}
    	S_{f,\Lalg}(x)\eqdef x + \argmin_{h\in \bbE} \left\lbrace f(x) + \la \nabla f(x),h\ra + \frac{1}{2} \la \nabla^2 f(x)h,h\ra +\frac{\Lalg}{6}\norm{h}_{x}^3\right\rbrace,\\
    	\label{eq:S_step}
    	x^{k+1} = S_{f,\Lalg}(x^{k}),
    	\end{gather} 
	for an estimate constant $\Lalg \geq \Lsemi$. 
	It turns out that subproblem \eqref{eq:step} is easy to solve. To get an explicit solution, we compute its gradient w.r.t. $h$. For solution $h^{\ast}$, it should be equal to zero, 
	\begin{gather}
	 \nabla f(x) + \nabla^2 f(x)h^{\ast} +\frac{\Lalg}{2}\norm{h^{\ast}}_{x}\nabla^2 f(x) h^{\ast} = 0
	 \label{eq:h_star_main},\\
	  h^{\ast} = -\left[\nabla^2 f(x)\right]^{-1}\nabla f(x)\cdot \left(\frac{\Lalg}{2}\norm{h^{\ast}}_{x}+ 1\right)^{-1}. \label{eq:h_star2}
	\end{gather}
	We get that step \eqref{eq:S_step} has the same direction as a \newton{} and is scaled by  $\alpha_k =\left(\frac{\Lalg}{2}\norm{h^{\ast}}_{x}+ 1\right)^{-1}. $ 
	Now, we substitute $h^{\ast}$ from \eqref{eq:h_star2} to \eqref{eq:h_star_main}
	\begin{gather}
	 \nabla f(x) - \nabla f(x) \alpha_k  -\frac{\Lalg}{2}\la \nabla f(x),\left[\nabla^2 f(x)\right]^{-1}\nabla f(x) \ra^{1/2} \nabla f(x) \alpha_k^2  = 0, \\
	 \nabla f(x) \left(1- \alpha_k  -\frac{\Lalg}{2}\norm{\nabla f(x)}_x^{\ast} \alpha_k^2\right)  = 0. \label{eq:h_quad_eq}
	\end{gather}
	We solve the quadratic equation \eqref{eq:h_quad_eq} for $\alpha_k$, and obtain explicit formula for stepsizes of \ain{}, as \eqref{eq:update}.
	We formalize this connection in theorem, for further explanation see proof in \Cref{sec:ap_theory}.
		\begin{theorem} \label{th:stepsize}
			For $\Lsemi$-semi-strong self-concordant function $f:\R^d \to \R$ and $\Lalg \geq \Lsemi$, update of \ain{} \eqref{eq:update},
			\begin{equation}
			    x^{k+1} = x^k - \alpha_k \h(x^k)^{-1} \g(x^k), \quad \text{where} \quad \alpha_k = \frac{-1+\sqrt{1+2 \Lalg \normMd {\g(x^k)} {x^k}}}{\Lalg \normMd {\g(x^k)} {x^k}},
			\end{equation}
			is a minimizer of upper bound \eqref{eq:step}, $x^{k+1} = S_{f,\Lalg}(x^k)$.
		\end{theorem}
\section{Convergence results}	
	\subsection{Global convergence} \label{sec:global}
	Next, we focus on global convergence guarantees.
	We use following assumption:
		\begin{assumption}[Bounded level sets] \label{as:level_sets}
			The objective function $f:\R^d\to\R$ has a unique minimizer $x^{*}$.
			Also, the diameter of the level set $\level\eqdef\lb x \in \bbE : f(x)\leq f(x^0) \rb$ is bounded by a constant $D_2$ as\footnote{We state it in $l_2$ norm for easier verification. In proofs, we use its variant $D$ in Hessian norms, \eqref{D_lebeg}.}, $\max\limits_{x\in\level} \normM {x-x^{*}} {} \leq D_2<+\infty$.
		\end{assumption}
 
	Our analysis proceeds as follows. Firstly, we show that one step of the algorithm decreases function value, and secondly, we use the technique from  \citep{ghadimi2017second} to show that multiple steps lead to $\cO\left ({k^{-2}} \right )$ global convergence.
	\begin{lemma}[One step globally] \label{lm:main_lemma}
		Let function $f:\R^d\to\R$ be $\Lsemi$--semi-strongly self-concordant, convex with positive-definite Hessian and $\Lalg \geq \Lsemi$. Then for any $x \in  \bbE$, holds
		\begin{equation} \label{eq:min_lemma}
		S_{f,\Lalg}(x) \leq \min \limits_{y \in \bbE} \left\{ f(y) + \frac{\Lalg}{ 3}\norm{y -x}^3_x \right\}.
		\end{equation}
	\end{lemma}
	This lemma implies that step \eqref{eq:step} decreases function value (take $y \leftarrow x$). Using notation of Assumption~\ref{as:level_sets}, $x^k \in \level $ for any $k\geq 0 $. Also, setting $y \leftarrow x$ and $x \leftarrow \opt$ in \eqref{eq:semi-strong-self-concordance} yields
	$\norm{x-\opt}_x \leq \ls \norm{x-\opt}_{\opt}^2 + \Lalg \norm{x-\opt}_{\opt}^3 \rs^{\frac{1}{2}}.$
	We denote those distances $D$ and $R$,
	\begin{equation} \label{D_lebeg}
	\gboxeq{D \eqdef \max\limits_{t \in [0;k+1]} \norm{x^t - \opt}_{x^t}}  \qquad \text{and} \qquad \gboxeq{R \eqdef \max\limits_{x\in \level} \ls \norm{x-\opt}_{\opt}^2 + \Lalg \norm{x-\opt}_{\opt}^3 \rs^{\frac{1}{2}}.}
	\end{equation}
	
	They are both affine-invariant and 
	$R$ upper bounds $D$. While $R$ depends only on the level set $\level$, $D$ can be used to obtain more tight inequalities. We avoid using common distance $D_2$, as $l_2$ norm would ruin affine-invariant properties.

		\begin{theorem} \label{thm:convergence}
			Let $f:\R^d \to \R$ be a $\Lsemi$--semi-strongly self-concordant convex function with positive-definite Hessian, constant $\Lalg$ satisfy $\Lalg \geq \Lsemi$ and Assumption~\ref{as:level_sets} holds. Then, after $k+1$ iterations of Algorithm \ref{alg:ain}, we have the following convergence: 
			\begin{equation}
			f(x^{k+1}) - f(\opt)\leq O\ls\frac{\Lalg D^{3}}{k^{2}}\rs\leq O\ls\frac{\Lalg R^{3}}{k^2}\rs.  \label{eq:convergence}
			\end{equation}	
		\end{theorem}
	Consequently, \ain{} converges globally with a fast rate $\cO \left( k^{-2} \right)$. We can now present local analysis.
	\subsection{Local convergence} \label{sec:local}
	For local quadratic convergence are going to utilise following lemmas.
	\begin{lemma} 
	\label{le:lsconv}
		For convex $\Lsemi$--semi-strongly self-concordant function $f:\R^d\to\R$ and for any $0<c<1$ in the neighborhood of solution \begin{equation}
		\label{eq_first_neigborhood}
		 x^k \in \left\{ x :  \normMd {\g (x)} {x} \leq \frac{(2c+1)^2-1}{2\Lalg} \right\}\,\, \text{holds}\,\,  \h(x^{k+1})^{-1} \preceq \left( 1 - c \right)^{-2} \h(x^k)^{-1}. 
		 \end{equation}
	\end{lemma}
	\Cref{le:lsconv} formalizes that a inverse hessians of a self-concordant function around the solution is non-degenerate.
	With this result, we can show one-step gradient norm decrease.
	
	\begin{lemma}[One step decrease locally] \label{le:one_step_local}
		Let function $f:\R^d\to\R$ be $\Lsemi$--semi-strongly self-concordant and $\Lalg \geq \Lsemi$. If $x^k$ such that \eqref{eq_first_neigborhood} holds, then for next iterate $x^{k+1}$ of \ain{} holds
		\begin{equation} \label{eq:one_step_local}
		\normMd {\g(x^{k+1})} {x^k}
		\leq \Lalg \alpha_k^2 \normsMd{\g(x^k)} {x^k}
		< \Lalg \normsMd{\g(x^k)} {x^k}.
		\end{equation}
		Using \Cref{le:lsconv}, we shift the gradient bound to respective norms,
		\begin{equation} \label{eq:one_step_local_shifted}
		\normMd {\g(x^{k+1})} {x^{k+1}}
		\leq  \frac{\Lalg \alpha_k^2}{1-c} \normsMd{\g(x^k)} {x^k}
		< \frac{\Lalg \alpha_k^2}{1-c}\normsMd{\g(x^k)} {x^k}.
		\end{equation}
		Gradient norm decreases $\normMd{\g(x^{k+1})} {x^{k+1}} \leq \normMd{\g(x^k)} {x^k}$ for $\normMd{\g(x^k)} {x^k} \leq \frac{(2-c)^2-1} {2\Lalg }$.
	\end{lemma}
	As a result, neighbourhood of the local convergence is $\Big\{x: \normMd{\g(x)} {x} \leq $ $ \min \left[\frac{(2-c)^2-1} {2\Lalg };\frac{(2c+1)^2-1}{2\Lalg}\right] \Big\}.$
	Maximizing by $c$, we get $c=1/3$ and neighborhood $\left\{x: \normMd{\g(x)} {x} \leq \frac{8}{9\Lalg } \right\}.$ One step of \ain{} decreases gradient norm quadratically, multiple steps leads to following decrease.
		\begin{theorem}[Local convergence rate] \label{th:local}
			Let function $f:\R^d\to\R$ be $\Lsemi$--semi-strongly self-concordant, 
			$\Lalg \geq \Lsemi$ and starting point $x^0$ be in the neighborhood of the solution such that $\normMd{\g(x^0)} {x^0} \leq \frac{8} {9\Lalg} $.
			For $k\geq 0$, we have quadratic decrease of the gradient norms,
			\begin{align}
			\normMd{\g(x^{k})}{x^{k}}
			&\leq \left( \frac{3} {2} \Lalg \right )^k 
			\left( \normMd{\g(x^0)} {x^0} \right)^{2^k} .
			\end{align}
		\end{theorem}
	\section{Numerical experiments} \label{sec:experiments_aicn}
	In this section, we evaluate proposed \ain{} (\Cref{alg:ain}) algorithm on the logistic regression task and second-order lower bound function. 
	We compare it with \rnewton{} sharing fast global convergence guarantees: \cnewton{} \citep{nesterov2006cubic}, and \rnewton{} (\citep{mishchenko2021regularized,doikov2021optimization}) with $L_2$--constant. Because \ain{} has a form of a \dnewton{}, we also compare it with \dnewton{} with fixed (tuned) stepsize $\alpha_k=\alpha$. We report decrease in function value $f(x^k)$, function value suboptimality $f(x^k)-f(\opt)$ with respect to iteration and time. The methods are  implemented as PyTorch optimizers. The code is available at \url{https://github.com/OPTAMI}. 
\subsection{Logistic regression} 
	For first part, we solve the following empirical risk minimization problem:
	\begin{equation}
	    \min \limits_{x \in \mathbb{R}^d} \left\{ f(x) \eqdef \frac{1}{m} \sum \limits_{i=1}^m  \log\left(1-e^{-b_i a_i^\top x}\right) + \frac{\mu}{2} \norm{x}^2 \right\},
	\end{equation}
	where $\{(a_i, b_i)\}_{i=1}^m$ are given data samples described by features $a_i$ and class $b_i \in \{-1, 1\}$. 
	
	In \Cref{fig:a9a}, we consider task of classification images on dataset \ana{} \citep{chang2011libsvm}. Number of features for every data sample is $d=123$, $m=20000$. We take staring point $x^0 \eqdef 10 [1,1, \dots, 1]^\top$ and $\mu = 10^{-3}$. Our choice is differ from $x^0=0$ (equal to all zeroes) to show globalisation properties of methods ($x^0=0$ is very close to the solution, as \newton{} converges in 4 iterations). Parameters of all methods are fine-tuned, we choose parameters $\Lalg,L_2, \alpha$ (of \ain{}, \cnewton{}, \dnewton{}, resp.) to largest values having monotone decrease in reported metrics. Fine-tuned values are $\Lalg = 0.97$ $L_2= 0.000215$, $\alpha = 0.285$. 
	\Cref{fig:a9a} demonstrates that \ain{} converges slightly faster than \cnewton{} by iteration, notably faster than Globally \rnewton{} and significantly faster than \dnewton{}. \ain{} outperforms every method by time.
   \begin{figure}[t]
	\includegraphics[width=0.49\linewidth]{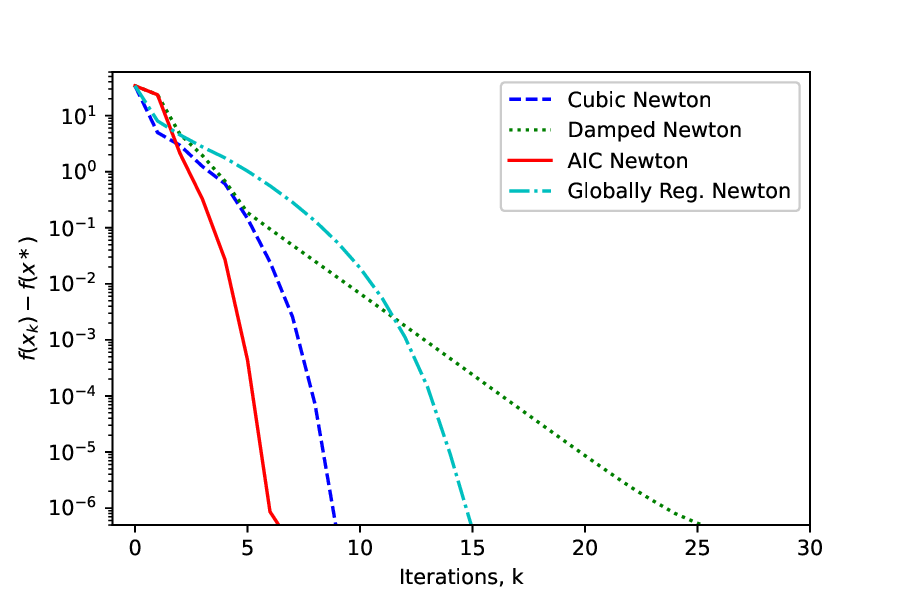}
	\includegraphics[width=0.49\linewidth]{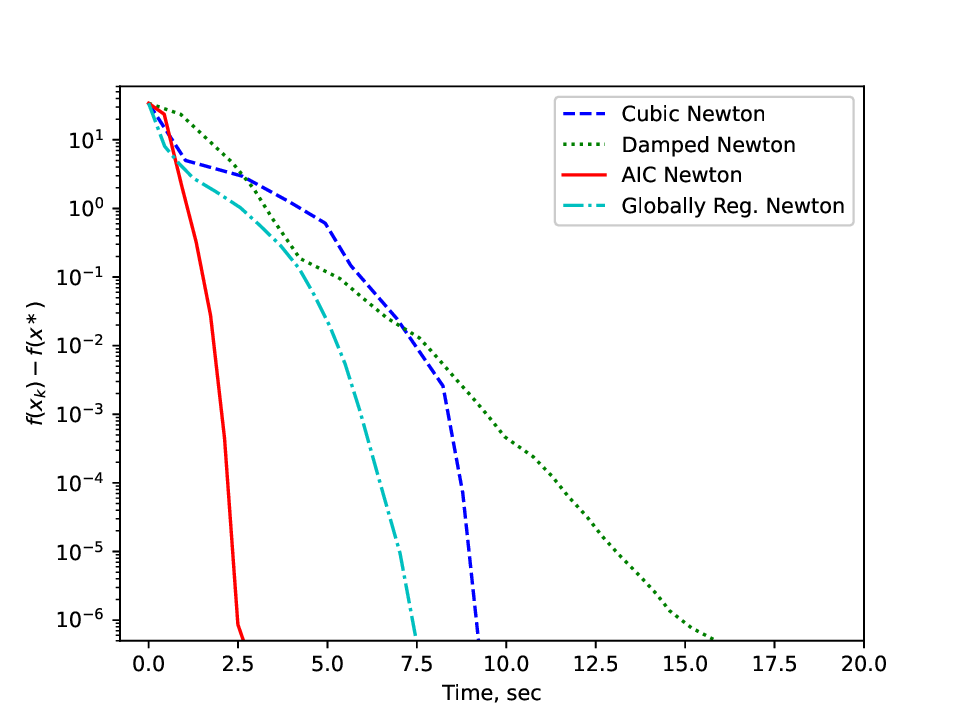}
	\caption[Convergence of \ain{} on logistic regression on a9a dataset]{Comparison of \rnewton{} methods and \dnewton{} for logistic regression task on  \ana{} dataset.}
	\label{fig:a9a}
	\end{figure}
	
	\begin{figure}
		\includegraphics[width=0.49\linewidth]{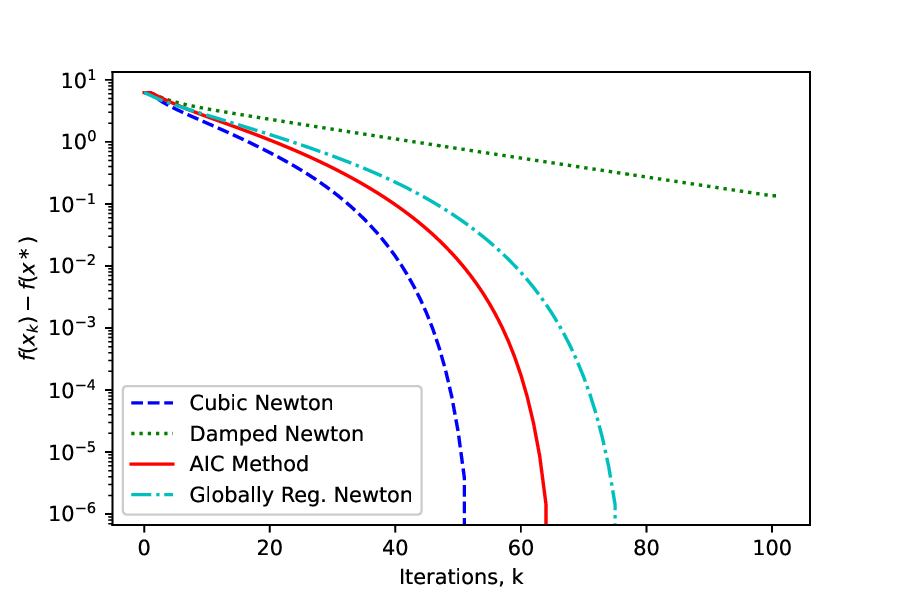}
		\includegraphics[width=0.49\linewidth]{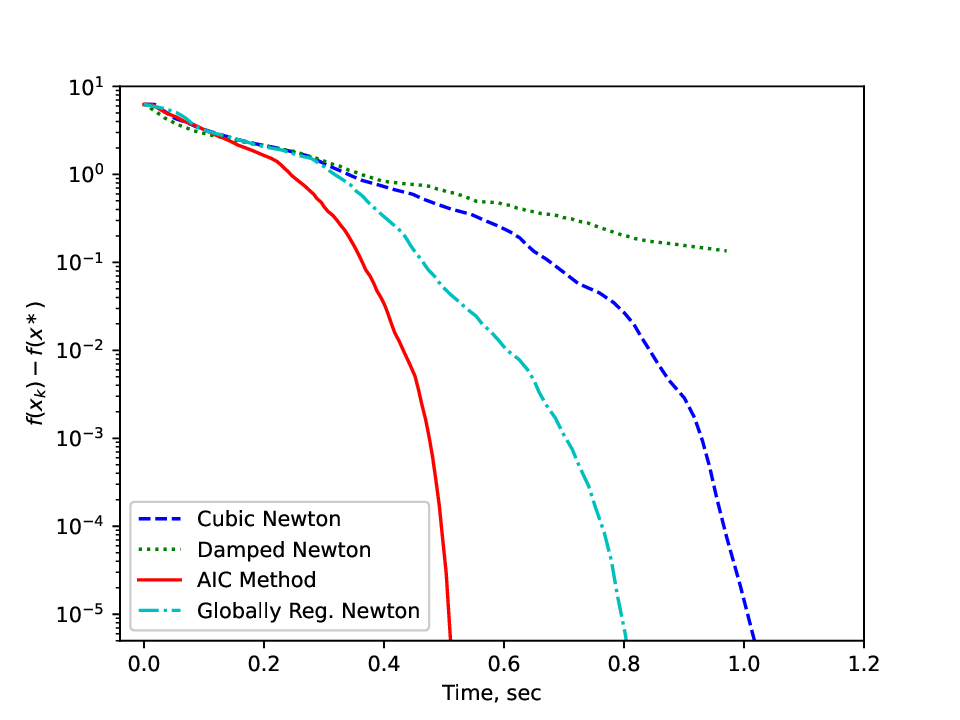}
		\caption[Convergence of \ain{} on second-order lower bound function] {Comparison of \rnewton{} methods and \dnewton{} for \emph{second-order lower bound function}.}
		\label{fig:lower_bound}
	\end{figure} 	
\subsection{Second-order lower bound function}
	For second part we solve the following minimization problem:
	\begin{equation}
	    \min_{x \in \mathbb{R}^d} \left\{ f(x) \eqdef \frac{1}{d} \sum_{j=1}^d  \left| \left[\mathbf A x\right]_j\right|^3 - x^1 + \frac{\mu}{2} \norm{x}^2 \right\}, \, \text{where }
	\mathbf A = \begin{pmatrix}
        1 & -1 & 0  &\dots & 0 \\
        0 & 1 & -1  &\dots & 0\\ 
        & & \dots & \dots &  \\
        0 & 0  &\dots & 0  & 1
        \end{pmatrix}.
	\end{equation}
	This function is a lower bound for a class of functions with Lipschitz continuous Hessian \eqref{eq:L2-smooth} with additional regularization \citep{nesterov2006cubic, nesterov2021implementable}. 
	In \Cref{fig:lower_bound}, we take $d=20$, $x^0 = 0$ (equal to all zeroes). Parameters $\Lalg,L_2, \alpha$ are fine-tuned to largest values having monotone decrease in reported metrics: $\Lalg = 662$ $L_2= 0.662$, $\alpha = 0.0172$. 
    \Cref{fig:lower_bound} demonstrates that \ain{} converges slightly slower than \cnewton{}, slightly faster than \rnewton{}, and significantly faster than \dnewton{}. ain{} outperforms every method by time. More experiments are presented in \Cref{sec:ap_experiments}. Note, that the iteration of the \cnewton{} needs an additional line-search, so one iteration of \cnewton{} is computationally harder than one iteration of \ain{}. More experiments are presented in \Cref{sec:ap_experiments}.

	\section{Extra comparisons to other methods} \label{sec:ap_experiments}
	
	\subsection{Stepsize comparison of \dnewton{}}
	In \Cref{fig:stepsizes}, we present comparison of stepizes of \ain{} with other \dnewton{} \citep{nesterov2018lectures}. Our algorithm uses stepsize bigger by orders of magnitude. For reader's convenience, we repeat stepsize choices. For \ain{} stepsize is $\alpha = \frac {-1 + \sqrt{1+ 2G}}{G}$, where $G \eqdef \Lsemi \normMd {\g(x^k)} {x^k}$. For \dnewton{} from \citet{nesterov2018lectures}, $\alpha_1 \eqdef \frac 1 {1 + G_1}, \alpha_2 \eqdef \frac {1+ G_1} {1 + G_1 + G_1^2}$, where $G_1 \eqdef \Lstandard \normMd {\g(x^k)} {x^k}$. 
	
	\begin{figure}
		\centering
		\includegraphics[width=0.49\textwidth]{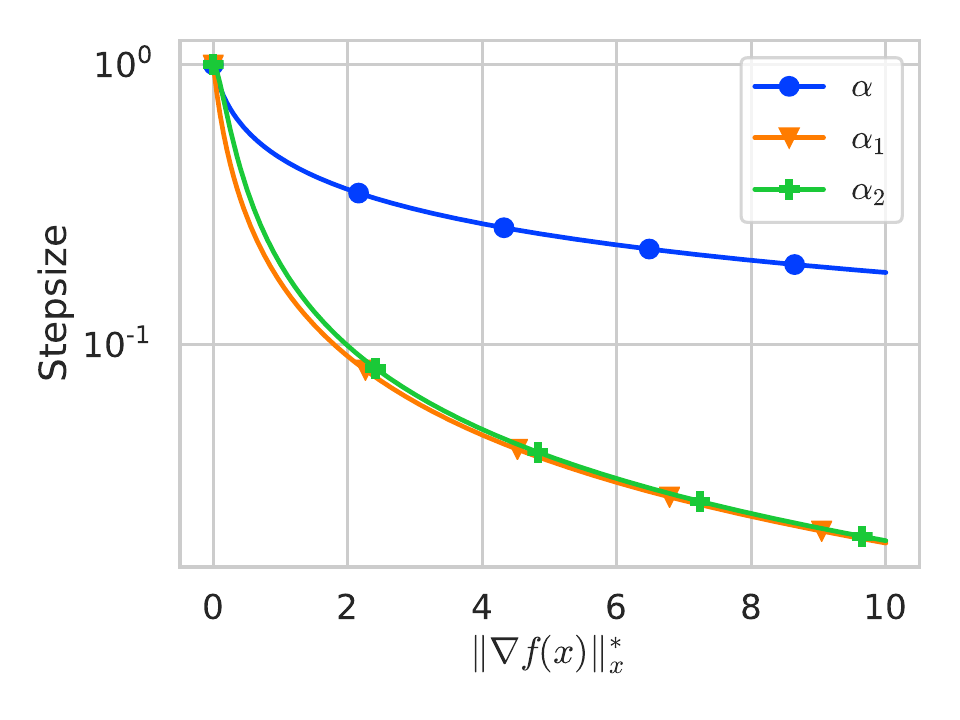}
		\includegraphics[width=0.49\textwidth]{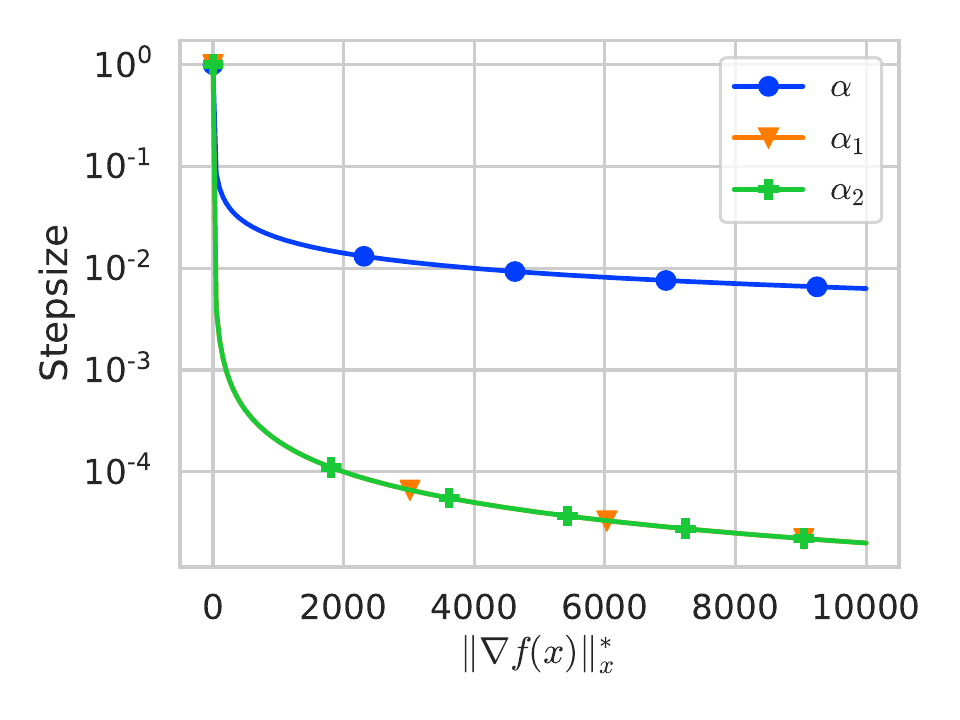}
		\caption[Comparison of stepsizes of \dnewton{} with quadratic local convergence.]{Comparison of stepsizes of affine-invariant \dnewton{} with quadratic local convergence. We compare \ain{} (blue), and stepsizes from \citet{nesterov2018lectures} in orange and green. We set $\Lsemi =\Lstandard=5.$}
		\label{fig:stepsizes}
	\end{figure}
	
	\subsection{Convergence comparison under various assumptions}
	In this subsection we present \Cref{tab:rates} -- comparison of \ain{} to \rnewton{} under different smoothness and convexity assumptions. 
	
	\begin{table*}[!t]
		\centering
		\setlength\tabcolsep{-0.5pt} 
		\begin{threeparttable}[b]
			{\scriptsize
				\renewcommand\arraystretch{2.2}
				\caption[Convergence of regularized Newton methods under various assumptions ]{
				Convergence guarantees under different versions of convexity and smoothness assumptions. For simplicity, we disregard dependence on bounded level set assumptions. All compared assumptions are considered for $\forall x,h \in \R^d$. We highlight the best know rates in blue.}
				\label{tab:rates}
				\centering 
				\begin{tabular}{ccccccc}\toprule[.1em]
					\bf Algorithm & \bf  \makecell{Strong \\ convexity \\ constant}  &\bf  \makecell{Smoothness assumption} &\bf  \makecell{Global \\ conv. \\ rate} & \bf \makecell{Local\tnote{\color{red}(1)} \\conv. \\exponent} & \bf Reference \\
					\midrule
				%
					\makecell{Dam. \newton{}  B} &$0$\tnote{\color{red}(2)} & self-concordance (\Cref{def:self-concordance}) & $\cO\left( k^{-\frac 12} \right)$ & {\color{blue} $ 2$} &  \citenum{nesterov2018lectures} \\
					\makecell{Dam. \newton{}  C} &$0$\tnote{\color{red}(2)} & self-concordance (\Cref{def:self-concordance}) & \xmark & {\color{blue} $ 2$} &  \citenum{nesterov2018lectures} \\
				    \makecell{\Cnewton{} } & $\mu$ & Lipschitz-continuous Hessian \eqref{eq:L2-smooth} & \makecell{$\cO\left( {\color{blue}{ k^{-2}} } \right) $} & \makecell{{\color{blue} $2$}} &   \makecell{\citenum{nesterov2006cubic}, \citenum{doikov2021local}} \\
				    \makecell{\Cnewton{} } & $\mu$--star-convex & Lipschitz-continuous Hessian \eqref{eq:L2-smooth} & \makecell{$\cO\left( {\color{blue}{ k^{-2}} } \right) $} & $\frac 32$ &   \citenum{nesterov2006cubic} \\
				    \makecell{\Cnewton{} } & \makecell{nonconvex, \\ lowerbounded} & Lipschitz-continuous Hessian \eqref{eq:L2-smooth} & \makecell{$\cO\left( {\color{blue}{ k^{-\frac 23}} } \right) $} & \xmark &   \citenum{nesterov2006cubic} \\
				    \makecell{ Gl.~Reg.~Newton} & $\mu$ & Lipschitz-continuous Hessian \eqref{eq:L2-smooth} 
				    & \makecell{$\cO\left( {\color{blue}{ k^{-2}} } \right) $} & \makecell{$\frac 32$} &   \makecell{\citenum{mishchenko2021regularized}, \citenum{doikov2021optimization}} \\
				    \makecell{ Gl.~Reg.~Newton} & $0$ & Lipschitz-continuous Hessian \eqref{eq:L2-smooth} 
				    & \makecell{$\cO\left( {\color{blue}{ k^{-2}} } \right) $} & \xmark &   \makecell{\citenum{mishchenko2021regularized}, \citenum{doikov2021optimization}} \\
					\midrule
					\makecell{\textbf{\ain{}}} & $0$\tnote{\color{red}(2)} & semi-strong self-concordance (\Cref{def:semi-self-concordance}) & \makecell{$\cO\left( {\color{blue}{ k^{-2}} } \right) $} & \makecell{{\color{blue} $2$}} &   Th.\ref{thm:convergence},\ref{th:local} \\
					\makecell{\textbf{\ain{}}} & $\mu$ & $f(x+h) - f(x) \leq \ip{\g(x)} {h} + \frac1 2 \normM h x ^2 + \frac\Lalt 6 \normM h x ^3$ & \makecell{$\cO\left( {\color{blue}{ k^{-2}} } \right) $} & \makecell{{\color{blue} $2$}} & Th.\ref{th:global},\ref{th:local}\tnote{\color{red}(3)} \\
					\makecell{\textbf{\ain{}}} & $0$ & $f(x+h) - f(x) \leq \ip{\g(x)} {h} + \frac1 2 \normM h x ^2 + \frac\Lalt 6 \normM h x ^3$ & \makecell{$\cO\left( {\color{blue}{ k^{-2}} } \right)$\tnote{\color{red}(4)}} & \xmark & Th.\ref{th:global} \\
					\bottomrule[.1em]
				\end{tabular}
			}
			\begin{tablenotes}
				{\scriptsize
					\item [\color{red}(1)] For a Lyapunov function $\Phi$ and a constant $c>0$, we report exponent $\beta>1$ of $\Phi(x^{k+1}) \leq c \Phi (x^k)^\beta$. Mark \xmark{} means that such $\beta, c, \Phi$ are not known.
					\item [\color{red}(2)] Self-concordance implies strong convexity locally.
					\item [\color{red}(3)] Under strong convexity, we can prove local convergence analogically to \Cref{th:local}.
					\item [\color{red}(4)] Convergence to a neighborhood of the solution.
				}
			\end{tablenotes}
		\end{threeparttable}
	\end{table*}
	
    \subsection{Logistic regression experiments}
	We solve the following minimization problem:
	\begin{equation}
	    \min_{x \in \mathbb{R}^d} \left\{ f(x) \eqdef \frac{1}{m} \sum_{i=1}^m  \log\left(1-e^{-b_i a_i^\top x}\right) + \frac{\mu}{2} \norm{x}^2 \right\}.
	\end{equation}
	To make problem and data balanced, we normalise each data point and get $\norm{a_i}=1$ for every $i \in \{1,\ldots, m\}$. Parameters of all methods are fine-tuned to get the fastest convergence. Note, that it is possible that for bigger $L$ method converges faster in practice.
	
	\begin{figure}[!h]
	\includegraphics[width=0.49\linewidth]{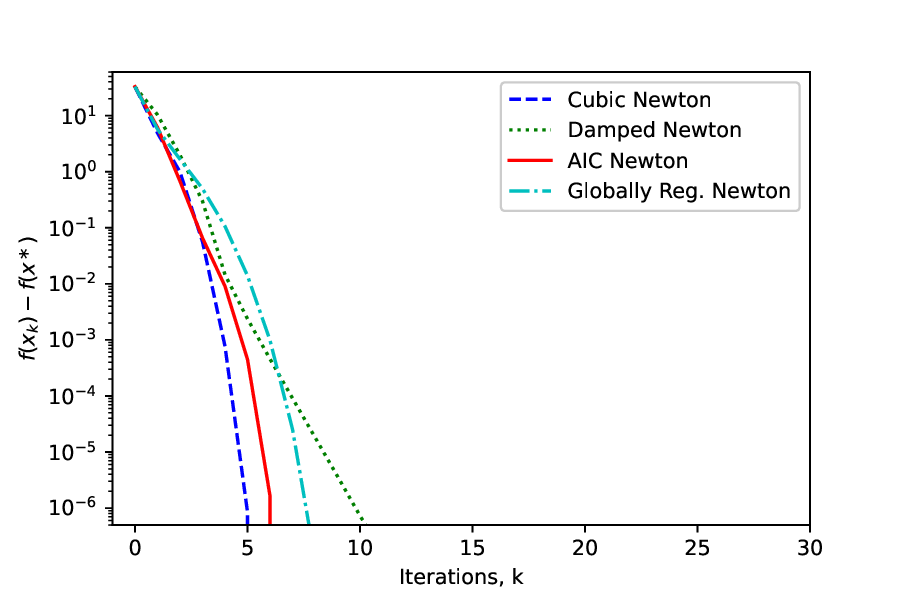}
	\includegraphics[width=0.49\linewidth]{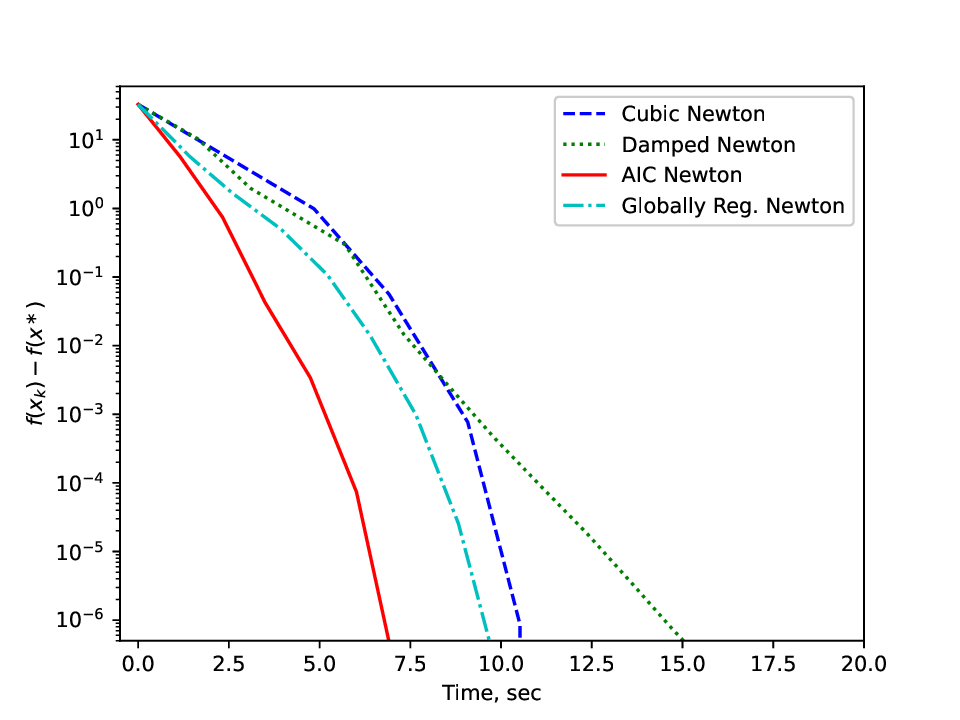}
	\caption[Convergence of \ain{} on logistic regression on w8a dataset]{Comparison of \rnewton{} methods and \dnewton{} for logistic regression task on  \wea{} dataset.}
	\label{fig:w8a}
	\end{figure}
	
	In \Cref{fig:w8a}, we consider classification task on dataset \wea{} \citep{chang2011libsvm}. Number of features for every data sample is $d=300$, $m=49749$. We take starting point $x^0 \eqdef 8 [1,1, \dots, 1]^\top$ and $\mu = 10^{-3}$. Fine-tuned values are $\Lalg = 0.6$, $L_2= 0.0001$, $\alpha = 0.5$. We can see that all methods are very close.  \dnewton{} has rather big step $0.5$, so it is fast at the beginning but still struggle at the end because of the fixed step-size.
 
	\begin{figure}[t]
	\includegraphics[width=0.49\linewidth]{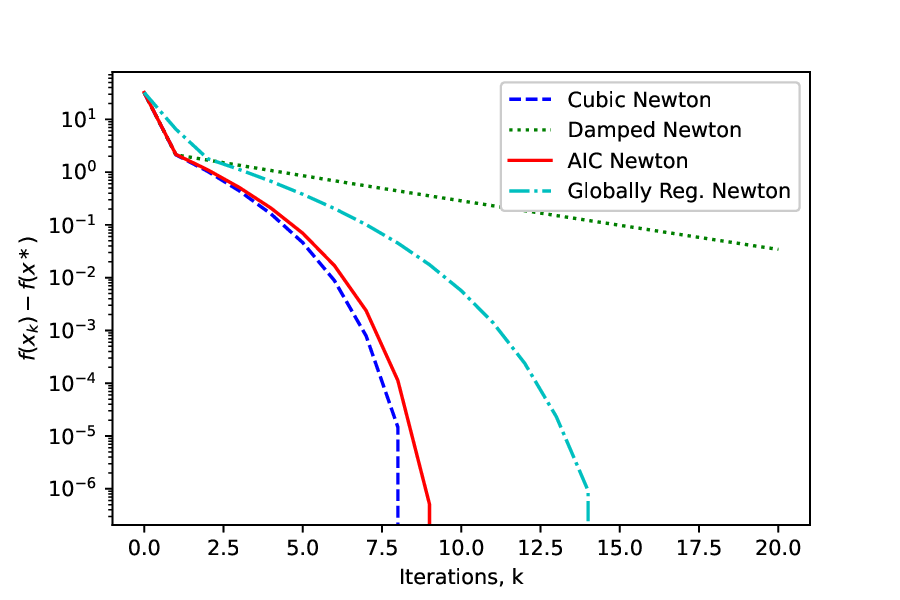}
	\includegraphics[width=0.49\linewidth]{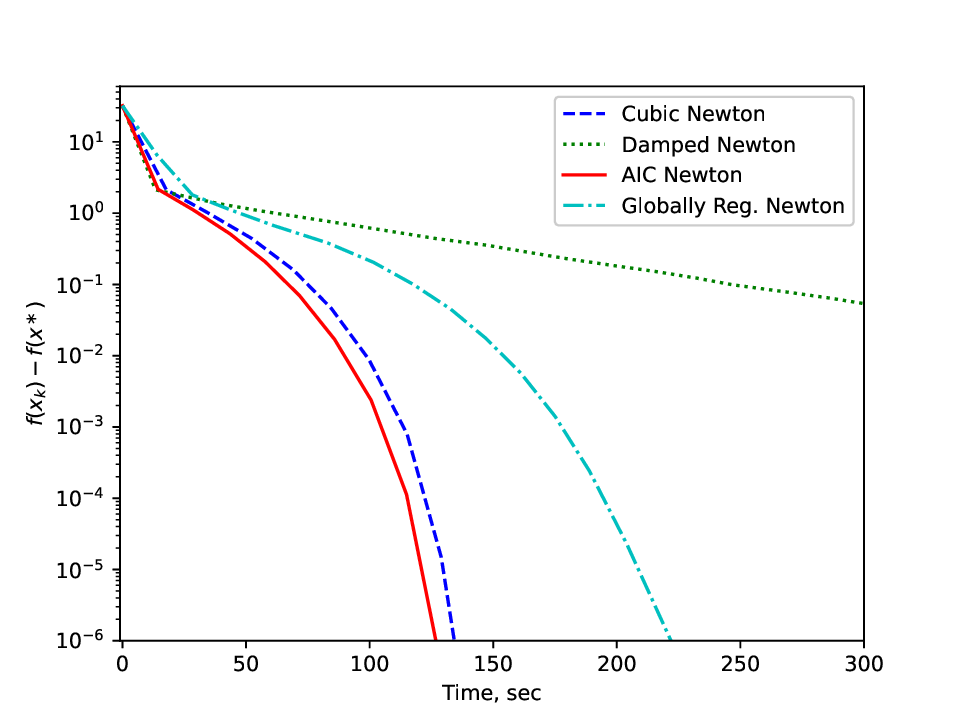}
	\caption[Convergence of \ain{} on binary classification task on \mnist{} dataset]{Comparison of \rnewton{} methods and \dnewton{} for logistic regression task on  \mnist{} dataset (0 vs. all other digits).}
	\label{fig:mnist}
	\end{figure}
	
	In \Cref{fig:mnist}, we consider binary classification task on dataset \mnist{} \citep{deng2012mnist} (one class contains images with 0, another one~--- all others). Number of features for every data sample is $d=28^2=784$, $m=60000$. We take starting point $x^0 \eqdef 3 \cdot [1,1, \dots, 1]^\top$ (such that \newton{} started from this point diverges) and $\mu = 10^{-3}$. Fine-tuned values are $\Lalg = 10$, $L_2 = 0.0003$ for Globally Reg. Newton and \cnewton{}, $\alpha = 0.1$. We see that \ain{} has the same iteration convergence as \cnewton{} but faster by time.
	
	\begin{figure}[H]
		\centering
		\includegraphics[width=0.49\textwidth]{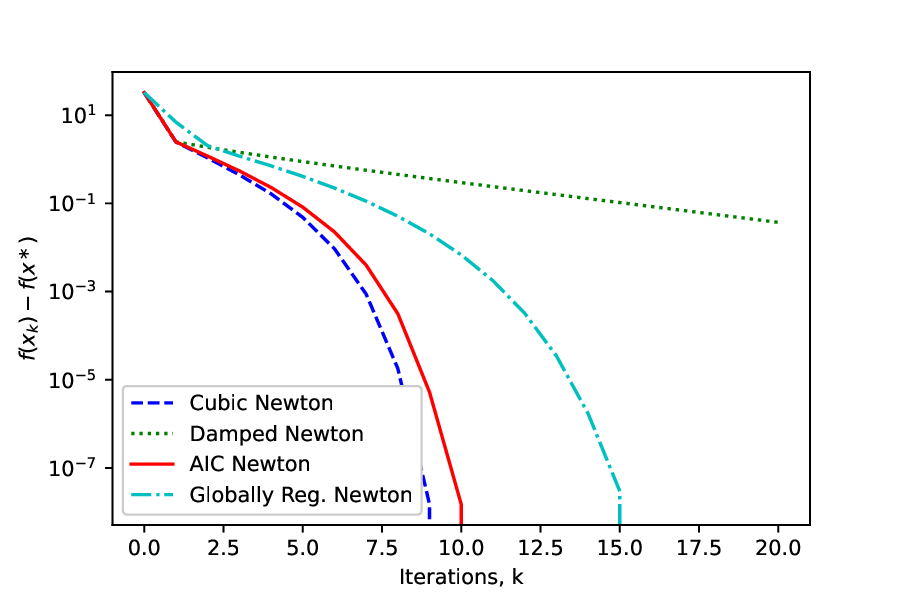}
		\includegraphics[width=0.49\textwidth]{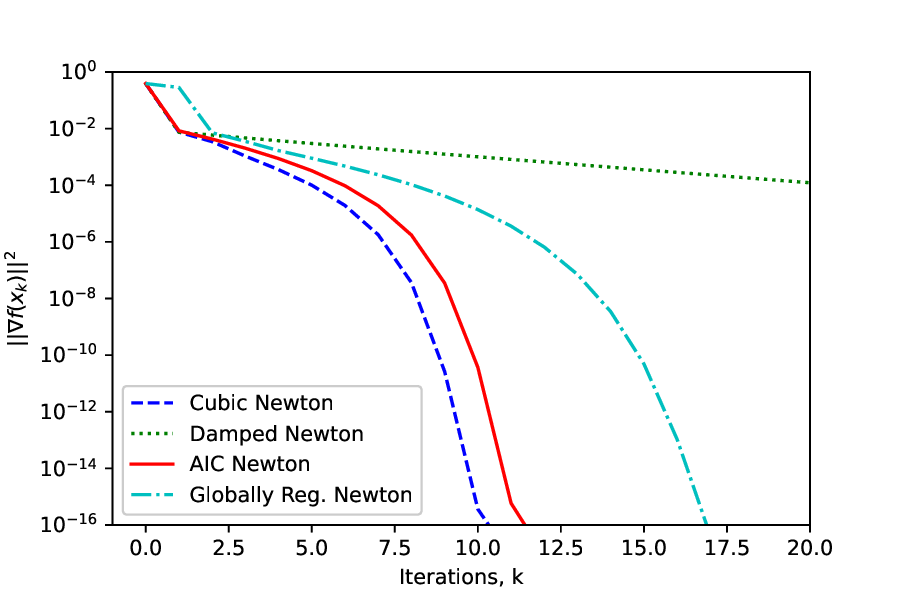}
		\caption[Convergence of \ain{} on multi-class classification task on w8a dataset] {Comparison of \rnewton{} methods and \dnewton{} for logistic regression task on \mnist{} dataset (10 models for $i$ vs. other digits problems with argmax aggregation).}
		\label{fig:mnist_all}
	\end{figure}
	
	In \Cref{fig:mnist_all}, we present the results for multi-class classification problem on dataset \mnist{}. We train 10 different models in parallel, each one for the problem of binary classification distinguishing $i$--th class out of others. Loss on current iteration for the plots is defined as average loss of $10$ models. Prediction is determined by the maximum ``probability'' predicted by $i$--th model. 
	The estimates for the parameters of methods are the same as in previous experiment.
	We see that \ain{} is the same speed as \cnewton{} and Globally Reg. Newton and much faster than \dnewton{} in both value of function and gradient norm.

	For normalized problem, we can analytically compute an upper bound for theoretical constant $L_2$,
	\begin{equation}
	    \norm{\nabla^3 f(x)}_{2}\leq L_2. 
	\end{equation}
	One can show that $L_2= \frac{\sqrt{3}}{18}\simeq 0.1$. In our experiments, we show that \cnewton{} can work with much lower constants:
	\madelon{} - $0.015$, \wea{} - $0.00003$, \ana{} - $0.000215$. It means that theoretical approximation of the constants can be bad and we have to tune them for all methods.

\chapter{Sketch-and-project meets Newton method: global $\okd$ convergence with low-rank updates} \label{sec:sgn}
\thispagestyle{empty}

\section{Introduction}
Second-order methods have always been a fundamental component of both scientific and industrial computing. Their origins can be traced back to the works \citet{Newton}, \citet{Raphson}, and \citet{Simpson}, and they have undergone extensive development throughout history \citep{kantorovich1948functional,more1978levenberg,griewank1981modification}. For the more historical development of classical methods, we refer the reader to \citet{ypma1995historical}. The amount of practical applications is enormous, with over a thousand papers included in the survey \citet{conn2000trust} on trust-region and quasi-Newton methods alone.

Second-order methods are highly desirable due to their invariance to rescaling and coordinate transformations, which significantly reduces the complexity of hyperparameter tuning. Moreover, this invariance allows convergence independent of the conditioning of the underlying problem. In contrast, the convergence rate of first-order methods is fundamentally dependent on the function conditioning. Moreover, first-order methods can be sensitive to variable parametrization and function scale, hence parameter tuning (e.g., step size) is often crucial for efficient execution.

On the other hand, even the simplest and most classical second-order method, Newton's method \citep{kantorovich1948functional}, achieves an extremely fast, quadratic convergence rate (precision doubles in each iteration) \citep{nesterov1994interior} when initialized sufficiently close to the solution. However, the convergence of the Newton method is limited only to the neighborhood of the solution. Several works, including \citet{jarre2016simple}, \citet{mascarenhas2007divergence}, \citet{bolte2022curiosities} demonstrate that when initialized far from optimum, the line search and trust-region Newton-like method can diverge on both convex and nonconvex problems.

\subsection{Demands of modern machine learning}
Despite the long history of the field, research on second-order methods has been thriving to this day. 
Newton-like methods with a fast $\okd$ global rate were introduced relatively recently under the names Cubic Newton method \citep{nesterov2006cubic} or Globally regularized Newton methods \citep{doikov2021local,mishchenko2021regularized,hanzely2022damped}. The main limitation of these methods is their poor scalability for modern large-scale machine learning. Large datasets with numerous features necessitate well-scalable algorithms. While tricks or inexact approximations can be used to avoid computing the inverse Hessian, simply storing the Hessian becomes impractical when the dimensionality $d$ is large.
This challenge has served as a catalyst for the recent developments in the field. To address the curse of dimensionality, \citet{SDNA}, \citet{luo2016efficient}, \citet{RSN}, \citet{RBCN}, and \citet{hanzely2020stochastic} proposed Newton-like method operating in random low-dimensional subspaces.
This approach, also known as \sap{} \citep{gower2015randomized}, has the advantage of substantially reducing the computational cost per iteration. However, this happens at the cost of slower, $\cO \left( k^{-1}\right)$, convergence rate \citep{gower2020variance, hanzely2020stochastic}.

\subsection{Contributions}
In this work, we argue that the \sap{} adaptations of second-order methods can be improved. To this end, we introduce the first \sap{} method (Sketchy Global Newton, \sgn{}, \Cref{alg:intro_sgn}) which boasts a global \textbf{$\okd$} convex convergence rate. This achievement aligns with the rapid global rate of full-dimensional \rnewton{} methods. 
In particular, sketching on $1$-dimensional subspaces engenders $\okd$ global convex convergence with an iteration cost of $\cO(1)$. 
As a cherry on the top, our approach offers a local linear convergence rate independent of the condition number	and a global linear convergence rate under the assumption of relative convexity (\Cref{def:rel}).
We summarize the contributions below and in Tables \ref{tab:three_ways}, \ref{tab:detailed}.

\begin{table*}[t!]
    \centering
    \setlength\tabcolsep{3pt} 
    \begin{threeparttable}[t]
        {\scriptsize
            \renewcommand\arraystretch{3}
            \caption[Global convex convergence rates of low-rank Newton methods]{
            Global convergence rates of low-rank Newton methods for convex and Lipschitz smooth functions. For simplicity, we disregard differences between various notions of smoothness. We use fastest full-dimensional algorithms as the baseline, we highlight best rate in blue.}
            \label{tab:setup_comparison}
            \centering 
            \begin{tabular}{?c?c|c?}
            \Xhline{2\arrayrulewidth}
             \makecell{\\ \\ \textbf{Update} \\ \textbf{direction}} \makecell{ \textbf{Update}\\ \textbf{oracle} \\ \\} & \makecell{\textbf{Full-dimensional} \\ (direction is deterministic)} & \makecell{\textbf{Low-rank} \\ (direction in expectation)}\\
             
             \Xhline{2\arrayrulewidth}
             \makecell{\textbf{Non-Newton} \\ \textbf{direction} }
             & \makecell{ {\color{blue} $\cO(k^{-2})$} \\
             Cubically regularized Newton \\  \citep{nesterov2006cubic}, \\
             Globally regularized Newton \\ \citep{mishchenko2021regularized}, \citep{ doikov2021optimization} } 
             
             & \makecell{$\cO(k^{-1})$ \\ 
              Stochastic Subspace Cubic Newton \\ \citep{hanzely2020stochastic} }\\               

             \hline
             \makecell{\textbf{Newton} \\ \textbf{direction}} 
             & \makecell{ {\color{blue} $\cO(k^{-2})$} \\ 
              Affine-Invariant Cubic Newton  \\ \citep{hanzely2022damped}}
             
             & \makecell{ {\color{blue} $\cO(k^{-2})$} \\ 
            \textbf{Sketchy Global Newton}  \textbf{(new)} \\ \hline
             {$ \cO(k^{-1})$}\\ 
              Randomized Subspace Newton \\ \citep{RSN}}\\
             \Xhline{2\arrayrulewidth}
            \end{tabular}
        }
    \end{threeparttable}
\end{table*}
\begin{algorithm} [t!]
    \caption{\sgn{}: Sketchy Global Newton \textbf{(new)}} \label{alg:sgn}
    \begin{algorithmic}[1]
        \State \textbf{Requires:} Initial point $x^0 \in \R^d$, distribution of sketch matrices $\cD$ such that $\mathbb E_{\s \sim \cD} \left[ \p \right] = \afrac \tau d \mI$, 
		constant $\Lalg$ upper bounding semi-strong self-concordance constants in the sketched directions (\Cref{def:intro_scs}) $\Lalg \geq \sup_{\s \sim \cD} \Ls$ 

        \Comment{Choose $\Lalg \geq 1.2 \sup_{\s} \Ls \Lrel _\s ^2 >0$ for global linear rate on rel. convex $f$}
        \For {$k=0,1,2\dots$}
        \State Sample $\s_k \sim \cD$
        \State $\als k = \afrac {-1 +\sqrt{1+2 \Lalg \normMSdk{\gSk(x^k)} {x^k} }}{\Lalg \normMSdk {\gSk(x^k)} {x^k} }$
        \State $x^{k+1} = x^k - \als k \s_k \left[\hSk(x^k)\right]^{\dagger} \gSk(x^k)$        
        \Comment {Equiv. to \eqref{eq:sgn_reg} and \eqref{eq:sgn_sap}}
        \EndFor
    \end{algorithmic}
\end{algorithm}

\begin{itemize}[leftmargin=*] 
\setlength\itemsep{0.2em}
\item \textbf{One connects all:} We present \sgn{} through three orthogonal viewpoints: \sap{} method, subspace \newton{} method, and subspace \rnewton{} method. Compared to established algorithms, \sgn{} can be viewed as \ain{} operating in subspaces, \sscn{} operating in local norms, or \rsn{} with the new stepsize schedule(\Cref{tab:three_ways}).

\item \textbf{Fast global convergence:}
\sgn{} is the first low-rank method that solves \textbf{convex} optimization problems with $\okd$ global rate. This matches the state-of-the-art rates of full-rank Newton-like methods. Other \sap{} methods, in particular, \sscn{} and \rsn{}, have slower $\mathcal O \left( k^{-1} \right)$ rate (\Cref{tab:setup_comparison}).

\item \textbf{Cheap iterations:} \sgn{} uses $\tau$--dimensional updates. Per-iteration cost is $\mathcal O \left( d\tau^2 \right)$ and in the $\tau=1$ case it is even $\mathcal O \left(1\right).$ Conversely, full-rank Newton-like methods have cost proportional to $d^3$  and $d \gg \tau$.

\item \textbf{Linear local rate:} \sgn{} has local linear rate $\cO \left( \frac d \tau \log \frac 1 \varepsilon \right)$ (\Cref{th:local_linear}) dependent only on the ranks of the sketching matrices. This improves over the condition-dependent linear rate of \rsn{} or any rate of first-order methods.

\item \textbf{Global linear rate:}  Under $\murel$--relative convexity, \sgn{} achieves global linear rate $\cO \left( \frac {\Lalg} {\rho \murel } \log \frac 1 \varepsilon \right)$ to a neighborhood of the solution (\Cref{th:global_linear})\footnote{$\rho$ is condition number of a projection matrix \eqref{eq:rho}, constant $\Lalg$ affects stepsize \eqref{eq:sgn_alpha}.}. 

\item \textbf{Geometry and interpretability:}
Update of \sgn{} uses well-understood projections\footnote{\citet{gower2020variance} describes six equivalent viewpoints.} of Newton method with stepsize schedule \ain{}. Moreover, those stochastic projections are affine-invariant and in expectation preserve direction \eqref{as:projection_direction}. 
On the other hand, implicit steps of regularized Newton methods including \sscn{} lack geometric interpretability.

\item \textbf{Algorithm simplicity:}
\sgn{} is affine-invariant and independent of the choice of the basis. This removes one parameter from potential parameter tuning.
Update rule \eqref{eq:sgn_aicn} is simple and explicit. Conversely, most of the fast globally-convergent Newton-like algorithms require an extra subproblem solver in each iteration.

\item \textbf{Analysis:}
The analysis of \sgn{} is simple, all steps have clear geometric interpretation. On the other hand, the analysis of \sscn{} \citep{hanzely2020stochastic} is complicated as it measures distances in both $l_2$ norms and local norms. This not only makes it harder to understand but also leads to worse constants, which ultimately cause a slower convergence rate.
\end{itemize}

\begin{table*}[b!]
    \centering
    \setlength\tabcolsep{2pt} 
    \begin{threeparttable}[b]
        {
        \scriptsize
            \caption[Three approaches for second-order global minimalization]{
            Three approaches for second-order global minimization. We denote $x^k\in \mathbb{R}^d$ iterates, $\sk \sim \cD$ distribution of sketches of rank $\tau \ll d$, $\als k, \als k$ stepsizes, $L_2, \Ls$ smoothness constants, $c_{\text{stab}}$ Hessian stability constant. For simplicity, we disregard differences in assumptions and report complexities for matrix inverses implemented naively. }
            \label{tab:three_ways}
            \centering 
            \begin{tabular}{?c?c|c|c?}
            \bottomrule
             \makecell{\textbf{Orthogonal}\\ \textbf{lines of}\\{\textbf{work}}} 
             & \makecell{\textbf{Sketch-and-Project} \citenum{gower2015randomized} \\ (various update rules)} 
             & \makecell{\textbf{Damped Newton} \citenum{nesterov1994interior},  \citenum{KSJ-Newton2018} } 
             & \makecell{\textbf{Globally Reg. Newton} \\ \citenum{nesterov2006cubic}, \citenum{polyak2009regularized},  \citenum{mishchenko2021regularized}, \citenum{doikov2021optimization}\tn{1} }\\
              
             \Xhline{2\arrayrulewidth}
             \makecell{\textbf{Update}\\ $x^{k+1} - x^k =$} 
             & \makecell{$\als k \pk k \left( \text{update}(x^k) \right)$,\\ for $\sk \sim \cD$}
             & $\als k [\h(x^k)]^\dagger \g(x^k)$
             & \makecell{$\argmin_{h\in \mathbb{R}^d} T(x^k, h)$, \,\\ for $T(x,h) \eqdef \la \g(x), h \ra +$\\$+ \frac 12 \normsM {h} x + \frac {L_2}6 \normM {h} {} ^3$}\\
                         
             \hline
             \makecell{\textbf{Characteristics}}
             &\makecell[l]{
             \color{mydarkgreen}
             + cheap, low-rank updates\\
             \color{mydarkgreen}
             + global linear convergence\\
             \, (conditioning-dependent)\\
             \color{mydarkred}
             -- optimal rate: linear\\
             }
             & \makecell[l]{
             \color{mydarkgreen}
             + affine-invariant geometry\\
             \color{mydarkred}
             -- iteration cost $\cO\left(d^3 \right)$\\
             Fixed $\als k = c_{\text{stab}}^{-1}$:\\
             \color{mydarkgreen}
             \, + global linear convergence\\
             Schedule $\als k \nearrow 1$:\\
             \color{mydarkgreen}
             \, + local quadratic rate\\
             }             
             & \makecell[l]{
             \color{mydarkgreen}
             + global convex rate $\okd$\\
             \color{mydarkgreen}
             + local quadratic rate\\
             \color{mydarkred}
             -- implicit updates\\
             \color{mydarkred}
             -- iteration cost $\cO\left(d^3 \log \frac 1 \varepsilon \right)$\\
             }\\
             \toprule
             
             \bottomrule
             \makecell{\textbf{Combinations}\\\textbf{+ benefits}} 
             & \makecell{\textbf{Sketch-and-Project}} 
             & \makecell{\textbf{Damped Newton} } 
             & \makecell{\textbf{Globally Reg. Newton} }\\

             \Xhline{2\arrayrulewidth}
             \makecell{\rsn{} \citenum{RSN} \\ \Cref{alg:rsn}}
             & \makecell{ \cmark\\ \makecell[l]{
             \color{mydarkgreen}
             + iter. cost $\cO\left(d \tau^2 \right)$\\ 
             \color{mydarkgreen}
             + iter. cost $\cO\left(1 \right)$ if $\tau=1$\\ 
             }} & \makecell{\cmark\\
             \color{mydarkgreen}
             + global rate $\cO\left( \frac 1\rho \frac \Lrel \murel \log \frac 1 \varepsilon\right)$ \\ }& \xmark\\
             \hline
             \makecell{\sscn{} \citenum{hanzely2020stochastic} \\ \Cref{alg:sscn}}
             &\makecell{\cmark\\ \makecell[l]{
             \color{mydarkgreen}
             + iter. cost $\cO\left(d\tau^2 +\tau^3 \log \frac 1 \varepsilon \right)$\\
             \color{mydarkgreen}
             + iter. cost $\cO\left(\log \frac 1 \varepsilon \right)$ if $\tau=1$\\
             \color{mydarkgreen}
             + local rate $\cO\left( \frac d \tau \log \frac 1 \varepsilon\right)$ \\ }} & \xmark & \makecell{ \cmark \\
             \color{mydarkgreen}
             + global convex rate $\okd$ \\}\\
             \hline
             \makecell{\ain{} \citenum{hanzely2022damped} \\ \Cref{alg:aicn}}
             & \xmark & \makecell{\cmark\\ \makecell[l]{
             \color{mydarkgreen}
             + affine-invariant geometry\\                          
             \color{mydarkred}
             -- no glob. lin. rate proof\tn{2}\\}} & \makecell{\cmark\\ \makecell[l]{
             \color{mydarkgreen}
             + global convex rate $\okd$\\
             \color{mydarkgreen}
             + local quadratic rate\\
             \color{mydarkgreen}
             + iteration cost $\cO\left(d^3 \right)$\\
             \color{mydarkgreen}
             + simple, explicit updates\\
             }}\\
             \hline
             \makecell{\textbf{\sgn{}}\\ \textbf{(this work)} \\ \Cref{alg:sgn}}
             & \makecell{ \cmark \\ \makecell[l]{
             \color{mydarkgreen}
             + iter. cost $\cO\left(d\tau^2 \right)$\\
             \color{mydarkgreen}
             + iter. cost $\cO\left(1 \right)$ if $\tau=1$\\
             \color{mydarkgreen}
             + local rate $\cO\left( \frac d \tau \log \frac 1 \varepsilon\right)$ \\             
             -- quadratic rate unachievable}}
             & \makecell{\cmark\\ \makecell[l]{
             \color{mydarkgreen}
             + affine-invariant geometry\\
             \color{mydarkgreen}
             + global rate $\cO\left( \frac 1\rho \frac \Lrel \murel \log \frac 1 \varepsilon\right)$\\}} 
             & \makecell{\cmark\\ \makecell[l]{
             \color{mydarkgreen}
             + global convex rate $\okd$\\
             \color{mydarkgreen}
             + simple, explicit updates\\
             \\}} \\

             \toprule
             \bottomrule
             \makecell{\textbf{Three views} \\\textbf{of} \color{mydarkgreen} \textbf{\sgn{}}} & \makecell{\textbf{Sketch-and-Project of}\\\textbf{\color{mydarkgreen} Damped Newton method}} & \makecell{\textbf{Damped Newton} \\ \textbf{\color{mydarkgreen}in sketched subspaces} } & \makecell{\textbf{{Affine-Invariant} Newton} \\ \textbf{\color{mydarkgreen} in sketched subspaces}}\\
             \Xhline{2\arrayrulewidth}
             \makecell{\textbf{Update}\\ $x^{k+1} - x^k =$} 
             & $\als k \pk k \color{mydarkgreen} [\h(x^k)]^\dagger \g(x^k)$ 
             & ${\color{mydarkgreen}\als k} {\color{mydarkgreen} \sk } [\nabla_{\color{mydarkgreen}\sk}f(x^k)]^\dagger \nabla_{\color{mydarkgreen} \sk}f(x^k)$
             & \makecell{${\color{mydarkgreen}\sk} \argmin_{h\in \mathbb{R}^d} T_{\color{mydarkgreen}\sk}(x^k, h)$, \\ for 
                         $T_{\color{mydarkgreen} \s} (x,h) \eqdef \la \g(x), {\color{mydarkgreen}\s} h \ra +$\\$+ \frac 12 \normsM {{\color{mydarkgreen}\s} h} x + \frac {\Ls}6 \Vert {{\color{mydarkgreen}\s} h} \Vert ^3 _ {\color{mydarkgreen}x} $} \\          
             \toprule
            \end{tabular}
        }
        \begin{tablenotes}
        	{\scriptsize
                \item [\color{red}(1)]Works \citep{polyak2009regularized}, \citep{mishchenko2021regularized}, \citep{doikov2021optimization} have explicit updates and iteration cost $\cO \left(d^3 \right)$, but for the costs of slower global rate, slower local rate, and slower local rate, respectively.
                \item [\color{red}(2)] \citet{hanzely2022damped} didn't show global linear rate of \ain{}. However, it follows from our Theorems \ref{th:global_linear}, \ref{th:local_linear} for $\sk = \mI$.
        	}
        \end{tablenotes}
    \end{threeparttable}
\end{table*}

\begin{table*}[t]
    \centering
    \setlength\tabcolsep{0pt} 
    \begin{threeparttable}[b]
        {\scriptsize
            \renewcommand\arraystretch{2.2}
            \caption[Summary of globally-convergent Newton methods]{             
            Globally convergent Newton-like methods. For simplicity, we disregard differences and assumptions - we assume strong convexity, $L$--smoothness, semi-strong self-concordance and bounded level sets. 
            We highlight the best know rates in blue. }
            \label{tab:detailed}
            \centering 
            \begin{tabular}{cccccccc}\toprule[.2em]
                \bf Algorithm & \bf \makecell{Stepsize \\ range} & \bf  \makecell{Affine \\ invariant \\ algorithm?} & \bf \makecell{Iteration \\ cost\tn{0}} &\bf  \makecell{Linear\tn{1} \\ convergence} & \bf \makecell{Global \\ convex \\convergence } & \bf Reference \\
                \bottomrule[.2em]
                \makecell{\newton{}} & $1$ & \cmark & $\mathcal O \left(d^3 \right)$ & \xmark & \xmark &  \citenum{kantorovich1948functional} \\
                \midrule
                \makecell{\Dnewton{} B}  & $(0,1]$ & \cmark & $\mathcal O \left(d^3 \right)$ & \xmark & $\mathcal O \left(k^{-\frac 12} \right)$ &  \citenum{nesterov1994interior} \\
                \midrule
                \makecell{\ain{}} & $(0,1]$ & \cmark & $\mathcal O \left(d^3 \right)$ & \xmark & $\cO\left( {\color{blue} {k^{-2}} } \right)$ &   \citenum{hanzely2022damped} \\
                \bottomrule[.2em]
                \makecell{\Cnewton{}} & 1 & \xmark &  $\mathcal O \left(d^3 \log \frac 1 \varepsilon \right)$\tn{3} &  \xmark & $\mathcal O \left( {\color{blue} k^{-2} } \right)$ &  \citenum{nesterov2006cubic} \\
                \midrule
                \makecell{\Gnewton{}} & $1$ & \xmark & $\mathcal O \left(d^3 \right)$ & \xmark & $\mathcal O \left(k^{-\frac 14} \right)$ &   \citenum{polyak2009regularized} \\
                \midrule
                \makecell{\Gnewton} & 1& \xmark & $\mathcal O \left(d^3 \right)$ & \xmark & $\mathcal O \left( {\color{blue} k^{-2} } \right)$  & \makecell{\citenum{mishchenko2021regularized},  \citenum{ doikov2021optimization}} \\
                \bottomrule[.2em]
                \makecell{\en{}} & $\frac 1 {L}$\tn{4} & \cmark & $\mathcal O \left(d^3 \right)$ & glob\tn{4} & \xmark &  \citenum{KSJ-Newton2018} \\
                \midrule
                \makecell{\rsn{}} & $\frac 1 {L}$ & \cmark & \makecell[l]{
                    $\mathcal O({\color{blue}d\tau^2}),$\\
                    $\mathcal O({\color{blue} 1})$ \quad  if $\tau=1$} & glob\tn{4} & $\mathcal O \left(k^{-1} \right)$  & \citenum{RSN} \\
                \midrule
                \makecell{\sscn{}} & 1 & \xmark & \makecell[l]{
                $\mathcal O(d\tau^2+\tau^3 \log \frac 1 \varepsilon),$\\
                $\mathcal O(\log \frac 1 \varepsilon)$ \quad if $\tau=1$}\tn{3} & loc & $\mathcal O \left(k^{-1} \right)$ & \citenum{hanzely2020stochastic} \\
                \bottomrule[.2em]
                \makecell{\textbf{\sgn{}}\\\textbf{(new)}} & $(0,1]$ & \cmark & \makecell[l]{
                    $\mathcal O({\color{blue}d\tau^2}),$\\
                    $\mathcal O({\color{blue} 1})$ \quad  if $\tau=1$} 
                     & \makecell{loc + glob\tn{5}} & $\mathcal O \left( {\color{blue} k^{-2} } \right)$ & \textbf{This work} \\
                \bottomrule[.2em]
            \end{tabular}
        }
        \begin{tablenotes}
            {\scriptsize
                \item [\color{red}(0)] $d$ is dimension, $\tau$ is rank of sketches $\sk $. We report rate of implementation using matrix inverses.
                \item [\color{red}(1)] ``loc"/``glob'' denotes algorithm's local/global linear rate 
                (under possibly stronger assumptions).
                \item [\color{red}(3)] Cubic Newton and SSCN solve implicit problem each iteration. Naively implemented, it requires $\times \log  \frac 1 \varepsilon$ matrix inverses to approximate sufficiently in order to converge to $\varepsilon$-neighborhood \citep{hanzely2022damped}. For larger $\tau$ or high precision $\varepsilon$ (case $\tau \log \frac 1 \varepsilon \geq d$), this becomes the bottleneck of the training..
                \item [\color{red}(4)] Authors assume $c$--stability, which is implied by Lipschitz smoothness + strong convexity \citep{RSN} .
                \item [\color{red}(5)] Separate results for local convergence (Th.~\ref{th:local_linear}) and global convergence to corresponding neighborhood (Th.~\ref{th:global_linear}).
            }
        \end{tablenotes}
    \end{threeparttable}
\end{table*}

\subsection{Notation}
In this chapter, we consider the optimization objective
\begin{equation} \label{eq:objective}
    \min_{x \in \mathbb{R}^d} f(x),
\end{equation}
where $f:\R^d \to \R$ is convex, twice differentiable, bounded from below, and potentially ill-conditioned. The number of features $d$ is potentially large.
Subspace methods use a sparse update
\begin{equation} \label{eq:subspace_generic}
x_+ = x + \s h,
\end{equation}
where $\s \in \mathbb{R}^{d \times \tau(\s)}, \s \sim \cD$ is a thin matrix and $h \in \mathbb{R}^{\tau(\s)}$. We denote gradients and Hessians along the subspace spanned by columns of $\s$ as {$\gS (x) \eqdef \st \g (x)$} and {$\hS (x) \eqdef \st \h(x) \s$}.
\citet{RSN} shows that $\s ^\top \h(x) \s$ can be obtained by twice differentiating function $\lambda \to f(x + \s \lambda)$ at cost of $\tau$ times of evaluating function $f(x+\s \lambda)$ by using reverse accumulation techniques \citep{christianson1992automatic, gower2012new}. This results in cost of $\cO \left(d\tau^2\right)$ arithmetic operations and if $\tau=1$, then the cost is even $\cO \left(1\right)$.\\

 We can define norms based on a symmetric positive definite matrix $\mathbf H \in \mathbb{R}^{d \times d}$:
\begin{equation}
    \norm{x}_ {\mathbf H} \eqdef  \la \mathbf Hx,x\ra^{1/2}, \, x \in \mathbb{R}, \qquad \norm{g}_{\mathbf H}^{\ast}\eqdef\la g,\mathbf H^{-1}g\ra^{1/2},  \, g \in \mathbb{R}.
\end{equation}
As a special case $\mathbf H= \mI$, we get $l_2$ norm $\norm{x}_\mI = \la x,x\ra^{1/2}$. We will be using local Hessian norm $\mathbf H = \nabla^2 f(x)$, with shorthand notation
\begin{equation} 
{\norm{h}_x \eqdef \la \nabla^2 f(x) h,h\ra^{1/2}, \, h \in \mathbb{R}^d,} \qquad{\norm{g}_{x}^{\ast} \eqdef \la g,\nabla^2 f(x)^{-1}g\ra^{1/2},  \, g \in \mathbb{R}^d.}
\end{equation}
Will be restricting iteration steps to subspaces, for simplicity we use shorthand notation {$\normMS h x = \normM h {\hS(x)}.$}

\section{Algorithm}
\subsection{Three faces of the algorithm}
Our algorithm combines the best of three worlds (\Cref{tab:three_ways}) and we can write it in three different ways.

\begin{theorem} 
\label{th:three}
    If $\g(x^k) \in \Range{\h(x^k)}$\footnote{$\Range{\mathcal A}$ denotes column space of the matrix $\mathcal A$.}, then the update rules \eqref{eq:sgn_reg}, \eqref{eq:sgn_aicn}, and \eqref{eq:sgn_sap} are equivalent.
    \begin{align}
    \text{\Rnewton{} step:} \quad x^{k+1} &= x^k + \sk\argmin_{h\in \mathbb{R}^d} \modelSk{x^k, h}, \label{eq:sgn_reg}\\
    \text{\Dnewton{} step:} \quad x^{k+1} &= x^k - \als k \sk [\hSk(x^k)]^\dagger \gSk(x^k), \label{eq:sgn_aicn}\\
    \text{\Sap{} step:} \quad x^{k+1} &= x^k - \als k \pk k [\h(x^k)]^\dagger \g(x^k), \label{eq:sgn_sap}
    \end{align}
    where $\pk k$ is a projection matrix onto $\Range{\sk}$ with respect to norm $\normM \cdot {x_k}$ (defined in equation \ref{def:px}),
    \begin{eqnarray}
        \modelS {x,h} &\eqdef& f(x) + \la \g(x), \s h \ra + \frac 12 \normsM {\s h} x + \frac {\Lalg}6 \normM {\s h} x ^3 \nonumber \\
        &=& f(x) + \la \gS(x), h \ra + \frac 12 \normsMS {h} x + \frac {\Lalg}6 \normMS {h} x ^3, \label{eq:sgn_Ts} \\
        \als k &\eqdef& \frac{-1 +\sqrt{1+2 \Lalg \normMSd{\gS(x^k)} {x^k} }}{\Lalg \normMSd {\gS(x^k)} {x^k} }. \label{eq:sgn_alpha}
    \end{eqnarray}
    We call this algorithm Sketchy Global Newton, \sgn{}, it is formalized as \Cref{alg:sgn}.
\end{theorem}
Notice that $\als k \in (0,1]$ and limit cases 
\begin{equation}
\als k \xrightarrow{\Lalg \normMSdk{\gSk(x^k)} {x^k} \rightarrow 0} 1
\end{equation}
and
\begin{equation}
\als k \xrightarrow{\Lalg \normMSdk{\gSk(x^k)} {x^k} \rightarrow \infty} 0.   
\end{equation}
For \sgn{}, we can easilyt transition between gradients and model differences
{$h^{k} \eqdef x^{k+1}-x^k$} by identities
\begin{equation} \label{eq:transition-primdual}
    h^k \stackrel{\eqref{eq:sgn_aicn}} = -\als k \sk [\hSk(x^k)]^\dagger \gSk(x^k), \qquad \normM {h^k}{x^k} =\als k \normMSdk {\gSk(x^k)}{x^k}.
\end{equation}

\subsection{Invariance to affine transformations}
We will use self-concordance in the range of $\s$ and a slightly stronger version, semi-strong self-concordance, introduced in \citep{hanzely2022damped}.
We will use a projection matrix on subspaces $\s$ with respect to to local norms $\normM \cdot x$. Denote
\begin{equation} \label{def:px}
    \p \eqdef \s \left( \st \h (x) \s \right)^\dagger \st \h(x).
\end{equation}

\begin{lemma} [\citep{gower2020variance}] \label{le:sketch_equiv}
    This $\p$ is a projection onto $\Range{\s}$ with respect to norm $\normM \cdot x$.    
\end{lemma}

We aim \sgn{} to preserve Newton's direction in expectation. In view of \eqref{eq:sgn_sap}, we can see that this holds as long $\s \sim \cD$ is such that its projection is unbiased.
\begin{assumption} \label{as:projection_direction}
    For distribution $\cD$ there exists $\tau>0$, so that
    \begin{equation}
        \mathbb E_{\s \sim \cD} \left[ \p \right] = \afrac \tau d \mI.
    \end{equation}
\end{assumption}
\begin{lemma} \label{le:tau_exp}
Assumption~\ref{as:projection_direction} implies $\mathbb E_{\s \sim \cD} \left[ \tau(\s) \right] = \tau$.
\end{lemma}

Note that Assumption~\ref{as:projection_direction} is formulated in the local norm, so it might seem restrictive. Next lemma demonstrates that such sketching matrices can be obtained from sketches with $l_2$--unbiased projection (which were used in \citep{hanzely2020stochastic}).

\begin{lemma} [Construction of sketch matrix $\s$] 
\label{le:setting_s}
    If we have a sketch matrix distribution $\tilde {\mathcal D} $ so that a projection on $\Range{\mathbf M}, \mathbf M \sim \mathcal D$ is unbiased in $l_2$ norms,
    \begin{equation}
        \mathbb E_{\mathbf M \sim \tilde {\cD} } \left[ \mathbf M ^ \top \left( \mathbf M ^ \top \mathbf M \right) ^ \dagger \mathbf M \right] = \afrac \tau d \mI,
    \end{equation}
    then distribution $\cD$ of $\s$ defined as $\s^\top \eqdef \mathbf M \left[ \h(x) \right] ^{-1/2}$ (for $\mathbf M \sim \tilde {\cD}$) satisfy 
    \begin{equation}
        \mathbb E_{\s \sim \cD} \left[ \p \right] = \afrac \tau d \mI.
    \end{equation}
\end{lemma}

\subsection{Insights into theory}
Before we present the main convergence results, we are going to showcase geometrically understandable auxiliary lemmas. Starting with contractive properties of a projection matrix $\p$.
\begin{lemma}[Contractivness of projection matrix $\p$] \label{le:projection_contract}
    For any $g,h \in \mathbb{R}^d$ we have
    \begin{eqnarray}
        \E{\normsM{\p h} x } 
        &=&  h ^\top \h(x) \E \p h
        \stackrel{As. \ref{as:projection_direction}} = \frac \tau d \normsM h x, \label{eq:proj_h}\\
        \E{\normsMd{\p g} x }
        &=&  g ^\top \E \p [\h(x)]^\dagger g
        \stackrel{As. \ref{as:projection_direction}} = \frac \tau d \normsMd g x, \label{eq:proj_g}\\
        \normsM {\p h} x &\leq&  \normsM {\p h} x + \normsM {(\mI - \p ) h} x = \normsM h x,\\
        \E{\normM{\p h} x ^3} &\leq& \E{ \normM h x \cdot \normsM{\p h} x } = \normM h x \E{\normsM{\p h} x } \stackrel{As. \ref{as:projection_direction}}= \frac \tau d \normM h x ^3.
    \end{eqnarray}
\end{lemma}
Now we show a key idea from Regularized Newton methods: that $\modelS{x,h}$ upper bounds loss function and minimizing it in $h$ decreases the function value.

\begin{proposition}[Lemma 2 in \citep{hanzely2022damped}]  \label{pr:semistrong_stoch}
    For $\Lsemi$--semi-strong self-concordant $f:\R^d \to \R$, and any $x \in \mathbb{R}^d, h \in \mathbb{R}^{\tau(\s)}$, sketches $\s \in \mathbb{R}^{d \times \tau(\s)}$ 
    and $x_+ \eqdef x+ \s h$ it holds
    \begin{gather}
        \left \vert f(x_+) - f(x) - \la \g(x), \s h \ra - \afrac 1 2 \normsM {\s h} x \right \vert \leq \afrac {\Lsemi} 6 \normM {\s h} x ^3, \label{eq:semistrong_approx} \\
        f(x_+) \leq \modelS{x,h},
    \end{gather}
    hence for $h^* \eqdef \argmin_{h \in \mathbb{R}^{\tau(\s)}} \modelS {x,h}$ and corresponding $ x_+ $, loss decreases,
    \begin{equation}
        f(x_+) \leq \modelS{x,h^*} = \min_{h \in \tau(\s)} \modelS {x,h} \leq \modelS{x,0} = f(x).
    \end{equation}
\end{proposition}

Finally, we show one step decrease in local sketched norms.
\begin{lemma} \label{le:one_step_dec}
    \sgn{} \eqref{eq:sgn_aicn} decreases loss of $\Ls$--self-concordant function $f:\R^d\to\R$ as
    \begin{equation}
        f(x^k)-f(x^{k+1}) \geq \left(2 \max \left \{\sqrt{\Lalg \normMSdk {\gSk(x^k)} {x^k} } ,2 \right\} \right)^{-1} \normsMSdk {\gSk(x^k)} {x^k}. \label{eq:one_step_dec} 
    \end{equation}
\end{lemma}

\section{Main results}
Now we are ready to present the main convergence results. Firstly, \Cref{sec:conv} presents the global \textbf{$\okd$} convergence rate for semi-strong self-concordant functions. Secondly, \Cref{sec:lin_loc} demonstrates a local linear rate that is independent of problem conditioning. Thirdly, \Cref{sec:lin_glob} shows the global linear convergence rate to a neighborhood of the solution under relative convexity assumption. Finally, \Cref{sec:limits} argues the optimality of those rates as \\ \sap{} methods cannot achieve superlinear rate. 

\subsection{Global convex $\okd$ convergence} \label{sec:conv}
Denote initial level set 
\begin{equation}
    \level \eqdef \left\{ x\in\mathbb{R}^d: f(x) \leq f(x^0) \right\}.
\end{equation}
Previous lemmas imply that iterates of \sgn{} stay in $\level$, $x^k \in \level \, \forall k\in \N$. Denote its diameter 
\begin{equation}
    R \eqdef \sup_{x,y \in \level} \normM {x-y}x.
\end{equation}
We will present a key decrease lemma for convex setup, and the global convergence rate theorem.

\begin{lemma} \label{le:global_step}
    Fix any $y \in \mathbb{R}^d$. Let the function $f:\R^d\to\R$ be $\Lsemi$--semi-strong self-concordant and sketch matrices $\sk \sim \cD$ have unbiased projection matrix, Assumption~\ref{as:projection_direction}. Then \sgn{} has decrease    
    \begin{equation}
        \E{f(x^{k+1} | x^k} \leq  \left(1- \afrac \tau d \right) f(x^k) + \afrac \tau d f(y) + \afrac \tau d \afrac {\max \Lalg + \Lsemi}  6 \normM{y-x^k} {x^k} ^3 .
    \end{equation}
\end{lemma}

\begin{theorem} \label{th:global_convergence}
    For $\Lsemi$--semi-strongly concordant function $f:\R^d\to\R$ with finite diameter of initial level set $\level$, $R<\infty$ and sketching matrices with Assumption~\ref{as:projection_direction}, \sgn{} has following global convergence rate
    \begin{align}        
        \E{f(x^k) -\fopt} 
        \leq  \afrac {4 d^3 (f(x^0)-\fopt)}{\tau^3 k^3} + \afrac {9(\max \Lalg + \Lsemi) d^2 R^3} {2 \tau^2 k^2} = \cO(k^{-2}).
    \end{align}
\end{theorem}

\subsection{Fast linear convergence}
Without further due, we can state the fast local linear convergence theorem.
\label{sec:lin_loc}
\begin{theorem}\label{th:local_linear}
    Let function $f:\R^d\to\R$ be $\Ls$--self-concordant in subspaces $\s \sim \cD$ and expected projection matrix be unbiased (Assumption~\ref{as:projection_direction}).
    For iterates of \sgn{} $x^0, \dots, x^k$ such that $\normMSdk {\gSk(x^k)} {x^k} \leq \frac 4 {L_{\sk}}$, we have local linear convergence rate
    \begin{equation}
        \E{f(x^{k})-\fopt)} \leq \left( 1- \afrac \tau {4d} \right)^k (f(x^0)-\fopt),
    \end{equation}
    and the local complexity of \sgn{} is independent on the conditioning, $\mathcal O \left( \frac d \tau \log \frac 1 \varepsilon \right).$
\end{theorem}

\subsection{Global linear convergence} \label{sec:lin_glob}
Our last convergence result is a global linear rate under relative smoothness in subspaces $\s$ and relative convexity. We describe intuition and present rates.
\begin{definition} \label{def:rel}
    We call relative convexity, and relative smoothness in subspace $\s$ positive constants $\murel, \Lrel _\s$ for which following inequalities hold $\forall x, y \in \mathbb \R^d$ and $y_\s = x+\s h$ for $h \in \mathbb{R}^{\tau(\s)}$:
    \begin{eqnarray}
        f(y) &\geq& f(x) + \la \g(x), y-x \ra + \afrac {\murel}2 \normsM{y-x} x, \label{eq:rel_conv}\\
        f(y_\s) &\leq& f(x) + \la \gS(x), y_\s - x \ra + \afrac {\Lrel _\s} {2}  \normsMS {y_\s - x}{x} \label{eq:rels_smooth}.
    \end{eqnarray}
\end{definition}

Note that \citet{RSN} shows that updates $x_+ = x + \s h,$ where $h$ is a minimizer of RHS of \eqref{eq:rels_smooth} converge linearly and can be written as Newton method with stepsize $\afrac 1 \Lrel$. Conversely, our stepsize $\als k$ varies, \eqref{eq:sgn_alpha}, so this result is not applicable to us. However, a small tweak will do the trick. Observe following:

\begin{itemize}[leftmargin=*] 
\setlength\itemsep{0.2em}
\item We already have fast local convergence (\Cref{th:local_linear}), so we just need to show linear convergence for points $x^k$ such that $\normMSdk {\gSk(x^k)} {x^k} \geq \frac 4 {L_{\sk}}$.

\item We can follow global linear proof of \rsn{} as long as stepsize $\als k$ is bounded from bellow and smaller than $\frac 1 {\Lrel_{\sk}}$. We are going to find $\Lalg$ that guarantees it.

\item Stepsize $\als k$ of \sgn{}, \eqref{eq:sgn_alpha}, is inversely proportional to $\Lalg \normMd {\gSk(x)}{x^k}$. Increasing $\Lalg$ decreases the convergence neighborhood arbitrarily. If we express stepsize bound requirements in terms of $\Lalg$, we might obtain global linear rate.

\item Constant $\Lalg$ is an estimate of semi-strong self-concordance constant $\Lsemi$, which is fundamentally different from relatvie smoothness constant $\Lrel$. 
Nevertheless, from regularized Newton method perspective \eqref{eq:sgn_reg}, \sgn{} can be written as
\begin{equation}
    x_+ = x + \s \argmin_{h \in \mathbb{R}^{\tau(\s)}} \left(f(x) +\la \gS(x), h \ra + \afrac 12 \left( 1 + \afrac {\Lalg}3 \normMS {h} x \right)  \normsMS {h} x \right). 
\end{equation}
If $1 + \afrac {\Lalg}3 \normMS {h} x \geq \Lrel _\s$, then \eqref{eq:sgn_Ts} upperbounds on RHS of \eqref{eq:rels_smooth}, and hence next iterate of \sgn{} really minimizes function upperbound. We can express $\Lalg$ using $\normMS {h} x = \alpha \normMSd {g} x$ as:
\begin{eqnarray}
    1 + \afrac {\Lalg}3 \normMS {h} x \geq  \Lrel _\s
    &\Leftrightarrow& \Lalg \geq  \afrac {3(\Lrel _\s-1)}{\alpha \normMSd {\gS(x)} x}\nonumber \\
    &\Leftrightarrow& 1 \geq  \afrac {3 (\Lrel _\s-1)}{-1 + \sqrt{1+2\Lalg \normMSd {\gS(x)} x}} \nonumber\\
    &\Leftrightarrow& \Lalg \geq \afrac 32 \afrac  {(\Lrel _\s -1) (3 \Lrel _\s -1)}{\normMSd {\gS(x)} x}. 
\end{eqnarray}

Thus choosing $\Lalg \geq \sup_{\s} \afrac 98 \Ls \Lrel _\s ^2 > \sup_\s \afrac 38 { \Ls (\Lrel _\s -1) (3 \Lrel _\s -1) }$ implies boundness of stepsize $\alpha$ for points $x$ such that $\normMSd {\gS(x)} x \geq \afrac 4 \Ls.$
\end{itemize}

Finally, we are ready to present global linear convergence rate. That depends on the conditioning of the expected projection matrix $\p$,

    \begin{eqnarray}
        \hat \p
        &\eqdef& [\h(x)]^{\frac 12} \s \left[\hS(x) \right]^\dagger \s^\top [\h(x)]^{\frac 12}
        = [\h(x)]^{\frac 12} \p [\h(x)]^{\dagger\frac 12}, \, \, \, \, \, \,\\
        \rho(x) &\eqdef& \min_{v \in \Range{\h(x)}} \afrac { v^\top \E{ \alpha \hat \p} v} {\normsM v \mI}
        = [\h(x)]^{\frac 12} \E{\alpha \p} [\h(x)]^{\dagger\frac 12}, \\
        \rho & \eqdef& \min_{x \in \level} \rho(x). \label{eq:rho}
    \end{eqnarray}

We can bound it and get a global convergence rate. 
\begin{theorem} \label{th:global_linear}
    Let $f:\R^d \to \R$ be $\Ls$--relative smooth in subspaces $\s$ and $\murel$--relative convex. Let sampling $\s\sim \cD$ 
    satisfy $\Null{\s^\top \h(x) \s} = \Null{\s}$
    and $\Range{\h(x)} \subset \Range {\bbE_{\s \sim \cD}\left[\s \s ^\top\right] }$.     
    Then $0 < \rho \leq 1$.
    Choose parameter $\Lalg = \sup_{\s \sim \cD} \afrac 98 \Ls \Lrel _\s^2$.
    
    While iterates $x^0, \dots, x^k$ satisfy $\normMSdk {\gSk(x^k)} {x^k} \geq \frac 4 {L_{\sk}}$, then \sgn{} has decrease
    \begin{equation}        
    \E{f(x^{k}) - \fopt}  \leq \left(1 - \afrac 43 \rho \murel \right)^k (f(x^0) - \fopt), 
    \end{equation}
    and global linear $\cO\left( \afrac 1 {\rho \murel} \log \afrac 1 \varepsilon \right)$ convergence.
\end{theorem}

\subsection{Local/linear convergence limit} \label{sec:limits}
Similarly to \ain{} \citep{hanzely2022damped}, we can show one step quadratic decrease of the gradient norm. However, as we perform subspace steps, quadratic decrease is the sketched subspace.
\begin{lemma} \label{le:local_step_subspace}
    For $\Lsemi$--semi-strong self-concordant function $f:\R^d\to\R$ and parameter choice $\Lalg \geq\Lsemi$, one step of \sgn{} has quadratic decrease in $\Range{\sk}$,
    \begin{equation} \label{eq:local_step_subspace}
        \normMSdk {\gSk(x^{k+1})} {x^k} \leq \Lalg \als k^2\normsMSdk{\gSk(x^k) }{x^k}.
    \end{equation}
\end{lemma}
Nevertheless, this is insufficient for superlinear local convergence; we can achieve a linear rate at best.
We can illustrate this on an edge case where $f:\R^d \to \R$ is a quadratic function: self-concordance assumption holds with $\Ls=0$ and as $\als k \xrightarrow {\Ls \rightarrow 0} 1$, \sgn{} stepsize becomes $1$ and \sgn{} simplifies to subspace Newton method. Unfortunately, subspace \newton{} method has just linear local convergence \citep{RSN}.

\section{Experiments}
We support our theory by comparing \sgn{} to \sscn{}.
To match practical considerations of \sscn{} and for the sake of simplicity, we adjust \sgn{} in unfavorable way:
\begin{enumerate}[leftmargin=*]
    \item We choose sketching matrices $\s$ to be unbiased in $l_2$ norms (instead of local hessian norms $\normM \cdot x$ from Assumption~\ref{as:projection_direction}),
    \item To disregard implementation specifics, we report iterations on the $x$--axis. 
Note that \sscn{} needs to use a subsolver (extra line-search) to solve implicit step in each iteration. If naively implemented using matrix inverses, iterations of \sscn{} are $\times \log \frac 1 \varepsilon$ slower.
We chose to didn't report time as this would naturally ask for optimized implementations and experiments on a larger scale -- this was out of the scope of the paper.
\end{enumerate}
Despite simplicity of \sgn{} and unfavourable adjustments, \Cref{fig:vssscn} shows that \sgn{} performs comparably to \sscn{}.

We can point out other properties of \sgn{} based on experiments in literature.

\begin{itemize}[leftmargin=*]
\setlength\itemsep{0em}
 
\item \textbf{Rank of $\s$ and first-order methods:} 
\citet{RSN} showed a detailed comparison of the effect of various ranks of $\s$. Also, \citet{RSN} showed that \rsn{} (fixed-stepsize Newton) is much faster than first-order Accelerated Coordinate Descent (\acd{}) for highly dense problems. For extremely sparse problems, \acd{} has competitive performance. As the stepsize of \sgn{} is increasing while getting close to the solution, we expect similar, if not better results.

\item \textbf{Various sketch distributions}:  
\citet{hanzely2020stochastic} considered various distributions of sketch matrices $\s \sim \cD$. In all of their examples, \sscn{} outperformed \cd{} with uniform or importance sampling and was competitive with \acd{}. As \sgn{} is competitive to \sscn{}, similar results should hold for \sgn{} as well.

\item \textbf{Local norms vs $l_2$ norms:} 
\citet{hanzely2022damped} shows that the optimized implementation of \ain{}  saves time in each iteration over the optimized implementation of \cnewton{}. As \sgn{} and \sscn{} use the same updates (but in subspaces), it indicates that \sgn{} saves time over \sscn{}.
\end{itemize}

\begin{figure}
    \centering
    \includegraphics[width=0.49\textwidth]{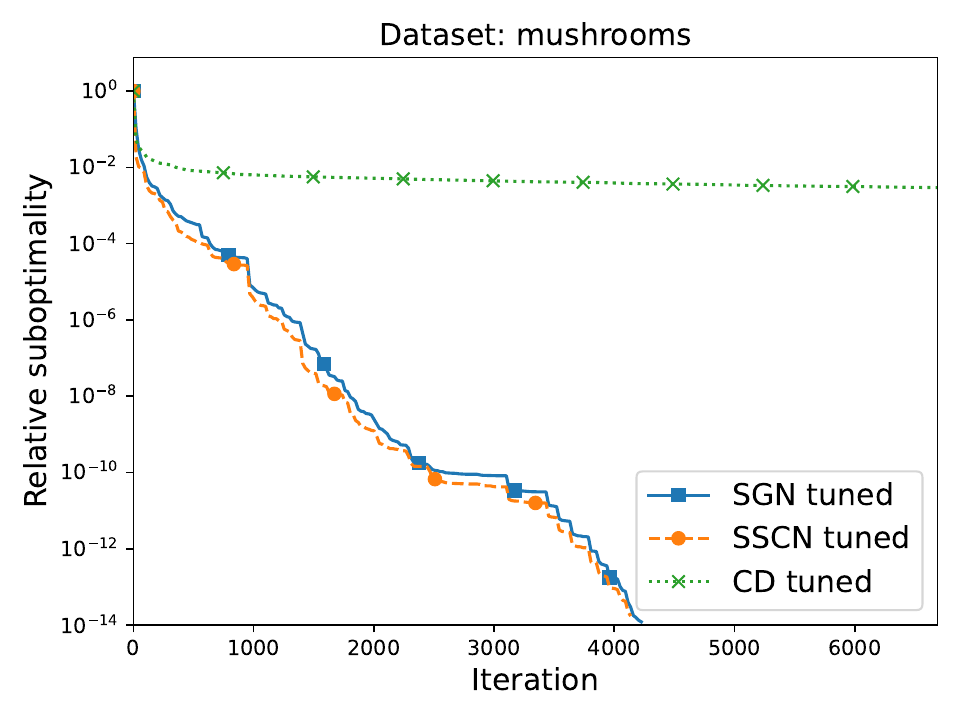}
    \includegraphics[width=0.49\textwidth]{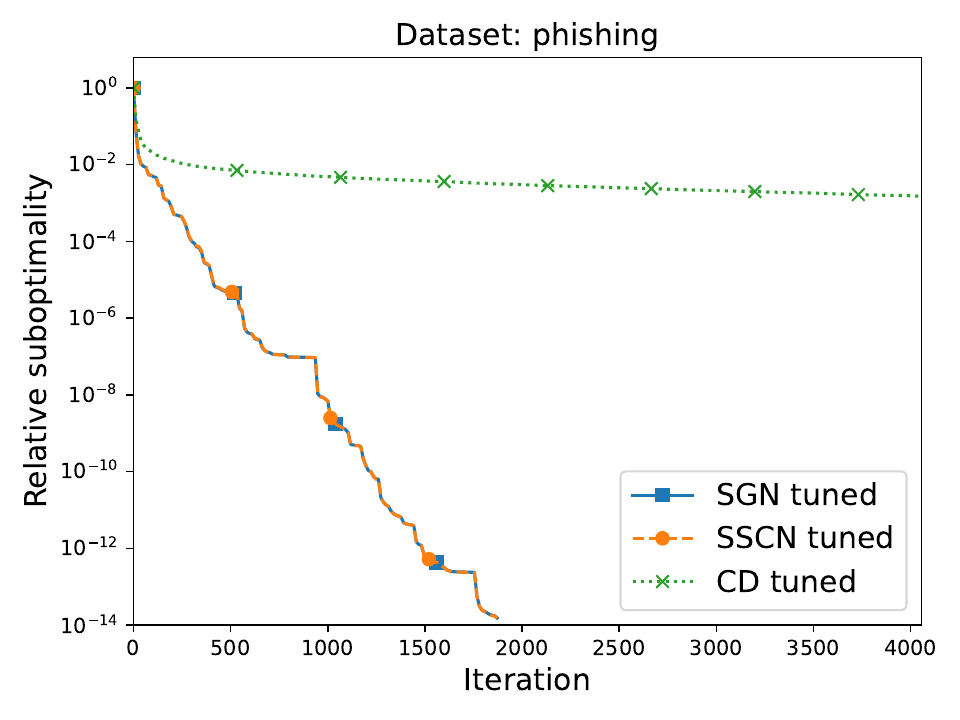}
    \includegraphics[width=0.49\textwidth]{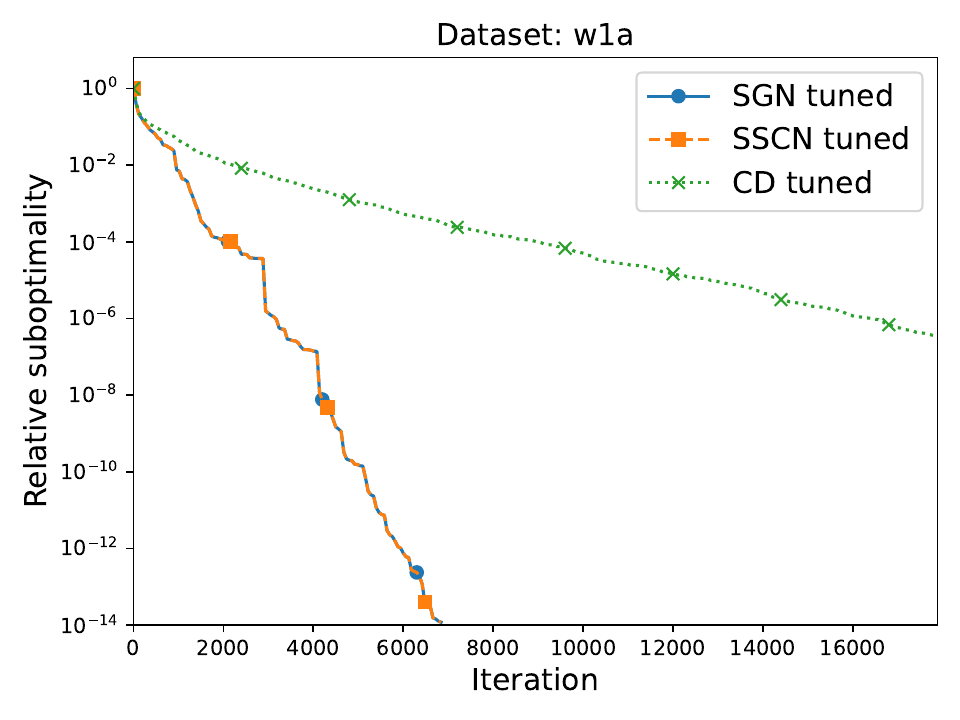}
    \includegraphics[width=0.49\textwidth]{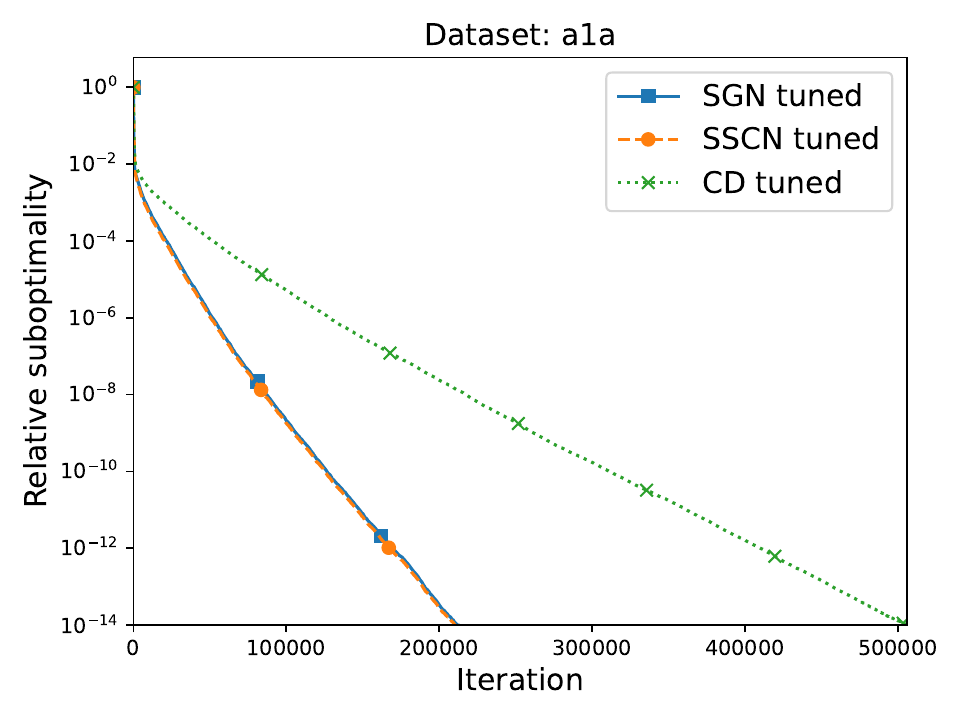}
    \includegraphics[width=0.49\textwidth]{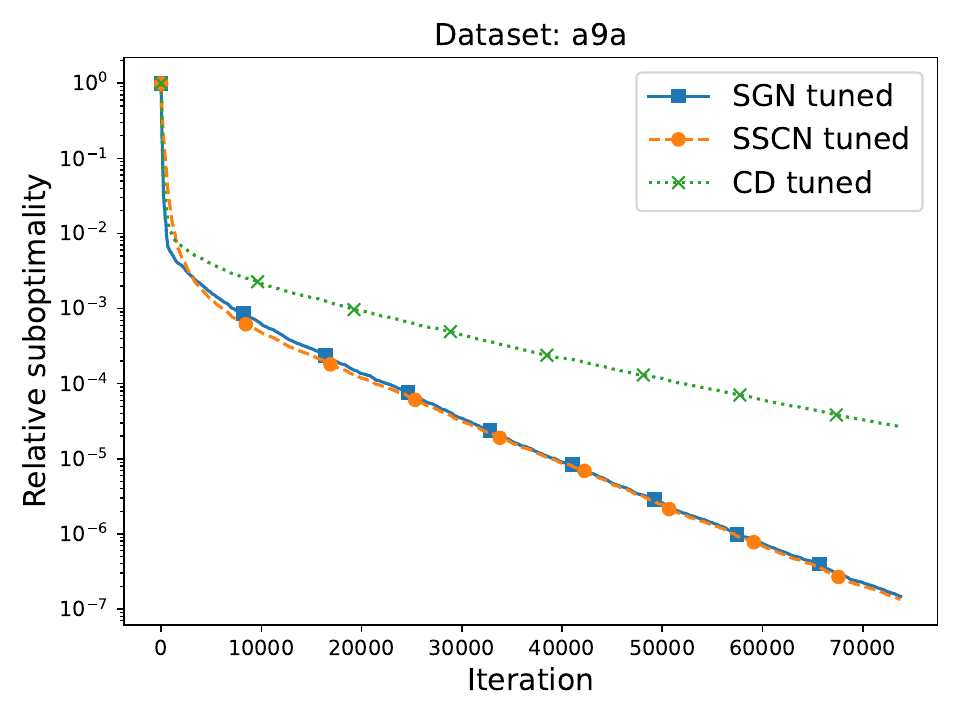}
    \includegraphics[width=0.49\textwidth]{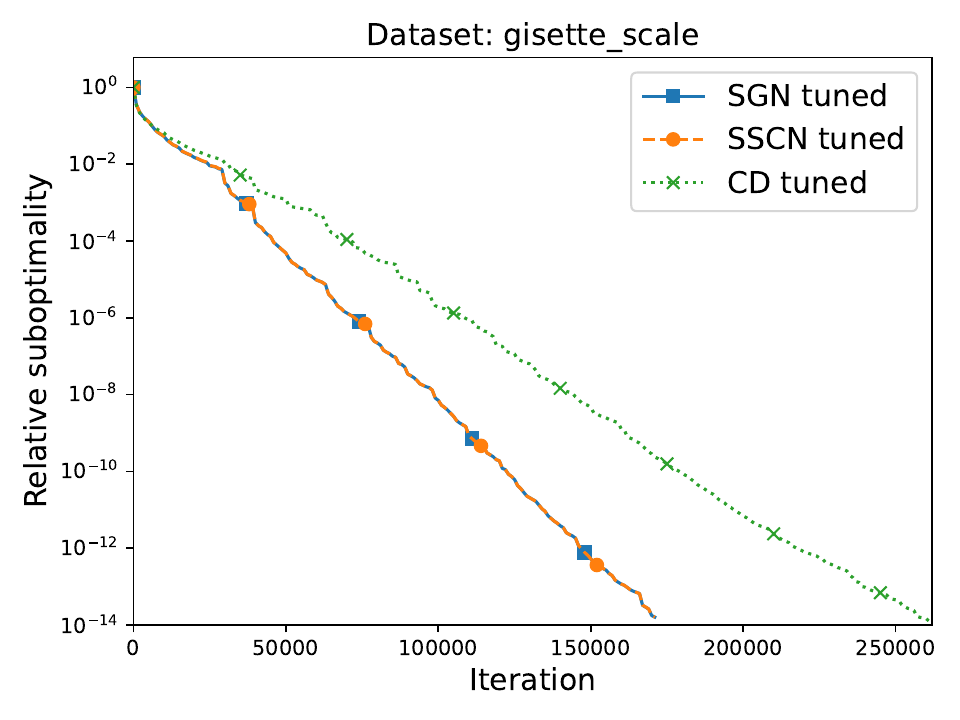}
    \includegraphics[width=0.49\textwidth]{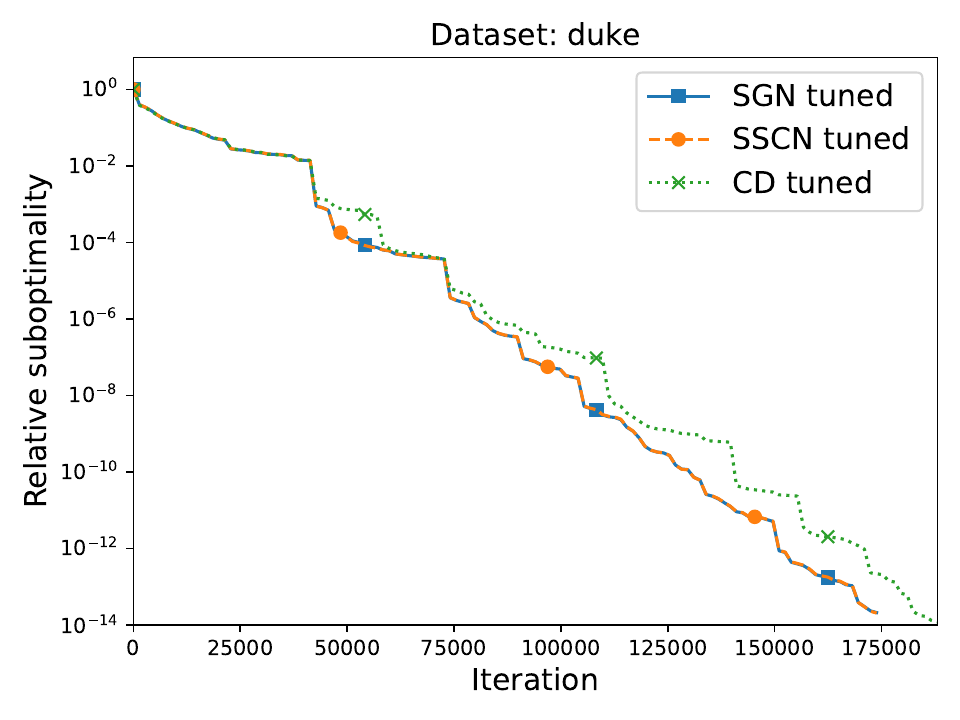}
    \includegraphics[width=0.49\textwidth]{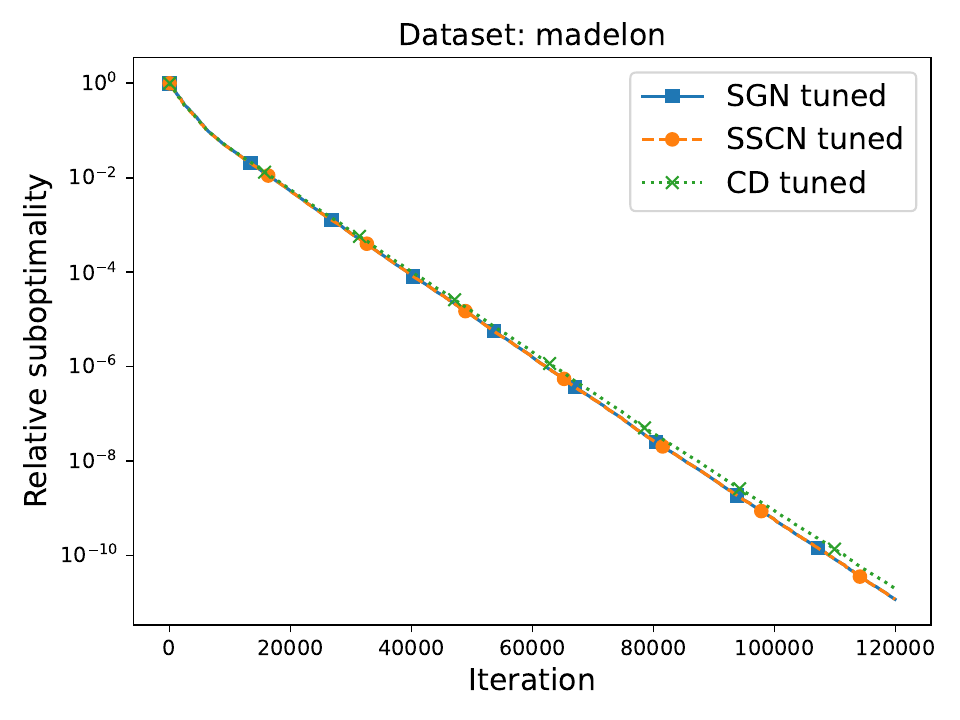}
    \caption[Comparison of low-rank second-order methods with fast convergence guarantees]{Comparison of \sscn{}, \sgn{} and \cd{} on logistic regression on \libsvm{} datasets for sketch matrices $\s$ of rank one. We fine-tune all algorithms for smoothness parameters.}
    \label{fig:vssscn}
\end{figure}

\end{singlespace}
\begin{onehalfspacing}
	\renewcommand*\bibname{\centerline{REFERENCES}} 
    \phantomsection
	\addcontentsline{toc}{chapter}{References}
	\newcommand{\BIBdecl}{\setlength{\itemsep}{0pt}}
		\bibliographystyle{plainnat}
		\bibliography{References}
\end{onehalfspacing}
\appendix
		\newpage
		\begingroup
			\let\clearpage\relax
			\begin{center}
			\vspace*{2\baselineskip}
			{ \textbf{{\large APPENDICES}}} 
			\end{center}
            \begin{singlespace}                
            

\chap{\ref{sec:mixture_optimal}}

\section{Table of frequently used notation}
To enhance the reader's convenience when navigating, we here reiterate our notation:

\begin{table}[ht]
	\setlength\tabcolsep{3pt}
    \centering
	\begin{threeparttable}
		{
			\caption{Summary of frequently used notation in \Cref{sec:mixture_optimal}.}
	          \label{tbl:notation}
			\begin{tabular}{|c|l|c|}
				\hline
				\multicolumn{3}{|c|}{{\bf General} }\\
				\hline
				$F: \R^{nd}\rightarrow \R$ & Global objective & \eqref{eq:main}\\
				$f_i: \R^{n}\rightarrow \R$ & Local loss on $i$--th node & \eqref{eq:main} \\
				$x_i\in \R^d$ & Local model on $i$--th node & \eqref{eq:main} \\
				$x\in \R^{nd}$ & Concatenation of local models $x = [x_1, x_2, \dots, x_n]$ & \eqref{eq:main} \\
				$f: \R^{nd}\rightarrow \R$ & Average loss over nodes $f(x) \eqdef \nicefrac{1}{n}\sum_{i=1}^n f_i(x_i)$ & \eqref{eq:main} \\
				$ \psi : \R^{nd}\rightarrow \R $ & Dissimilarity penalty $ \psi(x) \eqdef \frac{1}{2 n}\sum \limits_{i=1}^n \norm{x_i-\bar{x}}^2$ &   \eqref{eq:main} \\
				$\lambda \geq 0$ & Weight of dissimilarity penalty & \eqref{eq:main}\\
				$\text{Loc}(x_i,i)$ & local oracle: \{ proximal, gradient, summand gradient\} & \Cref{sec:lower} \\
				$\mu \geq 0$& Strong convexity constant of each $f_i$ ($\tilde{f}_{i,j}$)   &  \\
				${L_1} \geq 0$ &  Smoothness constant of each $f_i$  &  \\
				prox & Proximal operator  &  \eqref{eq:pgd} \\
				$m \geq 1$ & Number of local summands at clients $f_i = \frac{1}{m}\sum_{j=1}^m \tilde{f}_{i,j}$  & \Cref{sec:iapgd} \\
				$\tilde{f}_{i,j}:  \R^{n}\rightarrow \R$ & $j$--th summand of $i$--th local loss, $1 \leq j \leq m$ &  \Cref{sec:iapgd} \\
				$\tilde{{L_1}} \geq 0$ & Smoothness constant of each $\tilde{f}_{i,j}$ & \Cref{sec:iapgd}\\
				$\varepsilon \geq 0$ & Precision & \\
				$x^0\in \R^{nd}$ & Algorithm initialization  & \\
				$x^\star\in \R^{nd}$ & Optimal solution of  \eqref{eq:main},  $x^\star = [x_1^\star, x_2^\star, \dots, x_n^\star]$  & \\
				$F^\star \in \R$ & Function value at minimum, $F^\star = F(x^\star)$&  \\
				\hline
				\multicolumn{3}{|c|}{{\bf Algorithms}}\\
				\hline
				\apgdo{}& Accelerated Proximal Gradient Descent (Algorithm~\ref{alg:fista}) & \Cref{sec:apgd_simple}\\
				\apgdt{}& Accelerated Proximal Gradient Descent (Algorithm~\ref{alg:fista_2}) & \Cref{sec:apgd_simple}\\
				\iapgd{}& Inexact Accelerated Proximal Gradient Descent (Alg.~\ref{alg:fista_inex}) &  \Cref{sec:iapgd}\\
				\makecell{\iapgd{} \\ \tt + \agd{}}& \iapgd{} with \agd{} as a local sobsolver &  \Cref{sec:iapgd}\\
				\makecell{\iapgd{} \\ \tt + \katyusha{}{}}&\iapgd{} with \katyusha{} as a local subsolver  &  \Cref{sec:iapgd}\\
				\altsgdp{} & Accelerated Loopless Local Gradient Descent (Alg.~\ref{alg:acc_stoch})& \Cref{sec:al2sgd+} \\
				$p, \rho$  & Probabilities;  parameters of \altsgdp  & \Cref{sec:apgd_missing} \\
				\hline
			\end{tabular}
		}
		
	\end{threeparttable}
\end{table}

\clearpage

\section{Missing parts for Section~\ref{sec:upperbound} \label{sec:apgd_missing}}

In this section, we state the algorithms that were mentioned in the main paper: \apgdo{} as Algorithm~\ref{alg:fista}, \apgdt{} as Algorithm~\ref{alg:fista_2} and \altsgdp{} as Algorithm~\ref{alg:acc_stoch}. Next, we state the convergence rates of \apgdo{},  \apgdt{} as Proposition~\ref{prop:fista} and Proposition~\ref{prop:fista2} respectively. Lastly, we justify~\eqref{eq:pgd_specialized} via Lemma~\ref{lem:mnadjnjks}.

\begin{proposition}\citep{beck2017first}\label{prop:fista}
	Let $\{x^k\}_{k=0}^\infty$ be a sequence of iterates generated by Algorithm~\ref{alg:fista}. Then, we have for all $k\geq 0$:
	\[
	F(x^k) -  F^\star \leq \left( 1- \sqrt{\frac{\mu}{\lambda + \mu}}\right)^k \left(
	F(x^0) - F^\star + \frac{\mu}{2n}\norm{ x^0-x^\star}^2
	\right).
	\]
	
\end{proposition}

\begin{algorithm}[h]
	\caption{\apgdo{}}
	\label{alg:fista}
	\begin{algorithmic}[1]
		\State \textbf{Requires:} Starting point $y^0 = x^0\in\R^{nd}$
		\For{ $k=0,1,2,\ldots$ }
		\State{{\color{blue}Central server}  computes the average $\bar{y}^k = \frac1n \sum_{i=1}^n y^k_i$}
		\State For all {\color{red} clients} $i=1,\dots,n$: 
		\State \quad Solve the regularized local problem $x^{k+1}_i = \argmin_{z\in \R^d} f_i(z) + \frac{\lambda}{2} \norm{ z -\bar{y}^k }^2$  
		\State \quad Take the momentum step $y^{k+1}_i = x^{k+1}_i +\frac{ \sqrt{\lambda}- \sqrt{\mu}}{ \sqrt{\lambda}+ \sqrt{\mu}} ( x^{k+1}_i - x^{k}_i ) $
		\EndFor
	\end{algorithmic}
\end{algorithm}

\begin{proposition}\citep{beck2017first}\label{prop:fista2}
	Let $\{x^k\}_{k=0}^\infty$ be a sequence of iterates generated by Algorithm~\ref{alg:fista_2}. Then, we have for all $k\geq 0$:
	\[
	F(x^k) -  F^\star \leq \left( 1- \sqrt{\frac{\mu}{{L_1} + \mu}}\right)^k \left(
	F(x^0) - F^\star + \frac{\mu}{2n}\norm{ x^0-x^\star}^2
	\right).
	\]
	
\end{proposition}

\begin{algorithm}[h]
	\caption{\apgdt{}}
	\label{alg:fista_2}
	\begin{algorithmic}[1]
		\State \textbf{Requires:} Starting point $y^0 = x^0\in\R^{nd}$
		\For{ $k=0,1,2,\ldots$ }
		\State For all {\color{red} clients} $i=1,\dots,n$: 
		\State \quad Take a local gradient step $\tilde{y}^k_i = y^k_i -\frac1{L_1} \nabla f_i(y^k_i)$
		\State{{\color{blue}Central server} computes the average $\bar{y}^k = \frac1n \sum_{i=1}^n \tilde{y}^k_i$}
		\State For all {\color{red} clients} $i=1,\dots,n$: 
		\State \quad  Take a prox step w.r.t $\lambda \psi$: $x^{k+1}_i = \frac{{L_1}\tilde{y}^k_i + \lambda \bar{y}^k}{{L_1}+\lambda}$  
		\State \quad  Take the momentum step $y^{k+1}_i = x^{k+1}_i +\frac{ \sqrt{\frac{{L_1}}{\mu}}-1}{ \sqrt{\frac{{L_1}}{\mu}}+1} ( x^{k+1}_i - x^{k}_i ) $
		\EndFor
	\end{algorithmic}
\end{algorithm}

\begin{lemma}\label{lem:mnadjnjks}
	Let \begin{equation} \label{eq:pgd_}
		x^{k+1} = \prox_{\frac{1}{L_h }\phi}\left( x^k - \frac{1}{L_h} \nabla h(x^k)\right),
	\end{equation}
	for  $h (x)\eqdef \lambda \psi(x) + \frac{\mu}{2n}\norm{ x}^2$ and $\phi(x) \eqdef f(x) -  \frac{\mu}{2n}\norm{ x}^2$. Then, we have
	\[
	x^{k+1}_i  = \prox_{\frac{1}{\lambda}f_i}(\bar{x}^k).
	\]
	Further, the iteration complexity of the above process is $\cO\left(\frac{\lambda}{\mu} \log\frac1\varepsilon\right)$.
\end{lemma}
\begin{proof}
	Since function $\psi$ is $\frac1n$--smooth and $(\nabla \psi(x))_i = \frac1n(x_i -\bar{x})$~\citep{hanzely2020federated}, we have $L_h = \frac{\lambda +\mu}{n}, (\nabla h(x))_i =  \frac{\lambda}{n}(x_i -\bar{x}) + \frac{\mu}{n} x_i$ and thus 
	\begin{align*}
		x^{k+1}_i &= \argmin_{z\in \R^d} \frac1n f_i(z) -  \frac{\mu}{2n}\norm{ z}^2 + \frac{\lambda + \mu}{2n} \norm{ z - \left(x_i^k - \frac{n}{\lambda+\mu} \left(\frac{\lambda}{n}(x_i^k-\bar{x}^k )+ \frac{\mu}{n}x^k_i \right) \right) }^2 \\
		&
		=
		\argmin_{z\in \R^d}f_i(z) -  \frac{\mu}{2}\norm{ z}^2
		+ \frac{\lambda + \mu}{2} \norm{ z - \frac{\lambda}{\lambda+\mu} \bar{x}^k  }^2
		\\
		&=
		\argmin_{z\in \R^d}f_i(z) + \frac{\lambda }{2} \norm{ z -  \bar{x}^k  }^2 = \prox_{\frac{1}{\lambda}f_i}(\bar{x}^k).
	\end{align*}
	Let us now discuss the convergence rate. Given that function $h$ is $\mu_{h}$--strongly convex, iteration complexity of~\eqref{eq:pgd} to reach $\varepsilon$--suboptimality is $\cO\left(\frac{L_h}{\mu_{h}} \log\frac1\varepsilon\right)$. Since $L_h = \frac{\lambda + \mu}{n}$ (note that $\psi$ is $\frac{1}n$ smooth~\citep{hanzely2020federated}) and $\mu_h = \frac{\mu}{n}$, the iteration complexity of the process~\eqref{eq:pgd_specialized} becomes $\cO\left(\frac{\lambda}{\mu} \log\frac1\varepsilon\right)$, as desired.
\end{proof}

\begin{algorithm}[h]
	\caption{\altsgdp{}}
	\label{alg:acc_stoch}
	\begin{algorithmic}[1]
		\State \textbf{Requires:} $0< \theta_1, \theta_2 <1$, $\eta, \beta , \gamma > 0$, $\probx, \proby \in (0,1)$, $y^0 = z^0 = x^0 =w^0\in \R^{nd}$
		\For{$k=0,1,2,\ldots$}
		\State For all {\color{red} clients} $i=1,\dots,n$: 
		\State \hskip .3cm $x^k_i = \theta_1 z^k_i + \theta_2 w^k_i + ( 1 -\theta_1 -\theta_2) y^k_i$
		\State $\xi = 1$ with probability $\proby$ and $0$ with probability $1-\proby$
		\If {$\xi=0$}
		\State For all {\color{red} clients} $i=1,\dots,n$: 
		\State \hskip .3cm $g^k_\tR = \frac{1}{\TR(1-\proby)} \left(\nabla \flocc_{i,j}(x^k_{\tR})- \nabla \flocc_{i,j}(w^k_{\tR})\right) + \frac1n \nabla f_i(w_i^k)  +   \frac{\lambda}{n} (w_i^k - \bar{w}^k) $
		\State \hskip .3cm $y^{k+1}_\tR =  x^k_i - \eta g_i^k$
		\Else
		\State{\color{blue}Central server} computes the average $\bar{x}^k = \frac{1}{n}\sum_{i=1}^n x_i^k$ and sends it back to the clients
		\State For all {\color{red} clients} $i=1,\dots,n$: 
		\State \hskip .3cm $g^k_\tR=   \frac{\lambda }{\TR \proby} (x^k_{\tR} - \bar{x}^k) - \frac{(\proby^{-1} -1) \lambda}{ \TR} (w_i^k - \bar{w}^k) + \frac{1}{\TR} \nabla f_i(w^k_i)$
		\State \hskip .3cm Set $y_{\tR}^{k+1} =x_{\tR}^k  - \eta g_i^k $
		\EndIf
		\State For all {\color{red} clients} $i=1,\dots,n$: 
		\State \hskip .3cm $z^{k+1}_i = \beta z^k_i + (1-\beta)x^k_i + \frac{\gamma}{\eta}(y^{k+1}_i - x^k_i)$
		\State    $\xi' = 1$ with probability $\probx$ and $0$ with probability $1-\probx$
		\If {$\xi'=0$}
		\State For all {\color{red} clients} $i=1,\dots,n$: 
		\State \hskip .3cm $w^{k+1}_i = w^{k}_i$
		\Else
		\State For all {\color{red} clients} $i=1,\dots,n$: 
		\State \hskip .3cm $w^{k+1}_i = y^{k+1}_i$
		\State \hskip .3cm Evaluate and store $\nabla f_i(w^{k+1}_i) $
		\State {\color{blue}Central server} computes the average $\bar{w}^{k+1} = \frac{1}{n}\sum_{i=1}^n w_i^{k+1}$ and sends it back to the clients
		\EndIf
		\EndFor
	\end{algorithmic}
\end{algorithm}

\clearpage

\section{Proof of Theorem~\ref{thm:lb}}

In this section, we provide the proof of the Theorem \ref{thm:lb}. In order to do so, we construct a set of function $f_1, f_2, \dots, f_n$ such that for any algorithm satisfying Assumption~\ref{as:oracle} and the number of the iterations $k$, one must have $\norm{ x^{k} - x^\star}^2  \geq \frac 1 2 \left(1-10\max\left\{ \sqrt{\frac{\mu}{\lambda}}, \sqrt{\frac{\mu}{{L_1}-\mu}}\right \} \right)^{\comm(k)+1} \norm{ x^0 - x^\star}^2.$

Without loss of generality, we consider $x^0=0\in \R^{dn}$. The rationale behind our proof goes as follows: we show that the $nd$--dimensional vector $x^k$ has ``a lot of'' zero elements while $x^\star$ does not,  and hence we might lower bound $\norm{x^k-x^\star}^2$ by $\sum_{j: (x^k)_j=0} (x^\star)_j^2$, which will be large enough. As the main idea of the proof is given, let us introduce our construction.

Let $d=2T$ for some large $T$ and define the local objectives as follows for even $n$
{
	\footnotesize
	\begin{eqnarray*}
		f_1(y) = f_2(y) = \dots = f_{n/2}(y) &\eqdef&  \frac{\mu}{2}\norm{ y }^2  + ay_1  + \frac{\lambda}{2}c\left( \sum_{i=1}^{ T-1} (y_{2i} - y_{2i +1})^2  \right)+ \frac{\lambda b}{2}  y_{2T}^2  \\
		f_{n/2+1}(y) =  f_{n/2+2}(y) =\dots = f_{n}(y)&\eqdef& \frac{\mu}{2}\norm{ y }^2  +\frac{ \lambda}{2} c \left(\sum_{i=0}^{ T-1} (y_{2i + 1} - y_{2i+2})^2 \right)
	\end{eqnarray*}
}
and as
{
	\footnotesize
	\begin{eqnarray*}
		f_1(y) = \dots = f_{\nhalf}(y) &\eqdef&  \frac{\nhalf+1}{\nhalf}\frac{\mu}{2}\norm{ y }^2  + ay_1  + \frac{\lambda}{2}\frac{\nhalf+1}{\nhalf}c\left( \sum_{i=1}^{ T-1} (y_{2i} - y_{2i +1})^2  \right)+ \frac{\lambda b}{2}  y_{2T}^2  \\
		f_{\nhalf+1}(y) =  \dots = f_{n}(y)&\eqdef& \frac{\mu}{2}\norm{ y }^2  +\frac{ \lambda}{2} c \left(\sum_{i=0}^{ T-1} (y_{2i + 1} - y_{2i+2})^2 \right)
	\end{eqnarray*}
}
for $n=2\nhalf+1, \nhalf\geq 1$. Note that the smoothness of the objective is now effectively controlled by parameter $c$.

With such definition of functions $f_i(x_i)$, our objective is quadratic and can be written as
\begin{equation}\label{eq:dhjabhsudga}
	\frac{n}{\lambda}F(x) = \frac12x^\top \mM x + \frac{a}{\lambda}x_1,
\end{equation} where $\mM$ is matrix dependent on parity of $n$. For even $n$, we have
\begin{align*}
	\mM& \eqdef \left( \mI - \frac1n \ones \ones^\top \right) \otimes \mI +\frac{ \mu}{\lambda} \mI  + \begin{pmatrix}
		\mM_1 & 0 \\
		0 & \mM_2
	\end{pmatrix},  \text{ where}
	\\
	\mM_1 & \eqdef   \mI \otimes
	\begin{pmatrix}
		0 & 0 & 0& \dots &0  \\
		0 & \mat & 0 &\ddots  &\vdots\\
		0 & 0 & \mat &\ddots  &\vdots \\
		\vdots & \ddots & \ddots & \ddots& \vdots \\
		0 & \dots & \dots & \dots & b
	\end{pmatrix} \text{ and}
	\\
	\mM_2 & \eqdef  \mI \otimes
	\begin{pmatrix}
		\mat& 0 &  \dots  \\
		0 & \mat &\dots \\
		\vdots & \vdots &\ddots
	\end{pmatrix}.
\end{align*}

When $n$ is odd, we have
\begin{align*}
	\mM& \eqdef \left( \mI - \frac1n \ones \ones^\top \right) \otimes \mI +\frac{ \mu}{\lambda} \mI  + \begin{pmatrix}
		\mM_1 +\frac{\mu}{\nhalf\lambda} \mI& 0 \\
		0 & \mM_2
	\end{pmatrix}, \text{ where}
	\\
	\mM_1 & \eqdef   \mI \otimes
	\begin{pmatrix}
		0 & 0 & 0& \dots &0  \\
		0 & \mmat & 0 &\ddots  &\vdots\\
		0 & 0 & \mmat &\ddots  &\vdots \\
		\vdots & \ddots & \ddots & \ddots& \vdots \\
		0 & \dots & \dots & \dots & b
	\end{pmatrix} \text{ and}
	\\
	\mM_2 & \eqdef \mI \otimes
	\begin{pmatrix}
		\mat& 0 &  \dots  \\
		0 & \mat &\dots \\
		\vdots & \vdots &\ddots
	\end{pmatrix}.
\end{align*}

Note that our functions $f_k$ depends on parameters $a\in \R,b,c\in \R_+$. We will choose these parameters later in the way that the optimal solution can be obtained easily.

Now let's discuss optimal model for the objective. Since the the objective is strongly convex, the optimum $x^\star$ is unique. Let us find what it is. For the sake of simplicity, denote $y^\star \eqdef x_1^\star, z^\star =x^\star_{n}$. Due to the symmetry, we must have
\[y^\star = x_2^\star = \dots x_{n/2}^\star , \quad  z^\star = x_{n/2+1} = x_{n/2+2} = \dots x_{n-1}^\star \qquad \text{for even }n
\]
and
\[y^\star =x_2^\star = \dots x_{\nhalf}^\star,
\quad z^\star = x_{\nhalf+1} = \dots = x_{n-1}^\star \qquad \text{for odd }n.
\]
Now we use the following lemma to express elements of $y^\star, z^\star$ recursively.

\begin{lemma} \label{le:one_step_even}
	Let \[
	w_i \eqdef
	\begin{cases}
		\begin{pmatrix}
			z^\star_{i}\\
			y^\star_{i}
		\end{pmatrix} & \text{if } i \text{ is even} \\
		\begin{pmatrix}
			y^\star_{i}\\
			z^\star_{i}
		\end{pmatrix} & \text{if } i \text{ is odd}
	\end{cases}.
	\]
	Then, we have
	\[
	w_{i+1}
	=
	\mQ
	w_i
	\]
	where
	\[
	\mQ \eqdef \begin{pmatrix}
		-\frac{r}{c} &  \frac{c + \frac{\mu}{\lambda} +r}{c} \\
		- \frac{c + \frac{\mu}{\lambda} +r}{c}&  \frac{\left(c + \frac{\mu}{\lambda} + r\right)^2}{cr} - \frac{c}{r}
	\end{pmatrix}
	\]
	and
	\[
	r = \begin{cases}
		\frac12 & \text{if } n \text{ is even} \\
		\frac{\nhalf}{n} & \text{if } n \text{ is odd}
	\end{cases} .
	\]
\end{lemma}

To prove the lemma, we shall manipulate the first-order optimality conditions of~\eqref{eq:dhjabhsudga}. 

\begin{proof}
	\textbf {For even $n$}, the first-order optimality conditions yield
	
	\begin{eqnarray}
		\left(  c+\frac{1}2 + \frac{\mu}{\lambda} \right) z^\star_{2i-1} - cz^\star_{2i} - \frac{1}{2}y^\star_{2i-1} = 0 && \text{for } 1\leq i \leq T  \label{eq:z1even}\\
		\left(  c+\frac{1}2+ \frac{\mu}{\lambda} \right) z^\star_{2i} -  cz^\star_{2i-1} - \frac{1}{2}y^\star_{2i} = 0 && \text{for } 1\leq i \leq T  \label{eq:z2even}\\
		\left( c+\frac{1}2+ \frac{\mu}{\lambda} \right) y^\star_{2i} - cy^\star_{2i+1} - \frac12 z^\star_{2i} = 0 && \text{for } 1\leq i \leq T-1  \label{eq:y1even}\\
		\left( c+\frac{1}2 + \frac{\mu}{\lambda} \right) y^\star_{2i+1} -  cy^\star_{2i} - \frac12 z^\star_{2i+1} = 0 && \text{for } 1\leq i \leq T-1  \label{eq:y2even}
	\end{eqnarray}
	
	Equalities~\eqref{eq:z1even} and~\eqref{eq:z2even} can be equivalently written as
	
	\begin{equation}\label{eq:ehquiwhu}
		\begin{pmatrix}
			c & 0 \\
			- c-r- \frac{\mu}{\lambda}  & r  \\
		\end{pmatrix}
		\begin{pmatrix}
			z^\star_{2i}\\
			y^\star_{2i}
		\end{pmatrix}
		=
		\begin{pmatrix}
			c+r + \frac{\mu}{\lambda}  & - r \\
			-c& 0 \\
		\end{pmatrix}
		\begin{pmatrix}
			z^\star_{2i-1}\\
			y^\star_{2i-1}
		\end{pmatrix} \qquad \text{for } 1\leq i \leq T
	\end{equation}
	
	and consequently we must have for all $1\leq i \leq T$
	\begin{eqnarray*}
		\begin{pmatrix}
			z^\star_{2i}\\
			y^\star_{2i}
		\end{pmatrix}
		&=&
		\begin{pmatrix}
			c& 0 \\
			- c-r - \frac{\mu}{\lambda}  &  r  \\
		\end{pmatrix}^{-1}
		\begin{pmatrix}
			c+r + \frac{\mu}{\lambda}  & -r  \\
			-c& 0 \\
		\end{pmatrix}
		\begin{pmatrix}
			z^\star_{2i-1}\\
			y^\star_{2i-1}
		\end{pmatrix}
		\\
		&=&
		\begin{pmatrix}
			\frac{c + \frac{\mu}{\lambda} + r}{c}&    -\frac{r}{c} \\
			\frac{\left(c + \frac{\mu}{\lambda} + r\right)^2}{rc} - \frac{c}{r} & -\frac{c + \frac{\mu}{\lambda} + r}{c}
		\end{pmatrix}
		\begin{pmatrix}
			z^\star_{2i-1}\\
			y^\star_{2i-1}
		\end{pmatrix}
		\\
		&=&
		\mQ
		\begin{pmatrix}
			y^\star_{2i-1}\\
			z^\star_{2i-1}
		\end{pmatrix}.
	\end{eqnarray*}
	
	Analogously, from~\eqref{eq:y1even} and~\eqref{eq:y2even} we deduce that for all $1\leq i \leq T-1$
	\[
	\begin{pmatrix}
		y^\star_{2i+1}\\
		z^\star_{2i+1}
	\end{pmatrix}
	=
	\mQ
	\begin{pmatrix}
		z^\star_{2i}\\
		y^\star_{2i}
	\end{pmatrix}.
	\]

	\textbf{For odd $n$}, the first-order optimality conditions yield
	{
		\footnotesize
		\begin{align}
			\left(  c+\frac{\nhalf}{n} + \frac{\mu}{\lambda} \right) z^\star_{2i-1} - cz^\star_{2i} - \frac{\nhalf}{n}y^\star_{2i-1} = 0 & \quad \text{for } 1\leq i \leq T  \label{eq:z1odd}\\
			\left(  c+\frac{\nhalf}{n}+ \frac{\mu}{\lambda} \right) z^\star_{2i} -  cz^\star_{2i-1} - \frac{\nhalf}{n}y^\star_{2i} = 0 & \quad\text{for } 1\leq i \leq T  \label{eq:z2odd}\\
			\left( \frac{\nhalf+1}{\nhalf}c+\frac{\nhalf+1}{n}+ \frac{\nhalf+1}{\nhalf}\frac{\mu}{\lambda} \right) y^\star_{2i} - \frac{\nhalf+1}{\nhalf}cy^\star_{2i+1} - \frac{\nhalf+1}{n} z^\star_{2i} = 0 & \quad\text{for } 1\leq i \leq T-1  \label{eq:y1odd}\\
			\left( \frac{\nhalf+1}{\nhalf}c+\frac{\nhalf+1}{n} +\frac{\nhalf+1}{\nhalf} \frac{\mu}{\lambda} \right) y^\star_{2i+1} -  \frac{\nhalf+1}{\nhalf}cy^\star_{2i} - \frac{\nhalf+1}{n} z^\star_{2i+1} = 0 & \quad\text{for } 1\leq i \leq T-1  \label{eq:y2odd}
		\end{align}
	}

	Equalities~\eqref{eq:z1odd} and~\eqref{eq:z2odd} can be equivalently written as
	
	\[
	\begin{pmatrix}
		c & 0 \\
		- c-r- \frac{\mu}{\lambda}  & r  \\
	\end{pmatrix}
	\begin{pmatrix}
		z^\star_{2i}\\
		y^\star_{2i}
	\end{pmatrix}
	=
	\begin{pmatrix}
		c+r + \frac{\mu}{\lambda}  & - r \\
		-c& 0 \\
	\end{pmatrix}
	\begin{pmatrix}
		z^\star_{2i-1}\\
		y^\star_{2i-1}
	\end{pmatrix} \qquad \text{for } 1\leq i \leq T,
	\]
	which is identical to~\eqref{eq:ehquiwhu}, and thus
	
	\begin{eqnarray*}
		\begin{pmatrix}
			z^\star_{2i}\\
			y^\star_{2i}
		\end{pmatrix}
		&=&
		\mQ
		\begin{pmatrix}
			y^\star_{2i-1}\\
			z^\star_{2i-1}
		\end{pmatrix}.
	\end{eqnarray*}

	Similarly, ~\eqref{eq:y1odd} and~\eqref{eq:y2odd} imply that for all $1\leq i \leq T-1$
	\[
	\begin{pmatrix}
		y^\star_{2i+1}\\
		z^\star_{2i+1}
	\end{pmatrix}
	=
	\mQ
	\begin{pmatrix}
		z^\star_{2i}\\
		y^\star_{2i}
	\end{pmatrix}.
	\]
	
\end{proof}

As consequence of Lemma \ref{le:one_step_even}, we have that $w_k = \mQ ^{k-1}w_1$ with $\frac13 \leq r\leq \frac12$. 
Now we use the flexibility to choose $a \in \R, b \in \R_+$, so that we can find suitable $w_k$ (and thus suitable $x^\star$). Specifically, we aim to choose $a,b$, so that $w_1$ will be the eigenvector of $\mQ$, corresponding to a suitable eigenvalue $\gamma$ of  matrix $\mQ$. Then $w_k$ could be written as $w_k = \gamma^k w_1$.

\begin{lemma}
	Choose $c\eqdef \begin{cases}
		1 & \text{if } {L_1}\geq \lambda + \mu \\
		\delta\frac{\mu}{\lambda}, \delta \geq 1  & \text{if } {L_1}< \lambda + \mu
	\end{cases}
	$ and
	\begin{equation}
		b \eqdef
		\begin{cases}
			\frac{\frac{\mu^2}{\lambda^2} + 2\frac{\mu}{\lambda} + 2r +2r\frac{\mu}{\lambda}  + 2r^2+ \left( \frac{\mu}{\lambda} (\frac{\mu}{\lambda} +2r)(\frac{\mu}{\lambda}+2) (\frac{\mu}{\lambda} +2r+2)\right)^\frac12 }{2r(1+\frac{\mu}{\lambda}+r)} -1-\frac{\mu}{\lambda}  & \text{if } {L_1}\geq \lambda + \mu
			\\
			\frac{ \frac{\mu^2}{\lambda^2} + 2r^2+ 2 r \frac{\mu}{\lambda}+ 2\delta r \frac{\mu}{\lambda} +2\delta  \frac{\mu^2}{\lambda^2} +
				\frac{\mu}{\lambda} \left( (2\delta+1)(\frac{\mu}{\lambda}+2r) (\frac{\mu}{\lambda} +2r+2\delta \frac{\mu}{\lambda})\right)^\frac12}
			{2 r (\frac{\mu}{\lambda} + r + \delta \frac{\mu}{\lambda})}
			-1-\frac{\mu}{\lambda}  & \text{if } {L_1}< \lambda + \mu
		\end{cases} .
		\label{eq:bdef}
	\end{equation}
	
	Then, we have $b\geq 0$ and
	\[
	w_i = \gamma^{i-1} w_1 \neq \begin{pmatrix}
		0\\0
	\end{pmatrix}, \qquad \text{for } i=1,2,\dots, d,
	\]
	where
	\begin{equation}\label{eq:gamma_ineq}
		\gamma\eqdef
		\begin{cases}
			\frac{\frac{\mu^2}{\lambda^2} + 2\frac{\mu}{\lambda} + 2r +2r\frac{\mu}{\lambda}  - \left( \frac{\mu}{\lambda} (\frac{\mu}{\lambda} +2r)(\frac{\mu}{\lambda}+2) (\frac{\mu}{\lambda} +2r+2)\right)^\frac12}{2r}\geq 1-10\sqrt{\frac{\mu}{\lambda}}
			& \text{if } {L_1}\geq \lambda + \mu,
			\\
			\frac{\frac{\mu}{\lambda} + 2r+ 2\delta r +2\delta \frac{\mu}{\lambda}  - \left( (2\delta+1)(\frac{\mu}{\lambda}+2r) (\frac{\mu}{\lambda} +2r+2\delta \frac{\mu}{\lambda})\right)^\frac12}{2\delta r}
			\geq 1-10\sqrt{\frac{1}{\delta}}
			& \text{if } {L_1}< \lambda + \mu.
		\end{cases}
	\end{equation}

\end{lemma}

\begin{proof}
	First, note that if $c=1$, each local objective is $(\mu+\lambda)$--smooth, and thus also ${L_1}$--smooth (and therefore the choice of $c$ does not contradict the smoothness). Next, if ${L_1} \geq \lambda + \mu $, the vector
	\[
	v\eqdef \begin{pmatrix}
		\frac{\frac{\mu^2}{\lambda^2} + 2\frac{\mu}{\lambda} + 2r +2r\frac{\mu}{\lambda}  + 2r^2+ \left( \frac{\mu}{\lambda} (\frac{\mu}{\lambda} +2r)(\frac{\mu}{\lambda}+2) (\frac{\mu}{\lambda} +2r+2)\right)^\frac12 }{2r(1+\frac{\mu}{\lambda}+r)} \\
		1
	\end{pmatrix}
	\]
	is an unnormalized eigenvector of $\mQ$ corresponding to eigenvalue $\gamma$.\footnote{See a MatLab symbolic verification at file {\tt eigenvalues.m}.} Next, we prove $b\geq0$ and $\gamma \geq 1 - 10\sqrt{\frac{\mu}{\lambda}}$ using Mathematica, see the file {\tt proof.nb} and the screen shot below.
	
	\begin{figure}[H]
		\includegraphics[width=\textwidth]{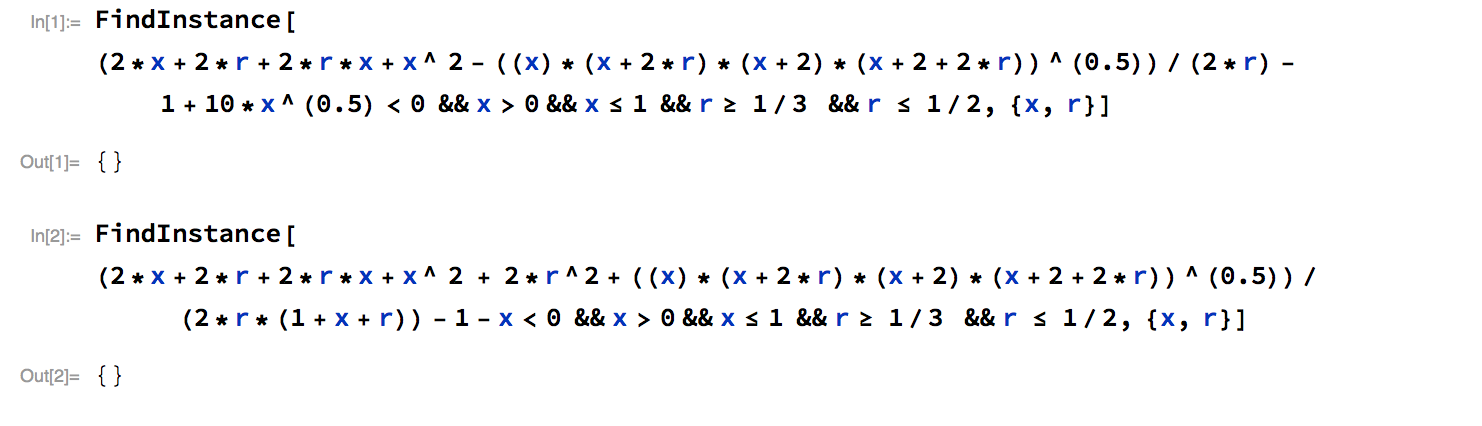}
	\end{figure}
	
	Let us look now at the case where ${L_1} \leq \lambda + \mu $. Now, the vector
	\[
	v\eqdef \begin{pmatrix}
		\frac{ \frac{\mu^2}{\lambda^2} + 2r^2+ 2 r \frac{\mu}{\lambda}+ 2\delta r \frac{\mu}{\lambda} +2\delta  \frac{\mu^2}{\lambda^2} +
			\frac{\mu}{\lambda} \left( (2\delta+1)(\frac{\mu}{\lambda}+2r) (\frac{\mu}{\lambda} +2r+2\delta \frac{\mu}{\lambda})\right)^\frac12}
		{2 r (\frac{\mu}{\lambda} + r + \delta \frac{\mu}{\lambda})}
		\\
		1
	\end{pmatrix}
	\]
	is an unnormalized eigenvector of $\mQ$ corresponding to eigenvalue $\gamma$.\footnote{See a MatLab symbolic verification at file {\tt eigenvaleus.m}.} Next, we prove $b\geq0$ and $\gamma \geq 1 - 10\sqrt{\frac{1}{\delta}}$ using Mathematica, see the file {\tt proof.nb} and the screen shot below.
	
	\begin{figure}[H]
		\includegraphics[width=\textwidth]{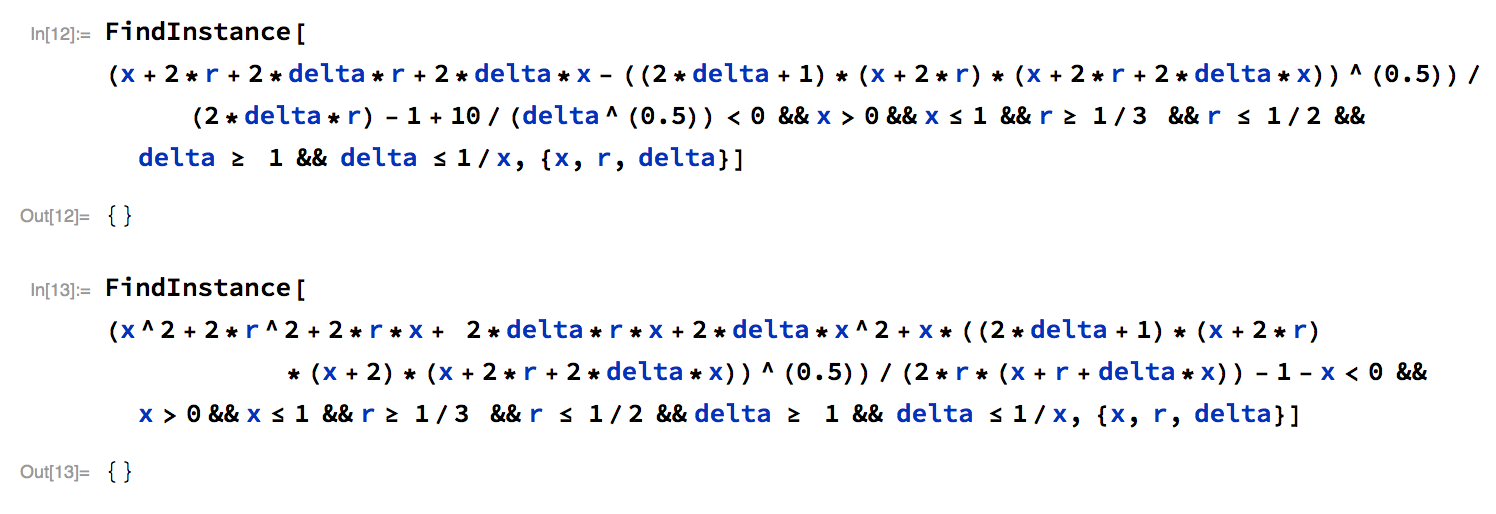}
	\end{figure}

	Setting $b$ according to~\eqref{eq:bdef} we assure that $w_i$ is a multiple of $v$ and consequently we  have
	\[
	w_i = \gamma^{i-1} w_1, \qquad \text{for } i=1,2,\dots, d,
	\]
	as desired. It remains to mention that $w_i\neq \begin{pmatrix}
		0\\0\end{pmatrix}$ regardless of the choice of $a \neq 0$.
	
\end{proof}

\begin{proof} \textbf{Theorom \ref{thm:lb}}
	
	Let $x^0 = 0\in \R^{nd}$. Note that our oracle allows us at most $K+1$ nonzero coordinates of $x^K$ after $K$ rounds of communications. Consequently,~
	\begin{eqnarray*}
		\frac{\norm{ x^{K}-x^\star  }^2}{\norm{x^0-x^\star }^2}
		&\geq &
		\frac12
		\frac{ \sum_{j=K+2}^d  \norm{ w_j }^2}{\sum_{j=1}^d  \norm{ w_j }^2}
		=
		\frac12
		\frac{ \sum_{j=K+2}^d  \gamma^{j-1}\norm{ w_1 }^2}{\sum_{j=1}^d \gamma^{j-1} \norm{ w_1 }^2}
		=
		\frac12
		\frac{ \gamma^{K+1}  \sum_{j=0}^{d-K-2}  \gamma^{j}}{\sum_{j=0}^{d-1} \gamma^{j} }
		\\
		&=&
		\frac12
		\gamma^{K+1}\frac{ 1-\gamma^{d-K-1}}{1-\gamma^d}
		\stackrel{(*)}{\geq}
		\frac 1 4 \left(1-10\max \left\{ \sqrt{\frac{\mu}{\lambda}}, \sqrt{\frac{1}{\delta}} \right\} \right)^{K+1}
		\\
		&=&
		\frac 1 4 \left(1-10\max\left\{ \sqrt{\frac{\mu}{\lambda}}, \sqrt{\frac{\mu}{{L_1}-\mu}}\right \} \right)^{K+1}
	\end{eqnarray*}
	where the inequality $(*)$ holds for large enough $T$ (and consequently large enough $d=2T$). \QED

\end{proof}


\section{Proofs for Section~\ref{sec:upperbound}}

\subsection{Towards the proof of Theorems~\ref{thm:inexact} and~\ref{thm:inexact_stoch}}

\begin{proposition} \label{prop:fista_inexact}
	Iterates of Algorithm~\ref{alg:fista_inex} satisfy
	{
		\footnotesize
		\begin{align}
			\nonumber
			&F(x^k) -  F^\star \\
			& \leq \left( 1- \sqrt{\frac{\mu}{\lambda}}\right)^k \left(
			\sqrt{2(F(x^0) - F^\star)}
			+
			2\sqrt{\frac{\lambda}{\mu} } \left( \sum_{i=1}^k \epsilon_i^\frac12  \left( 1- \sqrt{\frac{\mu}{\lambda}}\right)^{-\frac{i}{2}}\right)
			+
			\sqrt{
				\sum_{i=1}^k \epsilon_i \left( 1- \sqrt{\frac{\mu}{\lambda}}\right)^{-i}
			}
			\right)^2.
			\label{eq:inexact_prop}
		\end{align}
	}
\end{proposition}
\begin{proof}
	First, notice that the objective is $\frac{\lambda}{n}$ smooth and $\frac{\mu}{n}$--strongly convex. Next, the error in the evaluation of the proximal operator at iteration $k$ can be expressed as
	\[
	\sum_{i=1}^n \frac1n f_i(x_i^{k+1}) +  \frac{\lambda}{2n}\norm{x_i^{k+1} \bar{y}^k  }^2 \leq \sum_{i=1}^n \frac1n \epsilon_k = \epsilon_k.
	\]
	It remains to apply \citep[Proposition 4]{schmidt2011convergence}.
\end{proof}

\subsubsection{General convergence rate of \iapgd{}}
Theorem~\ref{thm:inexact_main} shows that the expected number of communications that Algorithm~\ref{alg:fista_inex} requires to reach $\varepsilon$--approximate solution is $\tilde{\cO}\left( \sqrt{\frac{\lambda}{\mu}}\right)$, given that~\eqref{eq:epsilon_bound_stoch} holds.

\begin{theorem}\label{thm:inexact_main}
	Assume that for all $k\geq0, 1\leq i \leq n$, the subproblem~\eqref{eq:algo_suboptimality} was  solved up to a suboptimality\footnote{See Algorithm~\ref{alg:fista_inex} for the exact meaning.} $\epsilon_k$ by a possibly randomized iterative algorithm such that
	\begin{equation}\label{eq:epsilon_bound_stoch}
		\E{\epsilon_k \mid x^k } \leq
		\left(1-\sqrt{\frac{\mu}{\lambda}} \right)^{2k} R^2,
	\end{equation}
	for some fixed $R>0$. Consequently, we have
	\begin{equation}\label{eq:dajnkbdbhas}
		\E{\left(F(x^k) -  F^\star \right)^{\frac12}}
		\leq
		\left(1-\sqrt{\frac\mu\lambda} \right)^{\frac{k}2} \left(
		\sqrt{2(F(x^0) - F^\star)}
		+
		2\left( 2\sqrt{\frac{\lambda}{\mu} } +1\right)\sqrt{\frac{\lambda}{\mu}}R
		\right).
	\end{equation}
\end{theorem}

\begin{proof}
	Let $\omega \eqdef 1- \sqrt{\frac{\mu}{\lambda}}$. Proposition~\ref{prop:fista_inexact} gives us
	
	\begin{align*}
		\left(F(x^k) -  F^\star \right)^{\frac12}
		&\stackrel{\eqref{eq:inexact_prop}}{\leq}
		\omega^{\frac{k}2} \left(
		\sqrt{2(F(x^0) - F^\star)}
		+
		2\sqrt{\frac{\lambda}{\mu} } \left( \sum_{i=1}^k \epsilon_i^\frac12 \omega^{-\frac{i}{2}}\right)
		+
		\sqrt{
			\sum_{i=1}^k \epsilon_i \omega^{-i}
		}
		\right)
		\\
		&\leq
		\omega^{\frac{k}2} \left(
		\sqrt{2(F(x^0) - F^\star)}
		+
		\left( 2\sqrt{\frac{\lambda}{\mu} } +1\right) \left( \sum_{i=1}^k \epsilon_i^\frac12  \omega^{-\frac{i}{2}}\right)
		\right).
	\end{align*}
	
	Taking the expectation, we get
	
	\begin{align*}
		\E{\left(F(x^k) -  F^\star \right)^{\frac12}}
		&\leq
		\omega^{\frac{k}2} \left(
		\sqrt{2(F(x^0) - F^\star)}
		+
		\left( 2\sqrt{\frac{\lambda}{\mu} } +1\right) \left( \sum_{i=1}^k \E{\epsilon_i^\frac12}  \omega^{-\frac{i}{2}}\right)
		\right)
		\\
		&\leq
		\omega^{\frac{k}2} \left(
		\sqrt{2(F(x^0) - F^\star)}
		+
		\left( 2\sqrt{\frac{\lambda}{\mu} } +1\right) \left( \sum_{i=1}^k \E{\epsilon_i}^\frac12  \omega^{-\frac{i}{2}}\right)
		\right)
		\\
		&\stackrel{\eqref{eq:epsilon_bound_stoch}}{\leq}
		\omega^{\frac{k}2} \left(
		\sqrt{2(F(x^0) - F^\star)}
		+
		\left( 2\sqrt{\frac{\lambda}{\mu} } +1\right) R \left( \sum_{i=1}^k  \omega^{\frac{i}{2}}\right)
		\right)
		\\
		&\leq 
		\omega^{\frac{k}2} \left(
		\sqrt{2(F(x^0) - F^\star)}
		+
		\left( 2\sqrt{\frac{\lambda}{\mu} } +1\right) R \left( \sum_{i=1}^\infty  \omega^{\frac{i}{2}}\right)
		\right)
		\\
		&= 
		\omega^{\frac{k}2} \left(
		\sqrt{2(F(x^0) - F^\star)}
		+
		\left( 2\sqrt{\frac{\lambda}{\mu} } +1\right) R \frac{\omega^{\frac12}}{1-\omega^{\frac12}}
		\right)
		\\
		&\leq 
		\omega^{\frac{k}2} \left(
		\sqrt{2(F(x^0) - F^\star)}
		+
		\left( 2\sqrt{\frac{\lambda}{\mu} } +1\right) 2R \frac{1}{1-\omega}
		\right)
		\\
		&= 
		\omega^{\frac{k}2} \left(
		\sqrt{2(F(x^0) - F^\star)}
		+
		\left( 2\sqrt{\frac{\lambda}{\mu} } +1\right) 2R\sqrt{\frac{\lambda}{\mu}}
		\right),
	\end{align*}
	which is exactly~\eqref{eq:dajnkbdbhas}.
\end{proof}

\subsubsection{Proof of Theorem~\ref{thm:inexact}}

Denote $\cS' \eqdef  \{x; F(x) \leq F^\star +  8 (F(x^0)-F^\star) \} $, $\cS \eqdef  \{(2-\alpha)x' - (1-\alpha)x''; x',x''\in \cS', 0\leq \alpha \leq 1\}$ and $ D\eqdef \diam(\cS)< \infty$. Consequently,

\begin{equation}\label{eq:mnadbdhsdbhakkj}
	D^2 \leq 36 \max_{x\in \cS} \norm{x-x^\star }^2 \leq \frac{18n}{\mu} \max_{x\in \cS} (F(x)-F(x^\star)) \leq  \frac{144n}{\mu} (F(x^0)-F^\star).
\end{equation}

Let us proceed with induction. Suppose that for all $0\leq t<k$  we have \[F(x^i) -  F^\star \leq 8 \left( 1- \sqrt{\frac{\mu}{\lambda}}\right)^t (F(x^0) - F^\star).\] Consequently, $x^t \in \cS' $ for all $0\leq t<k$. Thanks to the update rule of sequence $\{y\}_{t=1}^\infty$, we must have $y^{k-1}\in \cS$. Next, define $\hat{x}_i^{k} \eqdef \argmin_{z\in \R^d} f_i(z) + \frac{\lambda}{2n} \norm{ z -\bar{y}^{k-1} }^2$. Clearly, $\hat{x}^{k}\in \cS$, and consequently, $\norm{\hat{x}^{k} - y^{k-1} }^2\leq D^2$.

We will next show that
\begin{equation}
	\label{eq:eps_bound}
	\epsilon_k \leq R^2 \omega^{2k},
\end{equation}
where
\begin{equation}\label{eq:romega_def}
	R \eqdef \frac{\sqrt{2(F(x^0) - F^\star)}}{2\sqrt{\frac{\lambda}{\mu}} \left(2 \sqrt{\frac{\lambda}{\mu}} +1\right)}, \qquad \omega \eqdef 1-\sqrt{\frac{\mu}{\lambda}}.
\end{equation}

Define $h^k_i(z)\eqdef f_i(z) + \frac{\lambda}{2n} \norm{ z -\bar{y}^{k-1} }^2$. Since $h_i^{k}$ is $\frac1n ({L_1}+\lambda)$ smooth and $\frac1n(\mu + \lambda)$ strongly convex,  running \agd{} locally for $c_1 + c_2k$ iterations with\footnote{Inequality $(*)$ holds since for any $0\leq  a <1$ we have $\frac{-1}{\log(1-a)}\leq \frac1a$, while $(**)$ holds since $\log\left(\frac{1}{1-\sqrt{\frac{\mu}{\lambda}}} \right) \leq  2\sqrt{\frac{\mu}{\lambda}}$ thanks to $\lambda \geq 2\mu$. }
\begin{eqnarray*}
	c_1 &\eqdef&
	-\frac{\log \frac{4{L_1}D^2}{R^2}}{ \log \left(1-\sqrt{\frac{\mu+\lambda}{{L_1}+ \lambda}} \right)}
	\stackrel{(*)}{\leq}
	\sqrt{\frac{{L_1}+ \lambda}{\mu+\lambda}} \log \frac{4{L_1}D^2}{R^2}\\
	&\stackrel{\eqref{eq:mnadbdhsdbhakkj}}{\leq}&
	\sqrt{\frac{{L_1}+ \lambda}{\mu+\lambda}} \log \frac{1152 {L_1} \lambda n \left(2 \sqrt{\frac{\lambda}{\mu} }+1 \right)^2}{\mu^2}
	,
	\\
	c_2 &\eqdef&
	\frac{2 \log \omega}{\log  \left(1-\sqrt{\frac{\mu+\lambda}{{L_1}+ \lambda}} \right) }
	\stackrel{(*)+(**)}{\leq} 4\sqrt{\frac{\mu({L_1}+ \lambda)}{\lambda(\mu+\lambda)}}
\end{eqnarray*}
yields

\begin{eqnarray*}
	\epsilon_k
	&\stackrel{\text{\citenum{schmidt2011convergence}, Pr. }4 }{\leq}&
	\left(1-\sqrt{\frac{\mu+\lambda}{{L_1}+ \lambda}} \right)^{c_1+c_2 k} 4\left( \sum_{i=i}^n \left( h^k_i(y_i^{k-1}) - h^k_i(\hat{x}^k_i)\right)\right)
	\\
	&\leq&
	\left(1-\sqrt{\frac{\mu+\lambda}{{L_1}+ \lambda}} \right)^{c_1+c_2k} 4{L_1}D^2
	\\
	&=&
	\exp\left(
	c_2k \log\left(1-\sqrt{\frac{\mu+\lambda}{{L_1}+ \lambda}}\right) + c_1 \log \left(1-\sqrt{\frac{\mu+\lambda}{{L_1}+ \lambda}} \right)+ \log \left(4{L_1}D^2\right)\right)
	\\
	&=&
	\exp\left(
	2k \log\omega +  \log (R^2) \right)
	\\
	&=&
	R^2 \omega^{2k},
\end{eqnarray*}
as desired.

Next, Theorem~\ref{thm:inexact_main} gives us
{
	\begin{eqnarray*}
		F(x^k) -  F^\star
		&\stackrel{\eqref{eq:dajnkbdbhas}}{\leq} &
		\left( 1- \sqrt{\frac{\mu}{\lambda}}\right)^k \left(\sqrt{2(F(x^0) - F^\star)} +
		2\left( 2\sqrt{\frac{\lambda}{\mu} } +1\right)\sqrt{\frac{\lambda}{\mu}}R
		\right)^2
		\\
		&\stackrel{\eqref{eq:romega_def}}{=} &
		8 \left( 1- \sqrt{\frac{\mu}{\lambda}}\right)^k (F(x^0) - F^\star),
	\end{eqnarray*}
}
as desired.

Consequently, in order to reach $\varepsilon$ suboptimality, we shall set $k = \cO\left( \sqrt{\frac{\lambda}{\mu}}\log\frac1\varepsilon\right)$. The total number of local gradient computation thus is
\begin{eqnarray*}
	\sum_{i=1}^k( c_1  + c_2i )&=& kc_1 + c_2 \cO(k^2)
	\\
	&=&
	\cO\Bigg(
	\sqrt{\frac{{L_1}+ \lambda}{\mu+\lambda}} \log \frac{32 {L_1} \lambda n^2 \left(4 \sqrt{\frac{\lambda}{\mu} +1} \right)^2}{\mu^2}  \sqrt{\frac{\lambda}{\mu}}\log\frac1\varepsilon \\
	& & \qquad + 
	\sqrt{\frac{\mu({L_1}+ \lambda)}{\lambda(\mu+\lambda)}}
	\frac{\lambda}{\mu} \left( \log\frac1\varepsilon \right)^2
	\Bigg)
	\\
	&=&
	\cO\left(
	\sqrt{\frac{{L_1}+ \lambda}{\mu}} \log\frac1\varepsilon
	\left( \log \frac{ {L_1} \lambda n}{\mu} +\log\frac1\varepsilon \right)
	\right).
\end{eqnarray*}

\QED

\subsubsection{Proof of Theorem~\ref{thm:inexact_stoch}}

Next, since the sequence of iterates $\{ x^k\}_{k=0}^\infty$ is bounded, so is the sequence $\{ y^k\}_{k=0}^\infty$, and consequently, the initial distance to the optimum is bounded for each local subproblem too. As the local objective is $(\Lloc + \lambda)$--smooth and $(\mu + \lambda)$--strongly convex, in order to guarantee~\eqref{eq:epsilon_bound_stoch}, \katyusha{} requires
\begin{eqnarray*}
	\efa
	\cO\left( \left(m + \sqrt{m\frac{\Lloc + \lambda}{\mu + \lambda}}\right)\log\frac{1}{R^2 \omega^2k}  \right)\\
	&=&
	\cO\left(\left( m+  \sqrt{m\frac{\Lloc + \lambda}{\mu + \lambda}}\right)\left( \log\frac{1}{R^2} +  2k\log\frac{1}{ \omega}  \right)\right)
	\\
	&\stackrel{(***)}{=}&
	\cO\left(\left(m+  \sqrt{m \frac{\Lloc + \lambda}{\mu + \lambda}}\right)\left( \log\frac{1}{R^2} +  k\sqrt{\frac{\mu}{\lambda}} \right)\right)
\end{eqnarray*}
iterations.\footnote{Inequality $(***)$ holds since $\log\left(\frac{1}{1-\sqrt{\frac{\mu}{\lambda}}} \right) \leq  2\sqrt{\frac{\mu}{\lambda}}$ thanks to $\lambda \geq 2\mu$.}	

Lastly, since \katyusha{} requires $\cO(1)$ local stochastic gradient evaluations on average, the total local gradient complexity becomes
\begin{eqnarray*}
	\sum_{t=1}^{\tilde{\cO}\left(\sqrt{\frac{\lambda}{\mu}}\right)} \cO\left(\left(m+  \sqrt{m \frac{\Lloc + \lambda}{\mu + \lambda}}\right)\left( \log\frac{1}{R^2} +  t\sqrt{\frac{\mu}{\lambda}} \right)\right)
	=
	\tilde{\cO}\left(
	\left(m\sqrt{\frac{\lambda}{\mu}}
	+  \sqrt{m \frac{\Lloc + \lambda}{\mu }}\right)
	\right).
\end{eqnarray*}

\QED

\subsection{Towards the proof of Theorem~\ref{thm:a2} \label{sec:a2_proof}}

\begin{lemma}\label{lem:es_stoch}
	Suppose that $\flocc_{ij}$ is $\Lloc$ smooth for all $1\leq i\leq n,1\leq j\leq m $. Let $g^k$ be a variance reduced stochastic gradient estimator from Algorithm~\ref{alg:acc_stoch} and define
	\[
	\cL \eqdef \max \left\{ \frac{\Lloc}{n(1-\proby)}, \frac{\lambda}{n\proby} \right\}.
	\]
	Then, we have
	\begin{equation}\label{eq:exp:smooth_stoch}
		\E{\norm{g^k - \nabla F(x^k)}^2} \leq 2\cL D_F(w^k,x^k).
	\end{equation}
\end{lemma}

\begin{proof}

	\begin{align*}
		& \E{\norm{g^k - \nabla F(x^k)}^2} \\
		&\qquad =
		\frac{1-p}{m} \sum_{j=1}^m \sum_{i=1}^n\norm{\frac{1}{1-p} \left( \nabla \flocc_{ij}(x^k)  - \nabla \flocc_{ij}(w^k) \right)- \left( \nabla F(x^k)  - \nabla F(w^k)\right)}^2 \\
		&\qquad \, \qquad
		+
		p\norm{\frac{\lambda}{p} \left( \nabla \psi(x^k)  - \nabla \psi(w^k) \right)- \left( \nabla F(x^k)  - \nabla F(w^k)\right)}^2
		\\
		&\qquad \leq 
		\frac{1-p}{m} \sum_{j=1}^m \sum_{i=1}^n\norm{\frac{1}{1-p} \left( \nabla \flocc_{ij}(x^k)  - \nabla \flocc_{ij}(w^k) \right)}^2 \\
		&\qquad \, \qquad
		+
		p\norm{\frac{\lambda}{p} \left( \nabla \psi(x^k)  - \nabla \psi(w^k) \right))}^2
		\\
		&\qquad =
		\frac{1}{m(1-p)}\sum_{j=1}^m \sum_{i=1}^n\norm{\nabla \flocc_{ij}(x^k)  - \nabla \flocc_{ij}(w^k) }^2
		+
		\frac{\lambda^2}{p} \norm{\left( \nabla \psi(x^k)  - \nabla \psi(w^k) \right))}^2
		\\
		&\qquad \stackrel{(*)}{\leq}
		\frac{2\Lloc}{n m(1-p)} \sum_{j=1}^m \sum_{i=1}^n D_{\flocc_{ij}}(w^k,x^k)
		+
		\frac{2\lambda^2}{np}D_\psi(w^k,x^k)
		\\
		&\qquad =
		\frac{2\Lloc}{1-p}  D_{f}(w^k,x^k)
		+
		\frac{2\lambda^2}{np}D_\psi(w^k,x^k)
		\\
		\\
		&\qquad \leq
		2\max \left\{ \frac{\Lloc}{n(1-p)}, \frac{\lambda}{np} \right\} D_F(w^k,x^k)
		\\
		&\qquad = 2\cL D_F(w^k,x^k).
	\end{align*}
	Above, $(*)$ holds since $\flocc$ is $\frac{\Lloc}{n}$ smooth and $\psi$ is $\frac1n$ smooth~\citep{hanzely2020federated}.
	
\end{proof}

\begin{proposition} \label{prop:acc}
	Let $\flocc_{ij}$ be ${L_1}$ smooth and $\mu$ strongly convex for all $1\leq i\leq n,1\leq j\leq m $.
	Define the following Lyapunov function:
	\begin{eqnarray*}
		\Psi^k &\eqdef&  \norm{z^k - x^\star}^2 + \frac{2\gamma\beta}{\theta_1}\left[F(y^k) - F(x^\star)\right] + \frac{(2\theta_2 + \theta_1)\gamma\beta}{\theta_1\probx}\left[F(w^k) - F(x^\star)\right],
	\end{eqnarray*}
	and let
	\begin{eqnarray*}
		L_F &=&\frac1n( \lambda + \Lloc), \\
		\eta &=&  \frac14 \max\{L_F, \cL\}^{-1}, \\
		\theta_2 &=& \frac{\cL}{2\max\{L_F, \cL\}}, \\
		\gamma &=& \frac{1}{\max\{2\mu/n, 4\theta_1/\eta\}},\\
		\beta &=& 1 - \frac{\gamma\mu}{n} \; \mathrm{and} \\
		\theta_1 &=& \min\left\{\frac{1}{2},\sqrt{\frac{\eta\mu}{n} \max\left\{\frac{1}{2}, \frac{\theta_2}{\rho}\right\}}\right\} .
	\end{eqnarray*}
	Then the following inequality holds:
	\begin{equation*}
		\E{\Psi^{k+1}} \leq
		\left[1 -  \frac{1}{4}\min\left\{\probx, \sqrt{\frac{\mu}{2n\max\left\{L_{F}, \frac{\cL}{\rho}\right\}}} \right\} \right]\Psi^0.
	\end{equation*}
	As a consequence, iteration complexity of Algorithm~\ref{alg:acc_stoch} is
	\[
	\cO\left(\left(\frac{1}{\probx}   + \sqrt{\frac{ \max \left\{ \frac{\Lloc}{1-\proby}, \frac{\lambda}{\proby} \right\}}{\probx\mu}}  \right)\log\frac1\varepsilon\right).
	\]
\end{proposition}
At the same time, the communication complexity of \altsgdp{} is 
\[
\cO\left(  (\probx + \proby(1-\proby))\left( \frac{1}{\probx}   + \sqrt{\frac{ \max \left\{ \frac{\Lloc}{1-\proby}, \frac{\lambda}{\proby} \right\}}{\probx\mu}}\right)\log\frac1\varepsilon\right)\] and the local stochastic gradient complexity is
\[
\cO\left((
\probx m +
(1-\probx)
)\left(\frac{1}{\probx}   + \sqrt{\frac{ \max \left\{ \frac{\Lloc}{1-\proby}, \frac{\lambda}{\proby} \right\}}{\probx\mu}}  \right)\log\frac1\varepsilon
\right).
\]
\begin{proof}
	Note that \altsgdp{} is a special case of {\tt L-Katyusha} from~\citep{hanzely2020variance}.\footnote{
		Similarly, we could have applied different accelerated variance reduced method with importance sampling such as another version of {\tt L-Katyusha}~\citep{qian2019svrg}, for example.
	} In order to apply Theorem 4.1 therein directly, it suffices to notice that function $\meta:\R^d\to\R$ is $L_F =\frac1n( \lambda + \Lloc)$ smooth and $\frac1n \mu$ strongly convex, and at the same time, thanks to Lemma~\ref{lem:es_stoch} we have
	\[
	\E{\norm{g^k - \nabla F(x^k)}^2} \leq 2\cL D_F(w^k,x^k).
	\]
	Connsequently, we immediately get the iteration complexity. The local stochastic gradient complexity of a single iteration of \altsgdp{} is  
	0 if $\xi = 1, \xi' = 0$, 1 if $\xi=0, \xi'=0$, $m$ if $\xi=0, \xi'=1$ and $m+1$ if $\xi=1, \xi'=1$. Thus, the total expected local stochastic gradient complexity is bounded by
	\[
	\cO\left((
	\probx m +
	(1-\probx)
	)\left(\frac{1}{\probx}   + \sqrt{\frac{ \max \left\{ \frac{\Lloc}{1-\proby}, \frac{\lambda}{\proby} \right\}}{\probx\mu}}  \right)\log\frac1\varepsilon
	\right)
	\]
	as desired. Next, the total communication complexity is bounded by the sum of the communication complexities coming from the full gradient computation (if statement that includes $\xi$) and the rest (if statement that includes $\xi'$). The former requires a communication if $\xi'=1$, the latter if two consecutive $\xi$--coin flips are different (see~\citep{hanzely2020federated}), yielding the expected total communication $\cO(\probx + \proby(1-\proby))$ per iteration.
	
\end{proof}

\subsubsection{Proof of Theorem~\ref{thm:a2}}

For $\probx = \proby(1-\proby)$ and $\proby = \frac{\lambda}{\lambda+\Lloc}$, the total communication complexity of \altsgdp{} becomes
\begin{align*}
	\cO\left(  (\probx + \proby(1-\proby))\left( \frac{1}{\probx}   + \sqrt{\frac{ \max \left\{ \frac{\Lloc}{1-\proby}, \frac{\lambda}{\proby} \right\}}{\probx\mu}}\right)\log\frac1\varepsilon\right)
	&=
	\cO\left(  \sqrt{\frac{ \proby \Lloc + (1-\proby) \lambda}{\mu}}\log\frac1\varepsilon\right)
	\\
	&=
	\cO\left(  \sqrt{\frac{ \Lloc \lambda}{(\Lloc+ \lambda)\mu}}\log\frac1\varepsilon\right)
	\\
	&=
	\cO\left(  \sqrt{\frac{ \min\{ \Lloc , \lambda\}}{\mu}}\log\frac1\varepsilon\right),
\end{align*}
as desired. 

The local stochastic gradient complexity for $\proby = \frac{\lambda}{\lambda+\Lloc}$ and $\rho = \frac{1}{m}$ is
\begin{eqnarray*}
	&&
	\cO\left((
	\probx m +
	(1-\probx)
	)\left(\frac{1}{\probx}   + \sqrt{\frac{ \max \left\{ \frac{\Lloc}{1-\proby}, \frac{\lambda}{\proby} \right\}}{\probx\mu}}  \right)\log\frac1\varepsilon
	\right)
	\\
	&& \qquad =
	\cO\left( \left(
	m
	+ \sqrt{\frac{m (\Lloc + \lambda)}{\mu}}  \right)\log\frac1\varepsilon
	\right).
\end{eqnarray*}

\QED

\clearpage

\section{Related work on the lower complexity bounds}

\paragraph{Related literature on the lower complexity bounds.} We distinguish two main lines of work on the lower complexity bounds related to our paper besides already mentioned works~\citep{scaman2018optimal, hendrikx2020optimal}. 

The first direction focuses on the classical worst-case bounds for sequential optimization developed by Nemirovsky and Yudin~\citep{nemirovski1985optimal}. Their lower bound was further studied in~\citep{agarwal2009information,raginsky2011information, nguyen2019tight} using information theory. The nearly-tight lower bounds for deterministic non-Euclidean smooth convex functions were obtained in \citep{guzman2015lower}. A significant gap between the oracle complexities of deterministic and randomized algorithms for the finite-sum problem was shown in~\citep{woodworth2016tight}, improving upon prior works~\citep{agarwal2014lower, lan2018optimal}.

The second stream of work tries to answer how much a parallelism might improve upon a given oracle. This direction was, to best of our knowledge, first explored by the work of Nemirovski~\citep{nemirovski1994parallel} and gained a lots of traction decently~\citep{smith2017interaction, balkanski2018parallelization, woodworth2018graph, duchi2018minimax, diakonikolas2018lower} motivated by an increased interest in the applications in federated learning, local differential privacy, and adaptive data analysis.

\begin{remark} 
	Concurrently with our work,	a different variant of accelerated \\ \fedprox{}---\fedsplit{}---was proposed in~\citep{pathak2020fedsplit}. There are several key differences between our work:
	i) While Algorithm~\ref{alg:fista_inex} is designed to tackle the problem~\eqref{eq:main}, \fedsplit{} is designed to tackle~\eqref{eq:fl_standard}. ii) The paper \citep{pathak2020fedsplit} does not argue about optimality of \fedsplit{}, while we do and iii) Iteration/communication complexity of \fedsplit{} is $\cO\left( \sqrt{\frac{{L_1}}{\mu}}\log \frac1\varepsilon\right)$ under ${L_1}$ smoothness of $f_1, \dots f_n$; such a rate can be achieved by a direct application of \agd{}. At the same time, \agd{} does not require solving the local subproblem each iteration, thus is better in this regard. However \fedsplit{} is a local algorithm to solve~\eqref{eq:fl_standard} with the correct fixed point, unlike other popular local algorithms.
\end{remark}


\chap{\ref{sec:moreau_meta}}

\section{Table of frequently used notation}
For clarity, we provide a table of frequently used notation.
\begin{table}[h]
	\centering
    \begin{threeparttable}
	\caption{Summary of frequently used notation in \Cref{sec:moreau_meta}.}        
	\begin{tabular}{|c|l|}
		\hline
		\multicolumn{2}{|c|}{{\bf Functions} }\\
        \hline
		$f_i: \R^d \to \R$ & The loss of task $i$ \\
		$F_i(x) = \min_z \{f_i(z) + \frac{1}{2\alpha}\norm{z-x}^2\}$ & Meta-loss\\
		$F(x)=\frac{1}{n}\sum_{i=1}^n F_i(x)$ & Full meta loss\\
		$z_i(x)=\argmin_z \{f_i(z) + \frac{1}{2\alpha}\norm{z-x}^2\}$  & The minimizer of regularized loss\\
		\hline
        \multicolumn{2}{|c|}{{\bf Constants} }\\
        \hline
		$\mu$ & Strong convexity constant\\ 
		$L$ & Smoothness constant of $f$\\
		$\Lmeta$ & Smoothness constant of $\meta$\\
		$\alpha$ & Objective parameter\\
		$\outers$ & Stepsize of the outer loop\\
		$\locstep$ & Stepsize\\
		$s$ & Number of steps in the inner loop\\
		$\delta$ & Precision of the proximal oracle\\
		\hline
	\end{tabular}
    \end{threeparttable}
\end{table}

\section{Parameterization of the inner loop of \Cref{alg:maml2}}

Note that \Cref{alg:maml2} depends on only one parameter -- $\outers$. We need to keep in mind that parameter $\inners$ is fixed by the objective \eqref{eq:new_pb} and changing $\inners$ shifts convergence neighborhood. Nevertheless, we can still investigate the case wehn $\inners$ from \eqref{eq:new_pb} and $\inners$ from Line $6$ of \Cref{alg:maml2} are different, as we can see in the following remark. 
\begin{remark} \label{lem:locstep_different}
	If we replace line $6$ of \Cref{alg:maml2} by $z_{l+1}^k = x^k - \locstep \nabla f_i(z_{i,l}^k)$, we will have freedom to choose $\locstep$. However, if we choose stepsize $\locstep \neq \inners$, then similar analysis to the proof of \Cref{lem:approx_implicit} yields 
	\begin{align} \label{eq:locstep_inexact}
		\frac 1 \locstep \norm{z_{i,s}^k - (x^k - \locstep\nabla \meta_i(x^k))}
		&\leq  \left( (\locstep {L_1})^s + |\alpha - \locstep|{L_1}\right)  \norm{\nabla \meta_i(x^k)}.
	\end{align}
\end{remark}
Note that in case $\locstep \neq \inners$, we cannot set number of steps $s$ to make the right-hand side of~\eqref{eq:locstep_inexact} smaller than $\delta \norm{\nabla \meta_i(x^k)}$ when $\delta$ is small. In particular, increasing the number of local steps $s$ will help only as long as $\delta > |\inners - \locstep|{L_1}$.

This is no surprise, for the modified algorithm (using inner loop stepsize $\locstep$) will no longer be approximating $\nabla \meta_i(x^k)$. It will be exactly approximating $\nabla \tilde \meta_i(x^k)$, where  $\tilde \meta_i(x)\eqdef \min_{z\in\R^d} \left\{f_i(z) + \frac{1}{2\locstep}\norm{z - x}^2\right\}$ (see \Cref{lem:approx_implicit}).
Thus, choice of stepsize in the inner loop affects what implicit gradients do we approximate and also what objective we are minimizing.

\section{Proofs}

\subsection{Proof of \Cref{th:imaml_nonconvex}}
\begin{proof}
	The counterexample that we are going to use is given below:
	\begin{align*}
		f(x) &= \min\left\{\frac{1}{4}x^4 - \frac{1}{3}|x|^3 + \frac{1}{6}x^2, \frac{2}{3}x^2 - |x| + \frac{5}{12} \right\} \\
		&=
		\begin{cases}
			\frac{1}{4}x^4 - \frac{1}{3}|x|^3 + \frac{1}{6}x^2, &\textrm{if } |x|\le 1,\\
			\frac{2}{3}x^2 - |x| + \frac{5}{12}, & \textrm{otherwise}.
		\end{cases}
	\end{align*}
	See also Figure~\ref{fig:example} for its numerical visualization.
	\begin{figure}[h!]
		\minipage[t]{0.48\textwidth}
		\centering
		\includegraphics[width=0.95\linewidth]{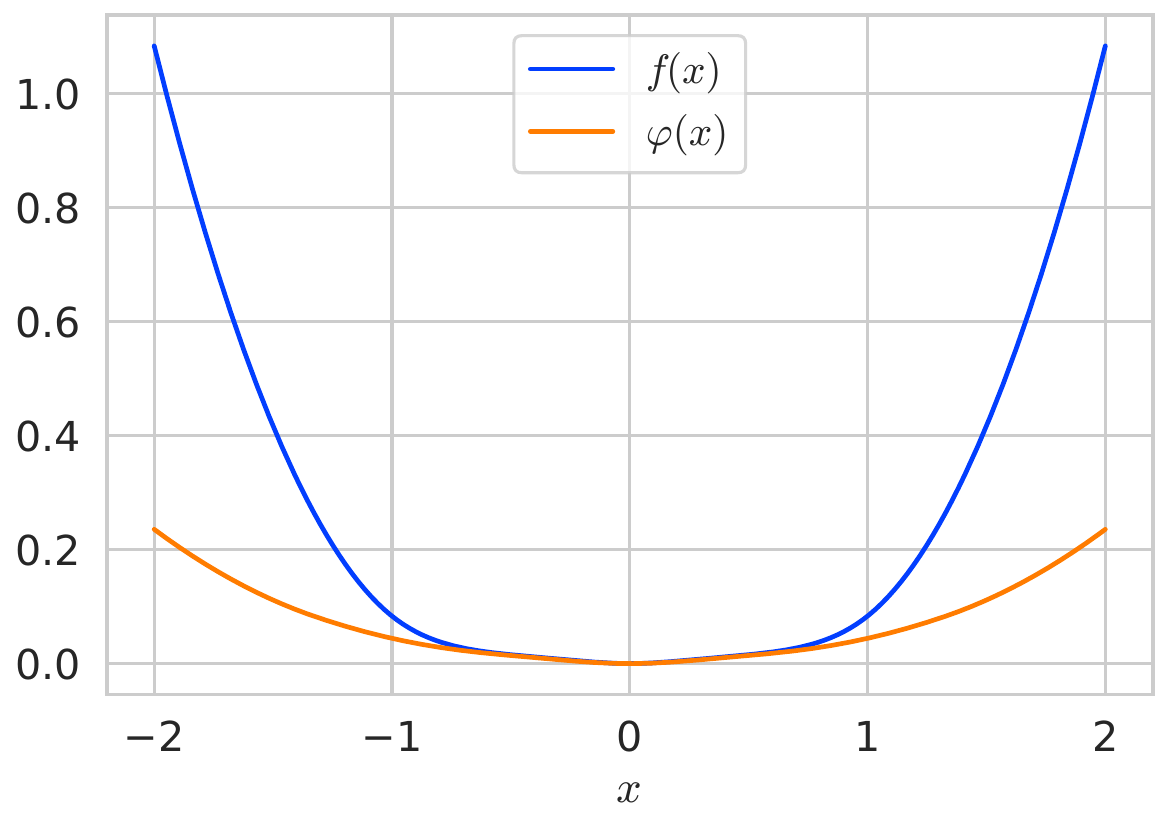}
		\caption[Visualization of \imaml{} meta-function]{Values of functions $f$ and $\varphi$.}
		\label{fig:example}
		\endminipage\hspace{0.2cm}
		\minipage[t]{0.48\textwidth}
		\centering
		\includegraphics[width=0.95\linewidth]{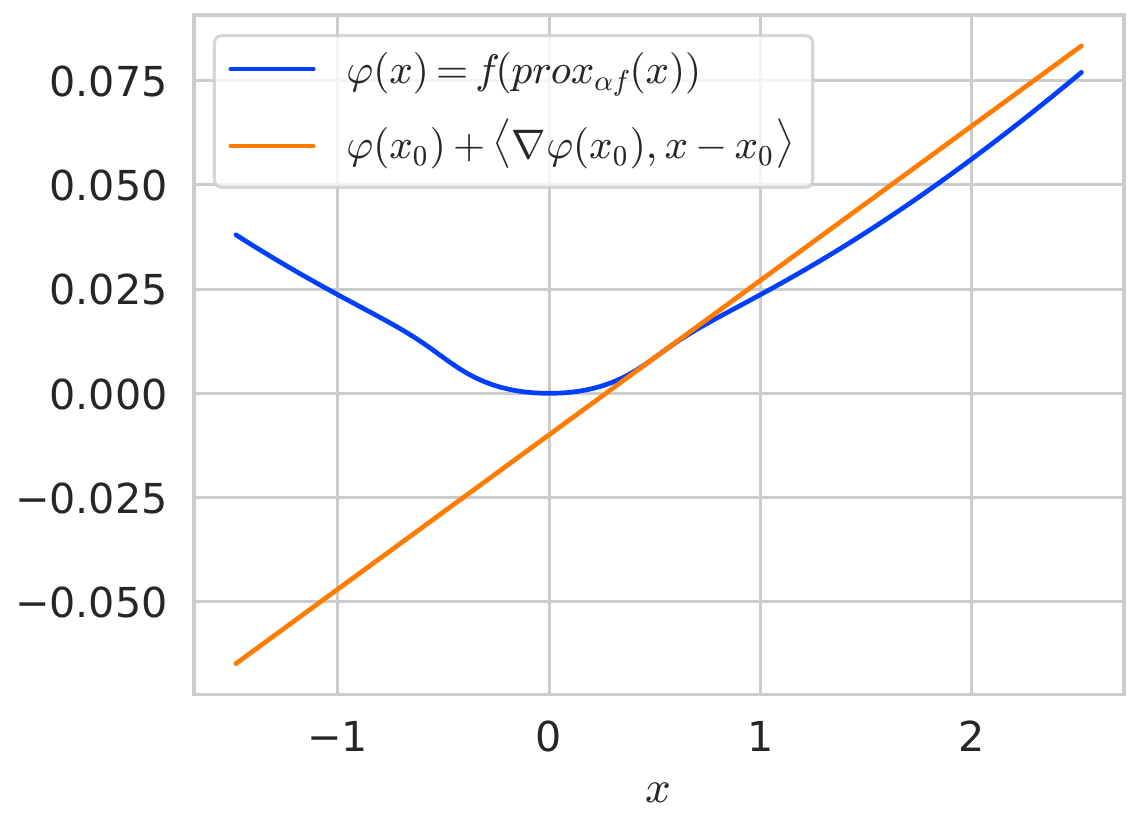}
		\caption[Example of \imaml{} meta-function nonconvexity]{Illustration of nonconvexity: the value of $\varphi$ goes below its tangent line from $x_0$, which means that $\varphi$ is nonconvex at $x_0$.}
		\endminipage
	\end{figure}
	
	It is straightforward to observe that this function is smooth and convex because its Hessian is
	\[
	f''(x)
	= \begin{cases}
		3x^2 - 2|x| + \frac{1}{3}, &\textrm{if } |x|\le 1,\\
		\frac{4}{3}, & \textrm{otherwise},
	\end{cases}
	\]
	which is always nonnegative and bounded. However, the function $\varphi(x) = f(z(x))$ is not convex at point $x_0=0.4 + \alpha\nabla f(0.4)$, because its Hessian is negative, i.e., $\varphi''(x_0)<0$, which we shall prove below. First of all, by definition of $x_0$, it holds that $0.4=x_0 - \alpha\nabla f(0.4)$, which is equivalent to the definition of $z(x)$, implying $z(x_0)=0.4$. Next, let us obtain the expression for the Hessian of $\varphi$. As shown in \citep{rajeswaran2019meta}, it holds in general that
	\[
	\nabla \varphi(x)
	= \frac{dz(x)}{dx}\nabla f(z(x)),
	\]
	where $\frac{dz(x)}{dx}$ is the Jacobian matrix of the mapping $z(x)$. Differentiating this equation again, we obtain
	\[
	\nabla^2 \varphi(x)
	= \frac{d^2z(x)}{dx^2}\nabla f(z(x)) + \nabla^2 f(z(x))\frac{dz(x)}{dx}\Bigl(\frac{dz(x)}{dx}\Bigr)^\top.
	\]
	Moreover, we can compute $\frac{d^2z(x)}{dx^2}$ by differentiating two times the equation $z(x) = x - \alpha \nabla f(z(x))$, which gives
	\[
	\frac{dz(x)}{dx} = \mathbf{I} - \alpha \nabla^2 f(z(x))\frac{dz(x)}{dx},
	\]
	where $\mathbf{I}$ is the identity matrix. Rearranging the terms in this equation yields 
	\[
	\frac{dz(x)}{dx} = (\mathbf{I} + \alpha \nabla^2 f(z(x)))^{-1}.
	\]
	At the same time, if we do not rearrange and instead differentiate the equation again, we get
	\[
	\frac{d^2z(x)}{dx^2}
	= -\alpha \nabla^2 f(z(x))\frac{d^2z(x)}{dx^2} - \alpha \nabla^3 f(z(x)) \left[\frac{dz(x)}{dx}, \frac{dz(x)}{dx}\right],
	\]
	where $\nabla^3 f(z(x))[\frac{dz(x)}{dx},\frac{dz(x)}{dx}]$ denotes tensor-matrix-matrix product, whose result is a tensor too. Thus,
	\[
	\frac{d^2z(x)}{dx^2}
	= -\alpha (\mathbf{I} + \alpha\nabla^2 f(z(x)))^{-1}\nabla^3 f(z(x))\left[\frac{dz(x)}{dx}, \frac{dz(x)}{dx}\right],
	\]
	and, moreover,
	\begin{eqnarray*}
		\nabla^2 \varphi(x)
		&=& -\alpha(\mathbf{I} + \alpha\nabla^2 f(z(x)))^{-1}\nabla^3 f(z(x))\left[\frac{dz(x)}{dx}, \frac{dz(x)}{dx}\right] \\
		&& \qquad + \nabla^2 f(z(x))\frac{dz(x)}{dx}\Bigl(\frac{dz(x)}{dx}\Bigr)^\top.
	\end{eqnarray*}
	
	For any $x\in(0, 1]$, our counterexample function satisfies $f''(x)=3x^2 - 2x + \frac{1}{3}$ and $f'''(x)=6x-2$. Moreover, since $z(x_0)=0.4$, we have $f''(z(x_0))=\frac{1}{75}$,  $f'''(z(x_0))=\frac{2}{5}$, $\frac{dz(x)}{dx}= \frac{1}{1+ \alpha/75}$, and 
	\[
	\varphi''(x)=- \frac{2\alpha}{5(1+ \alpha/75)^3} + \frac{1}{75(1+ \alpha/75)^2}.
	\]
	It can be verified numerically that $\varphi''(x)$ is negative at $x_0$ for any $\alpha>\frac{75}{2249}$. Notice that this value of $\alpha$ is much smaller than the value of $\frac{1}{{L_1}}=\frac{3}{4}$, which can be obtained by observing that our counterexample satisfies $f''(x)\le \frac{4}{3}$.
\end{proof}
Let us also note that obtaining nonconvexity of this objective for a fixed function and arbitrary $\alpha$ is somewhat challenging. Indeed, in the limit case $\alpha\to 0$, it holds that $\varphi(x)''\to f''(x)$ for any $x$. If $f''(x)>0$ then for a sufficiently small $\alpha$ it would also hold $\varphi''(x)>0$. Finding an example that works for any $\alpha$, thus, would require $f''(x_0)=0$.

\subsection{Proof of \Cref{th:imaml_nonsmooth}}
\begin{proof}
	Consider the following simple function
	\[
	f(x) = \frac{1}{2}x^2 + \cos(x).
	\]
	The Hessian of $f$ is $f''(x)=1 - \cos(x)\ge 0$, so it is convex. Moreover, it is apparent that the gradient and the Hessian of $f$ are Lipschitz. However, we will show that the Hessian of $\varphi$ is unbounded for any fixed $\alpha>0$. To establish this, let us first derive some properties of $z(x)$. First of all, by definition $z(x) $ is the solution of $\alpha f'(z(x))+(z(x)-x)=0$, where by definition of $f$, it holds $f'(z(x))=z(x) - \sin(z(x))$. Plugging it back, we get
	\[
	(\alpha + 1) z(x) - \alpha \sin (z(x)) = x.
	\]
	Differentiating both sides with respect to $x$, we get $(\alpha+1)\frac{d z(x)}{dx} - \alpha \cos(z(x))\frac{d z(x)}{dx} = 1$ and 
	\[
	\frac{dz(x)}{dx} = \frac{1}{1 + \alpha - \alpha \cos(z(x))}.
	\]	
	Thus, using the fact that $\varphi(x)=\varphi(z(x))$, we get
	\[
	\varphi'(x)
	= \frac{d\varphi(x)}{dx} 
	= \frac{df(z)}{dz}\frac{d z(x)}{dx}
	= \frac{z(x) - \sin(z(x)) }{1 + \alpha - \alpha \cos(z(x))}.
	\]
	Denoting, for brevity, $z(x)$ as $z$, we differentiate this identity with respect to $z$ and derive $\frac{d\varphi'(x)}{dz} = \frac{1 + 2\alpha - \alpha z \sin(z) - (1+2\alpha)\cos(z)}{(1+\alpha - \alpha\cos(z))^2}$.
	Therefore, for the Hessian of $\varphi$, we can produce an implicit identity,
	\[
	\varphi''(x)
	= \frac{d^2\varphi(x)}{dx^2} 
	= \frac{d\varphi'(x)}{dz}\frac{dz(x)}{dx}
	= \frac{1 + 2\alpha - \alpha z \sin(z) - (1+2\alpha)\cos(z)}{(1+\alpha - \alpha\cos(z))^3}.
	\]	
	The denominator of $\varphi''(x)$ satisfies $|1+\alpha - \alpha\cos(z)|^3\le (1 + 2\alpha)^3$, so it is bounded for any $x$. The numerator, on the other hand, is unbounded in terms of $z(x)$ since $|1 + 2\alpha - \alpha z \sin(z) - (1+2\alpha)\cos(z)|\ge \alpha |z\sin(z)| - 2(1+2\alpha)$. Therefore, $|\varphi''(x)|$ is unbounded. Moreover, $z(x)$ is itself unbounded, since the previously established identity for $z(x)$ can be rewritten as $|z(x)|=\left|\frac{1}{1+\alpha}x - \frac{\alpha}{1+\alpha}\sin(z(x))\right|\ge \frac{1}{1+\alpha}|x| - 1$. Therefore, $z(x)$ is unbounded, and since $\varphi''(x)$ grows with $z$, it is unbounded too. The unboundedness of $\varphi''(x)$ implies that $\varphi$ is not ${L_1}$--smooth for any finite ${L_1}$.
\end{proof}

\subsection{Proof of \Cref{lem:moreau_is_str_cvx_and_smooth}}
\begin{proof}
	The statement that $\meta_i$ is $\frac{\mu}{1+\alpha\mu}$--strongly convex is proven as \citep[Lemma 2.19]{planiden2016strongly}, so we skip this part.
	
	For nonconvex $\meta_i$ and any $x\in\R^d$, we have by first-order stationarity of the inner problem that $\nabla \meta_i(x) = \nabla f_i(z_i(x))$, where $z_i(x)=\arg\min_z \{f_i(z) + \frac{1}{2\alpha}\norm{z-x}^2\} = x - \alpha \nabla \meta_i(x)$. Therefore,
	\begin{align*}
		\norm{\nabla \meta_i(x) - \nabla \meta_i(y)}
		&= \norm{\nabla f_i(z_i(x)) - \nabla f_i(z_i(y))}
		\le{L_1}\norm{z_i(x) - z_i(y)}\\
		&= {L_1}\norm{x-y-\alpha(\nabla \meta_i(x) - \nabla \meta_i(y))}\\
		&\le {L_1}\norm{x-y} + \alpha{L_1}\norm{\nabla \meta_i(x) - \nabla \meta_i(y)}.
	\end{align*}
	Rearranging the terms, we get the desired bound:
	\[
	\norm{\nabla \meta_i(x) - \nabla \meta_i(y)}
	\le \frac{{L_1}}{1-\alpha {L_1}}\norm{x-y}.
	\]
	
	For convex functions, our proof of smoothness of $\meta_i$ follows the exact same steps as the proof of \citep[Lemma 2.19]{planiden2016strongly}. Let $f_i^*$ be the convex-conjugate of $f_i$. Then, it holds that $\meta_i = (f_i^* + \frac{\alpha}{2}\norm{\cdot}^2)^*$, see \citep[Theorem 6.60]{beck-book-first-order}. Therefore, $\meta_i^* = f_i^* + \frac{\alpha}{2}\norm{\cdot}^2$. Since $f_i$ is ${L_1}$--smooth, $f_i^*$ is $\frac{1}{{L_1}}$--strongly convex. Therefore, $\meta_i^*$ is $(\frac{1}{{L_1}} + \alpha)$--strongly convex, which, finally, implies that $\meta_i$ is $\frac{1}{\frac{1}{{L_1}}+\alpha}$--smooth.
	
	The statement $\frac{{L_1}}{1+\alpha {L_1}}\le {L_1}$ holds trivially since $\alpha>0$. In case $\alpha\le \frac{1}{\mu}$, we get the constants from the other statements by mentioning that $\frac{\mu}{1+\alpha\mu}\ge \frac{\mu}{2}$.
	
	The differentiability of $\meta_i$ follows from \citep[Theorem 4.4]{poliquin1996prox}, who show differentiability assuming $f_i$ is \emph{prox-regular}, which is a strictly weaker property than ${L_1}$--smoothness, so it automatically holds under the assumptions of \Cref{lem:moreau_is_str_cvx_and_smooth}.
\end{proof}

\subsection{Proof of \Cref{lem:approx_implicit}}
\textbf{\Cref{lem:approx_implicit}}.
Let task losses $f_i$ be ${L_1}$--smooth and $\alpha>0$. Given $i$ and $x\in\R^d$, we define recursively $z_{i,0}=x$ and $z_{i,j+1} = {\color{blue}x} - \alpha \nabla f_i({\color{mydarkred}z_{i,j}})$. Then, it holds for any $s\ge 0$
\begin{align*}
	\norm{ \nabla f_i(z_{i,s}) - \nabla \meta_i(x) }
	\le (\alpha {L_1})^{s+1} \norm{\nabla \meta_i(x)}.
\end{align*}
In particular, the iterates of \fomaml{} (\Cref{alg:fo_maml}) satisfy for any $k$
\begin{align*}
	\norm{ \nabla f_i(z_i^k) - \nabla \meta_i(x^k) }
	\le (\alpha {L_1})^2 \norm{\nabla \meta_i(x^k)}.
\end{align*}

\begin{proof}
	First, observe that by \cref{eq:implicit} it holds 
	\begin{align*}
		z_i(x)=x - \alpha \nabla\meta_i(x)=x-\alpha \nabla f_i(z_i(x)).
	\end{align*}
	For $s=0$, the lemma's claim then follows from initialization, $z_{i, 0}=x$, since 
	\[
	\norm{\nabla f_i(z_{i,s}) - \nabla \meta_i(x)} = \norm{\nabla f_i(x) - \nabla f_i(z_i(x))}
	\le{L_1}\norm{x - z_i(x)}
	= \alpha{L_1}\norm{\nabla \meta_i(x)}.
	\]
	For $s>0$, we shall prove the bound by induction.
	We have for any $l\ge 0$
	\begin{align*}
		\norm{z_{i,l+1} - (x - \alpha\nabla \meta_i(x))}
		&= \alpha \norm{\nabla f_i (z_{i,l}) - \nabla \meta_i(x)}
		= \alpha \norm{\nabla f_i (z_{i,l}) - \nabla f_i(z_i(x))}\\
		&\le \alpha {L_1}\norm{z_{i, l} - z_i(x)}\\
		&= \alpha {L_1}\norm{z_{i, l} - (x - \alpha\nabla \meta_i(x))}.
	\end{align*}
	This proves the induction step as well as the lemma itself.
\end{proof}

\begin{lemma}\label{lem:maml_approx_grad}
	If task losses $f_1,\dotsc, f_n$ are ${L_1}$--smooth and $\outers \le \frac{1}{{L_1}}$, then it holds
	\begin{align}
		\norm{\frac{1}{|T^k|}\sum_{i\in T^k} g_i^k}^2
		&\le \left(1 + 2 (\inners {L_1})^{2s} + \frac2 {|T|} \right)4{L_1}(\meta(x^k) - \meta(\opt)) \\
		& \qquad + 4\left(\frac1 {|T^k|} + (\inners {L_1})^{2s}\right)\sigma_*^2 \nonumber\\
		&\le 20{L_1}(\meta(x^k) - \meta(\opt)) + 4\left(\frac 1 {|T^k|} + \delta^2\right)\sigma_*^2. \label{eq:ave_grad_norm}
	\end{align}
\end{lemma}

\begin{proof}
	First, let us replace $g_i^k$ with $\nabla \meta_i(x^k)$, which $g_i^k$ approximates:
	\begin{eqnarray*}
		\norm{\frac{1}{|T^k|}\sum_{i\in T^k} g_i^k}^2
		&=& \norm{\frac{1}{|T^k|}\sum_{i\in T^k} \nabla\meta_i(x^k) + \frac{1}{|T^k|}\sum_{i\in T^k} (g_i^k - \nabla\meta_i(x^k))}^2 \\
		&\overset{\eqref{eq:a_plus_b}}{\le}& 2\norm{\frac{1}{|T^k|}\sum_{i\in T^k} \nabla\meta_i(x^k)}^2 + 2\norm{ \frac{1}{|T^k|}\sum_{i\in T^k} (g_i^k - \nabla\meta_i(x^k))}^2 \\
		&\overset{\eqref{eq:jensen_for_sq_norms}}{\le}& 2\norm{\frac{1}{|T^k|}\sum_{i\in T^k} \nabla\meta_i(x^k)}^2 + \frac{2}{|T^k|}\sum_{i\in T^k}\norm{ g_i^k - \nabla\meta_i(x^k)}^2 \\
		&\le& 2\norm{\frac{1}{|T^k|}\sum_{i\in T^k} \nabla\meta_i(x^k)}^2 + \frac{2}{|T^k|}\sum_{i\in T^k} \delta^2 \norm{\nabla\meta_i(x^k)}^2.
	\end{eqnarray*}	
	Taking the expectation on both sides, we get
	\begin{align*}
		\E{\norm{\frac{1}{|T^k|}\sum_{i\in T^k} g_i^k}^2}
		& \overset{\eqref{eq:rand_vec_sq_norm}}{\le} 2\norm{\nabla\meta(x^k)}^2 + 2\E{\norm{ \frac{1}{|T^k|}\sum_{i\in T^k}\nabla\meta_i(x^k) - \nabla\meta(x^k)}^2} \\
		& \qquad + \frac{2}{n}\sum_{i=1}^n \delta^2\norm{\nabla\meta_i(x^k)}^2.
	\end{align*}
	Moreover, each summand in the last term can be decomposed as
	\begin{align*}
		\norm{\nabla\meta_i(x^k)}^2
		&\overset{\eqref{eq:a_plus_b}}{\le} 2\norm{\nabla\meta_i(\opt)}^2 + 2\norm{\nabla \meta_i(x^k) - \nabla \meta_i(\opt)}^2\\ 
		&\overset{\eqref{eq:def_sigma}}{=} 2\sigma_*^2 + 2\norm{\nabla \meta_i(x^k) - \nabla \meta_i(\opt)}^2.
	\end{align*}
	Since $\meta_i$ is convex and ${L_1}$--smooth, we have for any $i$
	\begin{align*}
		\norm{\nabla \meta_i(x^k) - \nabla \meta_i(\opt)}^2 \le 2{L_1}(\meta_i(x^k) - \meta_i(\opt) - \la \nabla \meta_i(\opt), x^k-\opt \ra ).
	\end{align*}
	Averaging and using $\frac{1}{n}\sumin in \nabla \meta_i(\opt)=0$, we obtain
	\begin{align*}
		\frac{1}{n}\sum_{i=1}^n\norm{\nabla\meta_i(x^k)-\nabla\meta_i(\opt)}^2
		\le 2{L_1}(\meta(x^k) - \meta(\opt)).
	\end{align*}
	Thus,
	\begin{align}
		\frac{2}{n}\sum_{i=1}^n \delta^2 \norm{\nabla\meta_i(x^k)}^2
		&\le 4 \delta^2 \sigma_*^2 + 8L \delta^2 (\meta(x^k)-\meta(\opt)) \label{eq:meta_grad_norm} \\
		&\le 4 \delta^2 \sigma_*^2 + 8{L_1}(\meta(x^k)-\meta(\opt)). \notag
	\end{align}
	Proceeding to another term in our initial bound, by independence of sampling $i\in T^k$ we have
	\begin{eqnarray*}
		\E{\norm{ \frac{1}{|T^k|}\sum_{i\in T^k}\nabla\meta_i(x^k) - \nabla\meta(x^k)}^2}
		&=& \frac{1}{|T^k|} \avein in \E{\norm{\nabla\meta_i(x^k)}^2} \\
		&\overset{\eqref{eq:a_plus_b}}{\le}& \frac{2}{|T^k|} \avein in \Big(\E{\norm{\nabla\meta_i(x^k) - \nabla\meta_i(\opt)}^2} \\
		&& \hspace{3cm} + \E{\norm{\nabla \meta_i(\opt)}^2} \Big) \\
		&\overset{\eqref{eq:smooth_conv}}{\leq}& \frac 2 {|T^k|} \left(2L (\meta(x^k) - \meta(\opt)) + \sigma_*^2 \right)\\
		&\le& \frac {4{L_1}} {|T^k|} (\meta(x^k) - \meta(\opt)) + \frac 2 {|T^k|} \sigma_*^2.
	\end{eqnarray*}
	Finally, we also have $\norm{\nabla \meta(x^k)}^2\le 2{L_1}(\meta(x^k) - \meta(\opt))$. Combining all produced bounds, we get the claim
	\begin{align}
		\norm{\frac{1}{|T^k|}\sum_{i\in T^k} g_i^k}^2
		\le \left(1 + 2  \delta^2 + \frac 2 {|T|} \right)4{L_1}(\meta(x^k) - \meta(\opt)) + 4\left(\frac 1 {|T^k|} + \delta^2 \right)\sigma_*^2.
	\end{align}
\end{proof}

\subsection{Proof of \Cref{th:convergence_of_mamlP}}
\textbf{\Cref{th:convergence_of_mamlP}}.
Let task losses $f_1,\dotsc, f_n$ be ${L_1}$--smooth and $\mu$--strongly convex. If $|T^k|=\tau$ for all $k$, $\alpha\le\frac{1}{{L_1}}, \outers \leq \frac 1 {20{L_1}}$ and $\delta \leq \frac 1 {4 \sqrt \kappa}$, where $\kappa\eqdef \frac{{L_1}}{\mu}$, then the iterates of \Cref{alg:mamlP} satisfy
\begin{align*}
	\E{\norm{x^k-\opt}^2}
	\le \left(1 - \frac{\outers\mu}{4}\right)^k\norm{x^0-\opt}^2 + \frac{16}{\mu} 	\left( \frac {2\delta^2} \mu + \frac \outers {\tau} + \outers  \delta^2  \right) \sigma_*^2.
\end{align*}

\begin{proof}
	For the iterates of \Cref{alg:mamlP}, we can write
	\begin{align*}
		x^{k+1} 
		= x^k - \frac{\outers}{\tau}\sum_{i\in T^k} g_i^k.
	\end{align*}
	We also have by \Cref{lem:approx_implicit} that 
	\[
	\norm{g_i^k - \nabla \meta_i(x^k)}^2 \le (\alpha {L_1})^2\delta^2 \norm{\nabla \meta_i (x^k)}^2 \le \delta^2 \norm{\nabla \meta_i (x^k)}^2,
	\]
	so let us decompose $g_i^k$ into $\nabla \meta_i(x^k)$ and the approximation error:
	\begin{align*}
		\norm{x^{k+1} - \opt}^2
		&= \norm{x^k - \opt}^2 - \frac{2\outers}{\tau}\sum_{i\in T^k} \la g_i^k, x^k - \opt \ra + \outers^2\norm{\frac{1}{\tau}\sum_{i\in T^k} g_i^k}^2 \\
		&= \norm{x^k - \opt}^2 - \frac{2\outers}{\tau}\sum_{i\in T^k} \la\nabla\meta_i(x^k), x^k - \opt \ra \\
		&\qquad + \frac{2\outers}{\tau}\sum_{i\in T^k} \la \nabla\meta_i(x^k) - g_i^k, x^k - \opt \ra  + \outers^2\norm{\frac{1}{\tau}\sum_{i\in T^k} g_i^k}^2.
	\end{align*}
	First two terms can be upperbounded using strong convexity (recall that by \Cref{lem:moreau_is_str_cvx_and_smooth}, $\meta_i$ is $\frac{\mu}{2}$--strongly convex):
	\begin{align*}
		\norm{x^k - \opt}^2 - \frac{2\outers}{\tau}\sum_{i\in T^k} \la\nabla\meta_i(x^k), x^k - \opt \ra
		&\le \left(1 - \frac{\outers\mu}{2}\right)\norm{x^k - \opt}^2 \\
		& \qquad - \frac{2\outers}{\tau}\sum_{i\in T^k}(\meta_i(x^k) - \meta_i(\opt)).
	\end{align*}
	For the third term, we will need Young's inequality:	
	\begin{align*}
		2\la \nabla\meta_i(x^k) - g_i^k, x^k - \opt \ra 
		&\overset{\eqref{eq:young}}{\le}\frac{4}{\mu}\norm{\nabla\meta_i(x^k) - g_i^k }^2 + \frac{\mu}{4}\norm{x^k - \opt}^2 \\
		& \le\frac{4}{\mu} \delta^2 \norm{\nabla\meta_i(x^k)}^2 + \frac{\mu}{4}\norm{x^k - \opt}^2,
	\end{align*}
	which we can scale by $\beta$ and average over $i\in T^k$ to obtain
	\[
	\frac{2\outers}{\tau}\sum_{i\in T^k} \la \nabla\meta_i(x^k) - g_i^k, x^k - \opt \ra 
	\leq  \frac{4 \outers \delta^2}{\mu} \frac 1 \tau \sum_{i \in T^k} \norm{\nabla \meta_i(x^k)}^2 + \frac{\beta \mu}{4}\norm{x^k - \opt}^2.
	\]

	Plugging in upper bounds and taking expectation yields
	\begin{eqnarray*}
		\efa
		\E{\norm{x^{k+1}-\opt}^2} \\
		&\le& \left(1 - \frac{\outers\mu}{4}\right)\norm{x^k-\opt}^2 - 2\outers (\meta(x^k) - \meta(\opt)) \\
		&& \qquad + \frac{4}{\mu}\outers \delta^2 \avein in \norm{\nabla \meta_i(x^k)}^2 + \outers^2\norm{ \frac{1}{\tau}\sum_{i\in T^k} g_i^k}^2  \\
		&\overset{\eqref{eq:ave_grad_norm}}{\le} &\left(1 - \frac{\outers\mu}{4}\right)\norm{x^k-\opt}^2 - 2\outers(1 - 10\beta {L_1})  (\meta(x^k) - \meta(\opt)) \\
		&& \qquad + \frac{4}{\mu}\outers  \delta^2 \avein in\norm{\nabla \meta_i(x^k)}^2 + 4\outers^2 \left(\frac 1 {\tau} + \delta^2 \right)\sigma_*^2 \\
		&\overset{\eqref{eq:meta_grad_norm}}{\le}& \left(1 - \frac{\outers\mu}{4}\right)\norm{x^k-\opt}^2 - 2\outers(1 - 10\outers {L_1})  (\meta(x^k) - \meta(\opt))\\
		&&\qquad + \frac{8}{\mu}\outers \delta^2\left(\sigma_*^2 + 2{L_1}(\meta(x^k) - \meta(\opt))\right) + 4\outers^2 \left(\frac 1 {\tau} + \delta^2\right)\sigma_*^2 \\
		&=& \left(1 - \frac{\outers\mu}{4}\right)\norm{x^k-\opt}^2 - 2\outers \left(1 - 10\outers{L_1} - \frac {8{L_1}} \mu  \delta^2 \right)  (\meta(x^k) - \meta(\opt)) \\
		&& \qquad + \frac{8}{\mu}\outers \delta^2 \sigma_*^2 + 4\outers^2 \left(\frac 1 {\tau} + \delta^2 \right)\sigma_*^2.
	\end{eqnarray*}
	By assumption $\outers \leq \frac 1 {20{L_1}}, \delta \leq \frac 1 {4\sqrt{ \kappa}}$, we have $10\outers{L_1}\leq \frac 1 2$ and $8 \frac{{L_1}}{\mu} \delta^2\leq  \frac 1 2$, so $1 - 10\outers{L_1}- \frac {8{L_1}} \mu  \delta^2  \geq 0$, hence
	\begin{align*}
		\E{\norm{x^{k+1}-\opt}^2} 
		&\le \left(1 - \frac{\outers\mu}{4}\right)\norm{x^k-\opt}^2 + \frac{8}{\mu}\outers \delta^2 \sigma_*^2 + 4\outers^2 \left(\frac 1 {\tau} +  \delta^2 \right)\sigma_*^2.
	\end{align*}
	Recurring this bound, which is a standard argument, we obtain the theorem's claim.
	\begin{align*}
		\E{\norm{x^{k}-\opt}^2} 
		&\le \left(1 - \frac{\outers\mu}{4}\right)^k\norm{x^0-\opt}^2 \\
		& \qquad+ \left( \frac{8}{\mu}\outers \delta^2 \sigma_*^2 + 4\outers^2 \left(\frac1 {\tau} + \delta^2\right)\sigma_*^2 \right) \frac{1-\left( 1- \frac{\outers \mu} 4 \right)^k}{\frac{\outers \mu} 4}\\
		&\le \left(1 - \frac{\outers\mu}{4}\right)^k\norm{x^0-\opt}^2 + \frac{32}{\mu^2} \delta^2 \sigma_*^2  + \frac{16}{\mu \tau} \outers \sigma_*^2 + \frac{16}{\mu} \outers  \delta^2 \sigma_*^2\\
		&\le \left(1 - \frac{\outers\mu}{4}\right)^k\norm{x^0-\opt}^2 + \frac{16}{\mu} \left( \frac {2\delta^2} \mu + \frac \outers {\tau} + \outers  \delta^2  \right) \sigma_*^2.
	\end{align*}
\end{proof}

\subsection{Proof of \Cref{th:convengence_of_mamlP_no_stepsize}}

\textbf{\Cref{th:convengence_of_mamlP_no_stepsize}}.
Consider the iterates of \Cref{alg:mamlP} (with general $\delta$) or \Cref{alg:fo_maml} (for which $\delta=\alpha {L_1}$).
Let task losses be ${L_1}$--smooth and $\mu$--strongly convex and let objective parameter satisfy $\inners \leq \frac {1}{\sqrt 6 {L_1}}$. Choose stepsize $ \beta \leq \frac \tau {4 {L_1}}$, where $\tau = |T^k|$ is the batch size. Then we have
\begin{align*}
	\E{\norm{x^k-x^*}^2} \leq \left(1 - \frac {\beta \mu}{12}  \right)^k \norm{x^0 - x^*}^2 + \frac { 6\left( \frac \beta \tau + 3 \delta^2 \inners^2{L_1}\right) \varopt} {\mu}.
\end{align*}

\begin{proof}
	Denote $L_\meta$, $\mu_\meta, \kappa_\meta = \frac \Lmeta {\mu_\meta}$ smoothness constant, strong convexity constant, condition number of Meta-Loss functions $\meta_1, \dots, \meta_n$, respectively. 
	We have
	\begin{align*}
		\norm{x^{k+1}-x^*}^2
		&= \norm{x^k - x^* - \frac \beta \tau \sum_{i \in T^k}  \nabla \meta_i (y_i^k)}^2\\
		&=\norm{x^k - x^*}^2 - \frac {2 \beta} \tau \sum_{i \in T^k} \langle \nabla \meta_i (y_i^k), x^k-x^* \rangle + \beta^2 \norm{ \frac 1 \tau \sum_{i \in T^k} \nabla  \meta_i(y_i^k)}^2\\
		&\leq \norm{x^k - x^*}^2 - \frac {2 \beta} \tau \sum_{i \in T^k} \langle \nabla \meta_i (y_i^k) - \nabla \meta_i(x^*), x^k-x^* \rangle \\
		& \qquad + 2\beta^2 \norm{ \frac 1 \tau \sum_{i \in T^k} (\nabla  \meta_i(y_i^k) - \nabla  \meta_i(x^*) ) }^2\\
		& \qquad  - \frac {2 \beta} \tau \sum_{i \in T^k} \langle \nabla \meta_i (x^*), x^k-x^* \rangle + 2 \beta^2 \norm{\frac 1 \tau \sum_{i \in T^k} \nabla  \meta_i(x^*)}^2.
	\end{align*}
	
	Using Proposition~\ref{pr:three_point}, we rewrite the scalar product as
	$
	\langle \nabla \meta_i (y_i^k) - \nabla \meta_i(x^*), x^k-x^* \rangle  
	= D_{\meta_i}(x^*, y_i^k) + D_{\meta_i}(x^k, x^*) - D_{\meta_i}( x^k, y_i^k) 
	$,
	which gives
	\begin{align*}
		\norm{x^{k+1}-x^*}^2
		& \leq \norm{x^k - x^*}^2 - \frac {2 \beta} \tau \sum_{i \in T^k} \left[ D_{\meta_i}(x^*, y_i^k) + D_{\meta_i}(x^k, x^*) - D_{\meta_i}( x^k, y_i^k) \right]\\
		&+ 2\beta^2 \norm{\frac 1 \tau \sum_{i \in T^k} (\nabla  \meta_i(y_i^k) - \nabla  \meta_i(x^*) ) }^2  - \frac {2 \beta} \tau \sum_{i \in T^k} \langle \nabla \meta_i (x^*), x^k-x^* \rangle \\
		&\qquad + 2 \beta^2 \norm{\frac 1 \tau \sum_{i \in T^k} \nabla  \meta_i(x^*)}^2.
	\end{align*}
	Since we sample $T^k$ uniformly and $\{\nabla \meta_i (x^*)\}_{i\in T^k}$ are independent random vectors, 
	\begin{align*}
		\E{\norm{x^{k+1}-x^*}^2}
		&\leq \norm{x^k - x^*}^2 \\
		& \qquad + \frac{2 \beta} \tau \E{ \sum_{i \in T^k} \left[ D_{\meta_i}( x^k, y_i^k) - D_{\meta_i}(x^*, y_i^k) - D_{\meta_i}(x^k, x^*)\right] }\\
		& \qquad + \frac {2\beta^2} {\tau^2} \E{ \sum_{i \in T^k}  \norm{ \nabla  \meta_i(y_i^k) - \nabla  \meta_i(x^*) }^2 } + \frac {2 \beta^2} \tau \varopt.
	\end{align*}
	Next, we are going to use the following three properties of Bregman divergence:
	\begin{align}
		-D_{\meta_i}( x^*, y_i^k)
		& \overset{\eqref{eq:smooth_conv}}{\le} -\frac 1 {2 L_\meta} \norm{ \nabla  \meta_i(y_i^k) - \nabla  \meta_i(x^*) }^2  \notag \\
		-D_{\meta_i}(x^k, x^*)
		&\le  -\frac {\mu_\meta} 2\norm {x^k-x^*}^2\label{eq:div2} \\
		D_{\meta_i}(x^k, y_i^k)
		& \le \frac \Lmeta 2 \norm{x^k-y_i^k}^2. \notag 
	\end{align}
	Moreover, using identity $y_i^k = z_i^k + \alpha \nabla F_i(y_i^k)$, we can bound the last divergence as
	\begin{align*}
		D_{\meta_i}(x^k, y_i^k)
		&\le \frac \Lmeta 2 \norm{x^k-z_i^k - \inners \nabla \meta_i (y_i^k)}^2\\
		&= \frac 1 2 \inners^2 \Lmeta \norm{\frac 1 \inners (x^k-z_i^k) - \nabla \meta_i(y_i^k)}^2\\
		& \le \frac 3 2 \inners^2 \Lmeta \Big( \norm{\frac 1 \inners (x^k-z_i^k) - \nabla \meta_i(x^k)}^2 + \norm{\nabla \meta_i (x^k)- \nabla \meta_i(x^*)}^2 \\
		& \qquad + \norm{\nabla \meta_i (x^*)- \nabla \meta_i(y_i^k)}^2  \Big)\\
		& \le \frac 3 2 \inners^2 \Lmeta \Big( \delta^2 \norm{\nabla \meta_i(x^k)}^2 + \norm{\nabla \meta_i (x^k)- \nabla \meta_i(x^*)}^2 \\
		& \qquad + \norm{\nabla \meta_i (x^*)- \nabla \meta_i(y_i^k)}^2  \Big),
	\end{align*}
	where the last step used the condition in \Cref{alg:mamlP}. Using inequality~\eqref{eq:a_plus_b} on $\nabla \meta_i(x^k) = \nabla \meta_i(x^*) + (\nabla \meta_i(x^k) - \nabla \meta_i(x^*))$, we further derive
	\begin{align*}
		D_{\meta_i}(x^k, y_i^k)
		& \le \frac 3 2 \inners^2 \Lmeta \Big( 2\delta^2 \norm{\nabla \meta_i(x^*)}^2 + (1+2\delta^2) \norm{\nabla \meta_i (x^k)- \nabla \meta_i(x^*)}^2 \\
		& \qquad + \norm{\nabla \meta_i (x^*)- \nabla \meta_i(y_i^k)}^2  \Big)\\
		& \overset{\eqref{eq:smooth_conv}}{\le} \frac 3 2 \inners^2 \Lmeta \Big( 2\delta^2 \norm{\nabla \meta_i(x^*)}^2 + (1+2\delta^2)\Lmeta D_{\meta_i}(x^k,x^*) \\
		& \qquad + \norm{\nabla \meta_i (x^*)- \nabla \meta_i(y_i^k)}^2  \Big).	
	\end{align*}
	Assuming $\inners\leq \sqrt{\frac 2 3 (1+2\delta^2)} \frac 1 {L_\meta}$ so that $ 1- \frac 3 2 \inners^2 \Lmeta^2 (1+2\delta^2) >0$, we get
	\begin{eqnarray*}
		\efa D_{\meta_i}(x^k, y_i^k) -D_{\meta_i}(x^k, x^*)\\
		&\le& - \left( 1- \frac 3 2 \inners^2 \Lmeta^2 (1+2\delta^2) \right) D_{\meta_i}(x^k,x^*) \\
		&&\qquad + \frac3 2 \inners^2 \Lmeta \left( 2\delta^2 \norm{\nabla \meta_i(x^*)}^2 + \norm{\nabla \meta_i (x^*)- \nabla \meta_i(y_i^k)}^2  \right)\\
		&\stackrel{\eqref{eq:div2}}\le& - \left( 1- \frac 3 2 \inners^2 \Lmeta^2 (1+2\delta^2) \right) \frac {\mu_\meta} 2\norm {x^k-x^*}^2\\
		&&\qquad +  \frac3 2 \inners^2 \Lmeta \left( 2\delta^2 \norm{\nabla \meta_i(x^*)}^2 + \norm{\nabla \meta_i (x^*)- \nabla \meta_i(y_i^k)}^2  \right).
	\end{eqnarray*}
	
	Plugging these inequalities yields
	\begin{align*}
		\E{\norm{x^{k+1}-x^*}^2}	
		&\leq \left(1 - \beta \mu_\meta \left( 1- \frac 3 2 \inners^2 \Lmeta^2 (1+2\delta^2) \right)  \right) \norm{x^k - x^*}^2 \\
		& \quad+  \frac \beta \tau \left(3 \inners^2 \Lmeta +  \frac {2\beta} \tau - \frac 1 {L_\meta}  \right) \E{ \sum_{i \in T^k}  \norm{ \nabla  \meta_i(y_i^k) - \nabla  \meta_i(x^*) }^2 }\\
		& \quad + 2 \beta\left( \frac \beta \tau + 3 \inners^2 \delta^2 \Lmeta \right)\varopt.
	\end{align*}

	Now, if $\inners \leq \frac {1}{\sqrt 6 L_\meta}$ and $ \beta \leq \frac \tau {4 L_\meta}$, then $3 \inners^2 \Lmeta +  \frac {2\beta} \tau  - \frac 1 {L_\meta} \leq 0$, and consequently
	\begin{align*}
		\E{\norm{x^{k+1}-x^*}^2}
		&\leq \left(1 - \beta \mu_\meta \left( 1- \frac 3 2 \inners^2 \Lmeta^2 (1+2\delta^2) \right)  \right) \norm{x^k - x^*}^2 \\
		& \qquad + 2 \beta \left( \frac \beta \tau + 3 \inners^2 \delta^2 \Lmeta \right)\varopt.
	\end{align*}
	
	We can unroll the recurrence to obtain the rate
	\begin{eqnarray*}
		\efa \E{\norm{x^k-x^*}^2}\\
		& \le& \left(1 - \beta \mu_\meta \left( 1- \frac 3 2 \inners^2 \Lmeta^2 (1+2\delta^2) \right)  \right)^k \norm{x^0 - x^*}^2 \\
		&&\qquad + \left( \sum_{i=0}^{k-1} \left(1 - \beta \mu_\meta \left( 1- \frac3 2 \inners^2 \Lmeta^2 (1+2\delta^2) \right)  \right)^i \right) 2 \beta \left( \frac\beta \tau + 3 \inners^2 \delta^2 \Lmeta \right)\varopt\\
		&=& \left(1 - \beta \mu_\meta \left( 1- \frac 3 2 \inners^2 \Lmeta^2 (1+2\delta^2) \right)  \right)^k \norm{x^0 - x^*}^2 \\
		&&\qquad + \left( \frac {1- \left(1 - \beta \mu_\meta \left( 1- \frac3 2 \inners^2 \Lmeta^2 (1+2\delta^2) \right)  \right)^k }{1- \frac3 2 \inners^2 \Lmeta^2 (1+2\delta^2)} \right) \left( \frac\beta \tau + 3 \inners^2 \delta^2 \Lmeta \right) \frac{2\varopt} {\mu_\meta} \\
		&\leq& \left(1 - \beta \mu_\meta \left( 1- \frac 3 2 \inners^2 \Lmeta^2 (1+2\delta^2) \right)  \right)^k \norm{x^0 - x^*}^2 \\
		&& \qquad + \frac {2 \left( \frac \beta \tau + 3 \inners^2 \delta^2 \Lmeta \right) \varopt} {\mu_\meta (1- \frac 3 2 \inners^2 \Lmeta^2 (1+2\delta^2))}.
	\end{eqnarray*}
	
	Choice of $\delta$ implies $0 \leq \delta \leq 1$; \Cref{pr:moreau_is_smooth} yields $\frac \mu 2 \leq \frac \mu {1+\inners \mu} \leq \mu_\meta$ and $L_\meta \leq \frac{L_1}{1 + \inners {L_1}} \leq {L_1}$, so we can simplify
	\begin{align*}
		\E{\norm{x^k-x^*}^2} \leq \left(1 - \frac {\beta \mu}2 \left( 1- 5 \inners^2 {L_1}^2 \right)  \right)^k \norm{x^0 - x^*}^2 + \frac { 4\left( \frac \beta \tau + 3 \inners^2{L_1}\delta^2 \right) \varopt} {\mu (1 - 2 \inners^2 {L_1}^2 )}.
	\end{align*}
\end{proof}
\subsection{Proof of \Cref{th:nonconvex_fo_maml}}
\textbf{Theorem~\ref{th:nonconvex_fo_maml}} 
Let the variance of meta-loss gradients is uniformly bounded by some $\sigma^2$ (\Cref{def:bounded_var}), functions $f_1,\dotsc, f_n$ be ${L_1}$--smooth and $F$ be lower bounded by $F^*>-\infty$. Assume $\alpha\le \frac{1}{4{L_1}}, \beta\le \frac{1}{16{L_1}}$. If we consider the iterates of \Cref{alg:fo_maml} (with $\delta=\alpha {L_1}$) or \Cref{alg:mamlP} (with general $\delta$), then
\begin{align*}
	\min_{t\le k}\E{\norm{\nabla F(x^t)}^2}
	\le \frac{4}{\beta k}\E{F(x^0)-F^*}+ 16 \beta(\alpha {L_1})^2 \left(\frac{1}{|T^k|} + (\alpha {L_1})^2\delta^2\right) \sigma^2.
\end{align*}
\begin{proof}
	We would like to remind the reader that for our choice of $z_i^k$ and $y_i^k$, the following three identities hold. Firstly, by definition $y_i^k = z_i^k + \alpha \nabla f_i(z_i^k)$. Secondly, as shown in Lemma~\ref{lem:explicit_grad_to_implicit}, $ z_i^k = y_i^k - \alpha\nabla \meta_i(y_i^k)$. And finally, $\nabla f_i(z_i^k) = \nabla \meta_i(y_i^k)$. We will frequently  use these identities in the proof.
	
	Since functions $f_1,\dotsc, f_n$ are ${L_1}$--smooth and $\alpha \le \frac{1}{4{L_1}}$, functions $F_1,\dotsc, F_n$ are $(2{L_1})$--smooth as per \Cref{lem:moreau_is_str_cvx_and_smooth}. Therefore, by smoothness of $F$, we have for the iterates of \Cref{alg:mamlP}
	\begin{eqnarray*}
		\efa \E{F(x^{k+1})}\\
		&\overset{\eqref{eq:smooth_func_vals}}{\le}& \E{F(x^k) + \la \nabla F(x^k), x^{k+1}-x^k \ra + {L_1}\norm{x^{k+1}-x^k}^2 }\\
		&=& \E{F(x^k) - \beta  \langle\nabla F(x^k), \aveis i {T^k} \nabla f_i(z_i^k) \rangle + \beta^2{L_1}\norm{\aveis i {T^k} \nabla f_i(z_i^k)}}^2\\
		&=& F(x^k) - \beta \norm{\nabla F(x^k)}^2 + \beta\E{\langle\nabla F(x^k), \nabla F(x^k) - \avein in \nabla f_i(z_i^k)}  \\
		&& \qquad + \beta^2{L_1}\E{\norm{\aveis i{T^k} \nabla f_i(z_i^k)}^2 } \\
		&\overset{\eqref{eq:a_plus_b}}{\le}& F(x^k) - \frac{\beta}{2} \norm{\nabla F(x^k)}^2 + \frac{\beta}{2} \avein in \norms{\nabla F_i(x^k) - \nabla f_i(z_i^k)} \\
		&& \qquad + \beta^2{L_1}\E{\norm{\aveis i {T^k}\nabla f_i(z_i^k)}^2}.
	\end{eqnarray*}
	Next, let us observe, similarly to the proof of Lemma~\ref{lem:maml_approx_grad}, that the gradient approximation error satisfies
	\begin{align*}
		\norm{\nabla F_i(x^k) - \nabla f_i(z_i^k)}
		&= \norm{\nabla F_i(x^k) - \nabla \meta_i(y_i^k)}\\
		& \le {L_1}\norm{x^k - y_i^k} 
		={L_1}\norm{x^k - z_i^k - \alpha \nabla\meta_i(y_i^k)} \\
		&\le \alpha {L_1}\norm{ \nabla \meta(x^k) -\nabla\meta_i(y_i^k)} + \alpha {L_1}\norm{\frac{1}{\alpha}(x^k-z_i^k) - \nabla \meta_i(x^k)}\\
		&= \alpha {L_1}\norm{\nabla \meta(x^k) -\nabla f_i(z_i^k) } + \alpha {L_1}\norm{\frac{1}{\alpha}(x^k-z_i^k) - \nabla \meta_i(x^k)}.
	\end{align*}
	By rearranging and using our assumption on error $\delta$ as formulated in \Cref{alg:mamlP}, we have
	\begin{eqnarray*}	    
		\norm{\nabla \meta_i(x^k) -\nabla f_i(z_i^k)}
		&\le& \frac{\alpha {L_1}}{1-\alpha {L_1}}\norm{\frac{1}{\alpha}(x^k-z_i^k) - \nabla \meta_i(x^k)}\\
		&\le& \frac{\alpha {L_1}}{1-\alpha {L_1}}\delta\norm{\nabla \meta_i(x^k)}\\
		&\overset{\alpha\le \frac{1}{4{L_1}}}{\le}& \frac{4}{3}\alpha {L_1}\delta \norm{\nabla \meta_i(x^k)}.
	\end{eqnarray*}
	Squaring this bound and averaging over $i$, we obtain
	\begin{eqnarray*}
		\efa \avein in\norms{\nabla \meta_i(x^k) -\nabla f_i(z_i^k)}\\
		&\le& \frac{16}{9}(\alpha {L_1})^2\delta^2 \avein in\norm{\nabla \meta_i(x^k)}^2 \\
		&=& \frac{16}{9}(\alpha {L_1})^2\delta^2 \norm{\nabla \meta(x^k)}^2 \\
		&& \qquad + \frac{16}{9}(\alpha {L_1})^2\delta^2 \avein in\norm{\nabla \meta_i(x^k) - \nabla \meta (x^k)}^2 \\
		&\overset{\eqref{eq:bounded_var}}{\le}& \frac{16}{9}(\alpha {L_1})^2\delta^2 \norm{\nabla \meta(x^k)}^2 + \frac{16}{9}(\alpha {L_1})^2\delta^2 \sigma^2 \\
		&\le& \frac{1}{9}\norm{\nabla \meta(x^k)}^2 + 2 (\alpha {L_1})^2\delta^2 \sigma^2.
	\end{eqnarray*}
	For the other term in the smoothness upper bound, we can write
	\begin{eqnarray*}
		\efa \E{\norm{\frac{1}{|T^k|}\sum_{i\in T^k} \nabla f_i(z_i^k)}^2}\\
		& =& \E{\norm{\frac{1}{|T^k|}\sum_{i\in T^k} \nabla\meta_i(x^k) + \frac{1}{|T^k|}\sum_{i\in T^k} (\nabla f_i(z_i^k) - \nabla\meta_i(x^k))}^2} \\
		&\overset{\eqref{eq:a_plus_b}}{\le}& 2\E{\norm{\frac{1}{|T^k|}\sum_{i\in T^k} \nabla\meta_i(x^k)}^2} + 2\E{\norm{\frac{1}{|T^k|}\sum_{i\in T^k} (\nabla f_i(z_i^k) - \nabla\meta_i(x^k))}^2} \\
		& \overset{\eqref{eq:jensen_for_sq_norms}}{\le}& 2\E{\norm{\frac{1}{|T^k|}\sum_{i\in T^k} \nabla\meta_i(x^k)}^2} + \frac{2}{|T^k|}\E{\sum_{i\in T^k}\norm{\nabla f_i(z_i^k) - \nabla\meta_i(x^k)}^2} \\
		&\le& 2\E{\norm{\frac{1}{|T^k|}\sum_{i\in T^k} \nabla\meta_i(x^k)}^2} + \E{\frac{32}{9}\frac{1}{|T^k|}\sum_{i\in T^k} (\alpha {L_1})^2\delta^2 \norm{\nabla\meta_i(x^k)}^2}.
	\end{eqnarray*}
	Using bias-variance decomposition, we get for the first term in the right-hand side
	\begin{eqnarray*}
		\efa 2\E{\norm{\frac{1}{|T^k|}\sum_{i\in T^k} \nabla\meta_i(x^k)}^2}\\
		&\overset{\eqref{eq:rand_vec_sq_norm}}{=}& 2\norm{\nabla \meta (x^k)}^2+ 2\E{\norm{\frac{1}{|T^k|}\sum_{i\in T^k} \nabla\meta_i(x^k) - \nabla\meta(x^k)}^2} \\
		&=& 2\norm{\nabla \meta (x^k)}^2 + \frac{2}{|T^k|}\frac{1}{n}\sum_{i=1}^n\norm{\nabla\meta_i(x^k) - \nabla\meta(x^k)}^2.
	\end{eqnarray*}
	Similarly,  we simplify the second term using $\frac{32}{9}< 4$ and then obtain
	\begin{eqnarray*}
		\efa \frac{32}{9}\E{\frac{1}{|T^k|}\sum_{i\in T^k} (\alpha {L_1})^2\delta^2 \norm{\nabla\meta_i(x^k)}^2}\\
		&\overset{\eqref{eq:rand_vec_sq_norm}}{\le}& 4(\alpha {L_1})^2\delta^2\norm{\nabla\meta(x^k)}^2 + \frac{4(\alpha {L_1})^2\delta^2}{n}\sum_{i=1}^n\norm{\nabla\meta_i(x^k) - \nabla\meta(x^k)}^2.
	\end{eqnarray*}
	Thus, using $\alpha\le\frac{1}{4{L_1}}$ and $\delta\le 1$, we get
	\begin{eqnarray*}
		\efa\E{\norms{\frac{1}{|T^k|}\sum_{i\in T^k} \nabla f_i(z_i^k)}}\\
		&\le& 3\norm{\nabla \meta (x^k)}^2 
		+ \left(\frac{2}{|T^k|} + 4(\alpha {L_1})^2\delta^2\right)\sum_{i=1}^n\norm{\nabla\meta_i(x^k) - \nabla\meta(x^k)}^2 \\
		&\overset{\eqref{eq:bounded_var}}{\le}& 3\norm{\nabla \meta (x^k)}^2 + 4\left(\frac{1}{|T^k|} + (\alpha {L_1})^2\delta^2\right)\sigma^2.
	\end{eqnarray*}
	Now we plug these inequalities back and continue:
	\begin{eqnarray*}
		\E{F(x^{k+1})} - F(x^k)
		& \leq& - \frac{\beta}{2} \norm{\nabla F(x^k)}^2 + \frac{\beta}{18}\norm{\nabla\meta (x^k)}^2 + \beta(\alpha {L_1})^2\delta^2 \sigma^2 \\
		&&\qquad  + 3\beta^2 {L_1}\norm{\nabla \meta (x^k)}^2 \\
		&&\qquad + 4\beta^2 {L_1}\sigma^2\left(\frac{1}{|T^k|} + (\alpha {L_1})^2\delta^2\right)\sigma^2 \\
		&\overset{\beta\le \frac{1}{16{L_1}}}{\le }&   - \frac{\beta}{4} \norm{\nabla F(x^k)}^2 + 4\beta^2 {L_1}\sigma^2\left(\frac{1}{|T^k|} + (\alpha {L_1})^2\delta^2\right)\sigma^2 \\
		&& \qquad + \beta(\alpha {L_1})^2\delta^2 \sigma^2.
	\end{eqnarray*}
	Rearranging the terms and telescoping this bound, we derive
	\begin{align*}
		\min_{t\le k}\E{\norm{\nabla F(x^t)}^2}
		&\le \frac{4}{\beta k}\E{F(x^0)-F(x^{k+1})}+ 16 \beta \left(\frac{1}{|T^k|} + (\alpha {L_1})^2\delta^2\right) \sigma^2 \\
		& \qquad + 4(\alpha {L_1})^2\delta^2 \sigma^2\\
		&\le \frac{4}{\beta k}\E{F(x^0)-F^*}+ 16 \beta \left(\frac{1}{|T^k|} + (\alpha {L_1})^2\delta^2\right) \sigma^2 \\
		& \qquad + 4(\alpha {L_1})^2\delta^2 \sigma^2.
	\end{align*}
	The result for \Cref{alg:fo_maml} is obtained as a special case with $\delta=\alpha {L_1}$.
\end{proof}


\chap{\ref{sec:adaptive}}

\section{Proofs}
\subsection{Proof of Lemma \ref{le:L}}
For the considered partition sampling, the indices $1, \dots, n$ are distributed into the sets $\mathcal C_1, \dots, \mathcal C_K$ with each having a minimum cardinality of $\tau$. We choose each set $\mathcal C_j$ with probability $q_{\mathcal{C}_j}$ where $\sum \limits_{j} q_{\mathcal C_j}=1$. Denote

$$
\mathbf P_{ij} \eqdef
\begin{cases}
	0 & \text{if } i \in C_k, j \in C_l, k \neq l \\
	q_k \frac{\tau(\tau-1)}{n_k(n_k-1)} & \text{if } i \neq j,  i, j \in C_k, |C_k|=\tau_k\\
	q_k \frac {\tau} {n_k} & \text{if i=j}
\end{cases}.
$$
Therefore
\begin{align*}
	&\E{\norm{\nabla f_v(x) - \nabla f_v(y)}^2} \\
	&\vspace{2cm} = \frac{1}{n^2}\sum_{\mathcal C_k} \sum_{i,j \in \mathcal C_k} \frac{\mathbf P_{ij}}{p_ip_j} \Big \langle \nabla f_i(x)-\nabla f_i(y), \nabla f_j(x)-\nabla f_j(y) \Big \rangle \\
	&\vspace{2cm}= \frac{1}{n^2} \sum_{\mathcal C_k} \sum_{i \neq j \in \mathcal C_k} \frac{\mathbf P_{ij}}{p_ip_j} \Big \langle \nabla f_i(x)-\nabla f_i(y),\nabla f_j(x)-\nabla f_j(y) \Big \rangle   \\
	&\vspace{3cm} + \frac{1}{n^2} \sum_{\mathcal C_k} \sum_{i \in \mathcal C_k} \frac{1}{p_i} \Big \langle \nabla f_i(x)-\nabla f_i(y),\nabla f_i(x)-\nabla f_i(y) \Big \rangle\\
	& \vspace{2cm}= \frac{1}{n^2} \sum_{\mathcal C_k} \sum_{i \neq j \in \mathcal C_k} \frac{n_{\mathcal C_k}(\tau-1)}{q_{\mathcal C_k}\tau(n_{\mathcal C_k}-1)} \Big \langle \nabla f_i(x)-\nabla f_i(y),\nabla f_j(x)-\nabla f_j(y) \Big \rangle   \\
	& \vspace{3cm} + \frac{1}{n^2} \sum_{\mathcal C_k} \sum_{i \in \mathcal C_k} \frac{n_{\mathcal C_k}}{q_{\mathcal C_k}\tau} \Big \langle \nabla f_i(x)-\nabla f_i(y),\nabla f_i(x)-\nabla f_i(y) \Big \rangle\\
	&\vspace{2cm}= \frac{1}{n^2} \sum_{\mathcal C_k} \frac{n_{\mathcal C_k}(\tau-1)}{q_{\mathcal C_k}\tau(n_{\mathcal C_k}-1)} \sum_{i \neq j \in \mathcal C_k} \Big \langle \nabla f_i(x)-\nabla f_i(y),\nabla f_j(x)-\nabla f_j(y) \Big \rangle   \\
	& \vspace{3cm} + \frac{1}{n^2}\sum_{\mathcal C_k} \frac{n_{\mathcal C_k}}{q_{\mathcal C_k}\tau} \sum_{i \in \mathcal C_k} \norm{\nabla f_i(x)-\nabla f_i(y)}^2\\
	&\vspace{2cm}= \frac{1}{n^2} \sum_{\mathcal C_k} \frac{n_{\mathcal C_k}(\tau-1)}{q_{\mathcal C_k}\tau(n_{\mathcal C_k}-1)} \norm{\sum_{i \in \mathcal C_k} \nabla f_i(x)-\nabla f_i(y)}^2  \\
	& \vspace{3cm} + \frac{1}{n^2}\sum_{\mathcal C_k} \frac{n_{\mathcal C_k}(n_{\mathcal C_k}-\tau)}{q_{\mathcal C_k}\tau(n_{\mathcal C_k}-1)} \sum_{i \in \mathcal C_k} \norm{\nabla f_i(x)-\nabla f_i(y)}^2 \\
	& \vspace{2cm} \leq \frac{1}{n^2} \sum_{\mathcal C_k} \frac{n_{\mathcal C_k}^3(\tau-1)}{q_{\mathcal C_k}\tau(n_{\mathcal C_k}-1)} 2L_{\mathcal C_k}D_{f_{\mathcal C_k}}(x, y) + \frac{1}{n^2}\sum_{\mathcal C_k} \frac{n_{\mathcal C_k}(n_{\mathcal C_k}-\tau)}{q_{\mathcal C_k}\tau(n_{\mathcal C_k}-1)} \sum_{i \in \mathcal C_k} 2L_iD_{f_i}(x,y) \\
	&\vspace{2cm}\leq \frac{1}{n^2} \sum_{\mathcal C_k} \frac{n_{\mathcal C_k}^3(\tau-1)}{q_{\mathcal C_k}\tau(n_{\mathcal C_k}-1)} 2L_{\mathcal C_k}D_{f_{\mathcal C_k}}(x, y) + \frac{1}{n^2}\sum_{\mathcal C_k} \frac{n_{\mathcal C_k}^2(n_{\mathcal C_k}-\tau)}{q_{\mathcal C_k}\tau(n_{\mathcal C_k}-1)} 2\max_{i \in \mathcal C_k}{L_i}D_{f_{\mathcal C_k}}(x,y) \\
	& \vspace{2cm}= \frac{1}{n^2} \sum_{\mathcal C_k} 2\frac{n_{\mathcal C_k}^2(\tau-1)L_{\mathcal C_k} + n_{\mathcal C_k}(n_{\mathcal C_k}-\tau)\max_{i \in \mathcal C_k}{L_i} }{q_{\mathcal C_k}\tau(n_{\mathcal C_k}-1)} n_{\mathcal C_k}D_{f_{\mathcal C_k}}(x, y)  \\
	&\vspace{2cm} \leq 2\frac{1}{n}( \max_{\mathcal C_k}\frac{n_{\mathcal C_k}^2(\tau-1)L_{\mathcal C_k} + n_{\mathcal C_k}(n_{\mathcal C_k}-\tau)\max_{i \in \mathcal C_k}{L_i} }{q_{\mathcal C_k}\tau(n_{\mathcal C_k}-1)} n_{\mathcal C_k}) D_{f}(x,y),
\end{align*}
where $D_f(x,y) = f(x) - f(y) - \langle\nabla f(y), x-y\rangle$. 
Setting $y = x^*$, leads to the desired upper bound of the expected smoothness which is given by
\begin{eqnarray*}
	\mathcal L(\tau) = \frac{1}{n\tau}\left( \max_{\mathcal C_k}\frac{n_{\mathcal C_k}}{q_{\mathcal C_k}\left(n_{\mathcal C_k}-1\right)}\left(n_{\mathcal C_k}^2(\tau-1)L_{\mathcal C_k} + n_{\mathcal C_k}\left(n_{\mathcal C_k}-\tau\right)\max_{i \in \mathcal C_k}{L_i}\right)\right).
\end{eqnarray*}

Next, we derive a similar bound  for $\tau$--independent partition sampling.
\begin{align*}
	\E{\norm{\nabla f_v(x) - \nabla f_v(y)}^2} 
	&= \frac{1}{n^2}\sum_{\mathcal C_k} \sum_{i,j \in \mathcal C_k} \frac{\mathbf P_{ij}}{p_ip_j} \Big \langle \nabla f_i(x)-\nabla f_i(y), \nabla f_j(x)-\nabla f_j(y) \Big \rangle \\
	&= \frac{1}{n^2} \sum_{\mathcal C_k} \sum_{i \neq j \in \mathcal C_k} \frac{\mathbf P_{ij}}{p_ip_j} \Big \langle \nabla f_i(x)-\nabla f_i(y),\nabla f_j(x)-\nabla f_j(y) \Big \rangle   \\
	&\quad + \frac{1}{n^2} \sum_{\mathcal C_k} \sum_{i \in \mathcal C_k} \frac{1}{p_i} \Big \langle \nabla f_i(x)-\nabla f_i(y),\nabla f_i(x)-\nabla f_i(y) \Big \rangle\\
	&= \frac{1}{n^2} \sum_{\mathcal C_k} \sum_{i \neq j \in \mathcal C_k} \frac{1}{q_{\mathcal C_k}} \Big \langle \nabla f_i(x)-\nabla f_i(y),\nabla f_j(x)-\nabla f_j(y) \Big \rangle   \\
	&\quad + \frac{1}{n^2} \sum_{\mathcal C_k} \sum_{i \in \mathcal C_k} \frac{1}{q_{\mathcal C_k}p_i} \Big \langle \nabla f_i(x)-\nabla f_i(y),\nabla f_i(x)-\nabla f_i(y) \Big \rangle\\
	&= \frac{1}{n^2} \sum_{\mathcal C_k} \frac{1}{q_{\mathcal C_k}} \sum_{i \neq j \in \mathcal C_k} \Big \langle \nabla f_i(x)-\nabla f_i(y),\nabla f_j(x)-\nabla f_j(y) \Big \rangle   \\
	&\quad + \frac{1}{n^2}\sum_{\mathcal C_k} \frac{1}{q_{\mathcal C_k}} \sum_{i \in \mathcal C_k}\frac{1}{p_i} \norm{\nabla f_i(x)-\nabla f_i(y)}^2\\
	&= \frac{1}{n^2} \sum_{\mathcal C_k} \frac{1}{q_{\mathcal C_k}} \norm{\sum_{i \in \mathcal C_k} \nabla f_i(x)-\nabla f_i(y)}^2  \\
	&\quad + \frac{1}{n^2}\sum_{\mathcal C_k} \frac{1}{q_{\mathcal C_k}} \sum_{i \in \mathcal C_k}\frac{1-p_i}{p_i} \norm{\nabla f_i(x)-\nabla f_i(y)}^2 \\
	&\leq \frac{1}{n^2} \sum_{\mathcal C_k} \frac{n_{\mathcal C_k}^2}{q_{\mathcal C_k}} 2L_{\mathcal C_k}D_{f_{\mathcal C_k}}(x, y)  \\
	&\quad + \frac{1}{n^2}\sum_{\mathcal C_k} \frac{1}{q_{\mathcal C_k}} \sum_{i \in \mathcal C_k}\frac{1-p_i}{p_i} 2L_iD_{f_i}(x,y) \\
	&\leq \frac{1}{n^2} \sum_{\mathcal C_k} \frac{n_{\mathcal C_k}^2}{q_{\mathcal C_k}} 2L_{\mathcal C_k}D_{f_{\mathcal C_k}}(x, y)  \\
	&\quad + \frac{1}{n^2}\sum_{\mathcal C_k} \frac{n_{\mathcal C_k}}{q_{\mathcal C_k}} 2\max_{i \in \mathcal C_k}\frac{1-p_i}{p_i}{L_i}D_{f_{\mathcal C_k}}(x,y) \\
	&= \frac{1}{n^2} \sum_{\mathcal C_k} 2\left(\frac{n_{\mathcal C_k}L_{\mathcal C_k}}{q_{\mathcal C_k}} + \max_{i \in \mathcal C_k}{\frac{(1-p_i)L_i}{q_{\mathcal C_k}p_i} }\right) n_{\mathcal C_k}D_{f_{\mathcal C_k}}(x, y)  \\
	&\leq 2\frac{1}{n} \max_{i \in \mathcal C_k}\left({\frac{n_{\mathcal C_k}L_{\mathcal C_k}}{q_{\mathcal C_k}} + \max_{i \in \mathcal C_k}{\frac{(1-p_i)L_i}{q_{\mathcal C_k}p_i} }}\right) D_{f}(x,y).
\end{align*}

This gives the desired upper bound for the expected smoothness
\begin{eqnarray*}
	\mathcal L(\tau) = \frac{1}{n}\max_{i \in \mathcal C_k}{\left(\frac{n_{\mathcal C_k}L_{\mathcal C_k}}{q_{\mathcal C_k}} + \max_{i \in \mathcal C_k}\frac{(1-p_i)L_i}{q_{\mathcal C_k}p_i} \right)}
	.\end{eqnarray*}

\subsection{Proof of Lemma \ref{le:variance}}
Following the same notation, we move on to compute $\sigma$ for each sampling. First, for $\tau$--nice partition sampling we have 
\begin{align*}
	& \E{\norm{\nabla f_v(x^*)}^2} \\
	&\vspace{2cm}= \frac{1}{n^2}\sum_{\mathcal C_k} \sum_{i,j \in \mathcal C_k} \frac{\mathbf P_{ij}}{p_ip_j} \Big \langle \nabla f_i(x^*), \nabla f_j(x^*) \Big \rangle \\
	&\vspace{2cm}= \frac{1}{n^2} \sum_{\mathcal C_k} \sum_{i \neq j \in \mathcal C_k} \frac{\mathbf P_{ij}}{p_ip_j} \Big \langle \nabla f_i(x^*),\nabla f_j(x^*) \Big \rangle   + \frac{1}{n^2} \sum_{\mathcal C_k} \sum_{i \in \mathcal C_k} \frac{1}{p_i} \Big \langle \nabla f_i(x^*),\nabla f_i(x^*) \Big \rangle\\
	&\vspace{2cm}= \frac{1}{n^2} \sum_{\mathcal C_k} \sum_{i \neq j \in \mathcal C_k} \frac{n_{\mathcal C_k}(\tau-1)}{\tau(n_{\mathcal C_k}-1)q_{\mathcal C_k}} \Big \langle \nabla f_i(x^*),\nabla f_j(x^*) \Big \rangle   \\
	& \vspace{2cm} \qquad + \frac{1}{n^2} \sum_{\mathcal C_k} \sum_{i \in \mathcal C_k} \frac{n_{\mathcal C_k}}{\tau q_{\mathcal C_k}} \Big \langle \nabla f_i(x^*),\nabla f_i(x^*) \Big \rangle\\
	&\vspace{2cm}= \frac{1}{n^2} \sum_{\mathcal C_k} \frac{n_{\mathcal C_k}(\tau-1)}{\tau(n_{\mathcal C_k}-1)q_{\mathcal C_k}} \sum_{i \neq j \in \mathcal C_k} \Big \langle \nabla f_i(x^*),\nabla f_j(x^*) \Big \rangle   + \frac{1}{n^2}\sum_{\mathcal C_k} \frac{n_{\mathcal C_k}}{\tau q_{\mathcal C_k}} \sum_{i \in \mathcal C_k} h_i\\
	&\vspace{2cm}= \frac{1}{n^2} \sum_{\mathcal C_k} \frac{n_{\mathcal C_k}(\tau-1)}{\tau(n_{\mathcal C_k}-1)q_{\mathcal C_k}} \norm{\sum_{i \in \mathcal C_k} \nabla f_i(x^*)}^2  + \frac{1}{n^2}\sum_{\mathcal C_k} \frac{n_{\mathcal C_k}(n_{\mathcal C_k}-\tau)}{\tau(n_{\mathcal C_k}-1)q_{\mathcal C_k}}\sum_{i \in \mathcal C_k} h_i \\
	&\vspace{2cm}= \frac{1}{n^2} \sum_{\mathcal C_k} \frac{n_{\mathcal C_k}^3(\tau-1)}{\tau(n_{\mathcal C_k}-1)q_{\mathcal C_k}} {h_{\mathcal C_k}}  + \frac{1}{n^2}\sum_{\mathcal C_k} \frac{n_{\mathcal C_k}^2(n_{\mathcal C_k}-\tau)}{\tau(n_{\mathcal C_k}-1)q_{\mathcal C_k}}\overline{h}_{\mathcal C_k}.
\end{align*}
Where its left to rearrange the terms to get the first result of the lemma. Next, we compute $\sigma$ for $\tau$--independent partition.
\begin{align*}
	& \E{\norm{\nabla f_v(x^*)}^2} \\
	&\quad = \frac{1}{n^2}\sum_{\mathcal C_k} \sum_{i,j \in \mathcal C_k} \frac{\mathbf P_{ij}}{p_ip_j} \Big \langle \nabla f_i(x^*), \nabla f_j(x^*) \Big \rangle \\
	&\quad= \frac{1}{n^2} \sum_{\mathcal C_k} \sum_{i \neq j \in \mathcal C_k} \frac{\mathbf P_{ij}}{p_ip_j} \Big \langle \nabla f_i(x^*),\nabla f_j(x^*) \Big \rangle  + \frac{1}{n^2} \sum_{\mathcal C_k} \sum_{i \in \mathcal C_k} \frac{1}{p_i} \Big \langle \nabla f_i(x^*),\nabla f_i(x^*) \Big \rangle\\
	&\quad= \frac{1}{n^2} \sum_{\mathcal C_k} \sum_{i \neq j \in \mathcal C_k} \frac{1}{q_{\mathcal C_k}} \Big \langle \nabla f_i(x^*),\nabla f_j(x^*) \Big \rangle  + \frac{1}{n^2} \sum_{\mathcal C_k} \sum_{i \in \mathcal C_k} \frac{1}{q_{\mathcal C_k}p_i} \Big \langle \nabla f_i(x^*),\nabla f_i(x^*) \Big \rangle\\
	&\quad= \frac{1}{n^2} \sum_{\mathcal C_k} \frac{1}{ q_{\mathcal C_k}} \sum_{i \neq j \in \mathcal C_k} \Big \langle \nabla f_i(x^*),\nabla f_j(x^*) \Big \rangle   + \frac{1}{n^2}\sum_{\mathcal C_k} \frac{1}{q_{\mathcal C_k}} \sum_{i \in \mathcal C_k} \frac{1}{p_i}h_i\\
	&\quad= \frac{1}{n^2} \sum_{\mathcal C_k} \frac{1}{q_{\mathcal C_k}} \norm{\sum_{i \in \mathcal C_k} \nabla f_i(x^*)}^2  + \frac{1}{n^2}\sum_{\mathcal C_k} \sum_{i \in \mathcal C_k} \frac{(1-p_i)h_i}{q_{\mathcal C_k}p_i} \\
	&\quad= \frac{1}{n^2} \sum_{\mathcal C_k} \frac{n_{\mathcal C_k}^2}{q_{\mathcal C_k}} {h_{\mathcal C_k}} + \frac{1}{n^2}\sum_{\mathcal C_k} \sum_{i \in \mathcal C_k} \frac{(1-p_i)h_i}{q_{\mathcal C_k}p_i}.
\end{align*}

\subsection{Proof of Theorem \ref{le:tau_star}}
Recall that the optimal mini-batch size $\tau(x^*)$ is chosen such that the quantity $\max\left\{ \tau\mathcal L(\tau), \frac{2}{\epsilon\mu}\tau \sigma(x^*, \tau) \right\}$ is minimized. Note that in both $\tau$--nice partition, and $\tau$-- independent partition with $(p_i = \frac{\tau}{n_{\mathcal C_j}})$, $\tau \sigma(x^*, \tau)$ is a linear function of $\tau$ while $\tau \mathcal L(\tau)$ is a max across linearly increasing functions of $\tau$. To find the minimized in such a case, we leverage the following lemma.
\begin{lemma}
	Suppose that $l_1(x),l_2(x),...,l_k(x)$ are increasing linear functions of $x$, and $r(x)$ is linear decreasing function of $x$, then the minimizer of $\max(l_1(x),l_2(x),$ $\dots,l_k(x),r(x))$ is $x^* = \min_i(x_i)$ where $x_i$ is the unique solution for $l_i(x)=r(x)$.
\end{lemma}
\begin{proof}    
	Let $x^*$ be defined as above, and let $x$ be an arbitrary number. If $x \le x^*$, then  $r(x) \ge r(x^*) \ge r(x_i) = l_i(x_i) \ge l_i(x^*)$ for each $i$ which means $r(x) \ge \max(l_1(x^*), ..., l_k(x^*), r(x^*))$. On the other hand, if $x \ge x^*$, then let $i$ be the index s.t. $x_i=x^*$. We have $l_i(x) \ge l_i(x^*) = r(x^*) \ge r(x_j) = l_j(x_j) \ge l_j(x^*)$, hence $l_i(x) \ge \max(l_1(x^*), ..., l_k(x^*), r(x^*))$. This means that $x^* = \min_i(x_i)$ is indeed the minimizer of $\max(l_1(x),l_2(x),...,l_k(x),r(x))$.
\end{proof}

Now we can estimate optimal mini-batch sizes for proposed samplings.

\textbf{$\tau$--nice partition:} if $\sum_{\mathcal C_j} \frac{n_{\mathcal C_j}^2}{e_j} ( h_{\mathcal C_j}^*n_{\mathcal C_j}-\overline {h}^*_{\mathcal C_j}) \le 0$ then $\tau \sigma(\tau)$ is a decreasing linear function of $\tau$, and $\tau \mathcal L(\tau)$ is the max of increasing linear functions. Therefore, we can leverage the previous lemma with $r(\tau) = \frac{2}{\epsilon \mu}\tau\sigma(x^*, \tau)$ and $l_{\mathcal C_k}(\tau) = \frac{n_{\mathcal C_k}^2(\tau-1)L_{\mathcal C_k} + n_{\mathcal C_k}(n_{\mathcal C_k}-\tau)\max_{i \in \mathcal C_k}{L_i} }{q_{\mathcal C_k}(n_{\mathcal C_k}-1)} n_{\mathcal C_k}$ to find the optimal mini-batch size as $\tau^* = \min\limits_{\mathcal C_k}(\tau^*_{\mathcal C_k})$, where 
\begin{equation*} 
	\tau^*_{\mathcal C_k} = \frac{
		\frac{nn_{\mathcal C_r}^2}{e_{\mathcal{C}_r}} (L_{\mathcal C_r}-L_{\max}^{\mathcal C_r}) + \frac{2}{\epsilon \mu} \sum_{\mathcal C_j} \frac{n_{\mathcal C_j}^3}{e_{\mathcal C_j}}\left(\overline {h}_{\mathcal C_j}- h_{\mathcal C_j}\right)}
	{\frac {nn_{\mathcal{C}_r}} {e_{\mathcal C_r}}(n_{\mathcal C_r}L_{\mathcal C_r}-L_{\max}^{\mathcal{C}_r}) + \frac{2}{\epsilon \mu}\sum_{\mathcal C_j} 
		\frac{n_{\mathcal C_j}^2 }{e_{\mathcal C_j}}({\overline {h}_{\mathcal C_j} -n_{\mathcal C_j}  h_{\mathcal C_j} })}.
\end{equation*}
\textbf{$\tau$--independent partition:} Similar to $\tau$--nice partition, we have $\tau \sigma(\tau)$ is a decreasing linear function of $\tau$ if $\quad\sum_{\mathcal C_j} \frac{n_{\mathcal C_j}}
{ q_{\mathcal C_j}} (n_{\mathcal C_j}  h_{\mathcal C_j}^* - \overline {h}^*_{\mathcal C_j}) \le 0$, and $\tau \mathcal L(\tau)$ is the max of increasing linear functions of $\tau$. Hence, we can leverage the previous lemma with $r(\tau) = \frac{2}{\epsilon \mu}\tau\sigma(x^*, \tau)$ and $l_{\mathcal C_k}(\tau) = {\frac{n_{\mathcal C_k}L_{\mathcal C_k}\tau}{q_{\mathcal C_k}} + (n_{\mathcal C_k}-\tau)\max_{i \in \mathcal C_k}{\frac{L_i}{q_{\mathcal C_k}} }}$ to find the optimal mini-batch size as $\tau^* = \min\limits_{\mathcal C_k}(\tau^*_{\mathcal C_k})$, where
\begin{equation*} 
	\tau^*_{\mathcal C_k} = {\frac{
			\frac{2}{\epsilon \mu}
			\sum_{\mathcal C_j} \frac{n_{\mathcal C_j}^2} { q_{\mathcal C_j}} \overline {h}_{\mathcal C_j} - \frac{n}{q_{\mathcal C_r}} L_{\max}^{\mathcal{C}_r} }
		{
			\frac{2}{\epsilon \mu}\sum_{\mathcal C_j} \frac {n_{\mathcal C_j}} { q_{\mathcal C_j}}
			(\overline {h}_{\mathcal C_j} - n_{\mathcal C_j} h_{\mathcal C_j})
			+
			\frac n {q_{\mathcal C_r}} (n_{\mathcal C_r}L_{\mathcal C_r}-L_{\max}^{\mathcal C_r}) 
	}}.
\end{equation*}

\subsection{Proof of Lemma \ref{le:step_sizes_positive}}

For our choice of the learning rate we have

\begin{eqnarray*}
	\gamma^k &=& \frac 1 2\min \left\{\frac 1 {\lk}, \frac{\epsilon \mu}{\min(C,2\sg^k)} \right\}= \frac 1 2 \min \left\{\frac 1 {\lk}, \max \left\{ \frac{\epsilon \mu} C, \frac{\epsilon \mu}{2\sg^k} \right\} \right\} \\
	&\geq & \frac 1 2\min \left\{\frac 1 {\lk}, \frac{\epsilon \mu} C  \right\}
	.\end{eqnarray*}
Since $\lk$ is a linear combination between the smoothness constants of the functions $f_i$, then it is bounded. Therefore,
both $\lk$ and $C$ are upper bounded and lower bounded as well as $\frac 1 {\lk}$ and $\frac{\epsilon \mu} C$, thus also $\gamma^k$ is bounded by positive constants $\gamma_{\max}= \frac 1 2\max_{\tau \in [n]} \left\{ \frac 1 {\mathcal L(\tau)} \right\}$ and $ \gamma_{\min}= \frac 1 2\min \left\{ \min_{\tau \in [n]} \left\{\frac 1 {\mathcal L(\tau)} \right\}, \frac{\epsilon \mu} C \right\} $.

\subsection{Proof of Theorem \ref{th:convergence_our}}
Let $r^k = \norm{x^k-x^*}^2$, then
\begin{align*}
	\E{r^{k+1} |x^k} 
	&= \E{\norm{x^k-\gamma^k \nabla f_{v^k}(x^k)-x^*}^2 |x^k}\\
	&= r^k + (\gamma^k)^2\E{\norm{\nabla f_{v^k} (x^k)}^2|x^k} - 2 \gamma^k \langle \E{\nabla f_{v^k}(x^k)|x^k}, r^k \rangle  \\
	&=  r^k + (\gamma^k)^2 \E{\norm{\nabla f_{v^k}(x^k)}^2|x^k}-2\gamma^k \left(f(x^k)-f(x^*)+\frac \mu 2 r^k\right) \\
	&= (1-\gamma^k \mu)r^k+ (\gamma^k)^2 \E{\norm{\nabla f_{v^k} (x^k)}^2|x^k} - 2\gamma^k(f(x^k)-f(x^*))  \\
	&\leq  (1-\gamma^k \mu)r^k+ (\gamma^k)^2 (4 \mathcal L^k(f(x^k)-f(x^*)) +2\sg) - 2\gamma^k(f(x^k)-f(x^*))\\
	&= (1-\gamma^k \mu)r^k -2 \gamma^k((1-2\gamma^k \mathcal L^k)(f(x^k)-f(x^*))) +2(\gamma^k)^2\sg \\
	&\leq (1-\gamma^k \mu)r^k +2(\gamma^k)^2\sg \quad \text{ for  } \gamma_k \leq \frac 1 {2 \mathcal L^k} .
\end{align*}
From Lemma \ref{le:step_sizes_positive} there exist stepsizes bounds, $\gamma_{\min} \leq \gamma^k \leq \gamma_{\max}$, thus
\begin{equation*}
	\E{r^{k+1} |x^k} \leq (1-\gamma^k \mu)r^k +2(\gamma^k)^2\sg  \leq (1-\gamma_{\min} \mu)r^k +2\gamma_{\max}^2\sg
	.\end{equation*}
Therefore, unrolling the above recursion gives
\begin{eqnarray*}
	\E{r^{k+1}|x^k} &\leq& (1-\gamma_{\min} \mu)^k r^0+ 2 \gamma_{\max}^2 \sg \sum_{i=0}^k(1-\gamma_{\min}\mu)^k \\
	&\leq& (1-\gamma_\text{min} \mu)^{k}r^0+  \frac {2\gamma_{\max}^2\sg}{\gamma_\text{min} \mu}
	.\end{eqnarray*}

\subsection{Proof of convergence to linear neighborhood in $\epsilon$}

In this section, we prove that our algorithm converges to a neighborhood around the optimal solution with size upper bounded by an expression linear in $\epsilon$.
First of all, we prove that $\sg(x)$ is lower bounded by a multiple of the variance in the optimum $\sg^*$ (in Lemma \ref{le:sg_lower_bound}). Then, by showing an alternative upper bound on the stepsize, we obtain an upper bound for neighborhood size $R$ as expression of $\epsilon$.

\begin{lemma} \label{le:sg_lower_bound}
	Suppose $f$ is  $\mu$--strongly convex, $L$--smooth and with expected smoothness constant $\mathcal L$. Let $v$ be as in Section~\ref{sec:overview}, i.e., $\E{v_i=1}$ for all $i$.  Fix any $c>0$. The function $\sigma(x)\eqdef \E{\norm{\nabla{f_v(x)}}^2}$ can be lower bounded as follows: \[\sigma(x) \ge  \left\{ \mu^2 c   , 1 - 2 \sqrt{\mathcal L L c}  \right\}\sigma(x^*), \qquad \forall x\in \R^d.\]
	The constant $c$ maximizing this bound is $c=\left(\frac{\sqrt{\mathcal L L +\mu^2} - \sqrt{\mathcal L L} }{\mu^2}\right)^2$, giving the bound
	\[\sigma(x) \ge \frac{\left(\sqrt{\mathcal L L +\mu^2} - \sqrt{\mathcal L L} \right)^2}{\mu^2}  \sigma(x^*), \qquad \forall x\in \R^d.\]
\end{lemma}

\begin{proof}
	Choose $c>0$. If $\norm{x-x^*} \ge \sqrt{c\sigma(x^*)}$, then using Jensen's inequality and strong convexity of $f$, we get
	\begin{eqnarray}\label{eq:nif8g98gdud}
		\sigma(x) &\ge&  \norm{\E{\nabla{f_v(x)}}}^2 \\
		&=& \norm{\nabla{f(x)}}^2 \\
		&=& \norm{\nabla{f(x)}-\nabla{f(x^*)}}^2\\
		&\ge&  \mu^2 \norm{x-x^*}^2\\
		& \ge& \mu^2 c \sigma(x^*). 
	\end{eqnarray}
	If $\norm{x-x^*} \le \sqrt{c\sigma(x^*)}$, then expected smoothness and $L$--smoothness imply
	\begin{eqnarray}
		\E{\norm{\nabla{f_v(x)}-\nabla{f_v(x^*)}}^2}
		&\overset{\eqref{eq:exp_smoothn_cL}}{\leq}& 2 \mathcal L (f(x)-f(x^*)) \\
		&\leq&  \mathcal L L  \norm{x-x^*}^2 \\
		&\leq&  \mathcal L L c\sigma(x^*). \label{eqL:bui7gd97dd}
	\end{eqnarray}
	Further, we can write
	\begin{eqnarray*}
		\sigma(x^*)-\sigma(x) &=& \E{\norm{\nabla{f_v(x^*)}}^2}-\E{\norm{\nabla{f_v(x)}}^2}\\
		&=&
		- 2\E{ \left\langle\nabla{f_v(x)}-\nabla{f_v(x^*)},\nabla{f_v(x^*)} \right\rangle } - \E{\norm{\nabla{f_v(x)}-\nabla{f_v(x^*)}}^2} \\
		&\leq & - 2\E{ \left\langle\nabla{f_v(x)}-\nabla{f_v(x^*)},\nabla{f_v(x^*)} \right\rangle } \\
		&\leq & 2\E{ \norm{\nabla{f_v(x)}-\nabla{f_v(x^*)}} \norm{ \nabla{f_v(x^*)}  }}\\
		&\leq & 2 \sqrt{\E{ \norm{\nabla{f_v(x)}-\nabla{f_v(x^*)}}^2}} \sqrt{\E{\norm{ \nabla{f_v(x^*)}  }^2}}\\
		&\overset{\eqref{eqL:bui7gd97dd}}{\le}&  2 \sqrt{\mathcal L L c} \sqrt{\sigma(x^*)}  \sqrt{\sigma(x^*)} \\
		&=& 2 \sqrt{\mathcal L Lc} \sigma(x^*),\end{eqnarray*}
	where the first inequality follows by neglecting a negative term, the second by Cauchy-Schwarz and the third by H\"{o}lder inequality for bounding the expectation of the product of two random variables.
	The last inequality implies that 
	\begin{equation}\label{eq:i9hfgdhuf}
		\sigma(x) \ge \left(1 - 2 \sqrt{\mathcal L L c}  \right) \sigma(x^*) 
		.\end{equation}
	By combining  the bounds \eqref{eq:nif8g98gdud} and \eqref{eq:i9hfgdhuf}, we get
	\[\sigma(x) \ge \min \left\{ \mu^2 c   , 1 - 2 \sqrt{\mathcal L L c}  \right\}\sigma(x^*).\]

	Using Lemma \ref{le:sg_lower_bound}, we can upper bound stepsizes $\gamma^k$. Assume that $\sigma^*=\sigma(x^*)>0$. Let $\gamma_{\max}^{'}= \frac {\epsilon \mu} 2 \max \left\{ \frac 1 C, \frac 1 {2 \eta \sg^*} \right\}$, where $\eta = \frac{\left(\sqrt{\mathcal L L +\mu^2} - \sqrt{\mathcal L L} \right)^2}{\mu^2} $.
	We have
	\begin{eqnarray*}
		\gamma^k 
		&=& \frac 1 2\min \left\{\frac 1 {\lk}, \frac{\epsilon \mu}{\min(C,2\sg^k)} \right\} 
		\leq \frac {\epsilon \mu} 2 \max \left\{ \frac{1} C, \frac{1}{2\sg^k} \right\} \\
		&\leq& \frac {\epsilon \mu} 2 \max \left\{ \frac{1} C, \frac{1}{2 \eta \sg^*} \right\} 
		= \gamma_{\max}^{'}.
	\end{eqnarray*}
	
	Now, we use this alternative stepsizes upper bound to obtain alternative expression for residual term $R= \frac {2\gamma^{2}_{\max}\sg^*}{\gamma_\text{min} \mu}$ in Theorem \ref{th:convergence_our} (let's denote it $R^{'}$). Analogically to proof of Theorem \ref{th:convergence_our} (with upper bound of stepsizes $\gamma_{\max}^{'}$), we have $R^{'}= \frac {2\gamma^{'2}_{\max}\sg^*}{\gamma_\text{min} \mu}$.
	Finally, expanding expression for residual term $R^{'}$ yields the result: 
	\begin{align*}
		R^{'} &= \frac {2\gamma^{'2}_{\max}\sg^*}{\gamma_\text{min} \mu} = 
		\frac {2\left(\frac {\epsilon \mu} 2 \max \left\{ \frac{1} C, \frac{1}{2 \eta \sg^*} \right\} \right)^2 \sg^*}{\frac 1 2\min \left\{ \min_{\tau \in [n]} \left\{\frac 1 {\mathcal L(\tau)} \right\}, \frac{\epsilon \mu} C \right\} \mu} \\
		&=
		\epsilon^2 \mu \left( \max \left\{ \frac{1} C, \frac{1}{2 \eta \sg^*} \right\} \right)^2 \max \left\{ \max_{\tau \in [n]} \left\{ \mathcal L(\tau) \right\}, \frac C {\mu \epsilon}\right\} \sg^*\\
		&= \epsilon \mu \left( \max \left\{ \frac{1} C, \frac{1}{2 \eta \sg^*} \right\} \right)^2 \max \left\{ \epsilon \max_{\tau \in [n]} \left\{ \mathcal L(\tau) \right\}, \frac C \mu\right\} \sg^*
		. \end{align*}
	Consequently, $R^{'}$ as a function of $\epsilon$ is at least linear in $\epsilon$,  $R^{'} \in \mathcal O (\epsilon)$. 
	
\end{proof}

\section{Additional experimental results}-

\begin{figure}[h]
	\begin{center}
		\centerline{
			\includegraphics[width=0.5\columnwidth]{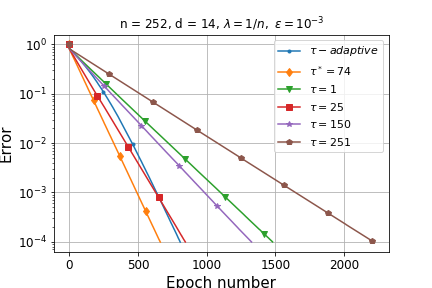}
			\includegraphics[width=0.5\columnwidth]{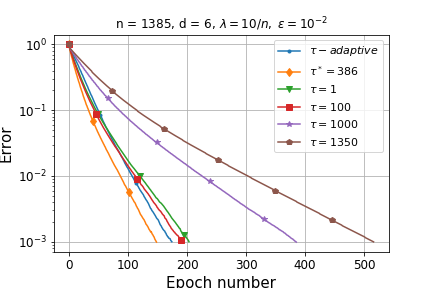}
		}
		\centerline{
			\includegraphics[width=0.5\columnwidth]{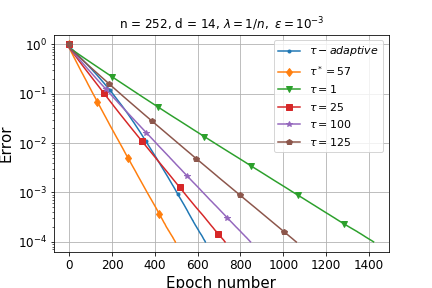}
			\includegraphics[width=0.5\columnwidth]{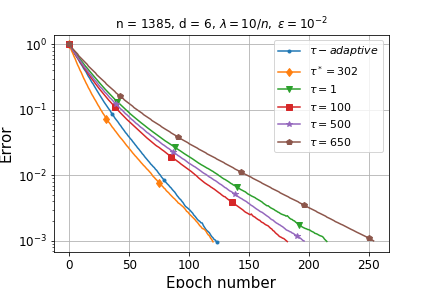}
		}
		\centerline{
			\includegraphics[width=0.5\columnwidth]{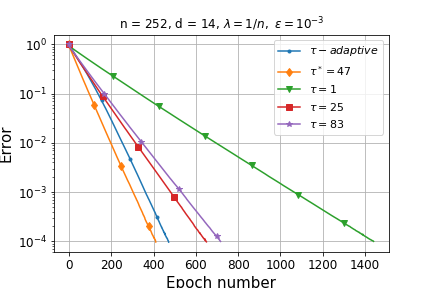}
			\includegraphics[width=0.5\columnwidth]{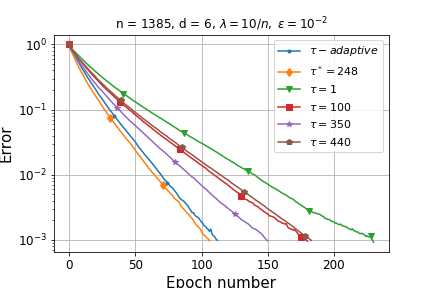}
		}
		\centerline{
			\includegraphics[width=0.5\columnwidth]{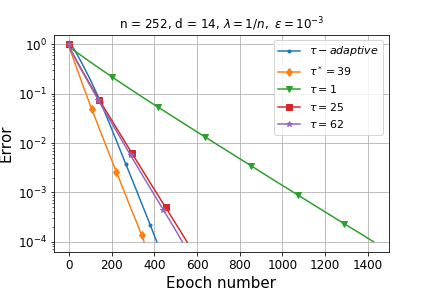}
			\includegraphics[width=0.5\columnwidth]{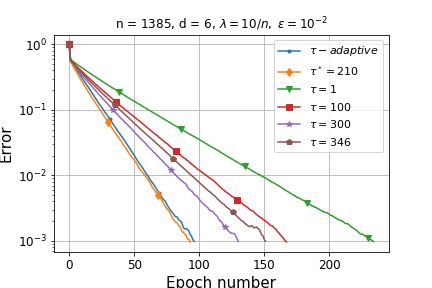}
		}
		\caption[Convergence of \sgd{} with adaptive mini-batch size, $\tau$--partition independent sampling on ridge regression]{\textbf{Convergence of Ridge regression} using $\tau$--partition independent sampling on \bodyfat{} dataset (first column) and $\tau$--partition nice sampling on \mg{} dataset (second column). In four rows, training set is distributed among $1$, $2$, $3$ and $4$ partitions, respectively.
		}
		\label{fig:ridge_sup}
	\end{center}
\end{figure}

\begin{figure}[h]
	\begin{center}
		\centerline{
			\includegraphics[width=0.5\columnwidth]{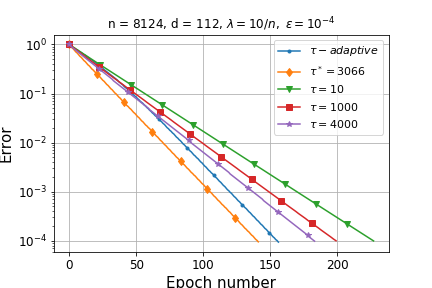}
			\includegraphics[width=0.5\columnwidth]{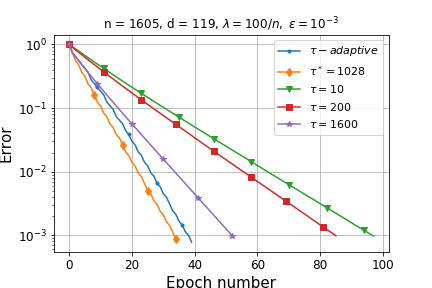}
		}
		\centerline{
			\includegraphics[width=0.5\columnwidth]{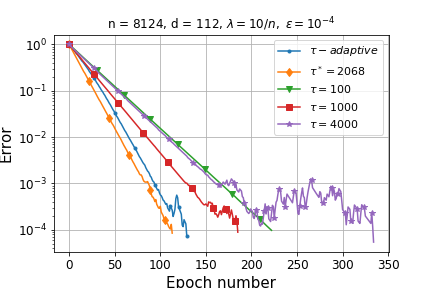}
			\includegraphics[width=0.5\columnwidth]{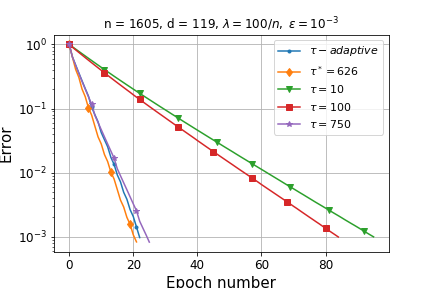}
		}
		\centerline{
			\includegraphics[width=0.5\columnwidth]{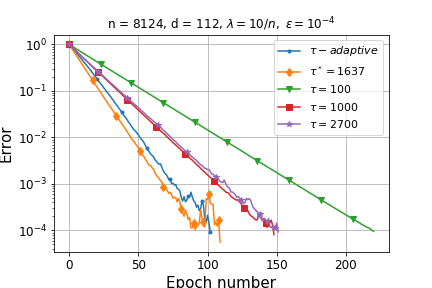}
			\includegraphics[width=0.5\columnwidth]{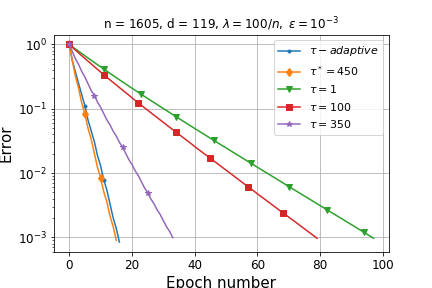}
		}
		\centerline{
			\includegraphics[width=0.5\columnwidth]{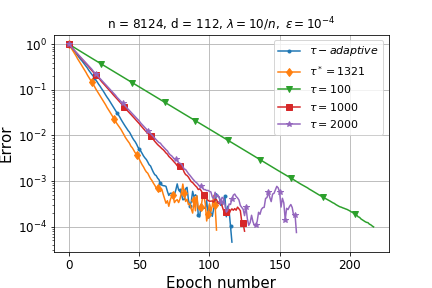}
			\includegraphics[width=0.5\columnwidth]{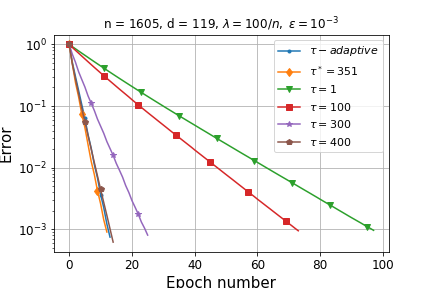}
		}
		\caption[Convergence of \sgd{} with adaptive mini-batch size, $\tau$--partition independent sampling on logistic regression]{\textbf{Convergence of Logistic regression} using $\tau$--partition independent sampling on \mushroom{} dataset (first column) and $\tau$--partition nice sampling on \aoa{} dataset (second column). In four rows, training set is distributed among $1$, $2$, $3$ and $4$ partitions, respectively.
		}
		\label{fig:logistic_sup}
	\end{center}
\end{figure}

\begin{figure}[h]
	\begin{center}
		\centerline{
			\includegraphics[width=0.5\columnwidth]{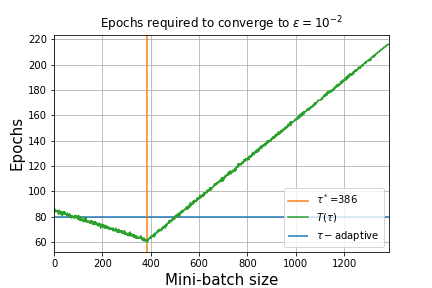}
			\includegraphics[width=0.5\columnwidth]{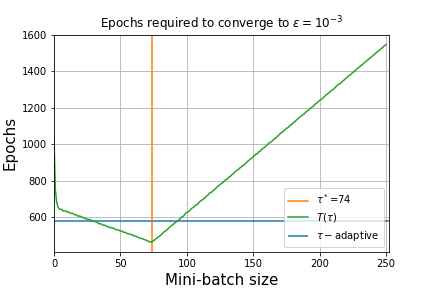}
		}
		\centerline{
			\includegraphics[width=0.5\columnwidth]{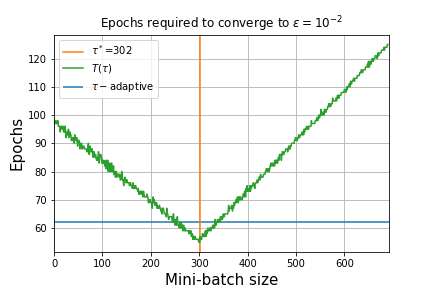}
			\includegraphics[width=0.5\columnwidth]{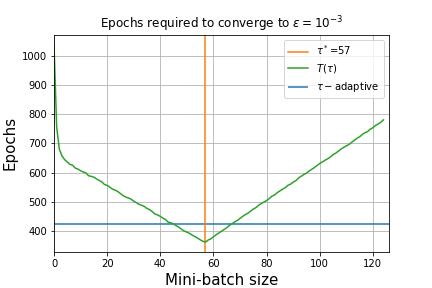}
		}
		\centerline{
			\includegraphics[width=0.5\columnwidth]{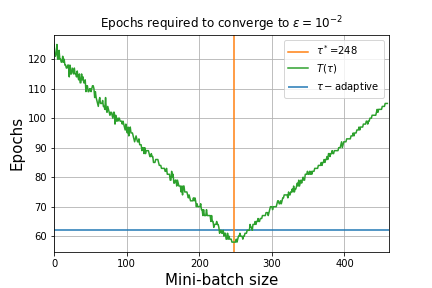}
			\includegraphics[width=0.5\columnwidth]{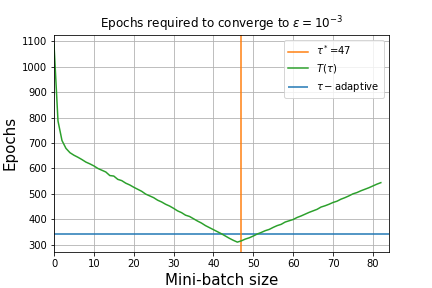}
		}
		\centerline{
			\includegraphics[width=0.5\columnwidth]{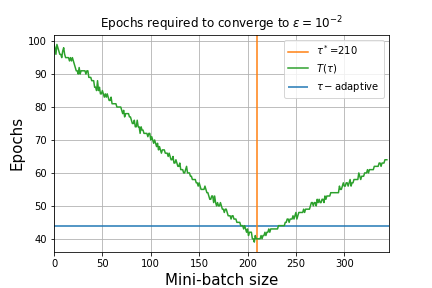}
			\includegraphics[width=0.5\columnwidth]{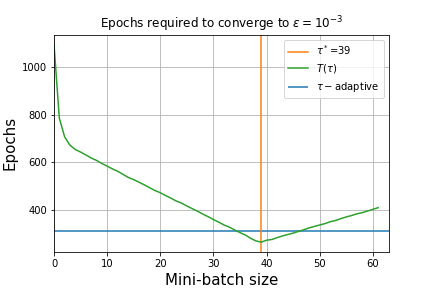}
		}
	\end{center}
    \caption[Effect of mini-batch size on the total iteration complexity]{\textbf{Effect of mini-batch size on the total iteration complexity}. First column: \mg{} dataset sampled using $\tau$--nice partition sampling. Second column: \bodyfat{} dataset sampled using $\tau$--independent partition sampling. In four rows, dataset is distributed across $1$, $2$, $3$, and $4$ partitions. This figure reflects the tightness of the theoretical result in relating the total iteration complexity with the mini-batch size, and the optimal mini-batch size. Moreover, This figure shows how close the proposed algorithm is to the optimal mini-batch size in terms of the performance.
    }
    \label{fig:taus_sup}
\end{figure}

In this section, we present additional experimental results. Here we test each dataset on the sampling that was not tested on. Similar to the earlier result, the proposed algorithm outperforms most of the fixed mini-batch size \sgd{}. Moreover, it can be seen as a first glance, that the optimal mini-batch is non-trivial, and it is changing as we partition the dataset through different number of partitions. For example, in \bodyfat{} dataset that is located in one partition, the optimal mini-batch size was $\tau^* = 74$. Although one can still sample $\tau = 74$ when the data is divided into two partitions, the optimal has changed to $\tau^* = 57$. This is clearly shown in Figure C.3 where it shows that the optimal batch size varies between different partitioning, and the predicted optimal from our theoretical analysis matches the actual optimal.


\chap{\ref{sec:aicn}}

\label{sec:ap_theory}

\section{Proofs}
\subsection{Proof of \Cref{le:AI_newton_nesterov}} 
\begin{proof}
	We follow \citep[Lemma 5.1.1]{nesterov2018lectures} to prove affine-invariance of \newton{}.
	Let $y^k = \mathbf A^{-1} x^k$ for some $k\geq 0$ and $\alpha_k$ be affine-invariant. Firstly,
	\begin{align*}
		y^{k+1} &= y^{k} - \alpha_k\left[\nabla^2 \phi(y^k) \right]^{-1} \nabla \phi(y^k)
		= y^{k} - \alpha_k\left[\mathbf A^\top\nabla^2 f(\mathbf A y^k)\mathbf A \right]^{-1} \mathbf A^\top \nabla f(\mathbf A y^k)\\
		&= \mathbf A^{-1}x^{k} - \alpha_k \mathbf A^{-1}\left[\nabla^2 f(x^k)\mathbf A \right]^{-1} \nabla f(x^k) = \mathbf A^{-1}x^{k+1}.
	\end{align*}
	
	Secondly note that $\normMd {\g(x)} {x}$ is affine invariant, as 
	\begin{align*}
		\normMd{\nabla g(y^k)} {y^k} 
		&=  \nabla g(y^k)^\top \nabla^2 g(y^k)^{-1} \nabla g(y^k) 
		= \nabla f(x^k)^\top \nabla^2 f(x^k)^{-1} \nabla f(x^k) \\
		&= \normMd{\nabla f(x^k)} {x^k}.
	\end{align*}
	Consequently, stepsizes $\alpha_k$ from \ain{} \eqref{eq:update} and \citet{nesterov2018lectures} are all affine-invariant. Hence those \dnewton{} algorithms are affine-invariant.\\
	
\end{proof}

\subsection{Proof of \Cref{le:model}}
\begin{proof}\\
	We rewrite function value approximation from the left hand side as
	\begin{align*}
		f(y) - f(x) - \nabla f(x)[y-x]
		&= \int\limits_0^1  \ls\nabla f(x+\tau(y-x))-\nabla f(x)\rs[y-x] d \tau \\
		&= \int\limits_0^1 \int\limits_0^{\tau}   \ls\nabla^2 f(x+\lambda(y-x))\rs[y-x]^2 d \lambda d \tau.
	\end{align*}
	
	Taking its norm, we can finish the proof as
	\begin{eqnarray*}
		\efa \Big |f(y) - f(x)- \nabla f(x)[y-x] - \frac{1}{2}\nabla^2 f(x)[y-x]^2\Big | \\
		& =& \left|\int\limits_0^1 \int\limits_0^{\tau}   \ls\nabla^2 f(x+\lambda(y-x))-\nabla^2 f(x)\rs[y-x]^2 d \lambda d \tau \right|\\
		&\leq & \int\limits_0^1 \int\limits_0^{\tau} \left|  \ls\nabla^2 f(x+\lambda(y-x))-\nabla^2 f(x)\rs[y-x]^2 \right|d \lambda d \tau \\
		& \stackrel{\eqref{eq:semi-strong-self-concordance}}{\leq} & \int\limits_0^1 \int\limits_0^{\tau} \Lsemi \lambda \norm{y-x}_x^3 d \lambda d \tau
		= \frac{\Lsemi}{6} \norm{y-x}_x^3 .
	\end{eqnarray*}
\end{proof}

\subsection{Proof of \Cref{th:stepsize}}
\begin{proof}		
	Proof is straightforward. To show that \ain{} model update minimizes $S_{f,\Lalg}(x)$, we compute the gradient of $S_{f,\Lalg}(x)$ at next iterate of \ain{}. Showing that it is $0$ concludes $x^{k+1} = S_{f,\Lalg}(x^k)$.		
	For simplicity, denote $h \eqdef y-x$. We can simplify the implicit update step $S_{f,\Lalg}(x)$,
	\begin{align*}
		S_{f,\Lalg}(x) 
		&= \argmin_{y\in \bbE} \Big\lbrace f(x) + \la \nabla f(x),y-x\ra + \frac{1}{2} \la \nabla^2 f(x)(y-x),y-x\ra \\
		& \qquad \qquad \qquad +\frac{\Lalg}{6}\norm{y-x}_{x}^3\Big\rbrace\\
		&= x + \argmin_{h  \in \bbE} \left\lbrace \la \nabla f(x),h\ra + \frac 1 2 \normsM h x. +\frac{\Lalg}{6}\norm{h}_{x}^3\right\rbrace.
	\end{align*}
	Taking gradient of the subproblem with respect to $h$,
	\begin{align*}
		\nabla_h \left( \la \nabla f(x),h\ra + \frac 1 2 \normsM h x +\frac{\Lalg}{6}\norm{h}_{x}^3 \right)= \g(x) + \h(x) h + \frac {\Lalg} 2 \h(x)h \normM h x.
	\end{align*}
	and setting $h$ according to \ain{}, $h = -\alpha \h(x)^{-1} \g(x)$, leads to
	\begin{eqnarray*} 
		\dots &=&\g(x) -\alpha \g(x) - \frac {\Lalg} 2 \alpha^2 \g(x) \normM {\h(x)^{-1} \g(x)} x\\  
		&=&-\g(x) \left( -1 + \alpha + \frac {\Lalg} 2 \alpha^2 \normMd {\g(x)} x \right).
	\end{eqnarray*}				
	Finally, choosing \ain{} as stepsize $\alpha$ \eqref{eq:update} is chosen as a root of quadratic function
	\begin{equation} \label{eq:alpha_star}
		\frac {\Lalg} 2 \normMd {\g(x)} x \alpha^2 + \alpha -1 =0
	\end{equation}
	makes the gradient of next iterate of \ain{}, \eqref{eq:step} $0$. This concludes the proof.
\end{proof}

\subsection{Proof of \Cref{lm:main_lemma}}
\begin{proof}
	This claim follows directly from \Cref{le:model}.
	\begin{eqnarray*}
		\efa f(S_{f,\Lalg}(x)) \\ 
		&\stackrel{\eqref{eq:wssc_ub}}{\leq}& f(x) + \la \nabla f(x), S_{f,\Lalg}(x)-x\ra + \frac{1}{2} \la \nabla^2 f(x)( S_{f,\Lalg} (x)-x),S_{f,\Lalg}(x)-x\ra \\
		&& \qquad +\frac{\Lalg}{6}\norm{ S_{f,\Lalg}(x)-x}_{x}^3 \\
		&\stackrel{\eqref{eq:step}}{=}&  \min \limits_{y \in \bbE} \left \{ f(x) + \la \nabla f(x),y-x\ra + \frac{1}{2} \la \nabla^2 f(x)(y-x),y-x\ra +\frac{\Lalg}{6}\norm{y-x}_{x}^3 \right \} \\
		&\stackrel{\eqref{eq:lower_upper_bound}}{\leq}&  \min \limits_{y \in \bbE} \left\{ f(y)   +  \frac{\Lalg}{3}\norm{y-x}_{x}^3 \right\} .  
	\end{eqnarray*}
\end{proof}

\subsection{Proof of \Cref{thm:convergence}}
Proof requires a technical lemma that we state and prove first.
\begin{lemma} \label{le:sequence_bound} from \citep[Equation (2.23)]{ghadimi2017second} \\ 
	For 
	\[ \eta_t \eqdef \frac{3}{t+3}, \quad t \geq 0, \qquad \text{and} \qquad A_t \eqdef 
	\begin{cases}
		1 , & t = 0,\\
		\prod \limits_{i=1}^t (1-\eta_i), & t\geq 1,
	\end{cases}
	\]
	we have
	\begin{align}
		A_t &= \frac 6 {(t+1)(t+2)(t+3)} \qquad \text{and} \qquad \sumin tk \frac {\eta_k^3} {A_t} &= \sumin tk \frac {9(t+1)(t+2)}{2(t+3)^2} \leq \frac {3k} 2.
	\end{align}    
\end{lemma}

\begin{proof}
	We have
	\begin{equation}\label{eq:A_t_bound}
		A_{k} =\prod_{t=1}^{k}\left(1-\eta_{t}\right)=\prod_{t=1}^{k} \frac{t}{t+3}=\frac{k! \, 3!}{(k+3) !}= 3! \prod_{j=1}^{3} \frac{1}{k+j}, 
	\end{equation}
	which gives,
	\begin{equation}
		\label{eq:sum_A_t_bound}
		\begin{aligned}
			\sum_{t=0}^{k} \frac{A_{k} \eta_{t}^{3}}{A_{t}} &=\sum_{t=0}^{T} \frac{3^{3}}{(t+3)^{3}} \prod_{j=1}^{3}\frac{t+j}{k+j} =3^{3} \prod_{j=1}^{3} \frac{1}{k+j} \sum_{t=0}^{k}  \frac{\prod_{j=1}^{3} (t+j)}{(t+3)^3}.
		\end{aligned}
	\end{equation}
	The sum is non-decreasing. Indeed, we have
	\begin{equation*}
		1 \leq 1 +  \frac{1}{t+3} \leq  1 + \frac{1}{t+j}, \quad \forall j \in \{1,2, 3\},
	\end{equation*}
	and, hence, for all $j \in \{1,2,3\}$,
	\begin{eqnarray*}
		&\ls 1 +  \frac{1}{t+3}\rs^3 &\leq \prod_{j=1}^{3} \ls 1 + \frac{1}{t+j} \rs \\
		\Leftrightarrow & \ls\frac{t+4}{t+3}\rs^3 & \leq \prod_{j=1}^{3} \frac{t+j+1}{t+j}\\
		\Leftrightarrow &  \frac{\prod_{j=1}^{3} (t+j)}{(t+3)^3} & \leq \frac{\prod_{j=1}^{3} (t+1+j)}{(t+1+3)^3} .
	\end{eqnarray*}		
	Thus, we have shown that the summands in the RHS of \eqref{eq:sum_A_t_bound} are growing, whence we get the next upper bound for the sum
	\begin{eqnarray}
		\sum_{t=0}^{k} \frac{A_{k} \eta_{t}^{3}}{A_{t}} 
		&=&3^{3} \prod_{j=1}^{3} \frac{1}{k+j} \sum_{t=0}^{k}  \frac{\prod_{j=1}^{3} (t+j)}{(t+3)^3} \nonumber\\
		&\leq& 3^3 \prod_{j=1}^{3} \frac{1}{k+j} (k+1)  \frac{ \prod_{j=1}^{3} (k+j)}{(k+3)^3}
		\leq\frac{(k+1)3^{3}}{(k+3)^{3}}\\
		&=& \mathcal O\ls \frac{1}{k^{2}}\rs.
		\label{eq:ATat/At_i0}
	\end{eqnarray}
	
\end{proof}

Now we can prove \Cref{thm:convergence}.

\begin{proof}
	We start by taking Lemma \ref{lm:main_lemma} for any $t \geq 0$, we obtain
	\begin{eqnarray*}
		f(x^{t+1}) 
		&\stackrel{\eqref{eq:min_lemma}}{\leq}& 
		\min \limits_{y \in \bbE} \left\{ f(y)  + \frac{\Lalg}{3}\norm{y-x^t}_{x^t}^3   \right \}  \\
		&\stackrel{\eqref{D_lebeg}}{\leq}&  \min \limits_{\eta_t \in [0, 1]} \left \{ f(x^{t} + \eta_t (\opt - x^t)) + \frac{\Lalg}{3}\eta_t^3 D^3  \right\} \\
		&\leq & \min \limits_{\eta_t \in [0, 1]} \left \{ (1-\eta_t)f(x^t) + \eta_t f(\opt) + \frac{\Lalg}{3}\eta_t^3 D^3 \right\},
	\end{eqnarray*}
	where for the second line we take $y=x^t + \eta_t (\opt-x^t)$ and use convexity of $f$ for the third line.
	Therefore, subtracting $f(\opt)$ from both sides, we obtain, for any $\eta_t \in [0,1]$, we get
	\begin{equation}\label{eq:conv_to_sum}
		f(x^{t+1}) - f(\opt) \leq (1-\eta_t)(f(x^t) - f(\opt)) +\frac{\Lalg}{3}\eta_t^3 D^3 .
	\end{equation}
	Let us define the sequence $\{A_t\}_{t\geq 0}$ as follows:
	\begin{equation} \label{A_t}
		\gboxeq{
			A_t \eqdef 
			\begin{cases}
				1 , & t = 0,\\
				\prod \limits_{i=1}^t (1-\eta_i), & t\geq 1.
			\end{cases}
		}
	\end{equation}
	Then $A_t = (1-\eta_t)A_{t-1}$. Also, we define \gbox{$\eta_0 \eqdef 1$.}
	Dividing both sides of \eqref{eq:conv_to_sum} by $A_t$, we get
	\begin{eqnarray}
		\frac{1}{A_t}\left(f(x^{t+1}) - f(\opt)\right)
		&\leq& \frac{1}{A_t} \left(1-\eta_t)(f(x^t) - f(\opt) \right)
		+ \frac{\eta_t^3}{A_t} \frac{\Lalg D^3}{3} \notag\\
		&=& \frac{1}{A_{t-1}}(f(x^t) - f(\opt)) 
		+ \frac{\eta_t^3}{A_t} \frac{\Lalg D^3}{3}  \label{eq:summing}.
	\end{eqnarray}
	Summing  both sides of inequality \eqref{eq:summing} for $t = 0, \ldots, k$ , we obtain
	\begin{eqnarray*}
		\frac{1}{A_k}(f(x^{{k}+1}) - f(\opt))  
		&\leq& \frac{(1-\eta_0)}{A_0}(f(x^0) - f(\opt))
		+ \frac{\Lalg D^3}{3}  {\sum \limits_{t=0}^k \frac{ \eta_t ^{3}}{A_t}} \notag\\
		&\stackrel{\eta_0=1}{=} &
		\frac{\Lalg D^3}{3}  {\sum \limits_{t=0}^k \frac{ \eta_t ^{3}}{A_t}}.
	\end{eqnarray*}
	As a result, 
	\begin{equation}
		f(x^{{k}+1}) - f(\opt) \leq
		\frac{\Lalg D^3}{3}  {\sum \limits_{t=0}^k \frac{A_k \eta_t ^{3}}{A_t}}. \label{eq:last} 
	\end{equation}
	To finish the proof, select $\eta_t$ so that $\sum_{t=0}^{k} \frac{A_{k} \eta_{t}^{3}}{A_{t}} = \okd$. This holds for\footnote{Note that formula of $\eta_0$ coincides with its definition above.}
	\begin{equation} \label{alpha_t}
		\gboxeq{\eta_t \eqdef \frac{3}{t+3}, \quad t \geq 0,}
	\end{equation}
	as stated in the \Cref{le:sequence_bound}.
	Plugging \Cref{le:sequence_bound} inequalities to \eqref{eq:last} concludes the proof of the \Cref{thm:convergence},
	\begin{align*}
		f(x^{{k}+1}) - f(\opt) 
		\leq \frac 6 {(k+1)(k+2)(k+3)} \frac{\Lalg D^3}{3} \frac {9k} 2
		\leq \frac{9\Lalg  D^3}{k^2}
		\leq \frac{9\Lalg  R^3}{k^2}.
	\end{align*}
	
\end{proof}

\subsection{Proof of \Cref{le:lsconv}}
\begin{proof}
	Claim follows from Theorem 5.1.7 of \citet{nesterov2018lectures}, which states that for $\Lstandard$--self-concordant function, hence also for $\Lsemi$--semi-strongly self-concordant function $f:\R^d\to\R$ and $x, y$ such that $\frac {\Lsemi} 2 \normM {y-x } {x} < 1$ holds
	\begin{equation*}
		\left( 1- \frac{\Lsemi} 2 \normM {y-x } {x} \right)^2 \h(y) 
		\preceq \h(x) 
		\preceq \left( 1- \frac{\Lsemi} 2 \normM {y-x } {x} \right)^{-2} \h(y).
	\end{equation*}
	Let $c$ be some constant, $0<c<1$.
	Then for updates of \ain{} in the neighborhood $$\left \{ x^k: c \geq \frac{\Lalg}{2}\alpha_k \normMd {\g (x^k)} {x^k} \right \}$$ holds
	\begin{equation} \label{eq:inv_hessian_bound_scf}
		\h(x^{k+1})^{-1} \preceq \left( 1 - c \right)^{-2} \h(x^k)^{-1}.
	\end{equation}
\end{proof}

\subsection{Proof of \Cref{le:one_step_local}}
In order to prove \Cref{le:one_step_local}, we first use semi-strong self-concordance to prove a key inequality -- a version of Hessian smoothness, bounding gradient approximation by difference of points.
\begin{lemma}\label{le:sssc_to_loc_bound}
	For semi-strongly self-concordant function $f:\R^d\to\R$ holds
	\begin{equation}
		\label{eq:semi_norm_ineq_}
		\norm{\nabla f(y) - \nabla f(x) - \nabla^2 f(x)[y-x] }^*_x
		\leq \frac{\Lsemi}{2} \norm{y-x}_x^2.
	\end{equation}
\end{lemma}
\begin{proof} [\Cref{le:sssc_to_loc_bound}]\\
	We rewrite gradient approximation on the left hand side as
	\begin{gather*}
		\nabla f(y) -  \nabla f(x) - \nabla^2 f(x)[y-x] =
		\int\limits_0^1  \ls\nabla^2 f(x+\tau(y-x))-\nabla^2 f(x)\rs[y-x] d \tau.
	\end{gather*}
	Now, we can bound it in dual norm as
	\begin{eqnarray*}
		\efa \norm{\nabla f(y) - \nabla f(x) - \nabla^2 f(x)[y-x] }_x^{\ast}\\
		&=& \norm{\int\limits_0^1   \ls\nabla^2 f(x+\tau(y-x))-\nabla^2 f(x)\rs[y-x] d \tau }_x^{\ast}\\
		&\leq& \int\limits_0^1 \norm{  \ls\nabla^2 f(x+\tau(y-x))-\nabla^2 f(x)\rs[y-x] }_x^{\ast}  d \tau \\
		&\leq& \int\limits_0^1  \norm{\nabla^2 f(x+\tau(y-x))-\nabla^2 f(x)}_{op} \norm{y-x}_x   d \tau \\
		&\stackrel{\eqref{eq:semi-strong-self-concordance}}{\leq}& \int\limits_0^1  \Lsemi \tau \norm{y-x}_x^2 d  \tau
		= \frac{\Lsemi}{2} \norm{y-x}_x^2,
	\end{eqnarray*}
	where in second inequality we used property of operator norm \eqref{eq:matrix_operator_norm}.\\
\end{proof}	

Finally, we are ready to prove one step decrease and the convergence theorem.\\

\begin{proof}[\Cref{le:one_step_local}; one step decrease locally]\\		
	We bound norm of $\nabla f(x^{k+1})$ using basic norm manipulation and triangle inequality as
	\begin{eqnarray*}
		\efa \normMd {\nabla f(x^{k+1})} {x^k}\\
		& \stackrel{\eqref{eq:update}}{=}& \normMd{\nabla f(x^{k+1}) - \nabla^2 f(x^k)(x^{k+1}-x^k)  - \alpha_k \nabla f(x^k) }{x^k}\\
		& =& \normMd{\nabla f(x^{k+1}) - \nabla f(x^{k}) - \nabla^2 f(x^k)(x^{k+1}-x^k)  + (1-\alpha_k) \nabla f(x^k) }{x^k}\\
		&\leq &\normMd{\nabla f(x^{k+1}) - \nabla f(x^{k}) - \nabla^2 f(x^k)(x^{k+1}-x^k)}{x^k}  + (1-\alpha_k) \normMd{\nabla f(x^k) }{x^k}.
	\end{eqnarray*}
	Using \Cref{le:sssc_to_loc_bound}, 
	we can continue
	\begin{eqnarray*}
		\normMd {\nabla f(x^{k+1})} {x^k}
		&\leq& \normMd{\nabla f(x^{k+1}) - \nabla f(x^{k}) - \nabla^2 f(x^k)(x^{k+1}-x^k)}{x^k} \\
		&& \qquad + (1-\alpha_k) \normMd{\nabla f(x^k) }{x^k}  \\
		&\stackrel{\eqref{eq:semi_norm_ineq_}}{\leq}& 
		\frac{\Lsemi}{2} \normsM{x^{k+1}-x^k} {x^k} + (1-\alpha_k) \normMd{\nabla f(x^k) }{x^k} \\
		&\stackrel{\eqref{eq:update}}{\leq} &
		\frac{\Lsemi \alpha_k^2}{2} \normsMd{\nabla f(x^k)} {x^k} + (1-\alpha_k) \normMd{\nabla f(x^k) }{x^k} \\
		& \leq&  \frac{\Lalg \alpha_k^2}{2} \normsMd{\nabla f(x^k)} {x^k} + (1-\alpha_k) \normMd{\nabla f(x^k) }{x^k} \\
		&=&  \left(\frac{\Lalg \alpha_k^2}{2} \normMd{\nabla f(x^k)} {x^k} -\alpha_k +1\right) \normMd{\nabla f(x^k) }{x^k}\\
		&\stackrel{\eqref{eq:alpha_star}}{=}& \Lalg \alpha_k^2\normsMd{\nabla f(x^k) }{x^k}.
	\end{eqnarray*}
	We use \Cref{le:lsconv} to shift matrix norms.
	\begin{eqnarray}
		\normMd{\g (x^{k+1})}{x^{k+1}}
		& \stackrel{\eqref{eq:inv_hessian_bound_scf}} \leq & \frac{1}{1-c} \normMd{\g (x^{k+1})} {x^k} \notag \\
		& \stackrel{\eqref{eq:one_step_local}} \leq & \frac{\Lalg \alpha_k^2}{1-c}  \normsMd{\g (x^k)} {x^k}\label{eq:local_shifted_scf_alt}\\
		& < & \frac{\Lalg \alpha_k}{1-c}  \normsMd{\g (x^k)} {x^k}. \notag
	\end{eqnarray}
	
	Neighborhood of decrease is obtained by solving $ \frac{\Lalg \alpha_k}{1-c}  \normMd{\g (x^k)} {x^k} \leq 1$.\\
\end{proof}

\subsection{Proof of \Cref{th:local}}
\begin{proof}	
	Let $c = \frac{1}{3}$,
	then for $\normMd{\g (x^0)} {x^0} < \frac {8 } {9\Lalg}$, we have
	$\frac{\Lalg \alpha_0}{1-c}  \normMd{\g (x^0)} {x^0} \leq 1$ and $c \geq \frac{\Lalg}{2}\alpha_0 \normMd {\g (x^0)} {x^0}$. Then from \Cref{le:one_step_local} we have  guaranteed the decrease of gradients $\normMd \gn {x^{k+1}} \leq \normMd \gk {x^k} < \frac {8} {9 \Lalg}.$ We finish proof by chaining \eqref{eq:local_shifted_scf_alt} and simplifying with $\alpha_i \leq 1$.
	\begin{align}
		\normMd{\g(x^{k})}{x^{k}}
		&\leq \left( \frac{3} {2} \Lalg \right )^k \left( \prod _{i=0}^{k} \alpha_i^{2} \right)
		\left( \normMd{\g(x^0)} {x^0} \right)^{2^k} .
	\end{align}
\end{proof}

\section{Global convergence under weaker assumptions} \label{sec:ap_global_nsc}
We can prove global convergence to a neighborhood of the solution without using self-concordance directly, just by utilizing the following convexity and Hessian smoothness in local norms (Assumption~\ref{as:self_con_glob}):
\begin{assumption}[Hessian smoothness, in local norms] \label{as:self_con_glob}
	Objective function $f:\R^d\to\R$ satisfy 
	\begin{equation} \label{eq:self_con_global}
		f(x+h) - f(x) \leq \ip{\g(x)} {h} + \tfrac 1 2 \normM h x ^2 + \tfrac \Lalt 6 \normM h x ^3, \qquad \forall x,h \in \bbE.
	\end{equation}
\end{assumption}

\begin{lemma}[One step decrease globally] \label{le:one_step_global}
	Let Assumption~\ref{as:self_con_glob} hold and let $\Lalg \geq \Lalt$. Iterates of \ain{} \cref{eq:update} yield function value decrease,
	\begin{equation} \label{eq:one_step_global}
		f(x^{k+1})-f(x^k)
		\leq \begin{cases}
			- \frac1 {2\sqrt \Lalg} \normM {\g(x^k)} {x^k} ^{*\frac32}, & \text{if } \normMd {\g(x^k)} {x^k} \geq \frac4 \Lalg, \\
			- \frac1 4 \normsMd {\g(x^k)} {x^k}, & \text{if } \normMd {\g(x^k)} {x^k} \leq \frac4 \Lalg,\\
			- \frac{\sqrt {c_1}} {2\sqrt \Lalg} \normM {\g(x^k)} {x^k} ^{*\frac32}, & \text{if } \normMd {\g(x^k)} {x^k} \geq \frac{4 c_1} \Lalg \text{ and } 0<c_1 \leq 1. 
		\end{cases}.
	\end{equation}
\end{lemma}

Decrease of \Cref{le:one_step_global} is tight up to a constant factor. As far as $\normMd {\g(x^k)} {x^k} \leq \frac{4 c_1} {\Lalg}$, we have functional value decrease as the first line of \eqref{eq:one_step_global}, which leads to $\cO \left( k^{-2} \right)$ convergence rate.

We can obtain fast convergence to only neighborhood of solution, because close to the solution, gradient norm is sufficiently small $\normMd {\g(x^k)} {x^k} \leq \frac{4c_1} {\Lalg}$ and we get functional value decrease from second line of \eqref{eq:one_step_global}. However, this convergence is slower then $\cO \left( k^{-2} \right)$ for $\normMd {\g(x^k)} {x^k} \approx 0$ and it is insufficient for $\cO(k^{-2})$ rate.

Note that third line generalizes first line; we use it to remove a constant factor gap from the neighborhood of fast local convergence.

\begin{theorem}[Global convergence] \label{th:global}
	Let function $f:\R^d\to\R$ be convex and Assumptions \ref{as:self_con_glob}, \ref{as:level_sets} hold, and constants $c_1, \Lalg$ satisfy $0<c_1\leq 1, \Lalg \geq \Lalt$. For $k$ until $\normMd {\g(x^k)} {x^k} \geq \frac{4c_1} {\Lalg}$,  \ain{} has global quadratic decrease, $f(x^k)-f^* \leq \mathcal O \left( \frac{\Lalg D^3}{k^2} \right)$. 
\end{theorem}
Note that the global quadratic decrease of \ain{} is only to a neighborhood of the solution. However, once \ain{} gets to this neighborhood, it converges using (faster) local convergence rate (\Cref{th:local}).

\subsection{Proofs of \Cref{sec:ap_global_nsc}}
Throughout the rest of proofs, we simplify expressions by denoting terms 
\begin{equation} \label{eq:hg_def}
	\gboxeq{\gk \eqdef \g(x^k)} \qquad \text{and} \qquad \gboxeq{\hk \eqdef \xdiff}, 
\end{equation}
for which holds
\begin{align} \label{eq:hg_relation} 
	\hk = - \alpha_k \h(x^k)^{-1} \gk, \quad \gk = -\frac1 {\alpha_k} \h(x^k) \hk \quad \text {and} \quad \normM {\hk } {x^k} = \sqrt \alpha \normMd {\gk} {x^k},
\end{align} 
and also \gbox{$\G \eqdef\Lalg \normMd \gk {x^k}.$}\\

\subsubsection{Proof of \Cref{le:one_step_global}}
\begin{proof}
	We can use Assumption~\ref{as:self_con_glob} to obtain
	\begin{eqnarray}
		f(x^{k+1})-f(x^k)
		&=& f(x^k+\hk) - f(x^k) \nonumber  \\
		&\stackrel {\eqref{eq:self_con_global}}\leq& \ip{\g(x^k)} {\hk} + \frac 1 2 \normsM \hk {x^k} + \frac \Lalg 6 \normM \hk {x^k} ^3 \label{eq:self_con_first} \\
		&\stackrel {\eqref{eq:hg_relation}}=& -\alpha_k \normsMd \gk {x^k} + \frac 1 2 \alpha_k^2 \normsMd \gk {x^k} + \frac \Lalg 6 \alpha_k^3 \normM \gk {x^k} ^{*3} \nonumber\\
		&=& -\alpha_k \normsMd \gk {x^k} \left( 1- \frac 1 2 \alpha_k - \frac \Lalg 6 \alpha_k^2 \normMd \gk {x^k} \right). \label{eq:self_con_last}
	\end{eqnarray}
	
	We can simplify bracket in \cref{eq:self_con_last} as
	\begin{eqnarray*}
		1- \frac 1 2 \alpha_k - \frac \Lalg 6 \alpha_k^2 \normMd \gk {x^k}
		&=& 1 - \frac 1 2 \frac {\sqrt{1+2\G} -1}{\G} - \frac {\G} 6 \left( \frac {\sqrt{1+2\G} -1}{\G} \right) ^2\\
		&=& \frac {4\G + 1 - \sqrt{1+2\G}}{6\G}
		\stackrel{\eqref{eq:AG}}\geq \frac 1 2.
	\end{eqnarray*}
	Also, we can simplify the other term in \cref{eq:self_con_last} as
	\begin{eqnarray*}
		\alpha_k \normsMd \gk {x^k}
		&=& \frac {\normMd \gk {x^k} } \Lalg \left( \sqrt{1+ 2\G} -1 \right)\frac{\sqrt{1+ 2\G} +1} {\sqrt{1+ 2\G} +1}
		= \frac {2 \normsMd \gk {x^k} }{\sqrt{1+ 2\G} +1}\\
		&\stackrel {\eqref{eq:jensen}} \geq&  \frac {2 \normsMd \gk {x^k} } { \sqrt {\G} + 1 + \frac 1 {\sqrt 2} } 
		\geq \frac {2 \normsMd \gk {x^k} } { \sqrt {\G} + 2 }
		\geq \frac { \normsMd \gk {x^k} } { \max \left( \sqrt {\G}, 2 \right) },
	\end{eqnarray*}
	and plug these two result into \cref{eq:self_con_last} to obtain first two lines of \eqref{eq:one_step_global}. Third line can be obtained from the first line of \eqref{eq:one_step_global}. For $c_1$ so that $0<c_1\leq 1$ and $x^k$ satisfying $\frac{4c_1} \Lalg \leq \normMd {\g(x^k)} {x^k} <\frac4 {\Lalg}$ holds
	
	\[f(x^{k+1})-f(x^k) \leq -\frac1 4 \normsMd {\g(x^k)} {x^k} \leq -\frac{\sqrt {c_1}} {2\sqrt \Lalg} \normM {\g(x^k)} {x^k} ^{*\frac32}.\]
\end{proof}

\subsubsection{Proof of \Cref{th:global}}
\begin{proof}
	As a consequence of convexity and bounded level sets (Assumption~\ref{as:level_sets}), we have 
	\begin{eqnarray*} 
		f(x^{k})-f^{*} 
		&\leq& \ip{\gk} {x^k-\opt}
		= \ip{\h(x^k)^{-1/2}\gk} {\h(x^k)^{1/2}(x^k-\opt)}\\
		& \leq& \normMd{\gk}{x^k} \normM{x^k-\opt}{x^k} 
		\leq D\normMd{\gk}{x^k}.
	\end{eqnarray*}
	Plugging it into \cref{eq:one_step_global} (which holds under Assumption~\ref{as:self_con_glob}) and noting that it yields
	\begin{equation*}
		f(x^{k+1})-f(x^k)
		\leq 
		- \frac {\sqrt{c_1}} {2\sqrt \Lalg D^{3/2}} \left( f(x^{k})-f^{*} \right) ^{3/2}.
	\end{equation*}
	Denote $\tau \eqdef \frac {\sqrt{c_1}} {2\sqrt \Lalg D^{3/2}}$ and $\err_k \eqdef \tau^2 (f(x^k)-f^*) \geq 0$.
	Terms $\err_k$ satisfy recurrence
	\begin{align*}
		\err_{k+1}
		&= \tau^2 (f(x^{k+1})-f^*)
		\stackrel{\eqref{eq:one_step_global}}\leq \tau^2 (f(x^k)-f^*) - \tau^3 (f(x^k)-f^*)^{3/2}
		= \err_k - \err_k^{3/2}.
	\end{align*}
	Because $\err_{k+1} \geq 0$, we have that $\err_k \leq 1$. 
	
	Proposition 1 of \citet{mishchenko2021regularized} shows that the sequence $\left\{ \err_k \right\} _{k=0} ^\infty, 0\leq \err_k \leq 1$ decreases as $\cO \left(\frac 1 {k^2} \right)$. Let $c_2$ be a constant satisfying $\err_k \leq \frac {c_2} {k^2}$ for all $k$ ($c_2 \approx 3$ seems to be sufficient). Finally, fol $k\geq \frac {\sqrt {c_2}}{\tau \sqrt{\varepsilon}} = 2\sqrt{ \frac {c_2\Lalg D^3}{c_1 \varepsilon} } = \cO \left( \sqrt{ \frac {\Lalg D^3}{\varepsilon} } \right)$ we have 
	\begin{equation*}
		f(x^k)-f^*
		= \frac {\err_k} {\tau^2}
		\leq \frac {c_2} {c_1 k^2 \tau^2}
		\leq \varepsilon.
	\end{equation*}
\end{proof}



\chap{\ref{sec:sgn}}

\section{Table of frequently used notation}
\begin{table}[h!]
	\centering
    \begin{threeparttable}        
	\caption[Summary of frequently used notation in \Cref{sec:sgn}]{Summary of frequently used notation}
	\label{tab:notation}
	\begin{tabular}{|c|l|}
        \hline 
        \multicolumn{2}{|c|}{{\bf General} }\\
        \hline
		$\mathbf A^\dagger$ & Moore-Penrose pseudoinverse of $\mathbf A$\\
		$\normM{\cdot}{op} $ & Operator norm\\
		$\normM{\cdot}x $ & Local norm at $x$\\
		$\normMd{\cdot}x $ & Local dual norm at $x$\\
		$x, x_+, x^k \in \mathbb{R}^d$ & Iterates\\
		$y \in \mathbb{R}^d$ & Virtual iterate (for analysis only)\\
		$h, h' \in \mathbb{R}^d$ & Difference between consecutive iterates\\
		$\als k$ & \sgn{} Stepsize\\
		\hline
        \multicolumn{2}{|c|}{{\bf Function specifics} }\\
        \hline
		$d$ & Dimension of problem\\
		$f: \mathbb{R}^d \rightarrow \mathbb{R}$ & Loss function \\
		$\modelS {\cdot, x} $ & Upperbound on $f$ based on gradient and Hessian in $x$\\
		$\opt$, $\fopt$ & Optimal model and optimal function value\\
		$\level$ & Set of models with functional value less than $x^0$\\
		$R, D, D_2$ & Diameter of $\level$\\
		$\Lstandard, \Lsemi$ & Self-concordance and semi-strong self-concordance constants\\
		$\Lalg$ & Smoothness estimate, affects stepsize of \sgn{}\\
		$\Lrel, \murel$ & Relative smoothness and relative convexity constants\\
		\hline
        \multicolumn{2}{|c|}{{\bf Sketching} }\\
        \hline
		$\gS, \hS, \normMS h x$  & Gradient, Hessian, local norm in range $\s$, resp.\\
		$\s \in \mathbb{R}^{d \times \tau(\s)}$ & Randomized sketching matrix\\
		$\tau(\s)$ & Dimension of randomized sketching matrix\\
		$\tau$ & Fixed dimension constraint on $\s$\\
		$\Ls$ & Self-concordance constant in range of $\s$\\
		$\p$ & Projection matrix on subspace $\s$ w.r.t. local norm at $x$\\         
		$\rho(x)$ & Condition numbers of expected scaled projection matrix $\E{\alpha \p}$\\
		$\rho$ & Lower bound on condition numbers $\rho(x)$\\
        \hline
	\end{tabular}
    \end{threeparttable}
\end{table}

\section{Experiments: technical details and extra comparison}
For completeness, in \Cref{fig:vsacd} we include comparison of \sgn{} and Accelerated Coordinate Descent on small-scale experiments.

We use comparison framework from \citep{hanzely2020stochastic}, including implementations of \sscn{}, Coordinate Descent and Accelerated Coordinate Descent.

Experiments are implemented in Python 3.6.9 and run on workstation with 48 CPUs Intel(R) Xeon(R) Gold 6246 CPU @ 3.30GHz. Total training time was less than $10$ hours. Source code and instructions are included in supplementary materials. As we fixed random seed, experiments are fully reproducible.

\begin{figure}
	\centering
	\includegraphics[width=0.49\textwidth]{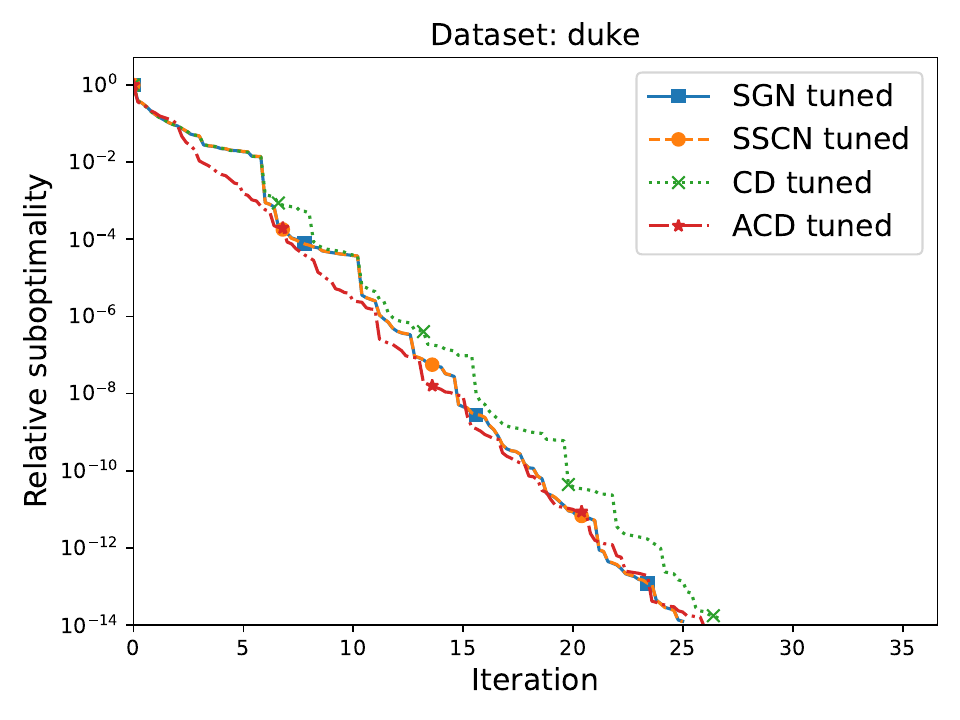}   
	\includegraphics[width=0.49\textwidth]{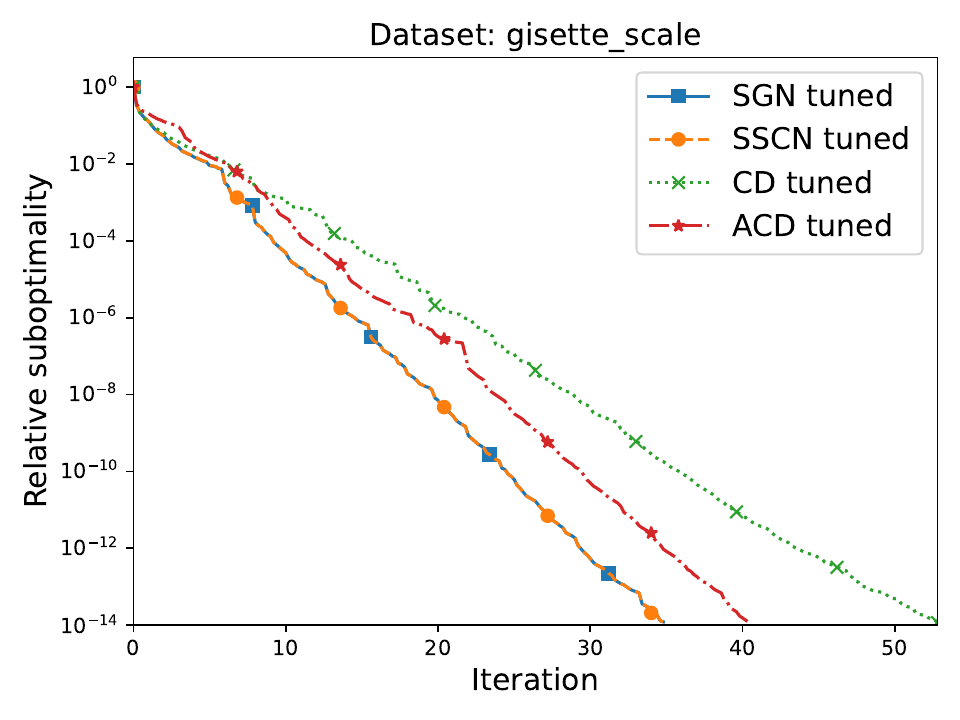}   
	\includegraphics[width=0.49\textwidth]{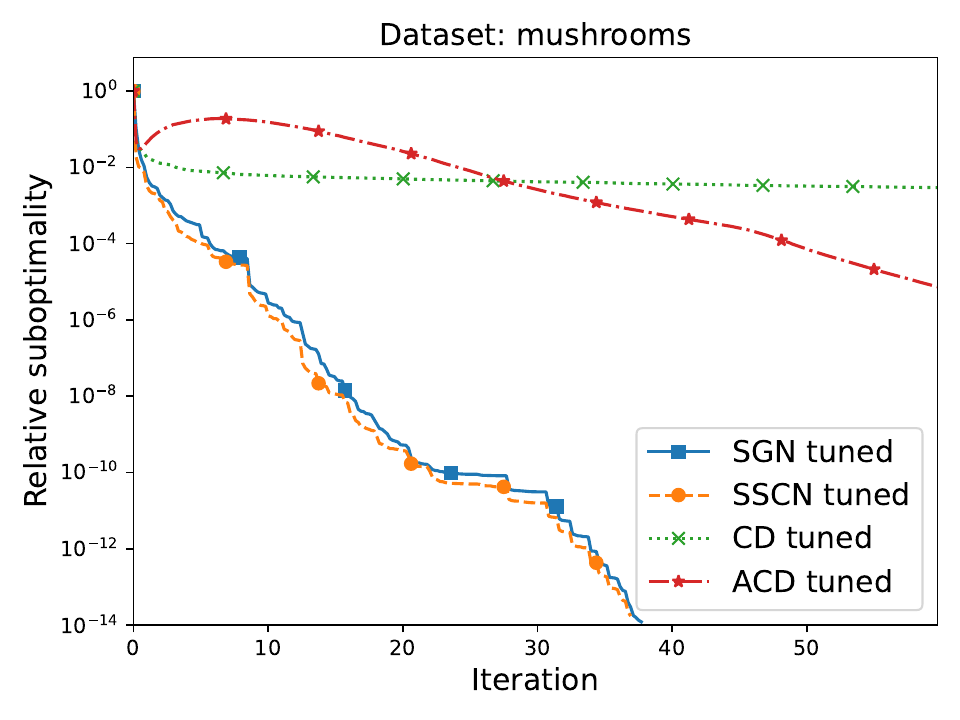}   
	\includegraphics[width=0.49\textwidth]{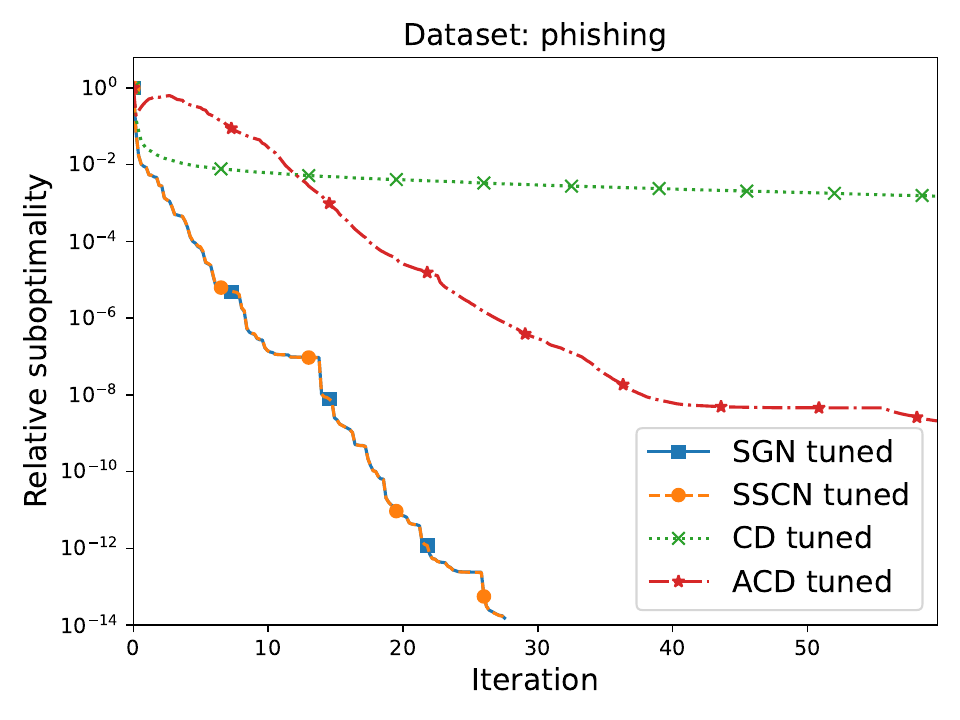}   
	\caption[Comparison of \sgn{} and first-order methods]{Comparison of \sscn{}, \sgn{}, \cd{} and \acd{} on logistic regression on \libsvm{} datasets for sketch matrices $\s$ of rank one. We fine-tune all algorithms for smoothness parameters.}
	\label{fig:vsacd}
\end{figure}

\section{Algorithm comparisons}
For readers convenience, we include pseudocodes of the most relevant baseline algorithms: Exact Newton Descent (\Cref{alg:end}), \rsn{} (\Cref{alg:rsn}), \sscn{} (\Cref{alg:sscn}), \ain{} (\Cref{alg:aicn}).

\begin{figure}[h]
	\begin{minipage}[t]{0.48\textwidth}
		\begin{algorithm}[H]
			\caption{Exact Newton Descent \citep{KSJ-Newton2018}} \label{alg:end}
			\begin{algorithmic}
				\State \textbf{Requires:} Initial point $x^0 \in \mathbb{R}^d$, $c$--stability bound $\sigma >c >0$
				\State
				\vspace{.5mm}
				\For {$k=0,1,2\dots$}
				\State
				\State $x^{k+1} = x^k - {\color{blue}\frac 1 {\sigma}} \left[\h(x^k)\right]^{\dagger} \g(x^k)$
				\EndFor
			\end{algorithmic}
		\end{algorithm}
	\end{minipage}
	\hfill
	\begin{minipage}[t]{0.55\textwidth}
		\begin{algorithm}[H]
			\caption{\rsn{}: Randomized Subspace Newton \citep{RSN}} \label{alg:rsn}
			\begin{algorithmic}
				\State \textbf{Requires:} Initial point $x^0 \in \mathbb{R}^d$, distribution of sketches $\cD$, relative smoothness constant $L_{\text{rel}} >0$
				\For {$k=0,1,2\dots$}
				\State {\color{mydarkgreen}Sample $\s_k \sim \cD$}
				\State $x^{k+1} = x^k - {\color{blue} \frac 1 {\Lrel}} {\color{mydarkgreen}\sk} \left[\nabla^2_{{\color{mydarkgreen}\sk}}f(x^k)\right]^{\dagger} \nabla_{{\color{mydarkgreen}\sk}}f(x^k)$
				\EndFor
			\end{algorithmic}
		\end{algorithm}
	\end{minipage}
	\begin{minipage}[t]{0.52\textwidth}
		\begin{algorithm} [H]
			\caption{\ain{}: Affine-Invariant Cubic Newton \citep{hanzely2022damped}} \label{alg:aicn}
			\begin{algorithmic}
				\State \textbf{Requires:} Initial point $x^0 \in \mathbb{R}^d$, estimate of semi-strong self-concordance $\Lalg \geq \Lsemi>0$
				\For {$k=0,1,2\dots$}
				\State {\color{blue}$\als k = \afrac {-1 +\sqrt{1+2 \Lalg \norm{\g(x^k)}_{x^k}^* }}{\Lalg \norm{ \g(x^k)}_{x^k}^* }$}
				\State $x^{k+1} = x^k - {\color{blue}\als k} \left[\h(x^k)\right]^{-1} \g(x^k)$\footnote {Equival.,  $x^{k+1} = x^k - {\color{brown}\argmin_{h \in \mathbb{R}^d} T(x^k, h)}$, \\ for {\color{brown}$T(x,h) \eqdef \la \g(x), h \ra + \frac 12 \normsM h x + \frac {\Lalg}6 \normM h x ^3$}.}
				\EndFor
			\end{algorithmic}
		\end{algorithm}
	\end{minipage}
	\hfill
	\begin{minipage}[t]{0.53\textwidth}
		\begin{algorithm} [H]
			\caption{\sscn{}: Stochastic Subspace Cubic Newton \citep{hanzely2020stochastic}} \label{alg:sscn}
			\begin{algorithmic}
				\State \textbf{Requires:} Initial point $x^0 \in \mathbb{R}^d$, distribution of random matrices $\cD$, Lipschitzness of Hessian constant $\Ls >0$
				\For {$k=0,1,2\dots$}
				\vspace{2.9mm}
				\State {\color{mydarkgreen}Sample $\sk \sim \cD$}
				\vspace{2.9mm}
				\State $x^{k+1} = x^k - {\color{mydarkgreen}\sk} {\color{brown}\argmin_{h \in \mathbb{R}^d} \hat T_{\color{mydarkgreen}\sk}(x^k, h)}$\footnote{\vspace{1mm}\newline for {\color{brown}$\hat T_{\color{mydarkgreen}\s}(x, h) = \la \g(x), {\color{mydarkgreen}\s} h \ra + \frac 12 \normsM {{\color{mydarkgreen}\s} h} x + \frac {\Ls}6 \normM {{\color{mydarkgreen}\s} h}{} ^3$}.}
				\EndFor
			\end{algorithmic}
		\end{algorithm}
	\end{minipage}
	\caption[Pseudocodes of algorithms related to \sgn{}]{Pseudocodes of algorithms related to \sgn{}. We highlight stepsizes of Newton method in blue, subspace sketching in green, regularized Newton step in brown.}
\end{figure}

\clearpage
\section{Proofs}

\subsection{Basic facts}
For any vectors $a, b\in\R^d$ and scalar $\nu>0$, Young's inequality states that
\begin{align}
	2\<a, b>\le \nu\|a\|^2+ \frac{1}{\nu}\|b\|^2. \label{eq:young}
\end{align}
Moreover, we have
\begin{align}
	\|a+b\|^2
	\le 2\|a\|^2 + 2\|b\|^2.\label{eq:a_plus_b}
\end{align}
More generally, for a set of $m$ vectors $a_1,\dotsc, a_m$ with arbitrary $m$, it holds
\begin{align}
	\Bigl\|\frac{1}{m}\sum_{i=1}^m a_i\Bigr\|^2
	\le \frac{1}{m}\sum_{i=1}^m\|a_i\|^2. \label{eq:jensen_for_sq_norms}
\end{align}
For any random vector $X$ we have
\begin{align}
	\ec{\|X\|^2}=\|\ec{X}\|^2 + \ec{\|X-\ec{X}\|^2}. \label{eq:rand_vec_sq_norm}
\end{align}
If $f$ is $L_f$-smooth, then for any $x,y\in\R^d$, it is satisfied
\begin{equation}
	f(y)
	\le f(x) + \<\nabla f(x), y-x> + \frac{L_f}{2}\|y-x\|^2. \label{eq:smooth_func_vals}
\end{equation}

Finally, for $L_f$-smooth and convex function $f:\R^d\to\R$, it holds
\begin{equation}
	f(x) \leq f(y)+\langle\nabla f(x), x-y\rangle-\frac{1}{2 L_f}\|\nabla f(x)-\nabla f(y)\|^{2} \label{eq:smooth_conv}.
\end{equation}

\begin{proposition} \label{pr:three_point} [Three-point identity]
	For any $u,v,w \in \R^d$, any $f$ with its Bregman divergence $D_f(x,y) = f(x)-f(y) - \langle \nabla f(y), x-y\rangle $, it holds
	\begin{align*}
		\langle \nabla f(u) - \nabla f(v), w-v \rangle
		&= D_f (v,u) + D_f(w,v) - D_f(w,u).
	\end{align*}
\end{proposition}

\begin{lemma}[Arithmetic mean -- Geometric mean inequality]
	For $c\geq0$ we have 
	\begin{equation} 
		1+ c = \frac {1 + (1+2c)}{2} \stackrel{AG} \geq \sqrt{1+2c}.
	\end{equation}
\end{lemma}

\subsection{Proof of \Cref{th:three}}
\begin{proof}
	Because $\g(x^k) \in \Range{\h(x^k)}$, it holds $\h(x^k)  [\h(x^k)]^\dagger \g(x^k) = \g(x^k)$.
	Updates \eqref{eq:sgn_aicn} and \eqref{eq:sgn_sap} are equivalent as
	\begin{eqnarray*}
		\pk k [\h(x^k)]^\dagger \g(x^k)
		&=& \sk \left( \sk^\top \h (x^k) \sk \right)^\dagger \sk^\top \h(x^k)  [\h(x^k)]^\dagger \g(x^k)\\
		&=& \sk \left( \sk^\top \h (x^k) \sk \right)^\dagger \sk^\top \g(x^k)\\
		&=&\sk [\hSk(x^k)]^\dagger \gSk(x^k).
	\end{eqnarray*}
	Taking gradient of $\modelSk{x^k, h}$ w.r.t. $h$ and setting it to $0$ yields that for solution $h^*$ holds
	\begin{gather} \label{eq:equiv_grad}
		\gSk(x^k) + \hSk(x^k) h^* + \frac \Lalg 2 \normM {h^*} {x^k, \sk} \hSk(x^k)h^* =0,
	\end{gather}
	which after rearranging is
	\begin{gather} \label{eq:equiv_hsol}
		h^* = - \left( 1+ \frac \Lalg 2 \normM {h^*} {x^k, \sk} \right)^ {-1} \left[\hSk(x^k) \right]^\dagger \gSk(x^k),
	\end{gather}
	thus solution of cubical regularization in local norms \eqref{eq:sgn_Ts} has form of Newton method with stepsize $\als k = \left( 1+ \frac \Lalg 2 \normM {h^*} {x^k, \sk} \right)^ {-1}$. We are left to show that this $\als k$ is equivalent to \eqref{eq:sgn_alpha}.
	
	Substitute $h^*$ from \eqref{eq:equiv_hsol} to \eqref{eq:equiv_grad} and $\als k = \left( 1+ \frac \Lalg 2 \normM {h^*} {x^k, \sk} \right)^ {-1}$ and then use $\h(x^k) [\h(x^k)]^\dagger \g(x^k) = \g(x^k)$, to get
	\begin{eqnarray*}
		0
		&=&\gSk(x^k) + \hSk(x^k) \left( - \als k \left[\hSk(x^k) \right]^\dagger \gSk(x^k) \right) \\
		&& \quad + \frac \Lalg 2 \left( \als k \normMSdk {\gSk(x^k)} {x^k}\right) \hSk(x^k) \left(- \als k \left[\hSk(x^k) \right]^\dagger \gSk(x^k) \right)\\
		&=& \left( 1 - \als k - \frac \Lalg 2 \als k^2 \normMSdk {\gSk(x^k)} {x^k} \right) \gSk(x^k).
	\end{eqnarray*}
	Finally, $\als k$ from \eqref{eq:sgn_alpha} is a positive root of polynomial $1 - \als k - \frac \Lalg 2 \als k^2=0$, which concludes the equivalence of \eqref{eq:sgn_aicn}, \eqref{eq:sgn_sap} and \eqref{eq:sgn_reg}.
\end{proof}

\subsection{Proof of \Cref{le:sketch_equiv}}
\begin{proof}
	For arbitrary square matrix $\mathbf M$ pseudoinverse guarantee $\mathbf M^ \dagger \mathbf M \mathbf M^ \dagger = \mathbf M^\dagger$. Applying this to $\mathbf M \leftarrow \left( \st \h (x) \s \right)$ yields $\la \p y, \p z \ra _ {\h(x)} = \la \p y, z \ra _{\h(x)} y,z \in \mathbb{R}^d $. Thus, $\p$ is really projection matrix w.r.t. $\normM \cdot x$. \\  
\end{proof}

\subsection{Proof of \Cref{le:tau_exp}}
\begin{proof}
	We follow proof of \citep[Lemma 5.2]{hanzely2020stochastic}.
	Using definitions and cyclic property of the matrix trace,
	\begin{eqnarray*}        
		\E{\tau(\s)}
		&=& \E{\Tr\left(\mI^{\tau(\s)}\right)}
		= \E{\Tr\left( \s^\top \h(x) \s \left( \s^\top \h(x) \s \right)^{\dagger} \right)}\\
		&=& \E{\Tr\left( \p \right)}
		= \Tr \left(\frac \tau d \mI^d\right)
		= \tau.
	\end{eqnarray*}
\end{proof}

\subsection{Proof of \Cref{le:setting_s}}

\begin{proof}
	We have
	\begin{eqnarray*}
		\EE_{\s \sim \cD}\left[\p\right]
		&=& \left[\h(x)\right]^{-1/2} \EE_{\mathbf M \sim \tilde {\cD}} \left[ \mathbf M^\top \left( \mathbf M^\top \mathbf M \right)^\dagger \mathbf M \right] \left[\h(x)\right]^{1/2}\\
		&=& \left[\h(x)\right]^{-1/2} \frac \tau d \mI \left[\h(x)\right]^{1/2}
		= \frac \tau d \mI.
	\end{eqnarray*}
\end{proof}

\subsection{Proof of \Cref{le:one_step_dec}}
\begin{proof}
	For $h^k = x^{k+1}-x^k$, we can follow proof of \citep[Lemma 10]{hanzely2022damped},
	\begin{eqnarray*}
		\efa f(x^k)-f(x^{k+1}) \\
		&\stackrel{\eqref{eq:semistrong_approx}} \geq& - \la \gS(x^k),h^k \ra  - \frac 12 \normsM {h^k} {x^k, \sk} - \frac \Lalg 6 \normM {h} {x^k, \sk}^3 \\
		&\stackrel{\eqref{eq:transition-primdual}} =& \als k \normsMSdk {\gSk(x^k)} {x^k} - \frac 12 \als k^2 \normsMSdk {\gSk(x^k)} {x^k} \\
		&& \qquad - \frac \Lalg 6 \als k^3 \normM {\gSk(x^k)} {x^k, \sk}^{*3} \\
		& =& \left(1 - \frac 12 \als k - \frac \Lalg 6 \als k^2 \normMSdk {\gSk(x^k)} {x^k} \right) \als k \normsMSdk {\gSk(x^k)} {x^k}\\
		& \geq& \frac 12 \als k \normsMSdk {\gSk(x^k)} {x^k}\\
		& \geq& \frac 1{2 \max \left \{\sqrt{\Lalg \normMSdk {\gSk(x^k)} {x^k} }, 2 \right\} } \normsMSdk {\gSk(x^k)} {x^k}.
	\end{eqnarray*}
\end{proof}

\subsection{Proof of \Cref{le:local_step_subspace}}
\begin{proof}
	We bound norm of $\gS(x^{k+1})$ using basic norm manipulation and triangle inequality as
	\begin{eqnarray*}
		\efa \quad \normMSdk {\gSk(x^{k+1})} {x^k}\\
		&=& \normMSdk{\gSk(x^{k+1}) - \hSk (x^k)(x^{k+1}-x^k)  - \als k \gSk (x^k) }{x^k}\\
		&=& \normMSdk{\gSk(x^{k+1}) - \gSk (x^{k}) - \hSk (x^k)(x^{k+1}-x^k) + (1-\als k) \gSk (x^k) }{x^k}\\
		&\leq& \normMSdk{\gSk(x^{k+1}) - \gSk(x^{k}) - \hSk(x^k)(x^{k+1}-x^k)}{x^k} \\
		&& \qquad  + (1-\als k) \normMSdk{\g(x^k) }{x^k}.
	\end{eqnarray*}
	
	Using $\Lsemi$--semi-strong self-concordance, we can continue
	\begin{eqnarray*}
		\cdots &\leq& \normMSdk{\gSk(x^{k+1}) - \gSk(x^{k}) - \hSk(x^k)(x^{k+1}-x^k)}{x^k} \\
		&& \qquad  + (1-\als k) \normMSdk{\gS(x^k) }{x^k}  \\
		&\leq& \frac{\Lsemi}{2} \normsM{x^{k+1}-x^k} {x^k, \sk} + (1-\als k) \normMSdk{\gSk(x^k) }{x^k} \\
		&=& \frac{\Lsemi \als k^2}{2} \normsMSdk{\gSk(x^k)} {x^k} + (1-\als k) \normMSdk{\gSk(x^k) }{x^k} \\
		&\leq&  \frac{\Lalg \als k^2}{2} \normsMSdk{\gSk(x^k)} {x^k} + (1-\als k) \normMSdk{\gSk(x^k) }{x^k} \\
		&=&  \left(\frac{\Lalg \als k^2}{2} \normMSd{\gSk(x^k)} {x^k} -\als k +1\right) \normMSdk{\gSk(x^k) }{x^k}\\
		&\stackrel{\eqref{eq:sgn_alpha}}{=}& \Lalg \als k^2\normsMSdk{\gSk(x^k) }{x^k}.
	\end{eqnarray*}
	Last equality holds because of the choice of $\als k$.
\end{proof}

\subsection{Technical lemmas}

\begin{lemma}[Arithmetic mean -- Geometric mean inequality]
	For $c\geq0$ we have 
	\begin{equation} \label{eq:AG}
		1+ c = \frac {1 + (1+2c)}{2} \stackrel{AG} \geq \sqrt{1+2c}.
	\end{equation}
\end{lemma}

\begin{lemma}[Jensen for square root]
	Function $f(x) = \sqrt x$ is concave, hence for $c\geq0$ we have 
	\begin{equation}\label{eq:jensen}
		\frac 1 {\sqrt 2} (\sqrt c+1) \leq \sqrt{c+1} \qquad \leq \sqrt c + 1.
	\end{equation}
\end{lemma}

\subsection{Proof of \Cref{le:global_step}}
\begin{proof}
	Denote 
	\begin{equation*}
		\Omega_\s (x, h') \eqdef f(x) + \la \g(x), \p h' \ra + \frac 12 \normsM {\p h'} x + \frac {\Lalg}6 \normM {\p h'} x ^3,
	\end{equation*}
	so that 
	\begin{equation*}
		\min_{h' \in \mathbb{R}^d} \Omega_\s (x,h') = \min_{h\in \mathbb{R}^{\tau(\s)}} \modelS{x,h}.
	\end{equation*}
	
	For arbitrary $y \in \mathbb{R}^d$ denote $h \eqdef y-x^k$. We can calculate
	\[f(x^{k+1})  \leq  \min_{h'\in \mathbb{R}^{\tau(\s)}} \modelS {x^k, h'} = \min_{h'' \in \mathbb{R}^d} \Omega_\s (x^k,h''),\]
	and
	\begin{eqnarray*}
		\efa \E{f(x^{k+1})}\\
		&\leq& \E{ \Omega_\s (x^k, h} )\\
		& =& f(x^k) + \frac \tau d \la \g(x^k), h \ra + \frac 12 \E{\normsM{\pk k h} {x^k} }  + \E{\frac \Lalg 6 \normM{\pk k h} {x^k} ^3}\\
		& \stackrel{\eqref{eq:proj_h}} \leq& f(x^k) + \frac \tau d \la \g(x^k), h \ra + \frac \tau {2d} \normsM h {x^k}  + \frac {\Lalg} 6 \frac \tau d \normM h {x^k} ^3\\
		& \stackrel{\eqref{eq:semistrong_approx}}\leq& f(x^k) + \frac \tau d \left( f(y)-f(x^k) + \frac {\Lsemi}6 \normM{y-x^k}{x^k}^3 \right)  + \frac {\Lalg} 6 \frac \tau d \normM h {x^k} ^3.
	\end{eqnarray*}
	
	In second to last inequality depends on unbiasedness of projection $\p$, Assumption~\ref{as:projection_direction}. In last inequality we used semi-strong self-concordance, \Cref{pr:semistrong_stoch} with $\s=\mI$.
\end{proof}

\subsection{Proof of \Cref{th:global_convergence}}
\begin{proof}
	Denote
	\begin{eqnarray}
		A_0 & \eqdef& \frac 43 \left(\frac d \tau \right)^3,\\
		A_k &\eqdef& A_0 + \sumin tk t^2 = A_0 -1 + \frac{k(k+1)(2k+1)}{6} \geq A_0 +\frac {k^3} 3, \\
		&& \qquad \text{...consequently } \quad  \sumin tk \frac {t^6}{A_t^2} \leq 9k, \\
		\eta_t & \eqdef& \frac d \tau \frac{(t+1)^2} {A_{t+1}} \qquad \text{implying } 1-\frac d \tau \eta_t = \frac {A_t} {A_{t+1}}.
	\end{eqnarray}
	Note that this choice of $A_0$ implies (as in \citep{hanzely2020stochastic})
	\begin{align}
		\eta_{t-1} 
		\leq \frac d \tau \frac {t^2}{A_0 + \frac {t^3} 3 }
		\leq \frac d \tau \sup_{t\in \N} \frac {t^2}{A_0 + \frac {t^3} 3 }
		\leq \frac d \tau \sup_{\zeta>0} \frac {\zeta^2} {A_0 + \frac {\zeta^3} 3}
		= 1
	\end{align}
	and $\eta_t \in [0,1]$.    
	Set $y \eqdef \eta_t \opt + (1-\eta_t) x^t$ in \Cref{le:global_step}. From convexity of $f$,
	\begin{eqnarray*}        
		\E{f(x^{t+1} | x^t)} 
		& \leq&  \left(1- \frac \tau d \right) f(x^t) + \frac \tau d \fopt \eta_t + \frac \tau d f(x^t) (1-\eta_t) \\
		& &\qquad + \frac \tau d \left( \frac {\max \Ls + \Lsemi}  6 \normM{x^t-\opt} {x^t} ^3 \eta_t^3 \right).
	\end{eqnarray*}    
	Denote $\delta_t \eqdef \E{f(x^t) -\fopt}$. Subtracting $\fopt$ from both sides and substituting $\eta_k$ yields
	\begin{align}        
		\delta_{t+1}
		& \leq  \frac {A_t} {A_{t+1}} \delta_t + \frac {\max \Ls + \Lsemi}  6 \normM{x^t-\opt} {x^t} ^3 \left( \frac d \tau \right)^2 \left(\frac {(t+1)^2}{A_{t+1}} \right)^3 .
	\end{align}
	Multiplying by $A_{t+1}$ and summing from from $t=0,\dots, k-1  $ yields
	\begin{align}        
		A_{k} \delta_{k}
		& \leq  A_0 \delta_0 + \frac {\max \Ls + \Lsemi} 6 \frac {d^2} {\tau^2} \sum_{t=0}^{k-1} \normM{x^t-\opt} {x^t} ^3 \frac {(t+1)^6}{A_{t+1}^2}.
	\end{align}    
	Using $\sup_{x\in \level} \normM{x-\opt} {x} \leq R$ we can simplify and shift summation indices,
	\begin{eqnarray}        
		A_{k} \delta_{k}
		& \leq & A_0 \delta_0 + \frac {\max \Ls + \Lsemi} 6 \frac {d^2} {\tau^2} D^3 \sumin tk \frac {t^6}{A_t^2}\\
		& \leq&  A_0 \delta_0 + \frac {\max \Ls + \Lsemi} 6 \frac {d^2} {\tau^2} D^3 9k,
	\end{eqnarray}
	and
	\begin{eqnarray}        
		\delta_{k}
		& \leq&  \frac {A_0 \delta_0}{A_k} + \frac {3(\max \Ls + \Lsemi) d^2 D^3 k} {2 \tau^2 A_k}\\
		& \leq & \frac {3A_0 \delta_0}{k^3} + \frac {9(\max \Ls + \Lsemi) d^2 D^3} {2 \tau^2 k^2},
	\end{eqnarray}
	which concludes the proof.
\end{proof}

\subsection{Proof of \Cref{th:local_linear}}
Before we start proof, we first formulate proposition that for self-concordant functions small change of model does not change Hessian drastically.
\begin{proposition} \citep[Lemma E.3]{hanzely2020stochastic}
	\label{pr:sscn_lemmas}
	For $\gamma>0$ and $x^k$ in neighborhood
	$x^k \in \left \{ x: \normMd {\g(x)} {x} < \frac 2{(1+\gamma^{-1})\Lstandard } \right \}$ for $\Lstandard$--self-concordant function $f:\R^d\to\R$, we can bound
	\begin{equation} \label{eq:sc_gamma}
		f(x^k)-\fopt \leq \frac 12 (1+ \gamma) \normsMd {\g(x^k)} {x^k}.
	\end{equation}
\end{proposition}

\begin{proof}
	\Cref{pr:sscn_lemmas} with $\gamma =2$ implies that in neighborhood $\normsMSd {\g(x^k)} {x^k} \leq \frac 4 \Ls$,
	\[f(x^k)-f(x^{k+1}) \stackrel{\eqref{eq:one_step_dec}} \geq \frac 14 \normsMSdk {\gSk(x^k)} {x^k}.\]
	With identity $\normsMSdk {\gSk(x^k)} {x^k} = \normsMd{\pk k \g(x^k)} {x^k} $, we can continue
	\begin{eqnarray*}
		\E{f(x^k)-f(x^{k+1})} 
		& \stackrel{\eqref{eq:one_step_dec}} \geq& \E{\frac 14 \normsMSd {\gSk(x^k)} {x^k}}
		= \E{ \frac 14 \normsMd {\pk k \g(x^k)} {x^k} } \\
		&\stackrel{\eqref{eq:proj_g}}=& \frac \tau {4d} \normsMd {\g(x^k)} {x^k}\\
		&\stackrel{\eqref{eq:sc_gamma}}\geq& \frac {\tau} {2d(1+ \gamma)} (f(x^k)-\fopt).
	\end{eqnarray*}
	Hence 
	\[ \E{f(x^{k+1})-\fopt)} \leq \left( 1- \frac \tau {2d(1+\gamma)} \right) (f(x^k)-\fopt), \]
	and to finish the proof, we use tower property across iterates $x^0, x^1, \dots, x^k$.
\end{proof}

\subsection{Towards proof of \Cref{th:global_linear}}

\begin{proposition}\citep[Equation (47)]{RSN}
	Relative convexity \eqref{eq:rel_conv} implies bound
	\begin{equation}
		\fopt \leq f(x^k) - \frac 1 {2\murel} \normsMd{\g(x^k)}{x^k}. \label{eq:rel_conv_dec}
	\end{equation}
\end{proposition}

\begin{proposition} Analogy to \citep[Lemma 7]{RSN} \label{pr:rho}
	For $\s \sim \cD $ satisfying conditions
	\begin{equation}
		\Null{\s^\top \h(x) \s} = \Null{\s} \quad \text{and} \quad \Range{\h(x)} \subset \Range {\E{\s_k \s_k ^\top} },
	\end{equation}
	also exactness condition holds
	\begin{equation}
		\Range{\h(x)} = \Range{\E{\hat \p}},
	\end{equation}
	and formula for $\rho(x)$ can be simplified
	\begin{equation}
		\rho(x) = \lambda_{\text{min}}^+ \left( \E{\alpha_{x, \s} \p} \right) >0
	\end{equation}
	and bounded $0 < \rho(x) \leq 1$. Consequently, $0 < \rho \leq 1$.
\end{proposition}

\begin{lemma}[Stepsize bound] \label{le:stepsize_bound}
	Stepsize $\als k$ can be bounded as 
	\begin{equation}
		\als k        
		\leq  \frac {\sqrt 2} {\sqrt{\Lalg \normMSdk{\gSk(x^k)} {x^k}}},
	\end{equation}
	and for $x^k$ far from solution, $\normMSdk{\gSk(x^k)} {x^k} \geq \frac 4 {L_{\sk}}$ and $\Lalg \geq \frac 98 \sup_\s {L_{\s}} \Lrel_{\s}^2$ holds
	$
	\als k \Lrels
	\leq \frac 23.
	$
\end{lemma}
\begin{proof}[Proof of \Cref{le:stepsize_bound}]\\
	Denote $G_k \eqdef \Lalg \normMSdk{\gSk(x^k)} {x^k}$. Using \eqref{eq:jensen} with $c \leftarrow 2 G > 0$ and
	\begin{equation}
		\als k
		= \frac {-1 + \sqrt{1+2 G}}{G}
		\leq \frac{\sqrt {2G}}{G}
		= \frac {\sqrt 2} {\sqrt G} 
		= \frac {\sqrt 2} {\sqrt{\Lalg \normMSdk{\gSk(x^k)} {x^k}}}
	\end{equation}
	and 
	\begin{align*}        
		\als k \Lrels
		&\leq \frac {\sqrt 2 \Lrels} {\sqrt{\Lalg \normMSdk{\gSk(x^k)} {x^k}}}\\
		&\leq \frac {\sqrt 2 \Lrels} {\sqrt{\frac 98 L_{\sk} \Lrels^2 \normMSdk{\gSk(x^k)} {x^k}}}\\
		& \leq \frac 43 \frac 1 {\sqrt{L_{\sk} \normMSdk{\gSk(x^k)} {x^k}}} 
		\leq \frac 23, \qquad \text{for } \normMSdk{\gSk(x^k)} {x^k} \geq \frac 4 {\Lrels}.
	\end{align*}
\end{proof}

\subsubsection{Proof of \Cref{th:global_linear}}
\begin{proof}
	Replacing $x\leftarrow x^k$ and $h \leftarrow \als k \pk k [\h(x^k)]^\dagger \g(x^k)$ so that $x^{k+1} = x^k + \s h$ in \eqref{eq:rels_smooth} yields
	\begin{eqnarray}
		f(x^{k+1})
		&\leq& f(x^k) - \als k \left(1 -  \frac 12 \Lrels \als k \right) \normsMSdk {\gSk(x^k)} {x^k}\\
		&\leq&  f(x^k) - \frac 23 \als k \normsMSdk {\gSk(x^k)} {x^k}.
	\end{eqnarray}
	In last step, we used that $\Lrels \als k \leq \frac 23$  holds for $\normMSdk {\gSk(x^k)} {x^k} \geq \frac 4 {\Lrels}$ (\Cref{le:stepsize_bound}). Next, we take expectation on $x^k$ and use definition of $\rho(x^k)$.
	\begin{eqnarray}
		\E{f(x^{k+1})} 
		&\leq& f(x^k) - \frac 23 \normsM {\g(x^k)} {\E{\als k \s \left[\hSk(x^k) \right]^\dagger \s^\top}}\\
		&\leq& f(x^k) - \frac 23 \rho(x^k) \normsMd {\g(x^k)} {x^k}\\
		& \stackrel{\eqref{eq:rel_conv_dec}} \leq& f(x^k) - \frac 43 \rho(x^k) \murel \left( f(x^k) - \fopt \right).
	\end{eqnarray}
	Now $\rho(x^k) \geq \rho$, and $\rho$ is bounded in \Cref{pr:rho}. Rearranging and subtracting $\fopt$ gives
	\begin{equation}        
		\E{f(x^{k+1}) - \fopt}  \leq \left(1 - \frac 43 \rho \murel \right) (f(x^k) - \fopt),
	\end{equation}
	which after towering across all iterates yields the statement.
\end{proof}
            \end{singlespace}
		\endgroup
\end{document}